%% file: ms.tex
\documentclass[11pt]{article} % For LaTeX2e

\pdfoutput=1

\usepackage{etoolbox}
\newtoggle{arxiv}
\toggletrue{arxiv}
\input{preamble/packages} % contains formatting under package geometry
\input{preamble/macros} % additional macros

\newlength\tindent
\allowdisplaybreaks

\title{Task-Optimal Exploration in Linear Dynamical Systems}
\author{Andrew Wagenmaker\footnote{University of Washington, Seattle. \href{mailto:ajwagen@cs.washington.edu}{ajwagen@cs.washington.edu}} \and Max Simchowitz\footnote{University of California, Berkeley. \href{mailto:msimchow@berkeley.edu}{msimchow@berkeley.edu}}  \and Kevin Jamieson\footnote{University of Washington, Seattle. \href{mailto:jamieson@cs.washington.edu}{jamieson@cs.washington.edu}} }
\date{\today}

\begin{document}

\maketitle

\begin{abstract}
\input{body/abstract}

\end{abstract}

\input{body/introduction}

\input{body/results_summary}

\input{body/comparison}

\input{appendix/general_lds_facts}

\input{body/synthesis_rates_arxiv}

\input{body/conclusion}
%\newpage
%\clearpage
\addcontentsline{toc}{section}{References}
\bibliographystyle{icml2021}
\bibliography{bibliography.bib}

\setcounter{tocdepth}{2}

\newpage
\appendix
\tableofcontents

\newpage
\input{appendix/organization_notation_arxiv}

\newpage
\input{body/remarks_and_extensions}

\newpage
\part{Martingale Decision Making}\label{part:general}
\input{body/synthesis_rates_arxiv_supp}

\input{body/general_decisions_arxiv}

\input{body/ce_upper_bound_ddm}

\part{Linear Dynamical Decision Making}\label{part:lin}
\input{appendix/lds_preliminaries}

\input{body/lds_lower_bound}

\input{body/ce_upper_bound}

\input{body/algorithm}

\newpage
\part{LQR and Further Examples}\label{part:examples}
\input{body/lqr}

\input{body/applications}

\end{document}

%% file: preamble/packages.tex
\usepackage[utf8]{inputenc}
\usepackage{mathrsfs}
\usepackage{microtype,upgreek}
\usepackage{subfigure}
\usepackage{booktabs} 
\usepackage{longtable,makecell}
\usepackage{amsmath, amsfonts, amsthm, amssymb, graphicx}
\usepackage{mathtools,verbatim}
\usepackage{enumitem,nicefrac}
\usepackage{bm}
\usepackage{thm-restate}
\usepackage{natbib}
\usepackage{xspace}
\usepackage[dvipsnames]{xcolor}

\usepackage{color-edits}

\addauthor{df}{ForestGreen}
\addauthor{ms}{orange}

\usepackage[noend]{algpseudocode}
\usepackage{algorithm}

\iftoggle{arxiv}
{
	\usepackage[scaled=.90]{helvet}
	
	\usepackage[vmargin=1.00in,hmargin=1.00in,centering,letterpaper]{geometry}
	\usepackage{fancyhdr, lastpage}

	\setlength{\headsep}{.10in}
	\setlength{\headheight}{15pt}
	\cfoot{}
	\lfoot{}
	\rfoot{\sc Page\ \thepage\ of\ \protect\pageref{LastPage}}

}

\usepackage{hyperref}
\usepackage{cleveref}
\hypersetup{colorlinks,citecolor=blue}
\newtoggle{upperbound}
\togglefalse{upperbound}

%% file: preamble/macros.tex
\newcommand{\Expop}{\operatornamewithlimits{\ensuremath{\mathbb{E}}}} %expectation

\newcommand{\matSig}{\bm{\Sigma}}

\newcommand{\calE}{\mathcal{E}}

\newcommand{\Mfraka}{M_{\fraka}}

\newcommand{\Gfraka}{G_{\fraka}}

\newcommand{\betataylor}{r_{\mathrm{quad}}}
\newcommand{\dlyap}{\mathsf{dlyap}}

\newcommand{\dimfraka}{d_\fraka}
\newcommand{\dimtheta}{d_{\theta}}
\newcommand{\betaexp}{r_{\mathrm{cov}}}
\newcommand{\inputset}{\mathcal{U}_{\gamma^2}}

\newcommand{\Lra}{L_{\mathrm{hess}}}
\newcommand{\imag}{\iota}

\newcommand{\wt}[1]{\widetilde{#1}}
\newcommand{\wh}[1]{\widehat{#1}}
\newcommand{\thetast}{\theta_\star}

\newcommand{\thethat}[1]{\widehat{}{\theta}}

\newcommand{\synth}{\mathsf{synth}}

\newcommand{\Aclst}{A_{\mathrm{cl},\star}}
\newcommand{\Thetast}{\Theta_{\star}}

\newcommand{\alphast}{L_{\mathrm{cov}}}
\newcommand{\betast}{r_{\mathrm{quad}}}
\newcommand{\calR}{\mathcal{R}}
\newcommand{\fraku}{\mathfrak{u}}

\newcommand{\frakw}{\mathfrak{w}}

\newcommand{\vectorize}{\mathrm{vec}}

\newcommand{\Kopt}{K_{\mathrm{opt}}}
\newcommand{\Jfunc}{\mathcal{J}}
\newcommand{\sigw}{\sigma_w}
\newcommand{\calN}{\mathcal{N}}

\newcommand{\calB}{\mathcal{B}}

\crefname{asm}{Assumption}{Assumptions} 
\crefname{protocol}{Protocol}{Protocols} 
\crefname{claim}{Claim}{Claims}

\newcommand{\Rx}{R_{\matx}}
\newcommand{\Ru}{R_{\matu}}

\newcommand{\poly}{\mathrm{poly}}

\newcommand{\calS}{\mathcal{S}}
\newcommand{\calT}{\mathcal{T}}

\newcommand{\matx}{\mathbf{x}}
\newcommand{\dimx}{d_{x}}
\newcommand{\dimu}{d_{u}}
\newcommand{\matu}{\mathbf{u}}

\numberwithin{equation}{section}

\newcommand{\fro}{\mathrm{F}}
\newcommand{\op}{\mathrm{op}}

\newcommand{\Ast}{A_{\star}}
\newcommand{\Bst}{B_{\star}}
\newcommand{\Kst}{K_{\star}}
\newcommand{\Pst}{P_{\star}}

\newcommand{\xcheck}{\check{x}}

{
 \theoremstyle{plain}
      \newtheorem{asm}{Assumption}
}
\creflabelformat{asm}{#2#1#3}

\theoremstyle{plain}
\newtheorem{thm}{Theorem}[section]
\newtheorem{lem}{Lemma}[section]
\newtheorem{cor}{Corollary}
\newtheorem{claim}{Claim}[section]
\newtheorem{fact}{Fact}[section]

\newtheorem{prop}[thm]{Proposition}

\theoremstyle{definition}
\newtheorem{defn}{Definition}[section]

\newtheorem*{task_prob}{Task-Specific Pure Exploration Problem}

\newtheorem{exmp}{Example}[section]

\newtheorem{rem}{Remark}[section]

\crefname{equation}{Eq.}{Eqs.}
\Crefname{equation}{Eq.}{Eqs.}

\renewcommand{\Pr}{\mathbb{P}}
\newcommand{\Exp}{\mathbb{E}}

\newcommand{\iidsim}{\overset{\mathrm{i.i.d.}}{\sim}}

\newcommand{\rmd}{\mathrm{d}}

\newcommand{\tr}{\mathrm{tr}}

\newcommand{\diag}{\mathrm{diag}}

\renewcommand{\implies}{\text{ implies }}

\newcommand{\matw}{\mathbf{w}}

\newcommand{\N}{\mathbb{N}}

\newcommand{\C}{\mathbb{C}}
\newcommand{\R}{\mathbb{R}}
\newcommand{\I}{\mathbb{I}}

\newcommand{\PD}[1][d]{\mathbb{S}_{++}^{#1}}

\DeclareMathOperator*{\argmax}{arg\,max}
\DeclareMathOperator*{\argmin}{arg\,min}
\DeclareMathOperator*{\logdet}{log\,det}

\def\ddefloop#1{\ifx\ddefloop#1\else\ddef{#1}\expandafter\ddefloop\fi}
\def\ddef#1{\expandafter\def\csname bb#1\endcsname{\ensuremath{\mathbb{#1}}}}
\ddefloop ABCDEFGHIJKLMNOPQRSTUVWXYZ\ddefloop

\def\ddefloop#1{\ifx\ddefloop#1\else\ddef{#1}\expandafter\ddefloop\fi}
\def\ddef#1{\expandafter\def\csname fr#1\endcsname{\ensuremath{\mathfrak{#1}}}}
\ddefloop ABCDEFGHIJKLMNOPQRSTUVWXYZ\ddefloop
\def\ddefloop#1{\ifx\ddefloop#1\else\ddef{#1}\expandafter\ddefloop\fi}
\def\ddef#1{\expandafter\def\csname scr#1\endcsname{\ensuremath{\mathscr{#1}}}}
\ddefloop ABCDEFGHIJKLMNOPQRSTUVWXYZ\ddefloop
\def\ddefloop#1{\ifx\ddefloop#1\else\ddef{#1}\expandafter\ddefloop\fi}
\def\ddef#1{\expandafter\def\csname b#1\endcsname{\ensuremath{\mathbf{#1}}}}
\ddefloop ABCDEFGHIJKLMNOPQRSTUVWXYZ\ddefloop
\def\ddef#1{\expandafter\def\csname c#1\endcsname{\ensuremath{\mathcal{#1}}}}
\ddefloop ABCDEFGHIJKLMNOPQRSTUVWXYZ\ddefloop
\def\ddef#1{\expandafter\def\csname h#1\endcsname{\ensuremath{\widehat{#1}}}}
\ddefloop ABCDEFGHIJKLMNOPQRSTUVWXYZ\ddefloop
\def\ddef#1{\expandafter\def\csname t#1\endcsname{\ensuremath{\widetilde{#1}}}}
\ddefloop ABCDEFGHIJKLMNOPQRSTUVWXYZ\ddefloop

\def\ddefloop#1{\ifx\ddefloop#1\else\ddef{#1}\expandafter\ddefloop\fi}
\def\ddef#1{\expandafter\def\csname mat#1\endcsname{\ensuremath{\mathbf{#1}}}}
\ddefloop abcdefgijklmnopqrstuvwxyzABCDEFGHIJKLMNOPQRSTUVWXYZ\ddefloop

\newcommand{\Ahat}{\wh{A}}
\newcommand{\Bhat}{\wh{B}}
\newcommand{\DelA}{\Delta_A}
\newcommand{\DelB}{\Delta_B}
\newcommand{\Hinf}{\mathcal{H}_{\infty}}

\newcommand{\Acl}{A_{\mathrm{cl}}}
\newcommand{\DelAcl}{\Delta_{\Acl}}

\newcommand{\DelK}{\Delta_K}
\newcommand{\covup}{\bar{\Gamma}_T}
\newcommand{\calA}{\mathcal{A}}

\newcommand{\Atil}{\wt{A}}

\newcommand{\util}{\wt{u}}

\newcommand{\Btil}{\wt{B}}
\newcommand{\rhool}{\rho_{\star}}
\newcommand{\tauol}{\tau_{\star}}

\newcommand{\phist}{\phi_{\star}}

\newcommand{\lammin}{\lambda_{\min}}

\newcommand{\epsop}{\epsilon_{\mathrm{op},i-1}}

\newcommand{\xui}{x^{\bmu,0}}
\newcommand{\xu}{x^{\bmu}}
\newcommand{\xutil}{\wt{x}^{\bmu}}
\newcommand{\calU}{\mathcal{U}}
\newcommand{\thetastbar}{\wt{\theta}_{\star}}
\newcommand{\Atilst}{\wt{A}_\star}
\newcommand{\Btilst}{\wt{B}_\star}
\newcommand{\Cexp}{C_{\mathrm{exp}}}
\newcommand{\Cinit}{C_{\mathrm{init}}}
\newcommand{\Lquad}{L_{\mathrm{quad}}}

\newcommand{\matGam}{\mathbf{\Gamma}}
\newcommand{\Gamnoise}{\Gamma^{\mathrm{noise}}}
\newcommand{\Gamin}{\Gamma^{\mathrm{in}}}
\newcommand{\Gamfreq}{\Gamma^{\mathrm{freq}}}
\newcommand{\Gamss}{\Gamma^{\mathrm{ss}}}
\newcommand{\matGamnoise}{\mathbf{\Gamma}^{\mathrm{noise}}}

\newcommand{\matGamfreq}{\mathbf{\Gamma}^{\mathrm{in,ss}}}
\newcommand{\matGamss}{\mathbf{\Gamma}^{\mathrm{ss}}}
\newcommand{\Rust}{R_{\matu,\star}}
\newcommand{\Ruopt}{R_{\matu,\mathrm{opt}}}
\newcommand{\lamnoise}{\lambda_{\mathrm{noise}}^{\star}}
\newcommand{\zu}{z^{\bmu}}
\newcommand{\zw}{z^{\bm{w}}}
\newcommand{\bmutil}{\bm{\util}}

\newcommand{\sigmaw}{\sigma_w}

\newcommand{\calD}{\mathcal{D}}
\newcommand{\matxi}{\bm{\xi}}

\newcommand{\calF}{\mathcal{F}}
\newcommand{\Rls}{\mathcal{R}_{\mathrm{ls}}}

\newcommand{\lddm}{$\mathsf{LDDM}$}
\newcommand{\lddmx}{$\mathsf{LDDM}$ }

\newcommand{\ddm}{$\mathsf{MDM}$}
\newcommand{\ddmx}{$\mathsf{MDM}$ }

\newcommand{\Phiss}{\Phi_{\mathrm{opt}}^{\mathrm{ss}}}
\newcommand{\cconj}{\mathrm{conj}}

\newcommand{\uext}{u^{\mathrm{ext}}}
\newcommand{\bmuext}{\bmu^{\mathrm{ext}}}
\newcommand{\Fourier}{\mathfrak{F}}
\newcommand{\ucheck}{\check{u}}
\newcommand{\Hermset}[1][d]{\mathbb{H}^d}
\newcommand{\Hermpsd}[1][d]{\mathbb{H}^d_{+}}
\newcommand{\herm}{\mathsf{H}}

\newcommand{\bmucheck}{\check{\bmu}}
\newcommand{\bmU}{\bm{U}}

\newcommand{\calM}{\mathcal{M}}
\newcommand{\lamst}{\lambda_{\min,T}^\star}
\newcommand{\lamstinf}{\lambda_{\min,\infty}^\star}
\newcommand{\tsum}{{\textstyle \sum}}
\newcommand{\xw}{x^{\bm{w}}}
\newcommand{\gamtil}{\wt{\gamma}^2}
\newcommand{\unoise}{u^w}

\newcommand{\kutil}{k_{\bm{\util}}}
\newcommand{\calUtil}{\wt{\calU}}
\newcommand{\thetatil}{\wt{\theta}}
\newcommand{\alphastlds}{L_{\mathrm{cov}}}
\newcommand{\betaexplds}{r_{\mathrm{cov}}}
\newcommand{\lamund}{\underline{\lambda}}

\newcommand{\bmu}{\bm{u}}
\newcommand{\xbmu}{x^{\bmu}}

\newcommand{\Gamnoisetil}{\wt{\Gamma}^{\mathrm{noise}}}

\newcommand{\Csys}{C_{\mathrm{sys}}}

\newcommand{\gamup}{\bar{\gamma}}
\newcommand{\policysetgood}{\Pi_{\gamma^2}^{\mathrm{sol}}}
\newcommand{\Ccov}{C_{\mathrm{cov}}}
\newcommand{\Clb}{C_{\mathrm{lb}}}
\newcommand{\ccexp}{c_{\mathrm{cov}}}
\newcommand{\rcov}{r_{\mathrm{cov}}}

%% file: body/abstract.tex
%!TEX root = ../main.tex

Exploration in unknown environments is a fundamental problem in reinforcement learning and control. In this work, we study task-guided exploration and determine what precisely an agent must learn about their environment in order to complete a particular task. Formally, we study a broad class of decision-making problems in the setting of linear dynamical systems, a class that includes the linear quadratic regulator problem. 
We provide instance- and task-dependent lower bounds which explicitly quantify the difficulty of completing a task of interest. Motivated by our lower bound, we propose a computationally efficient experiment-design based exploration algorithm. We show that it optimally explores the environment, collecting precisely the information needed to complete the task, and provide finite-time bounds guaranteeing that it achieves the instance- and task-optimal sample complexity, up to constant factors. Through several examples of the LQR problem, we show that performing task-guided exploration provably improves on exploration schemes which do not take into account the task of interest. Along the way, we establish that certainty equivalence decision making is instance- and task-optimal, and obtain the first algorithm for the linear quadratic regulator problem which is instance-optimal. We conclude with several experiments illustrating the effectiveness of our approach in practice.

%% file: body/introduction.tex
%!TEX root = ../main.tex
\newcommand{\thetahat}{\widehat{\theta}}
\newcommand{\thetals}{\thetahat_{\mathrm{ls}}}
\newcommand{\fraka}{\mathfrak{a}}
\newcommand{\aopt}{\fraka_{\mathrm{opt}}}
\newcommand{\frakahat}{\widehat{\fraka}}
\newcommand{\Phidesign}{\Phi}
\newcommand{\Phiopt}{\Phi_{\mathrm{opt}}}
\newcommand{\calO}{\mathcal{O}}
\newcommand{\deff}{d_{\Phi}}
\newcommand{\nablatwo}{\nabla^{\,2}}
\newcommand{\dima}{d_{\fraka}}
\newcommand{\lqrx}{\textsc{Lqr}\xspace}
\newcommand{\algname}{\textsc{Tople}\xspace}

\section{Introduction}

Modern reinforcement learning aims to understand how agents should best explore their environments in order to successfully complete assigned tasks. 
In the face of uncertainty about the environment, a naive strategy might be to explore the environment until it is uniformly understood (system identification), and then devise a plan to complete the task under this precise understanding of the environment (control). 
However, it is widely understood that such a two-phased approach of system identification followed by control can be wasteful since, depending on the task, some aspects of the environment ought to be estimated more accurately than others.
For instance, if a task requires a precise sequence of steps to be taken in order to be completed, one need not understand all possible outcomes leading to failure after a missed early critical step.
Since it may be very costly for an agent to estimate all facets of a complex or high dimensional environment to high precision, it is far preferable to direct agents' exploration only to those aspects most relevant to their tasks.
Motivated by this challenge, this paper aims to answer:
\begin{quote}
\emph{\hspace{-.5cm}Q1. What exactly must an agent learn about its environment to carry out a particular task? \\[6pt]
\hspace*{-.5cm}Q2. Given knowledge of the task, can the agent direct their exploration to speed up the process of learning this task-specific critical information? \\[6pt]
\hspace*{-.5cm}Q3. Having explored its environment, how can the agent best use the information gained to complete the task of interest? 
}
\end{quote}
Our work provides answers to the above questions for a family of decision-making problems in environments parameterized by a linear dynamical system, including synthesis of the linear quadratic regulator.
Specifically, for Q1 we show that accomplishing a variety of tasks amounts to maximizing a task-specific linear functional of the Fisher-information matrix, a quantity of fundamental importance to optimal experimental design. 
Indeed, our results naturally reduce to classical linear optimal experimental design criteria (for example, $A$-optimal) in the absence of dynamics. 
Answering Q2 in the affirmative amounts to being able to learn just enough about the environment to drive the system to a sequence of states that maximize this task-specific function as fast as possible. 
We accomplish this via a sequence of experimental design problems over control inputs given a successively improving estimate of the environment.
Finally, to answer Q3 we show that the \emph{certainty equivalence} decision rule---choosing the policy that would optimally complete the task if the estimate of the environment was correct---is the optimal decision rule in an instance-specific sense.

\subsection{Main Contributions}
Our primary contributions are as follows:
\begin{enumerate} 
	\item We develop task- and instance-specific lower bounds which precisely quantify how parameter estimation error translates to suboptimal task performance. 
	\item We cast the problem of optimal exploration as a surrogate experiment design problem we call \emph{task-optimal experiment design}. For linear dynamical systems, the task-optimal design problem can be solved efficiently by projected gradient descent. We demonstrate that  the solution to the design problem yields the information-theoretically optimal exploration strategy, in a strong, instance-dependent sense.
	\item  The task-optimal design depends on unknown problem parameters. We therefore propose a meta-algorithm, \algname, which sequentially solves empirical approximations to the design objective, and demonstrate that this approach matches the performance of the optimal design given knowledge of the true system parameters. As a consequence, we obtain the first instance-optimal algorithm for the \lqrx problem.
	\item We show through numerous mathematical examples that task-specific experiment design can perform arbitrarily better on a task of interest than uniform or task-agnostic exploration. We also rigorously prove a strong sub-optimality result for strategies with low regret for online LQR, such as optimism-under-uncertainty.  
	\item We show that, for \emph{any exploration strategy} which is sufficiently non-degenerate, in a very general class of decision-making problems which includes certain classes of nonlinear dynamical systems, the certainty equivalence decision rule is instance optimal.
	\item Finally, we show that our approach yields practical gains over naive exploration schemes through several numerical examples.
\end{enumerate}

All our results are non-asymptotic and polynomial in terms of the natural problem parameters.

%% file: body/results_summary.tex
%!TEX root = ../main.tex

\newcommand{\ohst}{\mathcal{O}^{\star}}
\newcommand{\Bop}{\calB_{\op}}

\newcommand{\Jlqr}{\Jfunc}
\newcommand{\mmax}{\mathfrak{M}}
\newcommand{\mmaxdec}{\mathfrak{M}_{\mathrm{dec}}}
\newcommand{\mmaxsy}{\mathfrak{M}_{\mathrm{synth}}}
\newcommand{\mmaxexp}{\mathfrak{M}_{\mathrm{exp}}}
\newcommand{\mmaxopt}{\mathfrak{M}_{\mathrm{opt}}}
\newcommand{\lqr}{\textsc{Lqr}}
\newcommand{\traj}{\bm{\uptau}}
\newcommand{\Rlqr}{\mathcal{R}}
\newcommand{\ce}{K_{\mathrm{ce}}}
\newcommand{\piexp}{\pi_{\mathrm{exp}}}
\newcommand{\rd}[1][d]{\R^{#1}}
\newcommand{\ace}{\mathsf{ce}}
\newcommand{\plan}{\mathsf{dec}}
\newcommand{\Piexp}{\Pi_{\mathrm{exp}}}
\newcommand{\frakE}{\mathfrak{E}}

\subsection{Task-Specific Pure Exploration}\label{sec:setting}

We consider linear dynamical systems of the form:
\begin{align}
x_{t+1} = \Ast x_t + \Bst u_t + w_t, \label{eq:our_dynamics} \quad x_{0} \equiv 0. 
\end{align}
where $x_t,w_t \in \R^{\dimx},u_t \in \R^{\dimu}$, $\Ast$ and $\Bst$ have appropriate dimensions, and 
where for simplicity we assume that $w_t \sim \calN(0,\sigw^2 I)$\footnote{See \Cref{sec:noise_cov} for a discussion on accommodating non-identity, possibly unknown noise covariance.}. We let $\thetast = (\Ast,\Bst)$ capture the true dynamical system; importantly, \emph{$\thetast$ is unknown to the learner}. We also define a policy $\pi$ as a mapping from past actions and states to future actions $\pi : (x_{1:t},u_{1:t-1}) \to u_t$. We let $\Exp_{\theta,\pi}[ \cdot ]$ denote the expectation over trajectories induced on instance $\theta$ playing policy $\pi$. While we show in \Cref{sec:overview_lds} that several of our results hold in a more general observation model which encompasses certain nonlinear systems, throughout \Cref{sec:results_summary,sec:interpret,sec:taskopt_improvements,sec:experiments_body} we assume we are in the linear dynamical system setting.

We are interested in a general decision making problem: given some smooth loss $\Jfunc_{\thetast}(\fraka)$ parameterized by $\thetast$, choose $\fraka \in \R^{\dima}$ such that $\Jfunc_{\thetast}(\fraka)$ is minimized. For every $\thetast$, we assume there exists some optimal decision $\aopt(\thetast)$ for which $\Jfunc_{\thetast}(\fraka)$ is minimized. We require that $\Jfunc_{\theta}(\fraka)$ and $\aopt(\theta)$ satisfy the following assumption.

\begin{asm}[Smooth Decision-Making, Informal]\label{asm:smooth_informal}
The loss $\Jfunc_{\theta}(\fraka)$ and optimal decision $\aopt(\theta)$ are each three times differentiable within a ball around $\aopt(\thetast)$ and $\thetast$, respectively, and their gradients can be absolutely bounded over this range. Furthermore, $\nabla_\fraka^2 \Jfunc_{\theta}(\fraka)$ varies smoothly in $\theta$. 
\end{asm}

\noindent Our interaction protocol is as follows.

\begin{task_prob} The  learner's behavior is specified by an exploration policy $\piexp : (x_{1:t},u_{1:t-1}) \to u_t$ and decision rule $\plan$ executed in the dynamics \Cref{eq:our_dynamics}.
\begin{enumerate}
	\item For steps $t = 1,\dots,T$, the learner executes $\piexp$ and collects a \emph{trajectory} $\traj = (x_{1:T+1},u_{1:T})$. 
	\item For a \emph{budget} $\gamma^2 \ge 0$, the inputs $u_{1:T}$ must satisfy the constraint $\Exp_{\piexp}[\sum_{t=1}^T \|u_t\|_2^2] \le T \gamma^2$.~\footnote{The upper bounds in this paper can be easily modified to ensure that the budget constraint $\sum_{t=1}^T \|u_t\|^2 \le T \gamma^2$ holds with high probability.}
	\item Finally, the learner proposes a decision $\frakahat = \plan(\traj)$  as a function of $\traj$.
\end{enumerate}
\end{task_prob}

The learner's performance is evaluated on the excess risk 
\begin{align*}
\calR(\frakahat;\thetast) := \Jlqr_{\thetast}(\frakahat) - \inf_{\fraka} \Jlqr_{\thetast}(\fraka),
\end{align*} 
and their goal is to choose an exploration policy $\piexp$ which induces sufficient exploration to propose a decision $\frakahat$ with as little excess risk as possible. For simplicity, we assume that the learner only collects a single trajectory. 
In contrast to online control, the performance of the exploration policy is only evaluated on its final decision $\frakahat$, not the trajectory generated during the learning phase.
To make this setting concrete, we consider several specific examples.

\subsection{Examples and Applications}\label{sec:ex_intro}

The task-specific pure exploration problem captures many natural settings. Under the assumed linear dynamics model of \eqref{eq:our_dynamics} with $\thetast = ( \Ast, \Bst)$, the unknown quantity of interest $\aopt(\thetast)$ can represent any function of the environment defined by $\thetast$. 
In the simplest case of system identification, we may have $\aopt(\thetast) = \thetast$, $\frakahat$  the least squares estimate of $\thetast$ given the trajectory $\traj = (x_{1:T+1},u_{1:T})$, and $\Jfunc_{\thetast}(\frakahat)$ a measure of loss, for example the Frobenius norm: $\Jfunc_{\thetast}(\frakahat) = \| \frakahat - \aopt(\thetast)\|_\fro^2$.
Even in this simple case, the learner can reduce $\Jfunc_{\thetast}(\frakahat)$ far faster with a deliberate exploration policy relative to a naive policy such as playing isotropic noise. 
The next several examples illustrate that the task specific pure exploration framework generalizes far beyond this simple system identification task.

\begin{exmp}[Pure Exploration \lqr]\label{ex:lqr}
In the \lqrx problem, the agent's objective is to design a policy that minimizes the infinite-horizon cumulative cost, with losses $\ell(x,u) := x^\top \Rx x + u^\top \Ru u$. The resultant cost function is 
\begin{align*}
	\Jlqr_{\lqr,\thetast}[\pi] := \lim_{T \to \infty} \Exp_{\thetast,\pi}\left[\frac{1}{T}\sum_{t=1}^Tx_t^\top \Rx x_t + u_t^\top \Ru u_t\right].
\end{align*}
It is well known that the optimal policies are of the form $u_t = K x_t$ where $K \in \R^{\dimu \times \dimx}$; we denote these policies $\pi^K$, and let $\Jlqr_{\lqr,\theta}(K) = \Jlqr_{\lqr,\theta}[\pi^K]$. Here our decision $\fraka$ is the controller $K$ and our loss $\Jlqr_{\theta}$ is $\Jlqr_{\lqr,\theta}$. Under standard conditions, $\Jlqr_{\theta}(\cdot)$ admits a unique minimizer, which we denote $\Kopt(\theta)$. Furthermore, both $\Jlqr_\theta$ and $\Kopt$ are smooth functions of $K$ and $\theta$, respectively, and can be shown to satisfy Assumption \ref{asm:smooth_informal}. 
\end{exmp}

\newcommand{\Irl}{\textsc{Irl}}
\newcommand{\Kagent}{K^{\mathrm{agent}}}
\newcommand{\uagent}{u^{\mathrm{agent}}}
\begin{exmp}[Inverse Reinforcement Learning]
In this setting, we assume there is some agent playing according to the control law $\uagent_t = \Kagent x_t$ in the system $\thetast = (\Ast,\Bst)$, inducing the closed-loop dynamics $\Aclst = \Ast + \Bst \Kagent$. Furthermore, we assume that $\Kagent = \Kopt(\thetast; \Rust)$ for some parameter $\Rust \in \R^{\dima}$ and a known map $\Kopt(\cdot ; \cdot)$. $\Ast,\Bst$ and $\Rust$ are unknown, but we are told the value of $\Kagent$ (e.g., estimated through observation of the agent's actions). We assume we have access to the \emph{closed-loop} system
\begin{align*}
x_{t+1} = \Aclst x_t + \Bst u_t + w_t
\end{align*}
and our goal is to infer the parameter, $\Rust$, the player is utilizing. This can be thought of as an inverse reinforcement learning problem, where we assume the agent is playing in order to minimize some cost parameterized by $\Rust$, and we want to determine what the cost is. In this setting our decision $\fraka$ is the cost vector $\Ru$ and we define our loss as:
$$ \Jlqr_{\Irl,\theta}(\Ru) = \| \Ru - \Ruopt(\theta) \|_\fro^2 $$
and the certainty equivalence estimate as:
$$ \Ruopt(\theta) = \argmin_{\Ru \in \R^{\dima}} \| \Kagent - \Kopt(\theta; \Ru) \|_\fro^2. $$
Under amenable parameterizations of $\Kopt(\thetast; \Rust)$, this will satisfy \Cref{asm:smooth_informal}.
\end{exmp}

\newcommand{\sid}{\textsc{Sid}}
\begin{exmp}[System Identification with Parametric Uncertainty]\label{ex:sysid_parametric}
Consider the system identification problem where we only care about estimating particular entries of $(\Ast,\Bst)$---for example, the gain of a particular actuator or the friction coefficient of a surface. In this setting, we choose our loss to be:
$$ \Jfunc_{\sid,\theta}(\thetahat) = \| \thetahat - \thetast \|_M^2 := \vectorize(\thetahat - \thetast)^\top M \vectorize(\thetahat - \thetast)$$
where $M \succeq 0$ has a value of 0 at coordinates which correspond to the known entries of $(\Ast,\Bst)$ and a value of 1 at coordinates which correspond to the unknown entries of $(\Ast,\Bst)$. Our decision, $\frakahat$, is the least squares estimate of $\thetast$.
\end{exmp}

\newcommand{\led}{\textsc{Led}}
\begin{exmp}[Linear Experimental Design]\label{ex:led}
If $\Ast = 0$, $\Bst^\top = \phist \in \R^{\dimu}$, \eqref{eq:our_dynamics} reduces to
\begin{align}\label{eq:linear_regression_design}
    y_t = \phist^\top u_t + w_t
\end{align}
for $y_t, w_t \in \R$, $u_t \in \R^{\dimu}$. This is the standard linear regression setting, and our framework therefore encompasses optimal linear experiment design in arbitrary smooth losses \citep{pukelsheim2006optimal}. For example, we may consider the $A$-optimal objective $\Jlqr_{\led,\phist}( \phi ) = \| \phi - \phist \|_2^2$. Alternatively, we could minimize the negative $\log$-likelihood relative to some reference distribution $\nu$ so that $\Jlqr_{\led,\phist}( \phi ) = \mathbb{E}_{U \sim \nu, Y \sim p(\cdot|U,\phist)}[ -\log( p(Y|U,\phi))]$ where $P(Y|U,\phist)$ is the likelihood of observations such that $y_t \sim p(\cdot| u_t,\phist)$ \citep{chaudhuri1993nonlinear,chaudhuri2015convergence,pronzato2013design}. Non-smooth $G$-optimal-like objectives such as $J_{\phist}( \phi ) = \max_{x \in \mathcal{X}} \langle x,  \wh{\phi} - \phist \rangle^2$ for some finite set $\mathcal{X} \subset \R^{d_u}$ can be captured in our framework by using an approximate smoothed objective such as $\Jlqr_{\led,\phist}( \phi ) = \frac{1}{\lambda} \log\left( \sum_{x \in \mathcal{X}} e^{\lambda \langle x, \phi - \phist \rangle^2} \right)$ for large $\lambda$. 
\end{exmp}

Many other examples exist---from more general control problems, to incentive design, and beyond. As we will show,   there is a provable gain to performing task-guided exploration on examples such as these. We present our results for general loss functions $\Jlqr$, but consider several of the examples stated here in more detail in \Cref{sec:interpret}.

\subsection{Related Works}
\paragraph{Experiment Design and Control.} Experiment design has over a century-old history in statistics, and numerous surveys have been written addressing its classical results (see e.g. \cite{pukelsheim2006optimal,pronzato2013design}).  More recently, \cite{chaudhuri2015convergence} gives a non-asymptotic active learning procedure for adaptive maximum likelihood estimation, again adapting the design to the unknown parameter of interest; unlike our work, their setting does not address dynamical systems. 

In the controls literature, there has been significant attention devoted to optimally exciting dynamical systems \citep{mehra1976synthesis,goodwin1977dynamic,jansson2005input,gevers2009identification,manchester2010input,hagg2013robust} to optimize classical design criteria for system identification.  More recent works \citep{hjalmarsson1996model,hildebrand2002identification,katselis2012application} have focused on designing inputs to meet certain task-specific objectives, as is the focus of this work.  In control, the optimal design depends on the unknown parameters of the system, and prior work rely on either robust experiment design  \citep{rojas2007robust,rojas2011robustness,larsson2012robust,hagg2013robust} or adaptive experimental design \citep{lindqvist2001identification,gerencser2005adaptive,barenthin2005applications,gerencser2007adaptive,gerencser2009identification}, the method of choice in this work, to address this challenge. Past results were often heuristic, and rigorous bounds are asymptotic in nature \citep{gerencser2007adaptive,gerencser2009identification}. In contrast, we provide finite sample upper bounds, and unconditional \emph{information-theoretic} lower bounds which validate the optimality of our approach. Our adaptive algorithm also admits an efficient implementation via projected gradient descent, whereas past designs require the solution of semi-definite programs, which may be prohibitive in high dimensions.

More recently, \cite{wagenmaker2020active} provided a finite sample analysis of system identification in the operator norm. Our work shows that designs which optimize operator norm recovery can fare arbitrarily worse for control tasks compared to task-optimal designs.
Moreover, the techniques in this work translate to providing an efficient implementation of the computationally inefficient procedure proposed by \cite{wagenmaker2020active}. In addition, our lower bounds consider a more realistic ``moderate $\delta$'' regime (see \Cref{rem:lb_compare} for comparison).

\paragraph{Non-Asymptotic Learning for Control.}
While the adaptive control problem has been extensively studied within the controls community \cite{aastrom2013adaptive}, machine learning has produced considerable recent interest in finite-time performance guarantees for system identification and control, guarantees which the classical adaptive control literature lacked. In the control setting, results have focused on finite time regret bounds for the \lqrx problem with unknown dynamics \citep{abbasi2011regret,dean2017sample,dean2018regret,mania2019certainty,dean2019safely,cohen2019learning}, with \cite{simchowitz2020naive} ultimately settling the minimax optimal regret in terms of dimension and time horizon. Others have considered regret in online adversarial settings \citep{agarwal2019online,simchowitz2020improper}. Recent results in system identification have focused on obtaining finite time high probability bounds on the estimation error of the system's parameters when observing the evolution over time \citep{tu2017non,faradonbeh2018finite,hazan2018spectral,hardt2018gradient,simchowitz2018learning,sarkar2018how,oymak2019non,simchowitz2019learning,sarkar2019finite,tsiamis2019finite}. Existing results rely on excitation from random noise to guarantee learning and do not consider the problem of learning with arbitrary sequences of inputs or optimally choosing inputs for excitation. Recent work  has begun to consider instance-optimal bounds with more targeted excitation \citep{wagenmaker2020active,ziemann2020uninformative}; the former is discussed above. The latter presents an asymptotic, instance-dependent lower bound for the online \lqrx problem. Our results, in contrast, consider offline pure-exploration for a class of tasks much more general than \lqrx, and are finite-time. Furthermore, \cite{ziemann2020uninformative} do not provide a matching upper bound for their lower bound.

\paragraph{Reinforcement Learning.} 
Viewing linear dynamical systems as a particular class of Markov Decision Processes (MDPs), our work can also be seen as studying PAC reinforcement learning (RL), where the goal is to find an $\epsilon$-good policy with probability $1-\delta$ on a fixed MDP and reward function. Existing literature on PAC RL has tended to focus on obtaining coarse, worst-case bounds \citep{dann2015sample,dann2017unifying,dann2019policy,menard2020fast}. Only recently has progress been made in obtaining instance-dependent bounds, and here the results are either restricted to the much simpler generative model setting \citep{zanette2019almost,marjani2020best}, or are asymptotic in nature and only apply to finding the \emph{optimal} policy \citep{marjani2021navigating}. In contrast, our work provides tight, non-asymptotic, and instance-dependent upper and lower bounds for finding $\epsilon$-good policies, albeit in a restricted class of continuous RL problems.

%\subsection{Organization}
%The remainder of this paper is organized as follows. In \Cref{sec:results_summary} we provide an overview of our results, state an informal version of \algname, and introduce the essential quantities used in our analysis. \Cref{sec:interpret} states several corollaries of our main result in specific settings and gives our bound for the \lqrx problem. \Cref{sec:taskopt_improvements} provides explicit examples where task-guided exploration yields provable gains over task-agnostic exploration, and \Cref{sec:experiments_body} presents numerical experiments demonstrating that this improvement occurs in practice as well. %We next move on to formal statements of \algname and our results in \Cref{sec:overview_lds}. Finally, \Cref{sec:M_norm_regression} and \Cref{sec:mdm_lb_body_arxiv} provide an overview of the technical details of this work, stating our main technical tool---upper and lower bounds on martingale regression in general norms---and providing a proof of our lower bound. 
%We close in \Cref{sec:conclusion} with several interesting questions motivated by this work.

\subsection{Organization}
The remainder of this paper is organized as follows. In \Cref{sec:results_summary} we provide an overview of our results, state an informal version of \algname, and introduce the essential quantities used in our analysis. \Cref{sec:interpret} states several corollaries of our main result in specific settings and gives our bound for the \lqrx problem. \Cref{sec:taskopt_improvements} provides explicit examples where task-guided exploration yields provable gains over task-agnostic exploration, and \Cref{sec:experiments_body} presents numerical experiments demonstrating that this improvement occurs in practice as well. We next move on to formal statements of \algname and our results in \Cref{sec:overview_lds}. Finally, \Cref{sec:M_norm_regression} and \Cref{sec:mdm_lb_body_arxiv} provide an overview of the technical details of this work, stating our main technical tool---upper and lower bounds on martingale regression in general norms---and providing a proof of our lower bound. We close in \Cref{sec:conclusion} with several interesting questions motivated by this work.

\section{Summary of Results}\label{sec:results_summary}
We now turn to the presentation of our results. We assume we are in the setting described in \Cref{sec:setting}. 
Throughout, we let $\ohst(\cdot)$ suppress terms polynomial in problem parameters, $\log \frac{1}{\delta}$, and $\log\log T$; we let $a \lesssim b$ if $a \le C \cdot b$ for a universal constant $C > 0$.
\newcommand{\taskhes}{\mathcal{H}}
\subsection{Optimality of Certainty Equivalence}
Before describing the optimal exploration policy for collecting data, we resolve the optimal procedure for synthesizing a decision, and its sample complexity. Given a trajectory $\traj = (x_{1:T+1},u_{1:T})$, the   \emph{least squares} estimator of $\thetast$ is
\begin{align*}
\thetals(\traj):= \argmin_{A,B}\sum_{t=1}^T \|x_{t+1}  - A x_t - B u_t \|_2^2.
\end{align*} 
Note that $\thetals$ is the maximum-likelihood estimator of $\thetast$. For our upper bounds, we propose the \emph{certainty-equivalent} decision rule:
\begin{defn}
The \emph{certainty equivalence} decision rule selects the optimal control policy for the least-squares estimate of the dynamics; $\ace(\traj) := \aopt(\thetals(\traj))$. 
\end{defn}
Certainty Equivalence has a long history in controller design \citep{theil1957note, simon1956dynamic}. 
To analyze this strategy, we quantify both the error in our least squares estimator, and how it translates into uncertainty about the control synthesis. The former is  quantified in terms of the expected covariance matrices under exploration policies:
\begin{align*}
& \Gamma_{T}(\pi;\theta) := \frac{1}{T}\Exp_{\theta,\pi}\left[\sum_{t=1}^T \begin{bmatrix}x_t\\
u_t \end{bmatrix}\begin{bmatrix}x_t\\
u_t \end{bmatrix}^\top \right],\\
& 
\boldsymbol{\Gamma}_T(\pi;\theta) := I_{\dimx} \otimes \Gamma_T(\pi;\theta).
\end{align*}
where $\otimes$ denotes the Kronecker product. The latter requires that we measure how uncertainty in $\theta$ translates into uncertainty about the optimal decision rule for the task of interest:
\begin{defn}[Model-Task Hessian and Idealized Risk] We define the \emph{model-task Hessian} as  
\begin{align*}\taskhes(\thetast) := \nablatwo_\theta \calR_{\thetast}(\aopt(\theta)) \big{|}_{\theta = \thetast},
\end{align*} 
and the idealized risk 
$$\Phi_T(\pi;\thetast) := \tr(\taskhes(\thetast) \matGam_{T}(\pi;\thetast)^{-1}).$$
\end{defn}
Intuiviely, the model-task Hessian measures the local curvature of  $\Jfunc_{\thetast}(\fraka)$, as the decision $\fraka$ varies along the directions of optimal policies $\aopt(\theta)$ for parameters $\theta$ in a neighborhood of $\thetast$. The idealized risk capture how the least-squares error propagates through this uncertainty. 

Our results will show that $\Phi$ characterizes the instance-optimal sample complexity for decision making, and consequently, by optimizing over $\pi$, of pure exploration. To formalize this, we require a notion of \emph{local minimax risk}:
\begin{defn}[Decision-Making Local Minimax Risk] Let $\calB \subset \Theta$ denote a subset of instances. The $T$-sample local minimax decision risk on $\calB$ under exploration policy $\piexp$ is
\begin{align*}
\mmax_{\piexp}(\calR; \calB) := \min_{\plan}\max_{\theta \in \calB} \Exp_{\traj \sim \theta,\piexp}\left[\Rlqr_{\theta}(\plan(\traj))\right]
\end{align*}
where the minimization is over all maps from trajectories to decisions. Typically, we shall let $\calB$ take the form $\calB_{\fro}(r;\thetast) := \{\theta: \|\theta - \thetast\|_{\fro} \le r\}$.
\end{defn}

By choosing $\calB$ to contain only instances close to $\thetast$, the local minimax risk captures the difficulty of completing our task on the specific instance $\thetast$, yielding an effectively instance-specific lower bound. Finally, we make the following assumption on the system dynamics.

\begin{asm}\label{asm:stable_systems}
Let $\Theta $ denote the set of all stable $\theta$: $\Theta := \{\theta = (A,B): \rho(A) < 1\}$, where $\rho(A)$ denotes the spectral radius of $A$. We assume that $\thetast \in \Theta$.
\end{asm}

While this assumption restricts our results to stable systems, similar assumptions are standard in much of the recent literature. Appendix \ref{sec:remarks} discusses generalization to unstable systems. Under this assumption we have the following result.

\begin{thm}[Optimality of Certainty Equivalence]\label{prop:ce_optimal} Let $\piexp$ be any sufficiently regular policy, and consider $\Jfunc_{\theta}(\fraka)$ and $\aopt(\theta)$ satisfying Assumption \ref{asm:smooth_informal} and $\thetast$ satisfying Assumption \ref{asm:stable_systems}. Then for all $\delta \in (0,1/3)$ and all $T$ sufficiently large, a trajectory $\traj$ generated by $\piexp$ and $\thetast$ satisfies the following with probability $1-\delta$,
\begin{align*}
\Rlqr_{\thetast}(\ace(\traj)) \lesssim  \sigw^2 \log( \tfrac{\dimx}{\delta}) \cdot {\color{blue}\frac{\Phi_T(\piexp;\thetast)}{T}}    + \ohst\left(\tfrac{1}{T^{3/2}}\right) .
\end{align*}
Moreover,  for some $r = \Omega(1/T^{5/12})$, and $\calB = \calB_{\fro}(r;\thetast)$, the synthesis minimax risk is lower bounded as
\begin{align*}
\mmax_{\piexp}(\calR; \calB) \ge \frac{\sigw^2}{3} \cdot {\color{blue}\frac{\Phi_T(\piexp;\thetast)}{T}}  - \ohst \left ( \tfrac{1}{T^{5/4}} \right ).
\end{align*}
\end{thm}

In \Cref{sec:overview_lds} we state the full version of this result, which holds in a more general \emph{martingale decision making} setting encompassing certain instances of the nonlinear system formulation considered in \cite{mania2020active}. This result establishes that, for a \textit{given} exploration policy $\piexp$ and for $T$ sufficiently large, the certainty equivalence decision $\ace(\traj)$ is the locally minimax optimal synthesis rule---there does not exist a more efficient way to utilize the acquired information to produce a decision. Note that, under some reasonable assumptions, an expectation bound can be obtained from the high probability bound. We precisely quantify what it means for a policy to be sufficiently regular in \Cref{sec:overview_lds}. In short, it entails that the policy sufficiently excites the system, and that the covariates concentrate to their mean. Lastly, note that our lower bound differs substantively from the $\delta \to 0$ lower bounds common in the adaptive estimation literature (see \Cref {rem:lb_compare}). Appendix \ref{sec:remarks} provides a more thorough discussion of these points.

\begin{proof}[Proof Sketch of \Cref{prop:ce_optimal}]
For the proof of the lower bound, we first show that for any decision $\frakahat = \plan(\traj)$ for which $\calR_\theta(\frakahat)$ is small, we can infer an instance $\thetahat(\frakahat)$ such that $\| \thetahat(\frakahat) - \theta \|_{\taskhes(\thetast)}^2$ is also small (see \Cref{sec:general_decision}). This equivalence reduces our problem to that of estimating $\theta$ in the $\taskhes(\thetast)$ norm. We then show a lower bound on a Gaussian martingale regression problem with general quadratic losses via a careful though elementary Bayesian MMSE computation. Unlike vanilla Cramer-Rao, this approach allows us to obtain a lower bound which holds for any estimator, not simply unbiased estimators (see \Cref{sec:M_norm_regression}). Combining these results gives the stated lower bound. The proof of our upper bound mirrors this: we approximate  $\calR_{\thetast}(\ace(\traj))$ as a quadratic, $\| \thetahat(\traj) - \thetast \|_{\taskhes(\thetast)}^2$, and prove an upper bound on the error of the least squares estimator for martingale regression in general norms (see \Cref{sec:ddm_ce_upper}).
\end{proof}

\subsection{Task-Optimal Experiment Design}

Given that the optimal risk for a fixed exploration policy $\piexp$ is governed by 
$$\Phi_T(\piexp;\thetast) = \tr(\taskhes(\thetast) \matGam_{T}(\piexp;\thetast)^{-1}),$$
it stands to reason that the optimal design procedure seeks to minimize this quantity. To this end, we introduce several quantities describing the optimality properties.

\newcommand{\policyset}{\Pi_{\gamma^2}}
\begin{restatable}[Power-Constrained Policies]{defn}{pigam}\label{def:pigamma}
Let $\policyset$ denote the set of causal polices that have expected average power bounded as $\gamma^2$. That is, for any $\pi \in \policyset$, we will have $\Exp_{\theta,\pi}[\sum_{t=1}^T \| u_t \|_2^2] \le T \gamma^2$ for all $\theta$. 
\end{restatable}

\begin{defn}[Optimal Risk]
We define:
\begin{align*}
\Phiopt(\gamma^2;\thetast) := \liminf_{T \to \infty} \inf_{\piexp \in \Pi_{\gamma^2}} \Phi_T(\piexp; \thetast),
\end{align*}
the risk obtained by the policy minimizing the complexity $\Phi_T(\piexp; \thetast)$. 
\end{defn}

\begin{defn}[Exploration Local Minimax Risk] Let $\calB \subset \Theta$ denote a subset of instances. The $T$-sample local minimax exploration risk on $\calB$ with budget $\gamma^2$ is
\begin{align*}
\mmax_{\gamma^2}(\calR;\calB) := \min_{\piexp \in \Pi_{\gamma^2}} \min_{\plan} \max_{\theta \in \calB} \Exp_{\traj \sim \theta, \piexp}[\calR_{\theta}(\plan(\traj))].
\end{align*}
\end{defn}

\paragraph{Algorithm Sketch.} We are now ready to state our algorithm. \algname proceeds in epochs. At each epoch it chooses a policy $\pi$ that minimize the certainty-equivalence design objective, $\Phi_T(\pi;\thetahat)$, based on the estimate of the system's parameters produced in the previous epoch. As the estimate of $\thetast$ is refined, the exploration policy is improved, and ultimately achieves near-optimal excitation of the system for the task of interest.

In the policy optimization step on \Cref{line:exp_des}, we optimize over a restricted class of policies, $\policyset^{\mathrm{p}}$, which contains only \emph{periodic} signals. As we show, this restriction is expressive enough to contain a near-optimal policy, while allowing us to represent $\matGam_{T}(\pi; \thetahat_i)$ in a convenient frequency-domain form. We then adopt a (sharp) \emph{convex relaxation} of these policies that transforms the experiment design into a convex program, admitting a simple, efficient projected gradient descent implementation. A formal definition of \algname and detailed explanation of these points is given in \Cref{sec:overview_lds}.

\newcommand{\tosed}{\textsc{Tosed}}
\begin{algorithm}[h]
  	\begin{algorithmic}[1]
  	\State{}\textbf{Input: } Initial epoch length $T_0$, budget $\gamma^2$
	\State $\pi_0 \leftarrow \calN(0,\gamma^2/\dimu \cdot I)$.
    \For{phase $i =0,1,2,\ldots$}
    \State Run system for $T_0 2^i$ steps, playing input 
   \Statex \hspace{3em} $u_t = \pi_i(x_{1:t},u_{1:t-1})$
    \State Compute least squares estimate
    \Statex \hspace{3em} $\thetahat_{i} \in \argmin_{A,B} {\textstyle \sum}_{t=1}^{T}\|x_{t+1} - Ax_t - Bu_t\|_2^2$
    \State Select policy for epoch $i+1$,\label{line:exp_des}
    \Statex \hspace{3em} $\pi_{i+1} \leftarrow \argmin_{\policyset^{\mathrm{p}}} \tr(\taskhes(\thetahat_i) \matGam_{T}(\pi; \thetahat_i)^{-1})$ 
    \EndFor
  \end{algorithmic}
  \caption{\textbf{T}ask-\textbf{OP}tima\textbf{L} \textbf{E}xperiment Design (\algname), Informal}
  \label{alg:tosed}
\end{algorithm}

\begin{thm}[Task-Optimal Experiment Design]\label{thm:exp_design_opt} Consider $\Jfunc_{\theta}(\fraka)$ and $\aopt(\theta)$ satisfying Assumption \ref{asm:smooth_informal} and $\thetast$ satisfying Assumption \ref{asm:stable_systems}. For sufficiently large $T$, the trajectory $\traj$ generated by \Cref{alg:tosed} enjoys the following guarantee with probability at least $1-\delta$:
\begin{align*}
\Rlqr_{\thetast}(\ace(\traj)) \lesssim \sigw^2 \log (\tfrac{\dimx}{\delta}) \cdot {\color{blue}\frac{\Phiopt(\gamma^2; \thetast)}{T}} + \ohst\left(\tfrac{1}{T^{3/2}}\right).
\end{align*}
Moreover, it produces inputs satisfying $\Exp_{\thetast,\algname}[\sum_{t=1}^T \| u_t \|_2^2] \le T \gamma^2$, and can be implemented in polynomial time. Finally, for $r = \ohst(1/T^{5/12})$ and $\calB = \calB_{\fro}(r;\thetast)$, the local minimax risk is lower bounded by
\begin{align*}
\mmax_{\gamma^2}(\calR;\calB) \ge \frac{\sigw^2}{64} \cdot {\color{blue}\frac{\Phiopt(\gamma^2;\thetast)}{T}} - \ohst \left ( \tfrac{1}{T^{5/4}} \right ).
\end{align*}
\end{thm}

We emphasize that the only assumptions needed for Theorem \ref{thm:exp_design_opt} to hold are that our system, $\thetast$, is stable, and that the loss we are considering, $\calR_{\thetast}(\fraka)$, is sufficiently smooth. For \textit{any} system and \textit{any} loss satisfying these minimal assumptions, including those stated in \Cref{sec:ex_intro}, Theorem \ref{thm:exp_design_opt} shows that certainty equivalence decision making is instance-wise optimal, and that \algname hits this optimal rate. Furthermore, while $\algname$ relies on experiment design, its sample complexity is also optimal over algorithms which incorporate feedback. We precisely quantify the lower order terms and burn-in times necessary for this result to hold in \Cref{sec:overview_lds}, and consider relaxations to our assumptions in Appendix \ref{sec:remarks}.

\begin{proof}[Proof Sketch of \Cref{thm:exp_design_opt}] The key technical difficulty lies in proving that our restricted class of policies, $\policyset^{\mathrm{p}}$, contains a near-optimal policy. We show this in \Cref{sec:lds_lb} by a careful truncation argument and application of Caratheodory's Theorem. Given this, the lower bound follows by a similar argument as in \Cref{prop:ce_optimal}. For the upper bound, we show that once $\thetast$ has been estimated well enough, the certainty equivalence experiment design on \Cref{line:exp_des} achieves the near-optimal rate (see \Cref{sec:lds_exp_design}). 
\end{proof}

%% file: body/comparison.tex
%!TEX root = ../main.tex

\section{Interpreting the Results}\label{sec:interpret}

To make our results more concrete, we return to the examples introduced in \Cref{sec:ex_intro}, and show how \algname applies in these settings.

\newcommand{\Rlqrb}{\calR_{\lqr}}
\newcommand{\Philqr}{\Phi_{\lqr}}
\subsection{Instance-Optimal LQR Synthesis}\label{sec:lqr_summary}
Consider the pure exploration \lqrx problem stated in \Cref{ex:lqr}. We define
\begin{align*}
\calR_{\lqr,\thetast}(K) := \Jlqr_{\lqr,\thetast}(K) - \min_{K} \Jlqr_{\lqr,\thetast}(K) 
\end{align*}
where $\Jlqr_{\lqr,\thetast}(K)$ is given in \Cref{ex:lqr}. Recall the \emph{discrete algebraic Ricatti equation}, defined for some $(A,B)$:
\begin{align*}
P = A^\top P A - A^\top P B(\Ru + B^\top P B) B^\top P A + \Rx
\end{align*}
If $\theta$ is stabilizable and $\Rx,\Ru \succ 0$, it is a well-known fact that this has a unique solution, $P \succeq 0$. We denote the solution for the instance $\thetast = (\Ast,\Bst)$ by $\Pst$. We also recall the definition of the $\mathcal{H}$-infinity norm of a system:
\begin{align*}
\| \Ast \|_{\Hinf} = \max_{\omega \in [0,2\pi]} \| (e^{\imag \omega} I - \Ast)^{-1} \|_\op
\end{align*}
Finally, we let $\Philqr(\gamma^2;\thetast) := \Phiopt(\gamma^2;\thetast)$ in the case when our loss is the \lqrx loss, $\Jlqr_{\thetast} = \Jlqr_{\lqr,\thetast}$. Given these definitions, the following corollary shows the performance of \algname on the pure exploration \lqrx problem, and that relevant quantities can be expressed in terms of the problem-dependent constants $\| \Pst \|_\op$, $\| \Bst \|_\op$, and $\| \Ast \|_{\Hinf}$.

\newcommand{\Clqra}{C_{\lqr}}
\newcommand{\Clqrb}{C_{\lqr,2}}
\newcommand{\Clqrc}{C_{\lqr,3}}
\begin{cor}
As long as  $T  \ge \Clqra (\dimx \log^2 T + \dimx^2)$, with probability at least $1-\delta$, \algname achieves the following rate for the \lqrx problem:
\begin{align*}
\calR_{\lqr,\thetast}(\ace(\traj)) \lesssim \sigw^2 \log(\tfrac{\dimx}{\delta}) \cdot{\color{blue}\frac{\Philqr(\gamma^2; \thetast)}{T}} + \tfrac{\Clqra \dimx^{5}   }{T^{3/2}}.
\end{align*}
Furthermore, any algorithm must incur the following loss:
\begin{align*}
\mmax_{\gamma^2}(\calR_{\lqr};\calB) \ge \frac{\sigw^2}{64} \cdot {\color{blue}\frac{\Philqr(\gamma^2;\thetast)}{T}} - \tfrac{\Clqra \dimx^5}{T^{5/4}}
\end{align*}
where $\calB$ is as in \Cref{thm:exp_design_opt} and $\Clqra = \Clqra'/$ $\min \{ \sigma_w^6, \gamma^6/\dimu^3, 1 \}$ for $\Clqra'$  polynomial in 
 $\| \Pst \|_\op$, $\| \Bst \|_\op$, $\| \Bst \|_\op^{-1}, \| \Ast \|_{\Hinf}, \| \Ru \|_\op, \gamma^2, \sigma_w^2, \dimu$, $\log \log T$, and $\log \tfrac{1}{\delta}$. 
\end{cor}

As this result shows, \algname is instance-optimal for the \lqrx problem, with sample complexity governed by the constant $\Philqr(\gamma^2;\thetast)$. To the best of our knowledge, this is the first algorithm provably instance-optimal for \lqrx---albeit in the offline \lqrx setting. 

\subsection{System Identification in Arbitrary Norms}
\newcommand{\Phisid}{\Phi_{\sid}}
\newcommand{\Csid}{C_{\sid}}
Next, we consider the case of system identification in arbitrary norms outlined in \Cref{ex:sysid_parametric}. In this setting our loss is $\calR_{\sid,\thetast}(\thetahat) := \Jfunc_{\sid,\thetast}(\thetahat) = \| \thetahat - \thetast \|_M^2$, and it can be shown our idealized risk is $\Phi_T(\piexp;\thetast) = \tr(M \matGam_{T}(\piexp;\thetast)^{-1})$. Defining
\begin{align*}
\Phisid(\gamma^2;\thetast) := \liminf_{T \rightarrow \infty} \inf_{\piexp \in \policyset} \tr(M \matGam_{T}(\piexp;\thetast)^{-1})
\end{align*}
\Cref{thm:exp_design_opt}  implies that 
\begin{align*}
\calR_{\sid,\thetast}(\ace(\traj)) \lesssim \sigw^2 \log(\tfrac{\dimx}{\delta}) \cdot {\frac{\Phisid(\gamma^2; \thetast)}{T}} + \tfrac{\Csid \tr(M)  \dimx^{3}  }{T^{3/2}}
\end{align*}
and that this rate is instance-optimal, for some constant $\Csid$ polynomial in $\| \Bst \|_\op$, $\| \Ast \|_{\Hinf}$, $\gamma^2, \sigma_w^2, \dimu$, $\log \tfrac{1}{\delta}$, and $\log \log T$. In particular, if $M = I$ our loss $\calR_{\sid,\thetast}(\thetahat)$ reduces to the Frobenius norm, implying that \algname is the optimal Frobenius norm identification algorithm.

\section{Task-Guided Exploration yields Provable Gains}\label{sec:taskopt_improvements}
\newcommand{\piop}{\pi_{\op}}
\newcommand{\pinoise}{\pi_{\mathrm{noise}}}
\newcommand{\pifro}{\pi_{\mathrm{fro}}}

We turn now to several examples which illustrate that taking into account the task of interest when performing exploration yields provable gains over task-agnostic exploration schemes. We focus on the \lqrx setting and compare against the following natural exploration baselines:
\begin{itemize}
    \item \textbf{System Identification in Operator Norm} \citep{wagenmaker2020active}: Let $\piop$ denote the exploration policy that is optimal for estimating $\thetast = (\Ast,\Bst)$ under the operator norm $\|\thetahat - \thetast \|_{\op} = \sup_{u: \|u\|_2 \leq 1} \| (\thetahat - \thetast )u \|_2$. Explicitly, $\piop = \argmax_{\pi \in \Pi_{\gamma^2}} \lammin(\matGam_{T}(\pi;\thetast))$.
    \item \textbf{System Identification in Frobenius Norm}: Let $\pifro$ denote the exploration policy that is optimal for estimating $\thetast = (\Ast,\Bst)$ under the Frobenius norm: $\pifro = \argmin_{\pi \in \Pi_{\gamma^2}} \tr(\matGam_{T}(\pi;\thetast)^{-1})$.
    \item \textbf{Task-Optimal Gaussian Noise}: Let $\pinoise$ denote the exploration policy such that $\pinoise =: \pinoise(\Lambda_\star)$ where $\pinoise(\Lambda_\star)$ plays the inputs $u_t \sim \calN(0,\Lambda_\star)$ and 
    $$\Lambda_\star = \arg\min_{\Lambda : \tr(\Lambda) \le \gamma^2} \tr(\taskhes(\thetast)\matGam_{T}(\pinoise(\Lambda);\thetast)^{-1}).$$
\end{itemize}
In stating our results, we overload notation and let $\calR_{\lqr,\thetast}(\piexp) = \calR_{\lqr,\thetast}(\ace(\traj))$ for $\traj \sim \piexp, \thetast$. We are concerned primarily in how the complexity scales with the dimension, $\dimx$, and $\frac{1}{1-\rho}$ where $\rho$ is the spectral radius of the system, and use $\Theta(\cdot)$ and $\calO(\cdot)$ to suppress lower order dependence on these terms. Our first example shows that, if $(\Ast,\Bst)$ is properly structured, \algname achieves a tighter scaling in $\frac{1}{1-\rho}$ than all naive exploration approaches.

\begin{prop}\label{lem:lqrex1}
Consider the system $\Ast = \rho \mathbf{e}_1 \mathbf{e}_1^\top$, $\Bst = b I$, $\Rx = \kappa I$, and $\Ru = \mu I$. There exist values of $b, \kappa,\mu,$ and $\sigmaw$ such that the loss of \algname,  optimal operator norm identification \citep{wagenmaker2020active},  optimal Frobenius norm identification, and optimally exciting Gaussian noise have the following scalings:
\begin{align*}
\calR_{\lqr,\thetast}(\algname) &= \calO \left ( \tfrac{\dimx^2}{(1-\rho)^{ 2} } \tfrac{\sigmaw^2}{\gamma^2 T} \right ) \\
\calR_{\lqr,\thetast}(\pifro) &= \Theta \left (  \Big (   \tfrac{\dimx^2}{(1-\rho)^2} + {\color{red}\tfrac{\dimx}{(1-\rho)^{ 5/2} }}  \Big ) \tfrac{\sigmaw^2}{\gamma^2 T} \right ) \\
\calR_{\lqr,\thetast}(\piop) &= \Theta \left (  \Big ( \tfrac{\dimx^2}{(1-\rho)^2} + {\color{red}\tfrac{\dimx}{(1-\rho)^{3} }}  \Big ) \tfrac{\sigmaw^2}{\gamma^2 T} \right ) \\
\calR_{\lqr,\thetast}(\pinoise) &= \Theta \left ( \Big ( \tfrac{{\color{red}(1-\rho)^{-1} + d_x^4(1-\rho)}}{(1-\rho)^{2} }  \Big ) \tfrac{\sigmaw^2}{\gamma^2 T} \right ).
\end{align*}
\end{prop}

\algname achieves the optimal scaling in $\frac{1}{1-\rho}$ and, as $\rho \rightarrow 1$, will outperform other approaches by an arbitrarily large factor. In addition, we note that Frobenius norm identification outperforms operator norm identification for this task. A key ingredient in the proof of this result is our convex relaxation of the optimal policy computation. Intuitively, on this instance, the first coordinate is easily excited and $\piop$ and $\pifro$ will therefore devote the majority of their energy to reducing the uncertainty in the remaining coordinates. However,  the \lqrx cost will primarily be incurred in the first coordinate due to the same effect---this coordinate is easily excited and therefore the first coordinate of the state grows at a much faster rate. As such, the task-optimal allocation does the opposite of $\piop$ and $\pifro$ and seeks to learn the first coordinate more precisely than the remaining coordinates so as to mitigate this growth.

In our next example, our system behaves isotropically but our costs are non-isotropic. As a result, certain directions incur greater cost than others, and the task-optimal allocation seeks to primarily reduce uncertainty in these directions.

\begin{prop}\label{prop:lqrex2}
Consider the system $\Ast = \rho I, \Bst = I, \Rx = I + \kappa e_1 e_1^\top$ and $\Ru = \mu I$. Then there exists a choice of $\mu,\kappa,$ and $\sigmaw$ such that 
\begin{align*}
\calR_{\lqr,\thetast}(\algname) &= \calO \left (  \tfrac{{\color{red} 1}}{(1-\rho)^4 } \tfrac{\sigmaw^2}{\gamma^2 T} \right ) \\
 \calR_{\lqr,\thetast}(\piop) = \calR_{\lqr,\thetast}(\pifro) &= \Theta \left (  \tfrac{{\color{red} \dimx }}{(1-\rho)^4 } \tfrac{\sigmaw^2}{\gamma^2 T} \right )\\
 \calR_{\lqr,\thetast}(\pinoise) &= \Theta \left ( \tfrac{{\color{red} \dimx^2}}{(1-\rho)^4 } \tfrac{\sigmaw^2}{\gamma^2 T} \right ).
\end{align*}
\end{prop}
We note that \algname improves on task-agnostic exploration by a factor of at least the dimensionality. These examples make clear that,  in the setting of a linear dynamical system, when our goal is to perform a specific task, exploration agnostic to this task can be arbitrarily suboptimal.

\subsection{Suboptimality of Low-Regret Algorithms}

\newcommand{\pilr}{\pi_{\mathrm{lr}}}
\newcommand{\Klr}{K_{\mathrm{lr}}}
In contrast to our pure-exploration setting, where we do not incur cost during exploration, a significant body of work exists on regret-minimization for the \emph{online} \lqrx problem with unknown $\Ast,\Bst$. Here the goal is to choose a \emph{low regret} policy $\pilr$ so as to minimize
\begin{align*}
\mathrm{Reg}_T := \Exp_{\thetast,\pilr} \Big [{\textstyle \sum}_{t=1}^T \ell(x_t,u_t) \Big ] - T \min_{K} J_{\lqr,\thetast}(K)
\end{align*}
for $\ell(x_t,u_t)$ as defined in \Cref{ex:lqr}. While our objectives differ, it would seem a natural strategy to run a low-regret algorithm for $T$ steps to obtain a controller $\Klr$, and then evaluate the cost $\Jlqr_{\lqr,\thetast}$ on this $\Klr$. The following result shows that there is a fundamental tradeoff between regret and estimation; in particular, the optimal $\Theta_{\star}(\sqrt{T})$ (see \cite{simchowitz2020naive}) regret translates to a (very suboptimal) $\Omega_{\star}(1/\sqrt{T})$ excess risk $\calR_{\thetast,\lqr}(\Klr)$.  
\begin{prop}[Suboptimality of Low Regret, Informal]\label{prop:informal_lr}
For any sufficiently large $T$ and any regret bound $R \in [\sqrt{T},T]$, any policy $\pilr$ with regret $\Exp_{\pilr,\thetast}[\mathrm{Reg}_T] \le R$ which returns a controller $\Klr$ as a function of its trajectory must have $\Exp_{\thetast,\pilr}[\calR_{\lqr,\thetast}(\Klr)]=  \Omega_{\star}(\frac{\dimu^2 \dimx}{R})$.
\end{prop}
In particular, \Cref{prop:informal_lr} implies that popular low-regret strategies, such as optimism-in-the-face-of-uncertainty \citep{abbasi2011regret,abeille2020efficient}, are highly suboptimal in our setting.  The key intuition behind the proof is that low regret algorithms converge to inputs $\matu_t \approx \Kst \matx_t$ approaching the optimal control policy; in doing so, they under-explore directions \emph{perpendicular} to the hyperplane $\{(x,u): u = \Kst x\}$, which are necessary for identifying the optimal control policy. We formally state and prove this result in \Cref{sec:insufficiency_low_reg}.

\section{Numerical Experiments}\label{sec:experiments_body}

\begin{figure*}[h]
     \centering
     \hfill
     \begin{minipage}[b]{0.33\textwidth}
         \centering
          \includegraphics[width=\linewidth]{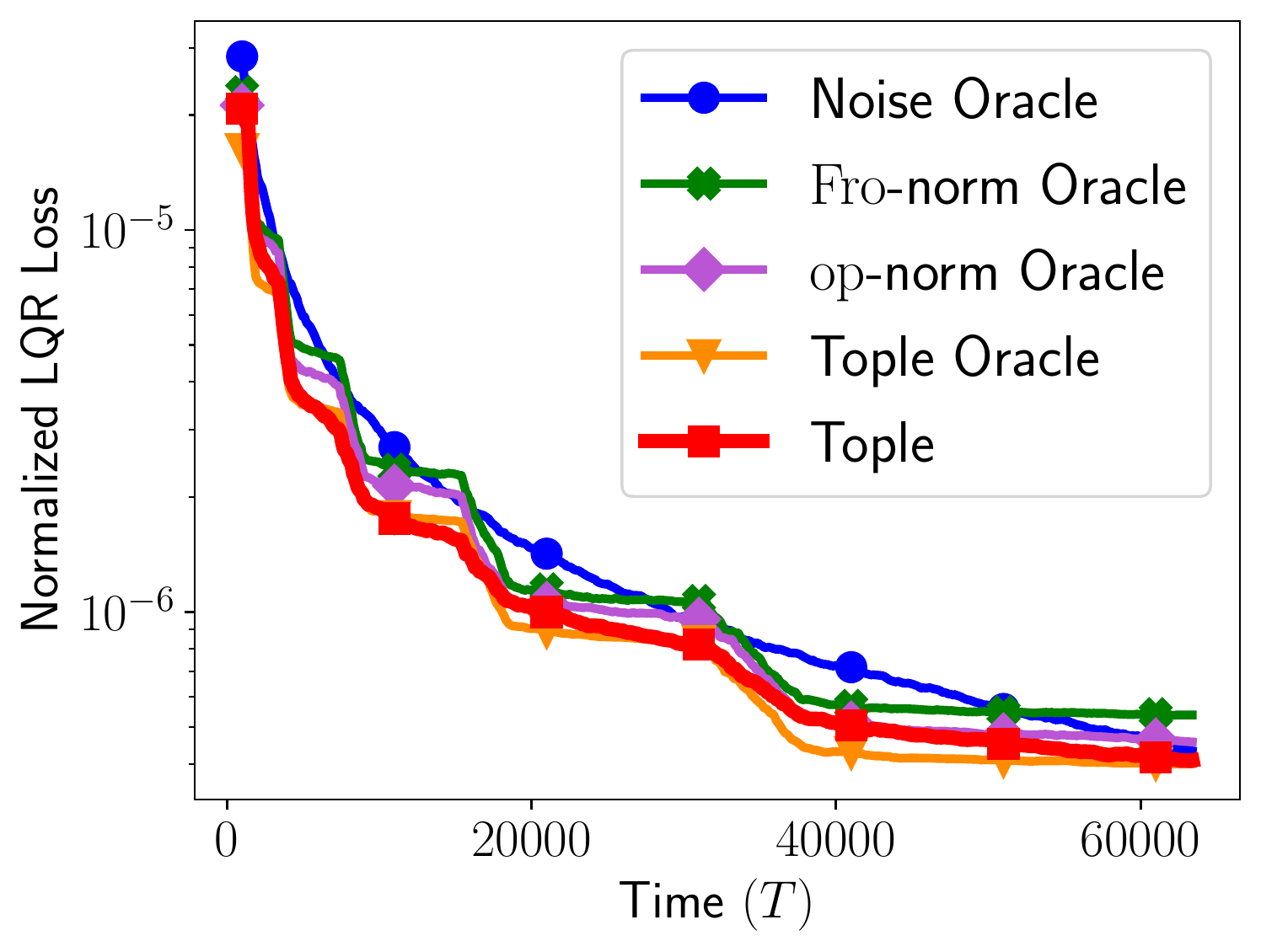}
  \caption{\lqrx loss vs time on $\Ast$ a Jordan block and $\Bst,\Rx,\Ru$ randomly generated.}
       \label{fig:jordan}
     \end{minipage}
     \hfill
     \begin{minipage}[b]{0.31\textwidth}
         \centering
          \includegraphics[width=\linewidth]{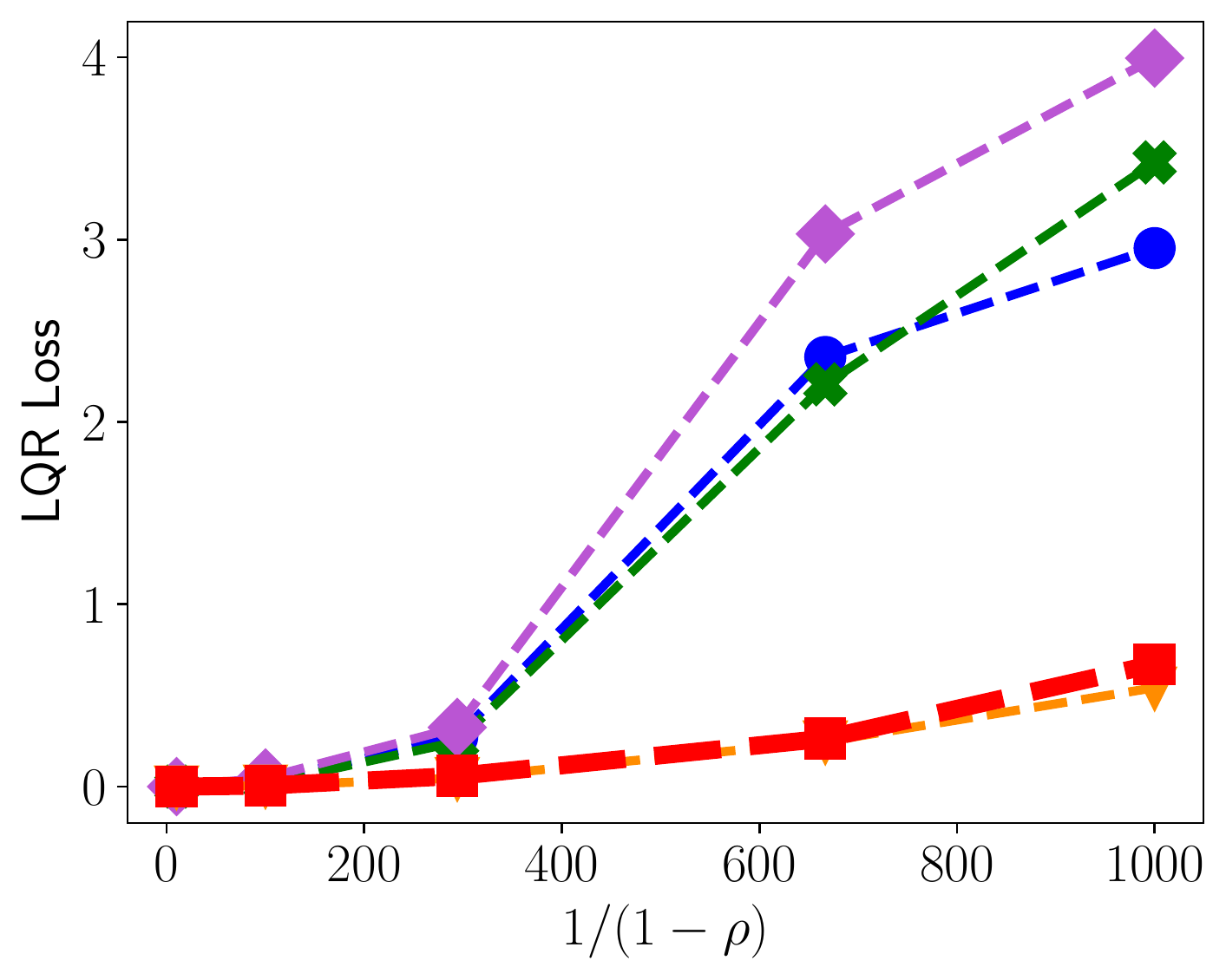}
  \caption{\lqrx loss when varying $\rho$ on example stated in \Cref{lem:lqrex1}.}
  \label{fig:ex1}
     \end{minipage}
     \hfill
     \begin{minipage}[b]{0.33\textwidth}
         \centering
          \includegraphics[width=\linewidth]{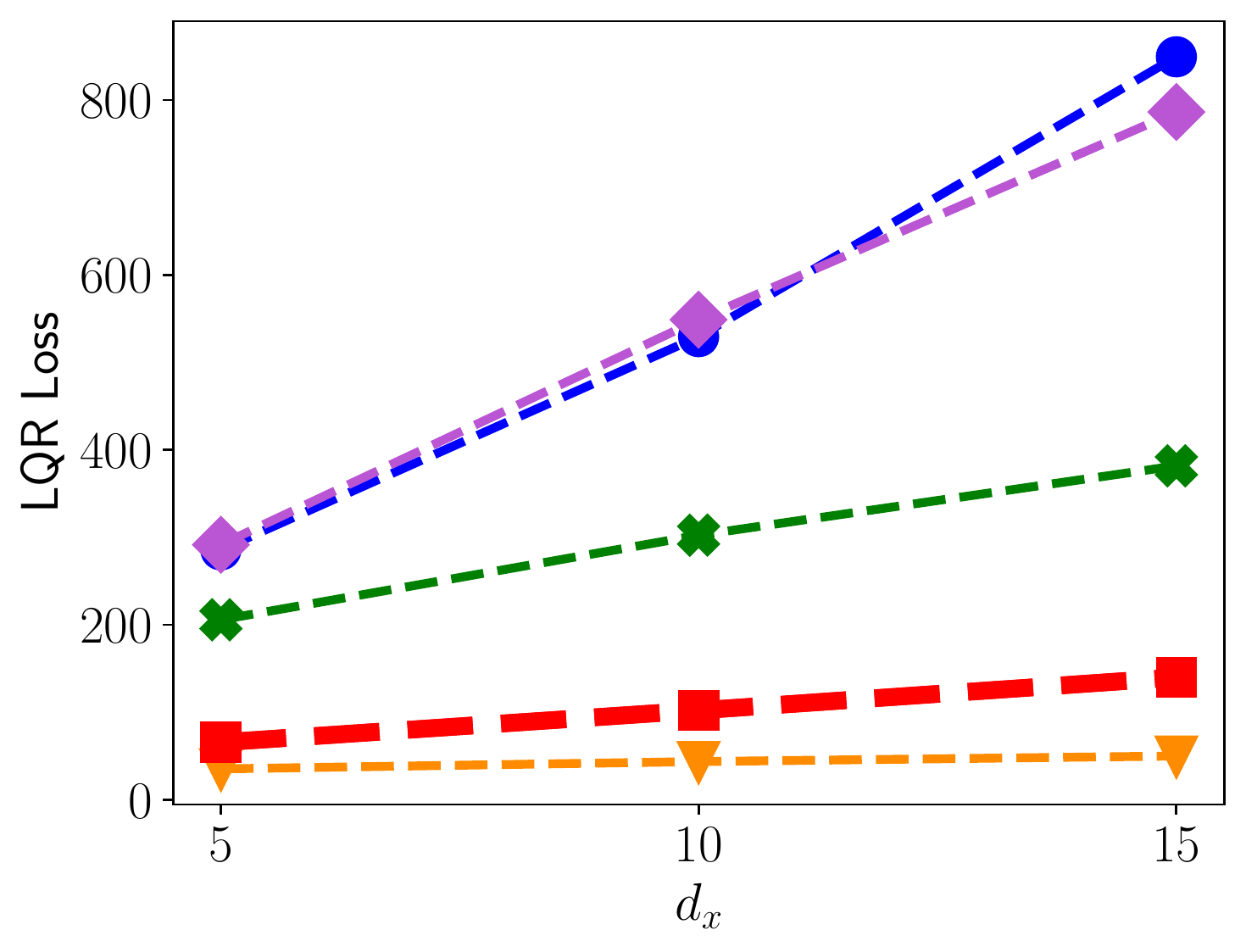}
  	\caption{\lqrx loss when varying $\dimx$ on example stated in \Cref{prop:lqrex2}. }
	\label{fig:ex2}
     \end{minipage}
\end{figure*}

Finally, we show that task-guided exploration yields practical gains. Figures \ref{fig:jordan}, \ref{fig:ex1}, and \ref{fig:ex2} illustrate the performance of \algname on several instances of the pure-exploration \lqrx problem. We compare against the baselines presented in \Cref{sec:taskopt_improvements} and the oracle task-optimal algorithm (which we refer to as ``\algname Oracle''). For all baselines, we compute the inputs in an oracle, offline manner, using knowledge of $\Ast$ and $\Bst$, and play them for the entire trajectory. Our implementation of \algname follows precisely the formal statement of the algorithm given in \Cref{sec:overview_lds}, and we rely on the aforementioned convex relaxation and a projected gradient descent solution to efficiently solve the experiment design problem. This convex relaxation can, in fact, also be applied to the optimal operator norm identification algorithm, rendering the algorithm from \cite{wagenmaker2020active} computationally efficient. We therefore rely on this relaxation and a projected subgradient descent method in our implementation of the operator norm identification algorithm. All data points correspond to averaging over at least 50 runs of the algorithm. Additional details and plots with error bars are provided in \Cref{sec:exp_details}. 

Figures \ref{fig:ex1} and \ref{fig:ex2} illustrate performance on the instances stated in \Cref{lem:lqrex1} and \Cref{prop:lqrex2}, respectively. Every point in the plot corresponds to the \lqrx loss obtained after $T = 60000$ steps. As these plots clearly illustrate, the theoretical gains stated in \Cref{lem:lqrex1} and \Cref{prop:lqrex2} appear in practice as well---there is a clear improvement in terms of the scaling in $\rho$ and $\dimx$ when performing task-guided exploration, even over moderate time regimes. Figure \ref{fig:jordan} illustrates the performance of \algname on a more ``typical'' problem instance: $\Ast$ a single Jordan block and $\Bst,\Rx$, and $\Ru$ randomly generated.  Figure \ref{fig:jordan} gives the average loss versus time obtained by averaging the performance over 15 different realizations of $\Bst,\Rx,\Ru$. As in the previous examples, \algname outperforms all other approaches.

%% file: appendix/general_lds_facts.tex
%!TEX root = ../main_arxiv.tex

\section{Formal Results and Algorithm}\label{sec:overview_lds}

In this section, we formally state the results given in \Cref{sec:results_summary} and present the full definition of \algname. This section is organized as follows. We first formally define our decision-making settings, \ddmx and \lddm, in \Cref{sec:summary_gendec}. Next, we present a lower bound on decision-making in the \ddmx setting in \Cref{sec:ddm_lb_summary}. In \Cref{sec:lddm_opt_lb}, we assume we are in the stronger \lddmx setting and present a lower bound on \emph{optimal} decision-making. \Cref{sec:ce_upper_summary} provides a sufficient condition on exploration policies and shows that, under this condition, certainty equivalence decision-making is optimal in the \ddmx setting. \Cref{sec:lddm_ce} then introduces a restricted set of policies in the \lddmx setting, \emph{sequential-open loop} policies, which we show contains \algname and is sufficiently regular. \Cref{sec:freq_domain_summary} provides an overview of frequency-domain representations of signals, an essential piece in our construction of \algname, and formally defines several sub-routines of \algname. Finally, in \Cref{sec:tople_results_summary} we formally state \algname and provide an upper bound on its performance.

\subsection{Martingale and Linear Dynamical Decision Making (\ddmx and \lddm)}\label{sec:summary_gendec}
Our decision making setting considers smooth loss functions parameterized by models $\theta \in \R^{\dimtheta}$,  $\Jfunc_{\theta}(\fraka): \R^{\dima } \to \R$. The loss function induces the \emph{excess risk} function 
\begin{align*}
\calR(\frakahat;\thetast) := \Jlqr_{\thetast}(\fraka) - \inf_{\fraka'} \Jlqr_{\thetast}(\fraka'),
\end{align*}
We denote the plug-in optimal decision 
\begin{align*}
\aopt(\theta) := \argmin_{\fraka} \calR(\fraka;\theta),
\end{align*}
that is,  the optimal decision when $\theta$ is the nominal parameter. We are, in particular, interested in the case when $\calR$ and $\aopt$ are smooth functions. Formally, we will stipulate that the excess risk function $\calR(\cdot;\cdot)$, and the plug-in optimal decision $\aopt(\theta)$ satisfy the following conditions:
\begin{restatable}[Smooth Decision-Making]{asm}{asmsmooth}\label{asm:smoothness}
There exist $\betataylor(\thetast)$ and constants $\mu > 0$, $L_{\fraka i},L_{\calR i}$, $i \in \{1,2,3\}$, and $\Lra$ such that for any $\theta$ and $\fraka$ satisfying 
\begin{align}
\| \theta - \thetast \|_2 \le \betataylor(\thetast), \quad \| \fraka - \aopt(\thetast) \|_2 \le L_{\fraka 1} \betataylor(\thetast), \label{eq:theta_fraka_condition}
\end{align}
the following conditions hold
\begin{itemize}
	\item The optimal action $\aopt(\theta)$ is unique, and moreover, there is a parameter $\mu$ such that $\calR(\fraka';\theta) \ge \frac{\mu}{2}  \|\fraka' - \aopt(\theta)\|_{2}^2$ for all $\fraka' \in \R^{\dima}$ (not restricted to $\fraka'$ satisfying \Cref{eq:theta_fraka_condition}). 
\item $\| \nabla_\fraka \calR(\fraka; \theta) \|_\op \le L_{\calR1}$, $\| \nabla_\fraka^2 \calR(\fraka; \theta) \|_\op \le L_{\calR2}$, and $\| \nabla_\fraka^3 \calR(\fraka; \theta) \|_\op \le L_{\calR3}$.
\item $\| \nabla_\theta \aopt(\theta) \|_\op \le L_{\fraka 1}$, $\| \nabla_\theta^2 \aopt(\theta) \|_\op \le L_{\fraka 2}$, and $\| \nabla_\theta^3 \aopt(\theta)[\updelta,\updelta,\updelta] \|_\op \le L_{\fraka 3}$ for all $\updelta \in \R^{\dimtheta}$ with $\| \updelta \|_2 = 1$.
\item $\nabla_\fraka^2 \calR(\fraka; \theta)$ is Lipschitz in $\theta$ with Lipschitz constant $\Lra$.
\end{itemize}
\end{restatable}
We also define:
\begin{align*}
\Lquad := \frac{1}{6} (L_{\calR3} L_{\fraka1}^3 + 3 L_{\calR2} L_{\fraka2} L_{\fraka1} + L_{\calR1} L_{\fraka3}), \quad L_\taskhes := 6 \Lquad  +  L_{\calR2} L_{\fraka1} + \Lra L_{\fraka1}^2.
\end{align*}

In the most general case, which we will refer to as \emph{martingale decision making}, we assume that we have observations of the form
\begin{align}\label{eq:ddm_observations}
y_t &= \langle \thetast, z_t \rangle + w_t, \quad 
w_t \mid \calF_{t-1} \sim \calN(0,\sigma_w^2), \quad z_t \text{ is } \calF_{t-1}\text{-adapted}.
\end{align}
for a filtration $(\calF_t)_{t \ge 1}$ and scalar observations $y_t$. We allow the distribution of the covariates $z_t$ to be arbitrary: for example, there may be some function $f(\dots)$ of appropriate shape such that $z_t = f(t, z_{1:t-1},y_{1:t-1},w_{1:t},u_{1:t},\thetast)$, for inputs of our choosing $u_{1:t}$. We are now ready to define our decision-making setting.

\begin{defn}[Martingale Decision Making (\ddm)]
Assume our excess risk $\calR(\frakahat; \thetast)$ satisfies \Cref{asm:smoothness} and our observations are generated by \eqref{eq:ddm_observations}. Then we call the problem of choosing a decision $\frakahat$ to minimize $\calR(\frakahat; \thetast)$ \emph{martingale decision making} (\ddm).
\end{defn}

Our goal in \ddmx is to estimate $\thetast$ from our observations well enough to find a decision rule $\frakahat$ that minimizes $\calR(\frakahat; \thetast)$. We will be interested in particular in the special case when \eqref{eq:ddm_observations} is a linear dynamical system:
\begin{align}\label{eq:lds_dec_making}
x_{t+1} = \Ast x_t + \Bst u_t + w_t 
\end{align}
for $w_t \sim \calN(0,\sigma_w^2 I)$. As we show in \Cref{sec:lds_vec}, linear dynamical systems are a special case of \eqref{eq:ddm_observations}. This special case defines the following restriction of \ddm. 

\begin{defn}[Linear Dynamical Decision Making (\lddm)]
Assume our excess risk $\calR(\frakahat; \thetast)$ satisfies \Cref{asm:smoothness} and that our observations are generated by a linear dynamical system, \eqref{eq:lds_dec_making}. Then we call the problem of choosing a decision $\frakahat$ to minimize $\calR(\frakahat; \thetast)$ \emph{linear dynamical decision making} (\lddm).
\end{defn}

Given these formalizations of our problem setting, we recall our interaction protocol:

\begin{task_prob} The  learner's behavior is specified by an exploration policy $\piexp : (x_{1:t},u_{1:t-1}) \to u_t$ and decision rule $\plan$ executed in the dynamics \Cref{eq:ddm_observations}.
\begin{enumerate}
	\item For steps $t = 1,\dots,T$, the learner executes $\piexp$ and collects a \emph{trajectory} $\traj = (y_{1:T},z_{1:T},u_{1:T})$. 
	\item For a \emph{budget} $\gamma^2 \ge 0$, the inputs $u_{1:T}$ must satisfy the constraint $\Exp_{\piexp}[\sum_{t=1}^T \|u_t\|^2] \le T \gamma^2$.
	\item Finally, the learner proposes a decision $\frakahat = \plan(\traj)$  as a function of $\traj$.
\end{enumerate}
\end{task_prob}

We emphasize the generality of this set of decision-making problems. While we will show that the \lqrx problem satisfies this assumption, many other decision-making problems can be cast as an instance of \ddmx or \lddm, as we discuss in \Cref{sec:ex_intro}.

Before stating our results, we remind the reader of our definition of power constrained policies:
\pigam*

Finally, recall that the $\Hinf$-norm of $\Ast$ is defined as:
\begin{align}
\|\Ast\|_{\Hinf} := \max_{\omega \in [0,2\pi]}\|(e^{\imag \omega} I - \Ast)^{-1}\|_{\op}
\end{align}
where $\imag$ denotes the imaginary number, $\sqrt{-1}$.

\subsection{Lower Bound for Decision Making in \ddm \label{sec:lb_decision_making}}\label{sec:ddm_lb_summary}

We first present a lower bound in the general \ddmx setting. We will assume we are playing a particular exploration policy, $\piexp \in \policyset$, and our goal is to derive lower bounds on decision-making given that our trajectory is generated by $\piexp$. Recall the definition of the local minimax risk:
\begin{align*}
\mmax_{\piexp}(\calR; \calB) := \min_{\plan} \max_{\theta \in \calB} \Exp_{\traj \sim \theta,\piexp} [\calR(\plan(\traj); \theta)] , 
\end{align*}
and the idealized risk:
$$\Phi_T(\pi;\thetast) := \tr(\taskhes(\thetast) \matGam_{T}(\pi;\thetast)^{-1}).$$
Our argument will show that the local minimax risk is lower bounded by the estimation error of $\theta$ in a relevant Mahalanobis norm, which yields the familiar ``inverse-trace of the covariance'' sample complexity. In the \ddmx setting, we denote our covariance as
\begin{align*}
\matSig_T := \sum_{t=1}^T z_t z_t^\top, \qquad \matGam_T(\pi;\theta) := \frac{1}{T} \Exp_{\theta,\pi} \left [ \matSig_T \right ] 
\end{align*}
For our lower bound to hold, the covariance matrices in question must satisfy two rather mild regularity conditions. 
\begin{asm}[Sufficient Excitation]\label{asm:suff_excite}
For some $\lamund > 0$ independent of $T$, and under our exploration policy $\piexp \in \policyset$:
\begin{align*}
\lammin(\matGam_T(\piexp;\thetast)) \ge \lamund .
\end{align*}
\end{asm}
In the special case of linear dynamical systems, \Cref{asm:suff_excite} can be enforced by adding a small amount of white noise to any exploration policy, and the budget constraint can still be met by scaling down inputs by a constant factor.

\begin{asm}[Smooth Response]\label{asm:smooth_covariates}
There exist parameters $\rcov(\thetast) > 0$, $\alphast(\thetast,\gamma^2) > 0$, $\Ccov > 0$, $\ccexp > 0$, and $\alpha > 0$ such that, under our exploration policy $\piexp \in \policyset$, for all $\theta$ satisfying $\| \theta - \thetast \|_2 \le \rcov(\thetast)$, we have:
\begin{align*}
\matGam_T(\piexp;\theta) \preceq  \ccexp \matGam_T(\piexp;\thetast) + \left(\alphast(\thetast,\gamma^2) \cdot \| \theta - \thetast \|_2 +  \frac{\Ccov}{T^\alpha}\right) \cdot I.
\end{align*}
\end{asm}
Intuitively, \Cref{asm:smooth_covariates} says that the covariance matrices do not vary too wildly in the ground truth instances. This will be true for any ``reasonable'' policy, and in fact, we can show that, without loss of generality, a comparable condition holds for the policies which perform near optimal experiment design in \lddm. Under these assumptions, we obtain the following lower bound.

\begin{thm}[Part 2 of  \Cref{prop:ce_optimal}]\label{cor:simple_regret_lb2_nice}
Assume we are in the \ddmx setting, that $\calR$ satisfies \Cref{asm:smoothness}, our exploration policy $\piexp  \in \policyset$ satisfies \Cref{asm:suff_excite} and \Cref{asm:smooth_covariates}, and suppose that the time horizon $T$ satisfies
\begin{align*}
\lamund T \ge \max \left \{  \left ( \tfrac{80 \dimtheta}{\betast(\thetast)^2} \right )^{6/5}, \left ( \tfrac{\sigma_w^2 L_{\calR 2}}{5 \mu} \right )^{6}, \left ( \tfrac{\alphast(\thetast, \gamma^2) \sqrt{ 5\dimtheta}}{\ccexp \lamund}  \right)^{12/5}, \left ( \tfrac{2 \Ccov}{ \ccexp \lamund^{1-\alpha}} \right )^{1/\alpha}, \left ( \tfrac{5 \dimtheta }{ \betaexp(\thetast)^2} \right)^{6/5} \right \}
\end{align*}
Then, defining the localizing ball $\calB_T := \{\theta : \| \theta - \thetast \|_{2}^2 \leq 5 \dimtheta /(\lamund T)^{5/6}\}$, and letting $\traj$ denote a trajectory generated by $\piexp$ on $\theta$, we have
\begin{align*}
\mmax_{\piexp}(\calR; \calB_T) = \min_{\plan} \max_{\theta \in \calB_T} \Exp_{\traj \sim \theta,\piexp} [\calR(\plan(\traj); \theta)]   \geq    \frac{\sigma_w^2}{1 + 2 \ccexp} \cdot & {\color{blue} \frac{\Phi_T(\piexp;\thetast)}{T} } - \frac{ \Clb }{ (\lamund T)^{5/4}} 
\end{align*}
where,
\begin{align*} 
\Clb &=  c_1 \Big ( (L_{\fraka1}  L_{\fraka2} L_{\calR2} + L_{\fraka1}^3 L_{\calR 3} +  \Lra ) \dimtheta^{3/2}  + L_{\fraka 1}^2 L_{\calR 2} \Big )
\end{align*}
for a universal constant $c_1$. 
\end{thm}

This result is itself a corollary of a more general result, \Cref{thm:simple_regret_lb2}, which provides a lower bound without \Cref{asm:suff_excite} or \Cref{asm:smooth_covariates}. We prove this result in \Cref{sec:general_decision}. We emphasize again that \Cref{cor:simple_regret_lb2_nice} does not require that the data be generated from a linear dynamical system---it holds for \emph{any} loss satisfying \Cref{asm:smoothness} so long as our observations follow \Cref{eq:ddm_observations}. However, as we show in \Cref{sec:reg_policy_lds}, \Cref{asm:suff_excite} and \Cref{asm:smooth_covariates} are met for a fairly general set of policies in linear dynamical systems, and a clean corollary of this result may be stated in the \lddmx setting.

\subsection{Lower Bound for Optimal Decision Making in \lddm}\label{sec:lddm_opt_lb}
We turn now to the \lddmx setting, and prove a lower bound that holds for \emph{all} exploration policies $\piexp \in \policyset$. We first define the following:
\begin{equation*}
    \lamnoise(\sigma_u) := \min \Big \{ \lammin \Big ( \sigma_w^2 \tsum_{s = 0}^{\dimx-1} \Ast^s (\Ast^s)^\top + \sigma_u^2 \tsum_{s=0}^{\dimx-1} \Ast^s \Bst \Bst^\top (\Ast^s)^\top \Big ) , \sigma_u^2 \Big \}
\end{equation*}
and in particular set:
\begin{equation*}
    \lamnoise := \lamnoise(\gamma/\sqrt{2\dimu})
\end{equation*}
Note that
$$\lamnoise(\sigma_u) = \lammin \Bigg ( \Exp \left [ \begin{bmatrix} x_{\dimx} \\ u_{\dimx} \end{bmatrix} \begin{bmatrix} x_{\dimx} \\ u_{\dimx} \end{bmatrix}^\top \mid u_s \sim \calN(0,\sigma_u^2 I), s = 0,\ldots,t \right ] \Bigg )$$
so it follows that $\lamnoise(\sigma_u)$ is the minimum eigenvalues of the covariates when we play isotropic noise, and can be thought of as a measure of how easily the system can be excited. We make the following assumption.

\begin{asm}\label{asm:lds_lb_suff_exci}
$\thetast$ and $\sigma_w$ are such  that $\lamnoise > 0$. In particular, it suffices that $\sigma_w > 0$, or the system is controllable.
\end{asm}
\newcommand{\Cinlb}{C_{\mathrm{lb}}^{\mathrm{init}}}
Before stating our result, we recall the definition of the exploration local minimax risk:
\begin{align*}
\mmax_{\gamma^2}(\calR;\calB) := \min_{\piexp \in \Pi_{\gamma^2}} \min_{\plan} \max_{\theta \in \calB} \Exp_{\traj \sim \theta, \piexp}[\calR(\plan(\traj); \theta)].
\end{align*}
and the optimal risk:
\begin{align*}
\Phiopt(\gamma^2;\thetast) := \liminf_{T \to \infty} \inf_{\piexp \in \Pi_{\gamma^2}} \Phi_T(\piexp; \thetast),
\end{align*}
We then have the following.
\begin{thm}[Part 2 of \Cref{thm:exp_design_opt}]\label{thm:lds_lb} Assume we are in the \lddmx setting and consider a loss function $\Jfunc_{\theta}(\fraka): \R^{\dima} \to \R$ with induced excess risk $\calR(\fraka;\thetast) := \Jlqr_{\thetast}(\fraka) - \inf_{\fraka'} \Jlqr_{\thetast}(\fraka')$. Fix a model $\thetast$ and time horizon $T$. Suppose that 
\begin{itemize}
	\item $\calR$ satisfies the smoothness condition, \Cref{asm:smoothness}. 
	\item The model $\thetast$ satisfies the excitation assumption \Cref{asm:lds_lb_suff_exci} with parameter $\lamnoise > 0$.
	\item The time horizon satisfies $T \ge \max \left \{ \Cinlb , \left ( \tfrac{80 (\dimx^2 + \dimx \dimu) }{(\lamnoise)^{5/6} \betast(\thetast)^2} \right )^{6/5},  \bigg (  \tfrac{\sigma_w^2 L_{\calR 2}}{5\mu} \bigg )^6 \right \} $.
\end{itemize}
Finally, define the localized ball of instances
\begin{align*}
\calB_T := \{\| \theta - \thetast \|_{\fro}^2 \le  5 (\dimx^2 + \dimx \dimu)/(\lamnoise T^{5/6}) \}
\end{align*}
Then, any decision rule $\plan(\traj)$ suffers the following lower bound 
\begin{align*}
\mmax_{\gamma^2}(\calR;\calB_T) =  & \min_{\piexp \in \policyset} \min_{\plan} \max_{\theta  \in \calB_T } \Exp_{\traj \sim \theta,\piexp} [\calR(\plan(\traj); \theta)]  \ge   \frac{\sigma_w^2}{64} \cdot   {\color{blue} \frac{ \Phiopt(\gamma^2;\thetast) }{T}} - \frac{\Clb}{(\lamnoise T)^{5/4}} 
\end{align*}
where above,
\begin{align*}
& \Cinlb = \poly \left ( \dimx, \dimu,  \| \Bst \|_\op, \| \Ast \|_{\Hinf}, \gamma^2, \sigma_w^2, \tfrac{1}{\lamnoise}, \log T \right )
\end{align*} 
and $\Clb$ is defined as in \Cref{cor:simple_regret_lb2_nice} with $\dimtheta = \dimx^2 + \dimx \dimu$. 
\end{thm}

We emphasize that this result holds for \emph{any} exploration policy with bounded power, $\piexp \in \policyset$. As such, it provides a lower bound on optimal decision-making. We prove \Cref{thm:lds_lb} in \Cref{sec:proof_lds_lb_opt}.

\newcommand{\tbreak}{\bar{t}}
\subsection{Upper Bound for Certainty Equivalence Decision Making in \ddm}\label{sec:ce_upper_summary}

We next consider upper bounds on decision making in the \ddmx setting when we are playing a fixed exploration policy $\piexp$. Given some data $\{(y_t,z_t,u_t)\}_{t=1}^T$ generated by playing $\piexp$ on \Cref{eq:ddm_observations}, we define our estimator of $\thetast$ as
\begin{align}
\thetals = \min_{\theta} \sum_{t=1}^T \| y_t - \theta^\top z_t  \|_2^2
\end{align}
The following is a sufficient assumption on $\piexp$ to guarantee the efficiency of certainty equivalence decision making. Recall that $\matSig_T$ denotes the random covariates.
\newcommand{\Texpcon}{T_{\mathrm{con}}}
\newcommand{\Texpse}{T_{\mathrm{se}}}
\newcommand{\Ccon}{C_{\mathrm{con}}}
\begin{asm}[Exploration Policy Regularity]\label{asm:minimal_policy}
We assume that the true instance $\thetast$ and policy $\piexp \in \policyset$ satisfy the following regularity conditions:
\begin{itemize}
    \item There exists some time $\Texpse(\piexp)$ such that for any $T \ge \Texpse(\piexp)$ the system is sufficiently excited. That is, if $T \ge \Texpse(\piexp)$:
    \begin{align*}
      \Pr_{\thetast,\piexp}  \Big [ \lammin(\matSig_T) \ge \lamund T, \matSig_T \preceq T \covup \Big ] \ge 1 - \delta
    \end{align*}
    for deterministic $\lamund > 0$ and $\covup \succeq 0$.
    \item There exists some time $\Texpcon(\piexp)$ such that, for any $T \ge \Texpcon(\piexp)$, the covariates have concentrated to their mean. That is, if $T \ge \Texpcon(\piexp)$:
    \begin{align*}
        \Pr_{\thetast,\piexp} \Big [ \| \matSig_T - \Exp_{\thetast,\piexp}[\matSig_T] \|_\op \le \tfrac{\Ccon}{T^\alpha} \lammin(\Exp_{\thetast,\piexp}[\matSig_T]) \Big ] \ge 1- \delta
    \end{align*}
    for deterministic $\Ccon > 0$ and $\alpha > 0$. 
\end{itemize}
\end{asm}

The following result precisely quantifies the loss of the certainty equivalence decision-making rule under this assumption on the policy. 
\newcommand{\Cupa}{C_{\mathrm{ce,1}}}
\newcommand{\Cupb}{C_{\mathrm{ce,2}}}
\begin{thm}[Part 1 of \Cref{prop:ce_optimal}]\label{thm:ce_upper_bound}
Assume we are in the \ddmx setting with some loss $\calR$ satisfying \Cref{asm:smoothness} and exploration policy $\piexp \in \policyset$ which satisfies \Cref{asm:minimal_policy} with minimal times $\Texpcon(\piexp)$ and $\Texpse(\piexp)$ and covariance lower bound $\lamund > 0$. If  
\begin{align}\label{eq:ce_opt_burnin}
    T >  \max \bigg \{\Texpcon(\piexp), \Texpse(\piexp), (4 \Ccon)^{1/\alpha}, \tfrac{c_1  ( \log(1/\delta) + \dimtheta +  \logdet(\covup/\lamund + I))}{\lamund \betast(\thetast)^2}
     \bigg \} 
\end{align}
then for $\delta \in (0,1/2)$, with probability $1- \delta,$ the certainty equivalence decision rule achieves the following rate,
\begin{align*}
\calR(\aopt(\thetals);\thetast)  \le  5 \sigma_w^2  \log\frac{24 \dimtheta}{\delta} \cdot {\color{blue} \frac{\Phi_T(\piexp;\thetast)}{T}} + \frac{\Cupa}{T^{3/2}} + \frac{\Cupb}{T^{1+2\alpha}} 
\end{align*}

where we let $c_1,c_2,c_3$ be universal numerical constants and set
\begin{align*}
    \Cupa & := \tfrac{c_2 \Lquad }{\lamund^{3/2}} \Big ( \log \tfrac{1}{\delta} + \dimtheta + \logdet (\covup/\lamund + I) \Big )^{3/2}, \quad \Cupb  :=  \tfrac{c_3 \sigma_w^2 \Ccon^2 \dimtheta \tr(\taskhes(\thetast))}{\lamund } \log \frac{\dimtheta}{\delta}.
\end{align*}
\end{thm}

We note that this upper bound matches the lower bound given in \Cref{cor:simple_regret_lb2_nice}. This shows that the certainty equivalence decision rule is instance optimal for \emph{any} decision-making problem in the \ddmx setting.

The proof of this result is given in \Cref{sec:ce_upper_pf}. The burn-in time \eqref{eq:ce_opt_burnin} and lower-order terms have transparent interpretations. For the burn-in, the requirement that $T$ be larger than $\Texpcon(\piexp)$ and $\Texpse(\piexp)$ is necessary to ensure that the concentration and excitation events stated in \Cref{asm:minimal_policy} hold with high probability. The requirement that $T$ be larger that $(4\Ccon)^{1/\alpha}$ is necessary to ensure that the covariates have concentrated enough for our $M$-norm estimation bound, \Cref{thm:Mnorm_est_bound}, to hold. Finally, the last term in the burn-in ensures that our estimate $\thetals$ is in a ball of radius $\betast(\thetast)$ around $\thetast$, which allows us to approximate $\calR(\aopt(\thetals);\thetast)$ by a quadratic. The lower order terms similarly yield intuitive explanations. $\Cupa/T^{3/2}$ quantifies the additional loss due to the error in our quadratic approximation of $\calR(\aopt(\thetals);\thetast)$, while $\Cupb/T^{1+2\alpha}$ is due to the lower order term given in our $M$-norm estimation bound, \Cref{thm:Mnorm_est_bound}.

\subsubsection{Corollary: Certainty-Equivalence Decision Making in \lddm}\label{sec:lddm_ce}

While \Cref{thm:ce_upper_bound} holds in a very general setting, our optimal decision-making algorithm, \algname, applies only to the \lddmx setting, and uses a highly structured set of policies. In order to facilitate the analysis of \algname, it is helpful to obtain a corollary of \Cref{thm:ce_upper_bound} in this more restricted setting. Towards making this precise, we introduce a set of policies in the \lddmx setting, \emph{sequential-open loop} policies, which we show contains \algname. Before formally defining these policies, we need the following piece of notation:
\begin{defn} Let
\begin{align*}
\gamup :=   c \Big ( (1 + \| \Bst \|_\op^2) \| \Ast \|_{\Hinf}^4 \Big ) \Big ( \sqrt{\dimx} \log \frac{T}{\delta} + \gamma^2 \| \Ast \|_{\Hinf}^2 \Big ).
\end{align*}
By \Cref{lem:cov_up}, $\covup := T \gamup \cdot I$ is a high probability upper bound on the covariates, assuming that $\piexp \in \policysetgood$, where we define $\policysetgood$ below.
\end{defn}
With this definition in place, we introduce the set of (sufficiently regular) sequential open loop policies satisfying the budged contstraint, denoted $\policysetgood$.
\begin{defn}[Sequential Open-Loop Policies]\label{asm:good_policy} We define a \emph{sequential open-loop policy} to be an exploration policy $\piexp \in \policyset$ satisfying the following conditions:
\begin{itemize}
    \item \textbf{(Open-Loop Gaussian)} There exist deterministic times $\{\tbreak_0,\tbreak_1,\ldots,\tbreak_{n-1},\tbreak_n\} \subseteq [T]$ with $\tbreak_0 = 0, \tbreak_n = T, \tbreak_{i+1} \ge \tbreak_i$, such that, for $t \in \{ \tbreak_i,\ldots,\tbreak_{i+1}-1\}$:
    \begin{align*}
        u_t | \calF_{\tbreak_i} \sim \calN(\util_t, \Lambda_{u,i})
    \end{align*}
    for $\calF_{\tbreak_i}$ measurable $\util_t$ and $\Lambda_{u,i} \succeq 0$ satisfying:
    \begin{align*}
        \sum_{t=0}^{T-1} \util_t^\top \util_t \le T \gamma^2, \quad \tr(\Lambda_{u,i}) \le  \gamma^2, \quad \lammin(\Lambda_{u,i}) \ge \sigma_u^2
    \end{align*}
    almost surely, for deterministic $\sigma_u$. 
    
    \item \textbf{(Low-Switching)} For any $t$, there exists some epoch $i$ such that $|\{ t,\ldots,t +  \Texpse(\piexp) \} \cap \{ \tbreak_i, \ldots, \tbreak_{i+1}-1 \}| \ge \frac{1}{2} \Texpse(\piexp)$ where
    \begin{align*}
        \Texpse(\piexp) := c_1 \dimx \Big ( (\dimx + \dimu) \log(\gamup/\lamnoise(\sigma_u) + 1) + \log \frac{n}{\delta} \Big ).
    \end{align*}
    In words, at least half of any length $\Texpse(\piexp)$ interval is contained in a single epoch.
	\end{itemize}
\end{defn}
\noindent We make several comments on this definition.
\begin{itemize}
	\item Any policy $\piexp \in \policysetgood$ satisfies \Cref{asm:minimal_policy}, which we prove in \Cref{sec:ce_upper_pf}.
	\item As we show in \Cref{sec:lds_exp_design}, \algname, an  optimal policy (up to constants), is in $\policysetgood$, with $n = \calO(\log T)$.
	\item The assumption that $u_t \mid \calF_{T_i}$ be Gaussian is for simplicity of analysis, and in general is not necessary---the noise could take different sub-Gaussian distributions if desired.
\end{itemize}

The following result instantiates \Cref{thm:ce_upper_bound} with any policy $\piexp \in \policysetgood$, and assuming we are in the \lddmx setting.
\begin{cor}\label{cor:ce_upper_bound_nice}
Assume we are in the \lddmx setting and consider some loss $\calR$ satisfying \Cref{asm:smoothness}, stable system $\thetast$, and exploration policy $\piexp \in \policysetgood$. If  
\begin{align}
    T \ge &  c_1 \Big ( \tfrac{\Csys n^2 (\sigma_w^4 + \gamma^4)}{ \lamnoise(\sigma_u)^2} + \tfrac{  \dimx }{ \lamnoise(\sigma_u) \betast(\thetast)^2} +\tfrac{\sqrt{\dimx}\sigma_w^2}{\gamma^2} + n + \dimx \Big ) \Big (\log \tfrac{n}{\delta} + d \log(\gamup/ \lamnoise(\sigma_u) + 3)  \Big )
   \label{eq:ce_opt_burnin_gsed}
\end{align}
then for $\delta \in (0,1/3)$, with probability $1- \delta$:
\begin{align*}
\calR(\aopt(\thetals);\thetast)  \le  5 \sigma_w^2  & \log\frac{24(\dimx^2 + \dimx\dimu)}{\delta} \cdot {\color{blue} \frac{\Phi_T(\piexp;\thetast)}{T}} + \frac{\Cupa}{T^{3/2}} + \frac{\Cupb}{T^{2}}
\end{align*}
where $\Csys = \poly(\| \Bst \|_\op, \| \Ast \|_{\Hinf})$, $d := \dimx + \dimu$, universal numerical constants $c_1,c_2$, and
\begin{align*}
    \Cupa & := \tfrac{c_2 \Lquad }{ \lamnoise(\sigma_u)^{3/2}} \Big ( \log \tfrac{1}{\delta} + \dimx d \log (\gamup/\lamnoise(\sigma_u) + 3) \Big )^{3/2}, \quad \Cupb :=  \tfrac{\Csys \sigma_w^2 (\sigma_w^4 + \gamma^4) \tr(\taskhes(\thetast)) d^3 n}{ \lamnoise(\sigma_u)^3}  \log^2 \tfrac{dn}{\delta}.
\end{align*}
\end{cor}

We prove this result in \Cref{sec:lds_ce_upper}.

\subsection{Efficient Experiment Design in Frequency Domain (\lddm)}\label{sec:freq_domain_summary}

Before presenting the formal definition of \algname, we establish the relevant experiment design preliminaries, which are best stated in frequency domain. From this, we will specify \algname, and then show that it is indeed a sequential open loop policy, in the sense of \Cref{asm:good_policy}. Recall that, in the \lddmx setting, we set
\begin{align*}
& \Gamma_{T}(\pi;\theta) := \frac{1}{T}\Exp_{\theta,\pi}\left[\sum_{t=1}^T \begin{bmatrix}x_t\\
u_t \end{bmatrix}\begin{bmatrix}x_t\\
u_t \end{bmatrix}^\top \right],\qquad \boldsymbol{\Gamma}_T(\pi;\theta) := I_{\dimx} \otimes \Gamma_T(\pi;\theta)
\end{align*}
Through the remainder of this section, we will use the convention that in the \lddmx setting bold matrices denote Kronecker products, $\matGam := I_{\dimx} \otimes \Gamma$. Note that, by mapping our linear dynamical system \eqref{eq:lds_dec_making} to our general regression setting \eqref{eq:ddm_observations} through the reduction given in \Cref{sec:lds_vec}, this definition is consistent with our definition of $\matGam$ in the \ddmx setting.

\subsubsection{Frequency-Domain Representations \label{ssec:fourier_preliminaries}} We let bold vectors $\bmu = (u_{i})_{i = 1}^k\in \C^{k\dimu}$ denote sequences of inputs, and denote their discrete-time Fourier transform (DFT) 
\begin{align}
\bmucheck = (\ucheck_s)_{s=1}^{k} = \Fourier(\bmu) \in \C^{k\dimu}, \quad \text{ where } \ucheck_s = \sum_{s=1}^{k} u_s \exp( \frac{2\pi \imag s }{k})
\end{align}
The mapping $\Fourier$ is invertible, though in general $\Fourier^{-1}: \C^{t\dimu} \to \C^{t\dimu}$. However, if our frequency-domain representation is \emph{symmetric}, we have that the inverse DFT is purely real.

\begin{defn}[Symmetric Signal]
We say that $\bmucheck = (\ucheck_s)_{s=1}^k \in \C^{k\dimu}$ is \emph{symmetric} if $\ucheck_s = \cconj(\ucheck_{k-s})$ for $s < k$ and $\ucheck_k$ is purely real, where $\cconj(\cdot)$ denotes the complex conjugate. 
\end{defn}

\begin{fact} $\Fourier^{-1}(\ucheck)$ is a vector with real coefficients if and only if $\ucheck$ is symmetric. 
\end{fact}
We now consider a convex relation of the outerproduct of this DFT. First, some preliminaries. For a complex vector $z \in \C^{d}$ (resp. matrix $A \in \C^{d_1\times d_2}$),  let $z^\herm$ (resp. $A^\herm$) denote its Hermitian adjoint; i.e., the complex conjugate of its transpose. We denote the set of \emph{Hermitian} matrices as $\Hermset := \{A \in \C^{d \times d}: A^\herm = A\}$, and the set of positive-semidefinite Hermitian matrices $\Hermpsd := \{A \in \Hermset: z^\herm A z \ge 0 , \quad \forall z \in \C^d\}$. Given $\ucheck = (\ucheck_\ell)_{\ell = 1}^k  \in \C^{k\dimu}$, we define its outerproduct as the sequence of complex-rank one Hermitian matrices, $\bmU = (U_{\ell})_{\ell=1}^k$, defined by
\begin{align}
\bmucheck \otimes \bmucheck := \bmU = (U_{\ell})_{\ell=1}^k, \quad \text{where} \quad U_{\ell} = \ucheck_\ell \ucheck^\herm_\ell \in \Hermpsd[\dimu]. \label{eq:outerproduct}
\end{align}
We now define the following set, which relaxes outer products to matrix sequences of the above form, with a total power constraint on their trace:
\begin{equation}\label{eq:input_set}
 \calU_{\gamma^2,k} := \left \{\bmU = ( U_\ell )_{\ell=1}^{k} \ : \ U_\ell \in \Hermpsd[\dimu], \quad \bmU \text{ is symmetric}, \quad\sum_{\ell=1}^{k} \tr(U_\ell) \le k^2 \gamma^2 \right \} 
\end{equation}
Critically, $\calU_{\gamma^2,k}$ is convex. We generalize the definition of symmetric signals here to matrices, defining it identically as we have defined symmetric vector signals. The following class of sequences $\bmU$ are of particular importance.
\begin{defn}[Rank One Relaxation] We say that $\bmU = \{U_{\ell} : 1 \le \ell \le k \} \in \calU_{\gamma^2,k}$ is \emph{rank one} if there exists a vector $\ucheck \in \C^{k \dimu}$ such that $\bmU = \ucheck \otimes \ucheck$.
\end{defn}
Lastly, we define a frequency-domain covariance operator defined on $\bmU \in \calU_{\gamma^2,k}$:
\begin{align} 
\Gamfreq_k(\theta,\bmU) &:= \frac{1}{k}  \sum_{\ell=1}^k (e^{\imag \frac{2\pi \ell}{k}} I - A)^{-1} B U_\ell B^\herm (e^{\imag \frac{2\pi \ell}{k}} I - A)^{-\herm}, \\\label{eq:Gamma_freq}
 \Gamfreq_{t,k}(\theta,\bmU) &:= \frac{t}{k}\Gamfreq_k(\theta,\bmU) 
\end{align}
We will overload notation, defining
\begin{align}
& \Gamfreq_k(\theta,\bmucheck) =  \Gamfreq_k(\theta,\bmucheck \otimes \bmucheck), \quad  \Gamfreq_{t,k}(\theta,\bmucheck) =  \Gamfreq_{t,k}(\theta,\bmucheck \otimes \bmucheck) \\
& \Gamfreq_k(\theta,\bmu) = \Gamfreq_k(\theta,\bmucheck), \quad \Gamfreq_{t,k}(\theta,\bmu) = \Gamfreq_{t,k}(\theta,\bmucheck)
\end{align}
for $\bmu = \Fourier^{-1}(\bmucheck)$. The following result shows that $\Gamfreq_k(\theta,\bmu)$ corresponds to the steady-state covariates of our system when an input $\bmu$ is played.

\begin{prop}\label{prop:steady_state_inputs}
 Let $\bmU \in \calU_{\gamma^2,k}$ be rank one, with $\bmU = \bmucheck \otimes \bmucheck$, $\bmucheck \in \C^{k \dimu}$. Let $\bmu =(u_t)_{t=1}^k = \Fourier^{-1}(\bmucheck)$. Define the extended inputs
\begin{align*}
\bmuext_{1:T} = (\uext_t)_{t\ge 1}, \text{ where }\uext_t = u_{\mathrm{mod}(t,k)}
\end{align*}
Finally, let $\xbmu$ denote the evolution of the dynamical system obtained by starting at initial state $x_0$ and  executing the input $\uext_t$. Then
\begin{enumerate}
\item $ \frac{1}{k} \Gamfreq_k(\theta,\bmU) = \lim_{T\rightarrow \infty} \frac{1}{T} \sum_{s=0}^{T-1} \xbmu_s (\xbmu_s)^\top = \lim_{T\rightarrow \infty} \frac{1}{T} \Gamin_T(\theta,\bmuext,x_0) $.
\item Let $k' \ge k$ be divisible by $k$.  Let $\bmu' = (\uext_{s})_{s=1}^{k'}$, and define the frequency domain quantities
\begin{align}
\bmucheck' = (\ucheck'_\ell)_{\ell=1}^{k'} = \Fourier(\bmu'), \quad \bmU' = \bmucheck' \otimes \bmucheck'.
\end{align}
 Then $\Gamfreq_{k'}(\theta,\bmU') = \frac{k'}{k} \Gamfreq_k(\theta,\bmU) $.
\end{enumerate}
\end{prop}

\paragraph{Noise-Augmented Covariances.}

We shall also study the covariance matrix that arises from exciting the system with white noise of covariance $\Lambda_u$, when the process noise has covariance $\Lambda_w$:
\begin{align}
 \Gamnoise_t(\theta,\Lambda_u) := \sum_{s =0}^{t-1} A^s \Lambda_w (A^s)^\top +  \sum_{s=0}^{t-1} A^s B \Lambda_u B^\top (A^s)^\top
\end{align}
We will overload notation and set 
\begin{align*}
\Gamnoise_t(\theta,\sigma_u) := \Gamnoise_t(\theta, \sigma_u^2 I)
\end{align*}
Since the Fourier transform preserves Gaussianity, the relevant covariance matrices become:
\begin{align}\label{eq:gamss_def} \Gamss_{T,k}(\theta,\bmU,\sigma_u) := \frac{1}{k} \Gamfreq_k(\theta,\bmU) + \frac{1}{T} \sum_{t=1}^T \Gamnoise_t(\theta,\sigma_u) 
\end{align}
If $\bmU$ is rank one, then \Cref{prop:steady_state_inputs} implies that $\Gamss_{T,k}(\theta,\bmU,\sigma_u)$ corresponds to the expected steady-state  covariates of the noisy system when playing inputs $\bmU$. If $\bmU$ is not rank one, then $\Gamss_{T,k}(\theta,\bmU,\sigma_u)$ corresponds to the expected steady-state covariates of the noisy system when playing a sequence of inputs formed by decomposing $\bmU$ into rank one inputs, as in \Cref{alg:construct_time_input}.

\subsubsection{The \texttt{SteadyStateDesign} Subroutine}

Using the preliminaries laid out in \Cref{ssec:fourier_preliminaries}, we define two subroutines of \algname before stating the full algorithm. We first state our experiment-design subroutine, which computes the certainty-equivalence task-optimal inputs.

\begin{algorithm}[H] 
\begin{algorithmic}[1]
\State{} \text{Input} time horizon $t$, signal length $k$, budget $\gamma > 0$, model estimate $\thetahat$
\State{} Let $\calU_{\gamma^2/2,k}$ be the lifted representation of inputs defined in \eqref{eq:input_set}
\State{}  Set $\bmU \in \calU_{\gamma^2/2,k} \subset (\Hermpsd)^{k}$ 
\begin{align}\label{eq:ss_opt2}
\bmU \leftarrow \min_{\bmU \in \calU_{\gamma^2/2,k}} \tr \left ( \taskhes(\thetahat) \cdot \matGamss_{t,t/\dimu}(\thetahat,\bmU,\gamma/\sqrt{2\dimu})^{-1} \right )
\end{align}
\State{} \textbf{return} $\bmU$
\end{algorithmic}
\caption{\texttt{SteadyStateDesign}($\thetahat,t,k,\gamma$)}\label{alg:steady_state_design}
\end{algorithm}

\Cref{alg:steady_state_design} chooses the input $\bmU$ to minimize a function of the steady-state covariates of the system $\thetahat$. In particular, observe that the objective is the steady-state analogue of the lower bound given in \Cref{thm:lds_lb}, and we can therefore interpret this routine as choosing the inputs that minimize the lower bound for our estimated system.

\paragraph{Implementation via Projected Gradient Descent.} Note that the set $\calU_{\gamma^2,k}$ is convex, and that the objective is also convex, due to the convexity of $\tr(X^{-1})$ and since $\Gamss_{t,t/\dimu}(\thetahat,\bmU,\gamma/\sqrt{2\dimu})$ is affine in $\bmU$. It follows that \eqref{eq:ss_opt2} can be efficiently solved with any SDP solver. The structure of $\calU_{\gamma^2,k}$, however, allows for an even more efficient solution. Note that any $\bmU$ can be projected onto $\calU_{\gamma^2,k}$ by computing the SVD of each $U_\ell \in \bmU$, an operation which takes time $\calO(k \dimu^3)$. Therefore, \eqref{eq:ss_opt2} can be efficiently solved by running the following projected gradient descent update:
\begin{align*}
    & \bmU_{i+1} \leftarrow \bmU_i - \eta \nabla_{\bmU} \Gamss_{t,t/\dimu}(\thetahat,\bmU_i,\gamma/\sqrt{2\dimu}) \\
    & \bmU_{i+1} \leftarrow \mathrm{proj}(\bmU_{i+1}; \calU_{\gamma^2/2,k})
\end{align*}
where $\mathrm{proj}(\bmU_{i+1}; \calU_{\gamma^2/2,k})$ denotes the projection of $\bmU_{i+1}$ onto $\calU_{\gamma^2/2,k}$.

\subsubsection{The \texttt{ConstructTimeInput} Subroutine}\label{sec:construct_time_input}

In order to efficiently solve our experiment design problem, we allow our input set to contain inputs that are not rank one. While this relaxation ensures our input set is convex, for a given $\bmU \in \calU_{\gamma^2/2,k}$ that is not rank one, it is not clear if $\bmU$ can be implemented in the time domain. Indeed, \Cref{prop:steady_state_inputs} shows that, if $\bmU$ is rank one, there exists some time domain input $\bmu = (u_t)_{t=1}^k$ such that
\begin{align*}
    \frac{1}{k} \Gamfreq_{k}(\theta,\bmU) = \lim_{T \rightarrow \infty} \frac{1}{T} \sum_{t=0}^{T-1} \xu_s (\xu_s)^\top
\end{align*}
which implies that we can approximately realize the response covariates $\Gamss_{T,k}(\theta,\bmU,\gamma/\sqrt{2\dimu})$ in the time domain, but this relationship no longer holds if $\bmU$ is not rank one. To remedy this, we propose the following procedure, which decomposes an arbitrary, not necessarily rank one, input $\bmU$ into a sequence of inputs that can be realized in the time domain. 

\begin{algorithm}[H] 
\begin{algorithmic}[1]
\State Denote eigendecompositions $U_\ell = \sum_{j=1}^{\dimu} \lambda_{\ell,j}v_{\ell,j} v_{\ell,j}^\herm, \ell = 1,\ldots,k, U_\ell \in \bmU$
\State $u_t = \mathbf{0} \in \R^{\dimu}$ for $t = 1,\ldots,\dimu T$
\For{$j=1,\ldots,\dimu$}
	\State $\ucheck_{\ell,j} \leftarrow \sqrt{\dimu \lambda_{\ell,j}} v_{\ell,j}$ for $\ell = 1,\ldots,k$
	\For{$n=1,\ldots,T/k  $}
		\State $u_{(j-1)T + (n-1) k + 1},\ldots,u_{(j-1)T + nk} \leftarrow \Fourier^{-1}(\ucheck_{1,j},\ldots,\ucheck_{k,j})$
	\EndFor
\EndFor
\State \textbf{return} $u_1,\ldots,u_{\dimu T}$
\end{algorithmic}
\caption{\texttt{ConstructTimeInput}($\bmU,T,k$)}\label{alg:construct_time_input}
\end{algorithm}

As the following result shows, \texttt{ConstructTimeInput} produces a time domain input which realizes the response covariates $\Gamfreq_{k}(\theta,\bmU)$ for arbitrary $\bmU$.

\begin{prop}\label{prop:constructtimeinput}
Let $\bmU \in \calU_{\gamma^2,k}$ not necessarily rank one. Let $\bmu_{m} = (u_t)_{t=1}^{\dimu m k}$ denote the time-domain input returned by calling \texttt{ConstructTimeInput}{\normalfont ($\bmU,mk,k$)} with $m$ an integer. Then
\begin{align*}
    \frac{1}{k} \Gamfreq_k(\theta,\bmU) = \lim_{m \rightarrow \infty} \frac{1}{\dimu m k} \sum_{t=0}^{\dimu mk  } x_t^{\bmu_{m}} ( x_t^{\bmu_{m}})^\top
\end{align*}
and, furthermore, the input satisfies $\sum_{t=1}^{\dimu mk} u_t^\top u_t \le \dimu m k \gamma^2$.
\end{prop}

\subsection{Optimal \lddmx Decision-Making: Formal Statement and Guarantee for \algname}\label{sec:tople_results_summary}

Finally, we provide a formal definition of \algname (\Cref{alg:lqr_simple_regret}), and a formal guarantee for its performance. Note that \algname applies in the \lddmx setting.

\begin{algorithm}[H] 
\begin{algorithmic}[1]
\State \textbf{Input:} Input power $\gamma^2$, initial epoch length $\Cinit \dimu$ ($\Cinit \in \N$)
\State $T_0 \leftarrow \Cinit \dimu$, $k_0 \leftarrow \Cinit$, $T \leftarrow T_0$
\State Run system for $T_0$ steps with $u_t  \sim \mathcal{N}(0,\frac{\gamma^2}{\dimu} I )$  
\For{$i=1,2,3,...$}
    \State $\thetahat_{i-1} \leftarrow \argmin_{\theta} \ \sum_{t=1}^{T} \| x_{t+1} - \theta [x_t; u_t] \|_2^2$
 	\State $T_i \leftarrow T_0 2^i$, $k_i \leftarrow k_0 2^{\lfloor i/4 \rfloor} $, $T \leftarrow T + T_i$
    \State $\bmU_{i} \leftarrow \texttt{SteadyStateDesign}(\thetahat_{i-1},T_{i},k_{i},\gamma)$
	\State $(\wt{u}_t^{i})_{t=1}^{T_{i}}\leftarrow$ \texttt{ConstructTimeInput}$(\bmU_{i},T_{i}/\dimu,k_{i})$
	\State Run system for $T_i$ steps with $u_t = \wt{u}_t^i + \unoise_t$, $\unoise_t \sim \mathcal{N}(0, \frac{\gamma^2}{2\dimu} I )$ \label{line:play_inputs}
\EndFor
\end{algorithmic}
\caption{\textbf{T}ask \textbf{OP}tima\textbf{L} \textbf{E}xperiment Design (\algname)}\label{alg:lqr_simple_regret}
\end{algorithm}

\algname begins by injecting isotropic Guassian noise into the system to achieve a minimum degree of excitation. It then solves a sequence of experiment design problems on the estimated system, $\thetahat_{i-1}$, and then plays the inputs that would optimally excite $\thetahat_{i-1}$. Due to the computational efficiency of \texttt{SteadyStateDesign}, \algname is computationally efficient. Note that by construction we will always have that $k_{i+1}$, $T_{i+1}/\dimu$, and $T_{i+1}/(\dimu k_{i+1})$ are integers, so all quantities in the algorithm and subroutines are well-defined. The following assumption quantifies how large $T$ must be to guarantee we achieve the optimal rate.

\newcommand{\Cinup}{C_{\algname}^{\mathrm{init}}}
\begin{asm}[Sufficiently Large $T$]\label{asm:upper_sufficient_T}
$T$ is large enough that the burn-in time of \Cref{cor:ce_upper_bound_nice}, \eqref{eq:ce_opt_burnin_gsed}, is met with $n = c_1 \log T$ and $\lamnoise(\sigma_u) = \lamnoise$, and
\begin{align}\label{eq:alg_bound_init2}
\begin{split}
T & \ge \max \Bigg \{  \frac{\Cinup \sqrt{\dimx} (\sigma_w^2 + 1)}{\lamnoise},    \frac{c_2( \log \frac{1}{\delta} + d \log(\gamup/\lamnoise + 1))}{\min \{ \betaexplds(\thetast)^{2}, \dimx^{-1}\betast(\thetast)^{2}, (\lamnoise)^2 \alphastlds(\thetast,\gamma^2)^{-2}\} \lamnoise} \Bigg \}
\end{split}
\end{align}
where
$$ \Cinup = \poly \left ( \| \Ast \|_{\Hinf}, \| \Bst \|_\op, \dimu, \gamma^2, \log \frac{1}{\delta}, \Cinit \right ),$$
$\betaexplds(\thetast) =\Csys^{-1}$ is defined as in \Cref{lem:smooth_covariates} for some $\Csys = \poly(\| \Ast \|_{\Hinf}, \| \Bst \|_\op)$, $d = \dimx + \dimu$, and $c_1,c_2$ are universal numerical constants.
\end{asm}

Then we have the following theorem, upper bounding the loss achieved by \algname.
\newcommand{\Ctopa}{C_{\algname,\mathrm{1}}}
\newcommand{\Ctopb}{C_{\algname,\mathrm{2}}}
\begin{thm}[Part 1 of \Cref{thm:exp_design_opt}]\label{thm:tople_upper_formal}
Assume we are in the \lddmx setting, that $\calR$ satisfies \Cref{asm:smoothness}, $\delta \in (0,1/3)$, and that $T$ is large enough for \Cref{asm:upper_sufficient_T} to hold. Then with probability at least $1- \delta$, the estimate $\thetahat_T$ produced by \Cref{alg:lqr_simple_regret} satisfies: 
\begin{align*}
\calR(\aopt(\thetahat_T); \thetast)  \le  480\sigma_w^2 \log & \frac{72(\dimx^2 + \dimx \dimu)}{\delta}  \cdot {\color{blue} \frac{ \Phiopt(\gamma^2;\thetast) }{T} } + \frac{\Cupa + \Ctopa}{T^{3/2}} + \frac{\Cupb + \Ctopb}{T^2}
 \end{align*}
and, furthermore,  $\Exp[\tsum_{t=1}^T u_t^\top u_t ] \le T \gamma^2$. Here $\Cupa$ and $\Cupb$ are defined as in \Cref{cor:ce_upper_bound_nice},
\begin{align*}
    & \Ctopa := c_1 \left ( \tfrac{\sqrt{\dimx} d^2 \Lra }{(\lamnoise)^{3/2}}+ \tfrac{\alphastlds(\thetast,\gamma^2) \tr(\taskhes(\thetast))}{(\lamnoise)^{5/2}} \right ) \sqrt{\log(1/\delta) + d \log(\gamup/\lamnoise + 1)}, \\
    & \Ctopb := c_2 \tfrac{ d^2 \alphastlds(\thetast,\gamma^2) \Lra}{(\lamnoise)^3} \Big (\log(1/\delta) + d \log(\gamup/\lamnoise + 1) \Big ),
\end{align*}
$\alphastlds(\thetast,\gamma^2) = \Csys (\sigma_w^2 +  \gamma^2)$ is defined as in \Cref{lem:smooth_covariates}, $d := \dimx + \dimu$, and $c_1,c_2$ are universal numerical constants.
\end{thm}

We note that this upper bound matches the lower bound on \lddmx decision making given in \Cref{thm:lds_lb}, up to constants. We prove this result in \Cref{sec:tople_up_pf}. The additional burn-in required and additional lower-order terms are required to quantify how close to optimal the inputs being played are. In particular, when $T$ satisfies \eqref{eq:alg_bound_init2}, we are able to show that the inputs being played achieve near-optimal performance. The additional lower order terms, $\Ctopa/T^{3/2}$ and $\Ctopb/T^2$, both quantify the loss incurred by performing certainty equivalence experiment design with an estimate of $\thetast$.

%% file: body/synthesis_rates_arxiv.tex
%!TEX root = ../main_arxiv.tex

\newcommand{\pd}[1][d]{\mathbb{S}_{++}^{\dimtheta}}
\newcommand{\psd}[1][d]{\mathbb{S}_{+}^{\dimtheta}}
\newcommand{\frakD}{\mathfrak{D}}
\newcommand{\Prbar}{\overline{\Pr}}

\newcommand{\Sigavg}{\Sigma_{\mathrm{avg}}}

\newcommand{\minhalf}{-\nicefrac{1}{2}}
\newcommand{\nicehalf}{\nicefrac{1}{2}}

\newcommand{\Sigbar}{\overline{\Sigma}}
\newcommand{\Gammast}{\Gamma_{\star}}
\newcommand{\PDset}[1][d]{\mathbb{S}^{#1}}
\newcommand{\bmtheta}{\bm{\theta}}
\newcommand{\Rlsbar}{\overline{\calR}_{\mathrm{ls}}}

\newcommand{\Ntrunc}{\calN_{\mathrm{tr}}}
\newcommand{\Dtrunc}{\calD_{\mathrm{trunc}}}
\newcommand{\Dfull}{\calD_{\mathrm{full}}}
\newcommand{\barSigMt}{\overline{\matSig}_{T;M}}
\newcommand{\matZtil}{\wt{\matZ}}

\section{Optimal Rates for Martingale Regression in General Norms}\label{sec:M_norm_regression}
We now provide an overview of the key technical ideas employed in this work. A critical piece in our analysis is establishing upper and lower bounds on martingale linear regression in arbitrary norms. In particular, we are interested in the regression setting employed by \ddm, where we have observations of the form
\begin{align*}
y_t &= \langle \thetast, z_t \rangle + w_t, \quad 
w_t \mid \calF_{t-1} \sim \calN(0,\sigma_w^2), \quad z_t \text{ is } \calF_{t-1}\text{-adapted}.
\end{align*}
for a filtration $(\calF_t)_{t \ge 1}$, true parameter $\thetast \in  \R^{\dimtheta}$, covariates $z_t \in \R^{\dimtheta}$, scalar observations $y_t$, and noise $w_t$. As long as the above observation model holds, we allow the distribution of the covariates to be arbitrary: for example, that there is some function $f(\dots)$ of appropriate shape such that $z_t = f(t, z_{1:t-1},y_{1:t-1},w_{1:t},u_{1:t},\thetast)$.  Our aim is to produce an estimator $\thetahat \in \R^{\dimtheta}$ so as to minimize the following weighted least-squares risk:
\begin{align}
\Rls(\thetahat; \theta) := \| \thetahat - \theta \|_M^2, \quad M \succeq 0. \label{eq:rls}
\end{align}

In this section, we state our upper and lower bounds on $\Rls(\thetahat; \theta)$, and provide a proof of our lower bound. Throughout, we let $\Exp_{\theta}$ and $\Pr_{\theta}$ denote probabilities and expectations with respect to the above law when $\thetast = \theta$.

\subsection{Upper and Lower Bounds on $M$-norm Regression}

We will show that the least squares estimator 
\begin{align}
\thetals := \textstyle\left(\sum_{t=1}^T z_t z_t^\top\right)^{-1} \sum_{t=1}^T z_t y_t \label{eq:thetals}
\end{align}
is the optimal estimator of $\thetast$ for the risk in \Cref{eq:rls}, in a very strong, instance dependent sense. Throughout, the central object of our analysis is the random covariance matrix:
\begin{align*}
\matSig_T := \textstyle\sum_{t=1}^T z_t z_t^\top.
\end{align*}

Let us start with the lower bound. We will call an estimator $\thetahat$ \emph{measurable} if $\thetahat$ is a measurable function of the covariates and responses $(y_t,z_t : 1 \le t \le T)$, and possibly some internal randomness. We consider the localized risk in a Euclidean ball of radius $r > 0$ around a nominal instance $\theta_0$

\begin{thm}[Truncated van Trees]\label{thm:gauss_assouad_M} Let $\thetahat$ be an arbitrary measurable  estimator. Moreover, fix a covariance parameter $\Gamma \in \pd$, nominal instance $\theta_0 \in \R^{\dimtheta}$, and radius $r \ge \sqrt{5\tr(\Gamma^{-1})}$.  Let $\calB := \{\theta:\|\theta - \theta_0\|_2 \le r\}$ denote a Euclidean ball around $\theta_0$. Then, it holds that 
\begin{align*}
 \inf_{\thetahat} \max_{\theta \in \calB} \Exp_\theta \Rls(\thetahat;\theta)  \ge \sigma_w^2\min_{\theta \in \calB} \tr\left(M \cdot \left( \Exp_{\theta}[\matSig_T] + \Gamma\right)^{-1} \right) - \Psi(r;\Gamma,M),
\end{align*}
where $\Psi(r;\Gamma, M) := \frac{32\| M \|_\op}{\lambda_{\min}(\Gamma)}\exp( - \frac{r^2 }{5 }\lambda_{\min}(\Gamma))$.
\end{thm}
\begin{proof}[Proof Sketch of \Cref{thm:gauss_assouad_M}]The proof is given shortly below in \Cref{sec:thm_gauss_assouad_proof}; it is derived from a Bayes-risk lower bound deriving from an explicit computation of the conditional variance (and thus minimal mean square error in estimation) of a parameter $\bmtheta$ drawn from a normal distribution centered at $\theta_0$, with covariance $\Lambda = \Gamma^{-1}$. This distribution is then carefully truncated at radius $r$ to ensure the local minimax bound holds when restricted to the ball $\calB$. 
\end{proof}
\begin{rem}[Comparison to Previous Lower Bounds] \label{rem:lb_compare}
Lower bounds for regression are typically derived from the Cramer-Rao bound (e.g, in \cite{chaudhuri2015convergence}), which applies only to unbiased estimators, and does not rule out more efficient estimation by allowing bias. In contrast, our work provides an \emph{unconditional} information theoretic lower bound, derived from a closed-form computation of an expected Bayes risk in linear regression with a Gaussian prior (\Cref{thm:gauss_assouad_M}). This technique is similar in spirit to the Van Trees inequality \citep{gill1995applications} 
which was used in concurrent work to understand instance-optimal regret in LQR when $\Ast$ is known but $\Bst$ is not \citep{ziemann2020uninformative}. 

Another common technique for adaptive estimation lower bounds is Assouad's method \citep{arias2012fundamental,simchowitz2020naive}, typically yielding worst-case (though not sharp, instance-dependent) lower bounds.  The lower bounds for adaptive experiment design in linear systems due to \citep{wagenmaker2020active,jedra2019sample} hold in the asymptotic regime where the tolerated probability of failure $\delta$ tends to $0$, a regime pioneered by \cite{kaufmann2016complexity} for pure-exploration multi-arm bandits,  and extended to reinforcement learning in \cite{ok2018exploration}. For continuous parameter estimation (such as the linear control setting control),  the $\delta \to 0$ asymptotic lower bounds differs from non-asymptotic upper bounds by as much as a dimension factor, unless $\delta$ is taken to be exponentially small in dimension \citep{simchowitz2017simulator}. In particular, taking $\delta \to 0$ yields a qualitatively inaccurate picture of the expected error of the estimators in question. In contrast, this work achieves matching bounds in the (arguably more natural) ``moderate $\delta$'' regime, where the tolerated failure probability is no smaller than inverse polynomial in the time horizon. 
\end{rem}

In our applications, we shall choose $r$ sufficiently large and $\Gamma$ sufficiently small so that the lower bound reads
\begin{align}
\sigma_w^{-2} \inf_{\thetahat} \max_{\theta \in \calB} \Exp_\theta \Rls(\thetahat;\theta)  \gtrsim \min_{\theta \in \calB} \tr\left(M  \Exp_{\theta}[\matSig_T]^{-1} \right); \label{eq:lb_approx}
\end{align}
in other words, that the  $M$-weighted  trace of the inverse covariance matrix lower bounds the risk. Even though the right-hand side considers the minimum over $\theta \in \calB$, the radius $r$ of $\calB$ can be chosen small enough that this quantity does not vary significantly, under certain regularity conditions. For a sense of scaling  $\Exp_{\theta}[\matSig_T]$ will typically scale like $\Omega(T)$, by choosing $\Gamma \preceq o(T)$, and $r^2 \propto \tr(\Gamma^{-1}) \log^2 (T)$, the term $\Psi(r,\Gamma,M)$ vanishes as $T^{-\omega(1)}$, and the approximation \Cref{eq:lb_approx} holds. Moreover, since this scaling of $r$ vanishes at a rate of $\log^2 T/ \sqrt{T}$, $r$ is small enough so as to ensure $\Exp_{\theta}[\matSig_T]$ does not vary significantly on $\calB$. We turn now to our upper bound.

\begin{thm}\label{thm:Mnorm_est_bound} Fix any matrices $\Gamma \in \pd[d], M \in \psd[d]$, with $M \ne 0$. Given a parameter $\beta \in (0,1/4)$, define the event 
\begin{equation*}
\calE := \left \{  \| \matSig_T - \Gamma  \|_\op  \le \beta \lambda_{\min}(\Gamma) \right \}
\end{equation*}
Then, if $\calE$ holds, the following holds with probability $1 - \delta$:
\newcommand{\loworder}{\mathrm{LowOrder}}
\begin{align*}
\| \thetals - \thetast \|_{M}^2 &\le 5(1+\alpha) \cdot \sigma_w^2\log \frac{6\dimtheta}{\delta} \cdot \tr(M \Gamma^{-1}), \text{ w.p. } 1 - \delta , \quad \text{where }\\
\alpha &:= 26 \beta^2 \lambda_{\max}(\Gamma) \tr(\Gamma^{-1}).
\end{align*}
\end{thm}
\begin{proof}[Proof Sketch] Like many results of this flavor, the proof is based on the self-normalized martingale inequality \citep{abbasi2011improved}. Unlike related results, however, our proof must relate in the error in the $M$-norm $\| \thetals - \thetast \|_{M}^2$ to the $\Gamma$-geometry so as to recover $\tr(M \Gamma^{-1})$. It turns out that, due to the fact that matrix square does not preserve the Lowner order (i.e., it is possible to have $0 \prec A \prec B$, but $A^2 \not \preceq B^2$ ), we require the empirical matrix $\matSig_T$ to concentrate around $\Gamma$ for this argument to go through. This forces us to require the above event $\calE$ to hold, and to suffer the error term $\alpha$. The complete proof is given in \Cref{sec:proof_thm_Mnorm_est}.
\end{proof}

\subsection{Proof of $M$-norm Regression Lower Bound (Theorem \ref{thm:gauss_assouad_M}) \label{sec:thm_gauss_assouad_proof}}
We now prove \Cref{thm:gauss_assouad_M}. The proofs of all lemmas are deferred to \Cref{sec:thm_gauss_assouad_proof_lemmas}.

Without loss of generality, set $\sigma_w^2 = 1$. The proof of the lower bound is a Gaussian-specialization of the Van Trees inequality (see, e.g. \cite{gill1995applications}), a Bayes-risk lower bound which considers the risk of estimating a quantity under a certain prior. For our prior, we use a normal distribution, which we truncate to a radius $r$. In what follows, we set
\begin{align*}
\Lambda := \Gamma^{-1} \in \pd.
\end{align*}
We let $\Ntrunc$ denote the following \emph{truncated normal} distribution: the distribution of $Z \sim \calN(\theta_0, \Lambda)$, conditioned on the event $\|Z - \theta_0\|^2 \le r$. 
We further define the full data $\frakD_T := (y_{1:T},z_{1:T})$, and let 
\begin{itemize}
\item $\Dfull(\frakD_T)$ denote the posterior of $\bmtheta$ given $\frakD_t$, when $\bmtheta$ is drawn from $\calN(\theta_0,\Lambda)$;
\item Let $\Dtrunc(\frakD_T)$ denote the distribution of $\bmtheta \mid \frakD_T$. 
\end{itemize}
Throughout, we assume that our posited estimator $\thetahat$ is a deterministic function of $\frakD_T$; this is without loss of generality for a Bayes-risk lower bound. Then, since the distribution $\Ntrunc$ is supported on the ball $\calB := \{\theta:\|\theta - \theta_0\|_2 \le r\}$, 
\begin{align*}
\inf_{\thetahat} \max_{\theta \in \calB} \Exp_\theta \Rls(\thetahat;\theta)  \ge \Expop_{\bm{\theta}\sim \Ntrunc}\Exp_{\bmtheta}\Rls(\thetahat;\bm{\theta}) = \Expop_{\bm{\theta} \sim \Ntrunc}\Exp_{\frakD_T \sim \bm{\theta}}\| \thetahat(\frakD_T) - \bmtheta\|_{2}^2.
\end{align*}
\begin{lem}[Replica Lemma]\label{lem:replica} Let $X,Y$ be abstract random variables, with $X \sim \Pr_{X}$, $Y \sim \Pr_{Y \mid X}(\cdot \mid X)$, and let $f(x,y)$ be an integrable function. Moreover, suppose that $X \mid Y$ has density $\Pr_{X \mid Y}(\cdot \mid Y)$. Then
\begin{align*}
\Exp_{X \sim \Pr_{X}}\Exp_{Y \sim \Pr_{Y \mid X} (\cdot \mid X)}[f(X,Y)] = \Exp_{X \sim \Pr_{X}}\Exp_{Y \sim \Pr_{Y \mid X}(\cdot \mid X)}\Exp_{X' \sim \Pr_{X \mid Y}( \cdot \mid Y)}[f(X',Y)].
\end{align*}
\end{lem}

By \Cref{lem:replica}, the above is equal to 
\begin{align}
\Expop_{\bm{\theta}\sim \Ntrunc}\Exp_{\bmtheta}\Rls(\thetahat;\bm{\theta}) &= \Expop_{\bm{\theta}\sim \Ntrunc}\Expop_{\frakD_T \sim \bmtheta}\Exp_{\bmtheta' \sim \Dtrunc(\frakD_T)} \| \thetahat(\frakD_T) - \bmtheta'\|_{M}^2. \label{eq:replica_trunc}
\end{align}

For a random vector $Z$ and any fixed $a$, $\Exp \| Z - a \|_M^2 \geq \Exp \| Z - \Exp Z \|_M^2$; that is, the Bayes estimator is optimal. Denoting the event $\calE := \{\| \bmtheta' - \theta_0 \|_2 \leq r\}$ (over the randomness of $\bmtheta'$), we lower bound \Cref{eq:replica_trunc} by
\begin{align}
\text{\Cref{eq:replica_trunc}} \ge & \Expop_{\bm{\theta}\sim \Ntrunc}\Expop_{\frakD_T \sim \bmtheta}\Exp_{\bmtheta' \sim \Dtrunc(\frakD_T)} \|\bmtheta' - \Exp_{\bmtheta' \sim \Dtrunc(\frakD_T)} \bmtheta'\|_{M}^2 \nonumber\\
& = \Expop_{\bm{\theta}\sim \Ntrunc}\Expop_{\frakD_T \sim \bmtheta}\Exp_{\bmtheta' \sim \Dfull(\frakD_T)} \left [  \left\|\bmtheta' - \Exp_{\bmtheta' \sim \Dfull(\frakD_T)}[ \bmtheta' \mid \calE ] \right\|_{M}^2 \mid \calE \right] \label{eq:replica_variance}
 \end{align}
 To handle this expression, we use the following technical lemma:
 \begin{lem}\label{lem:cond_exp_trunc_M} Consider a square-integrable random vector $Z \in \R^{\dimtheta}$, fixed $\mu \in \R^{\dimtheta}$, $r \ge 0$. Define the event $\calE := \{ \| Z - \mu \|_2 \le r \}$. Then, 
\begin{align*}
\Exp[\| Z - \Exp[ Z \mid \calE] \|_M^2 \mid \calE] \ge \Exp\left\|Z - \Exp[Z]\right\|_M^2 - 4 \| M \|_{\op}\Exp[\I\{\calE^c\}\|Z - \mu\|^2_2].
\end{align*}
\end{lem}

Instantiating \Cref{lem:cond_exp_trunc_M},
\begin{align*}
\text{\Cref{eq:replica_trunc}} & \ge \Expop_{\bm{\theta}\sim \Ntrunc}\Expop_{\frakD_T \sim \bmtheta}\Exp_{\bmtheta' \sim \Dfull(\frakD_T)} \left [  \|\bmtheta' - \Exp_{\bmtheta'' \sim \Dfull(\frakD_T)} [\bmtheta'' ] \|_{M}^2  \right ] \\
&\quad - 4 \| M \|_{\op}^2 \Expop_{\bm{\theta}\sim \Ntrunc}\Expop_{\frakD_T \sim \bmtheta}\Exp_{\bmtheta' \sim \Dfull(\frakD_T)} \Exp_{\bmtheta' \sim \Dfull(\frakD_T)} \left [  \|\bmtheta' - \theta_0 \|_{2}^2  \cdot \I \{ \| \bmtheta' - \theta_0 \|_2^2 > r^2 \} \right ].
\end{align*}
Hence, retracing our steps thus far, 
\begin{align}
\inf_{\thetahat} \max_{\theta \in \calB} \Exp_\theta \Rls(\thetahat;\theta)
&\ge  \underbrace{\Exp_{\bm{\theta}\sim \Ntrunc}\Exp_{\frakD_T \sim \bmtheta} \Exp_{\bmtheta' \sim \Dfull(\frakD_T)} \left [  \|\bmtheta' - \Exp_{\bmtheta' \sim \Dfull(\frakD_T)} [\bmtheta' ] \|_{M}^2  \right ]}_{\text{(a)}} \nonumber\\
&\quad - 4 \| M \|_{\op}^2 \cdot\underbrace{\left(\Exp_{\bm{\theta}\sim \Ntrunc}\Exp_{\frakD_T \sim \bmtheta}\Exp_{\bmtheta' \sim \Dfull(\frakD_T)}\left[\|\bmtheta' - \theta_0 \|_2^2 \cdot\I\{\|\bmtheta' - \theta_0\|_2^2 > r^2\}\right]\right)}_{\text{(b)}} \label{eq:term_a_b_decomp}
\end{align}
Let us control the two resulting terms.

\paragraph{Computing term \text{\normalfont(a)}:} First, we bound the dominant term $\text{(a)}$:
\begin{lem}\label{lem:Gaussian_expectation_M} The following identity holds:
\begin{align*}
\Exp_{\bmtheta' \sim \Dfull(\frakD_T)} \left [  \|\bmtheta' - \Exp_{\bmtheta'' \sim \Dfull(\frakD_T)} [\bmtheta'' ] \|_{M}^2 \right ] = \sigma_w^2 \tr(M^{1/2} (\matSig_T + \Lambda^{-1})^{-1} M^{1/2}).
\end{align*}
\end{lem}

As a direct consequence of the above lemma, we find that
\begin{align}
\text{term (a)} &= \Expop_{\bmtheta\sim \Ntrunc}\Expop_{\frakD_T \sim \bmtheta}\Expop_{\bmtheta' \sim \Dfull(\frakD_T)} \Exp_{\bmtheta'}\left[\sigma_w^2 \tr(M^{1/2} (\matSig_T + \Lambda^{-1})^{-1} M^{1/2})\right] \nonumber\\
&=  \Expop_{\bmtheta\sim \Ntrunc}\Expop_{\bmtheta}\left[\sigma_w^2 \tr(M^{1/2} (\matSig_T + \Lambda^{-1})^{-1} M^{1/2})\right]
\label{eq:terma},
\end{align}
where in the last line, we have invoked \Cref{lem:replica}.

\paragraph{Upper bounding term \text{\normalfont (b) }:} Let $p_r := \Pr_{\bmtheta\sim \calN(\theta_0,\Lambda)}[\|\bmtheta - \theta_0\| > r]$ denote the probability that $\bmtheta$ lies within the truncation region. Then, for any nonnegative function $f(\theta) \ge 0$, 
\begin{align*}
\Exp_{\bm{\theta}\sim \Ntrunc}[f(\bmtheta)] &= \frac{1}{1-p_r}\Exp_{\bm{\theta}\sim \calN(\theta_0,\Lambda)}[ f(\bmtheta) \cdot \I\{\|\bmtheta - \theta_0\| \le r\} ] \le \frac{1}{1-p_r}\Exp_{\bm{\theta}\sim \calN(\theta_0,\Lambda)}[ f(\bmtheta)  ]
\end{align*}
Thus,
\begin{align}
\text{term (b)} &= \Exp_{\bm{\theta}\sim \Ntrunc}\Exp_{\frakD_T \sim \bmtheta}\Exp_{\bmtheta' \sim \calD(\frakD_T)}\left[\|\bmtheta' - \theta_0\|_2^2 \cdot\I\{\|\bmtheta' - \theta_0\|_2^2 > r^2\}\right] \nonumber\\
&\le  \frac{1}{1-p_r}\Exp_{\bm{\theta}\sim \calN(\theta_0,\Lambda)}\Exp_{\frakD_T \sim \bmtheta}\Exp_{\bmtheta' \sim \calD(\frakD_T)}\left[\|\bmtheta' - \theta_0\|_2^2 \cdot\I\{\|\bmtheta' - \theta_0\|_2^2 > r^2\}\right] \label{eq:b_upper_bound}
\end{align}
By the replica lemma (\Cref{lem:replica}), the second line is equal to 
\begin{align*}
\text{\Cref{eq:b_upper_bound}} &= \frac{1}{1-p_r}\Exp_{\bm{\theta}\sim \calN(\theta_0,\Lambda)}\left[\|\bmtheta - \theta_0\|_2^2 \cdot\I\{\|\bmtheta - \theta_0\|_2^2 > r\}\right].
\end{align*}
We now bound the above. Note that $\bmtheta - \theta_0$ has the same distribution as $\Lambda^{\nicehalf}\matg$, where $\matg \sim \calN(0,I_d)$. Hence, $p_r = \Pr_{\matg \sim \calN(0,I_d)}[ \matg^\top \Lambda \matg > r^2]$. For $r^2 \ge 2\tr(\Lambda)$, Markov's inequality therefore implies $p_r \le 1/2$, so that $\frac{1}{1-p_r} \le 2$. Moreover, by the same change of variables, 
\begin{align*}
\Exp_{\bm{\theta}\sim \calN(\theta_0,\Lambda)}\left[\|\bmtheta -\theta_0\|_2^2 \cdot\I\{\|\bmtheta - \theta_0\|_2^2 > r^2\}\right] = \Exp_{\matg\sim \calN(0,I_d)}\left[ \matg^\top \Lambda \matg \cdot\I\{\matg^\top \Lambda \matg > r^2\}\right],
\end{align*}
Thus, for $r^2 \ge 2 \tr(\Lambda) = 2 \tr(\Gamma^{-1})$ (which follows from the condition of the theorem, $r \ge \sqrt{5 \tr(\Gamma^{-1})}$, it holds that
\begin{align*}
\text{term (b)} &\le 2\Exp_{\matg\sim \calN(0,I_d)}\left[ \matg^\top \Lambda \matg \cdot\I\{\matg^\top \Lambda \matg > r^2\}\right] = 2\int_{r^2}^{\infty} \Pr[\matg^\top \Lambda \matg > r^2] d\tau. 
\end{align*}
We now invoke a coarse consequence of the Hanson-Wright inequality:
\begin{lem}[Consequence of Hanson-Wright]\label{lem:hansonwright_consequence} For any $u \ge 4\tr(\Lambda)$, we have
\begin{align*}
\Pr[\matg^\top \Lambda \matg > \tr(\Lambda) + u] \le e^{-\frac{u}{4\|\Lambda\|_{\op}}}. 
\end{align*}
\end{lem}
	
Hence, under the assumption of the theorem, $r^2 \ge 5\tr(\Gamma^{-1}) = 5\tr(\Lambda) $, we may bound
\begin{align}
\text{term (b)} &\le 2\int_{r^2}^{\infty} \Pr[\matg^\top \Lambda \matg > \tau] d\tau \nonumber \\
&= 2\int_{r^2 - \tr(\Lambda)}^{\infty} \Pr[\matg^\top \Lambda \matg > \tr(\Lambda) + \tau] d\tau \nonumber\\
&\le 2\int_{r^2 - \tr(\Lambda)}^{\infty} e^{-\frac{u}{4\|\Lambda\|_{\op}}}  = 8\|\Lambda\|_{\op} e^{- \frac{r^2 - \tr(\Lambda)}{4\|\Lambda\|_{\op}}} \nonumber\\
&\le 8\|\Lambda\|_{\op} e^{- \frac{r^2}{5\|\Lambda\|_{\op}}}. \label{eq:term_b_bound}
\end{align}
\textbf{Concluding the Proof:}
Combining \Cref{eq:term_a_b_decomp,eq:terma,eq:term_b_bound}, we have
\begin{align*}
 \inf_{\thetahat} \max_{\theta \in \calB} \Exp_\theta \Rls(\thetahat;\theta) &\ge \Expop_{\bmtheta\sim \Ntrunc}\Expop_{\bmtheta}\left[\sigma_w^2 \tr(M^{1/2} (\matSig_T + \Lambda^{-1})^{-1} M^{1/2})\right]  \\
 &\qquad- 32\|\Lambda\|_{\op}\|M\|_{\op} e^{- \frac{r^2}{5\|\Lambda\|_{\op}}}
\end{align*}
Since $\Lambda = \Gamma^{-1}$, the last line of the above display as $\Psi(r;\Gamma, M) := \frac{32\| M \|_\op}{\lambda_{\min}(\Gamma)}\exp( - \frac{r^2 }{5 }\lambda_{\min}(\Gamma))$. Finally, we lower bound the first line of the above display crudely via Jensen's inequality:  indeed, since $X \mapsto \tr(X^{-1})$ is a convex function (on the domain of positive-definite matrices), and since convexity is preserved under affine transformation, we have
\begin{align*}
\Expop_{\bmtheta\sim \Ntrunc}\Expop_{\bmtheta}\left[\sigma_w^2 \tr(M^{1/2} (\matSig_T + \Lambda^{-1})^{-1} M^{1/2})\right]   \ge \Expop_{\bmtheta\sim \Ntrunc}\left[\sigma_w^2 \tr(M^{1/2} (\Expop_{\bmtheta}[\matSig_T] + \Lambda^{-1})^{-1} M^{1/2})\right]  
\end{align*}
Subsituting in $\Gamma = \Lambda^{-1}$, and noting that the distribution $\Ntrunc$ is supported on the ball $\calB$ concludes the bound. 
\qed

\section{Lower Bounds for Martingale Decision Making}\label{sec:mdm_lb_body_arxiv}

We next wish to apply this lower bound on $M$-norm regression to obtain a lower bound on decision making with smooth losses, our \ddmx setting.

\paragraph{Smoothness Assumptions and Consequences.} We first recall the smoothness assumption on our loss $\calR$ in the \ddmx setting.
\asmsmooth*

The above assumption directly yields the following Lipschitz conditions.
\begin{prop}\label{prop:gd_lipschitz}
Assume that $\calR,\aopt$ satisfy Assumption \ref{asm:smoothness}. Then for any model $\theta \in \R^{\dimtheta}$ and action $\fraka \in \R^{\dimfraka}$ satisfying \eqref{eq:theta_fraka_condition}, it holds that
\begin{itemize}
\item $\nabla_\fraka^{(i)} \calR(\fraka;\theta)$ is Lipschitz in the operator norm with Lipschitz constant $L_{\calR(i+1)}$ for $i=0,1,2$
\item  $\nabla_\theta^{(i)} \aopt(\theta)$ is Lipschitz in the operator norm with Lipschitz constant $L_{\fraka(i+1)}$, for $i=0,1$.
\end{itemize}
In the above, we adopted the convention $\nabla_x^{(0)} f(x) = f(x)$. 
\end{prop}

\paragraph{Relating Smooth Decision Making to $M$-norm Estimation.} The next step is to relate smooth decision making to $M$-norm estimation. We begin by introducing the revelant gradients and Hessians, and in particular, the \emph{task Hessian} $\taskhes(\theta)$ introduced in \Cref{sec:results_summary}. 

\begin{defn}[Key Gradients and Hessians]\label{def:gen_dec_properties} For some $\thetast$ and function $\calR,\aopt$, let:
\begin{itemize}
    \item $M_{\fraka}(\thetast) := \nabla_{\fraka}^2 \calR(\fraka; \thetast )$ at $\fraka = \aopt(\thetast)$.
    \item $\Gfraka(\thetast) := \nabla_{\theta}  \aopt(\theta)$ at $\theta = \thetast$.
    \item $\taskhes(\thetast) =  \nabla_{\theta}^2 \calR(\aopt(\theta); \thetast )$ at $\theta = \thetast$.  In particular, $\taskhes(\thetast) = \Gfraka(\thetast)^\top \Mfraka(\thetast) \Gfraka(\thetast) $. 
\end{itemize}
\end{defn}

The following result utilizes Assumption \ref{asm:smoothness} to guarantee that $\taskhes(\theta)$ is itself a smooth map, and that the norm induced by $\taskhes(\thetast)$ can be used to approximate $\calR(\aopt(\thetahat);\thetast)$, both of which are critical pieces in our analysis.

\begin{restatable}{prop}{propquad}\label{quad:certainty_equivalence}
Assume that $\calR,\aopt$ satisfy Assumption \ref{asm:smoothness} and that $\thetahat$ satisfies $\| \thetahat - \thetast \|_2 \le \betataylor(\thetast)$. Then the following hold:
\begin{align*}
  &\left| \calR(\aopt(\thetahat);\thetast) - \|\thetahat - \thetast\|_{\taskhes(\thetast)}^2 \right|\le \Lquad \cdot \|\thetahat - \thetast\|_{2}^3, \quad \| \taskhes(\thetast) - \taskhes(\thetahat) \|_\op \leq L_\taskhes \cdot \|\thetahat - \thetast\|_{2}.
\end{align*}
where:
\begin{align*}
    & \Lquad := \frac{1}{6} (L_{\calR3} L_{\fraka1}^3 + 3 L_{\calR2} L_{\fraka2} L_{\fraka1} + L_{\calR1} L_{\fraka3}), \quad L_\taskhes := 6 \Lquad  +  L_{\calR2} L_{\fraka1} + \Lra L_{\fraka1}^2.
\end{align*} 
\end{restatable}

We now state our key lemma, which allows us to reduce smooth decision making to $M$-norm estimation and obtain a lower bound on the local minimax risk in terms of estimation in a particular norm.

\begin{lem}\label{thm:simple_regret_lb}
Assume that the excess risk $\calR$ and optimal-decision function $\aopt$ satisfy \Cref{asm:smoothness} with smoothness parameters $L_{\calR i}$ and $L_{\fraka i}$ and radius parameters $\betataylor$ dictating the region in which the smoothness holds. Let $r > 0$ be a radius parameter satisfying
\begin{align*}
r \le \betataylor(\thetast)/4
\end{align*}
and define the associated balls $\calB_T(\thetast) := \{\theta : \| \theta - \thetast \|_{2} \leq r\}$. Then, 
\begin{align*}
& \min_{\frakahat} \max_{\theta \in \calB_T(\thetast)} \Exp_{\theta,\piexp} [\calR(\frakahat; \theta)]  \geq \min \Bigg \{ \min_{\thetahat} \max_{\theta \in \calB_T(\thetast)} \frac{1}{2} \Exp_{\theta,\piexp} \left [\| \thetahat - \theta \|_{\taskhes(\thetast) }^2 \right ]   - C_1 r^{3},   \mu L_{\fraka1}^2 r^2 \Bigg \}
\end{align*}
where we define the constant, for a universal numerical constant $c_1$, 
\begin{align*}
C_1 = c_1 \Big (  L_{\fraka1} L_{\fraka2} L_{\calR2}  +  L_{\fraka1}^3 L_{\calR 3} +  \Lra \Big ).
\end{align*}
\end{lem}
\begin{proof}[Proof Sketch]
Our goal is to show that $\Exp_{\theta,\piexp} [\calR(\frakahat; \theta)]$ can be lower bounded by the estimation error of $\theta$ in the $\| \cdot \|_{\taskhes(\thetast)}$ norm. While \Cref{quad:certainty_equivalence}, shows that this equivalence is true when $\frakahat$ is the certainty equivalence estimate, here we want to show that this is true for \emph{any} estimate. To this end, we define
\begin{align*}
\updelta_{\star}(\frakahat) &: = \argmin_{\updelta} \| \Mfraka(\thetast)^{1/2} \left (  \frakahat -  \aopt(\thetast) - \Gfraka(\thetast) \updelta \right ) \|_2 \\
&= (\Mfraka(\thetast)^{1/2} \Gfraka(\thetast))^\dagger \Mfraka(\thetast)^{1/2} ( \frakahat -  \aopt(\thetast))
\end{align*}
and define the induced estimate
\begin{align*}
\thetahat(\frakahat) &:= \thetast + \updelta_{\star}(\frakahat)
\end{align*}
Intuitively, we would expect $\thetahat(\frakahat)$ to be close to $\theta$ in an appropriate metric if our decision $\frakahat$ achieves a small excess risk, $\calR(\frakahat;\theta)$. By carefully Taylor expanding both $\calR(\frakahat;\theta)$ and $\aopt(\theta)$, we show that this is the case, writing the excess risk $\calR(\frakahat;\theta)$ as the sum of $\| \thetahat(\frakahat) - \theta \|_{\taskhes(\thetast)}^2$ and a $\calO(r^3)$ term. This implies that if $\calR(\frakahat;\theta)$ is small, $\| \thetahat(\frakahat) - \theta \|_{\taskhes(\thetast)}^2$ will also be small, which reduces the problem of minimizing $\calR(\frakahat;\theta)$ to that of estimating $\theta$ in the $\| \cdot \|_{\taskhes(\thetast)}$ norm. As $\thetahat(\frakahat)$ is a particular estimator of $\theta$ given our trajectory, it follows that the resulting loss is lower bounded by minimizing over all estimators, $\thetahat$, which gives the result. We defer the details of this argument to \Cref{sec:lem_quad_approx_lb_pf}.
\end{proof}

By tuning the radius parameter $r(T)$ appropriately, we achieve the following general purpose lower bound on the excess risk:
\begin{thm}\label{thm:simple_regret_lb2}
Suppose the smoothness assumption \Cref{asm:smoothness} holds with its stated smoothness parameters. In addition, fix a regularization parameter $\lambda > 0$, and suppose that $T$ satisfies
\begin{align*}
\lambda T \ge \max \left \{ \left ( \tfrac{80 \dimtheta}{\betataylor(\thetast)^2}  \right )^{6/5}, \left ( \tfrac{ \sigma_w^2 L_{\calR 2}}{5 \mu} \right )^{6} \right \}
\end{align*}
Finally, define the localizing ball $\calB_T := \{\theta : \| \theta - \thetast \|_{2}^2 \leq 5 \dimtheta /(\lambda T)^{5/6} \}$.  Then, for any $\theta_0 \in \calB_T(\thetast)$,
\begin{align*}
\min_{\frakahat} \max_{\theta  \in \calB_T} \Exp_{\theta,\piexp} [\calR(\frakahat; \theta)] &  \geq  \sigma_w^2 \min_{\theta \in \calB_T}  \tr\left(\taskhes(\thetast) \left( \Exp_{\theta,\piexp}[\matSig_T] + \lambda T \cdot I \right)^{-1}  \right)   - \frac{ C_2 }{ (\lambda T)^{5/4}}.
\end{align*}
where we have defined the constant, for a universal numerical constant $c_2$,
\begin{align*}
\small C_2 &= c_2 \Big (    (L_{\fraka1}  L_{\fraka2} L_{\calR2} + L_{\fraka1}^3 L_{\calR 3} +  \Lra) \dimtheta^{3/2} + L_{\fraka1}^2 L_{\calR2}  \Big ).
\end{align*}
\end{thm}
\begin{proof}[Proof Sketch]
This result follows by applying \Cref{thm:simple_regret_lb} to lower bound the local minimax risk by estimation of $\theta$ in the $\| \cdot \|_{\taskhes(\thetast)}$ norm. We then apply \Cref{thm:gauss_assouad_M} to lower bound the estimation error of $\theta$ in this norm, which yields the stated bound. The details of this argument are given in \Cref{sec:proof_thm_simpl_reg_lb2}.
\end{proof}

\Cref{thm:simple_regret_lb2} is our most general lower bound and serves as the basis for the lower bounds stated in \Cref{sec:overview_lds}. Indeed, \Cref{cor:simple_regret_lb2_nice} can be derived as a simple corollary of this result under \Cref{asm:suff_excite} and \Cref{asm:smooth_covariates}. We explicitly state this argument in \Cref{sec:smpl_regret_lb2_nice_pf}.

%% file: body/conclusion.tex
%!TEX root = ../main_arxiv.tex

\section{Conclusion}\label{sec:conclusion}
In this work, we have shown that task-guided exploration of an unknown environment yields significant improvements over task-agnostic exploration. Furthermore, we have derived an instance- and task-optimal exploration algorithm which applies to a wide range of decision making problems, and derived corresponding instance- and task-dependent lower bounds. Our results also establish that certainty equivalence decision making is optimal, and we obtain the first instance-optimal algorithm for the \lqrx problem. This work raises several interesting questions.
\begin{itemize}
\item While our martingale decision making setting encompasses certain classes of nonlinear systems, all our results fundamentally rely on linear observations of the parameter of interest, $\thetast$. Task-optimal exploration remains an open question for general nonlinear systems, and is an interesting future direction.
\item We show that the smoothness conditions on our loss are met by a wide range of decision making problems. However, it remains an interesting future direction to obtain an optimal algorithm that holds without these smoothness assumptions. As \cite{wagenmaker2020active} shows, when the loss is the operator norm---which we note does not satisfy our smoothness assumption---the optimal algorithm takes a form very similar to \algname. Does a general algorithm and analysis exist for both smooth and non-smooth losses?
\item Our work focuses on the offline, pure-exploration setting. Extending our analysis to obtain instance- and task-optimal rates in the \emph{online} setting is an interesting direction of future work. For the online \lqrx problem in particular, \cite{simchowitz2020naive} obtain the optimal scaling in terms of dimension but their rates are suboptimal in terms of other problem-dependent constants. On the lower bound side, \cite{ziemann2020uninformative} provide an instance-dependent lower bound but give no upper bound. Solving this problem may require new algorithmic ideas, and we leave this for future work. 
\end{itemize}

\subsubsection*{Acknowledgements} 
The work of AW is supported by an NSF GFRP Fellowship DGE-1762114. MS is generously supported by an Open Philanthropy AI Fellowship. The work of KJ is supported in part by grants NSF RI 1907907 and NSF CCF 2007036.

%% file: appendix/organization_notation_arxiv.tex
%!TEX root = ../main_arxiv.tex
\newcommand{\specialcell}[2][l]{\begin{tabular}[#1]{@{}l@{}}#2\end{tabular}}

\section{Organization and Notation}

\subsection{Organization}
We break the appendix up into three parts. First, \Cref{sec:remarks} sketches out various extensions to our results, provides additional remarks, and states and proves the formal version of the lower bound on low-regret algorithms.

\Cref{part:general} covers martingale decision making. \Cref{sec:M_norm_regression} completes the proofs of our upper and lower bounds for martingale regression in general norms considered in \Cref{sec:M_norm_regression}. Next, \Cref{sec:general_decision} formally proves a locally minimax lower bound on martingale decision making with smooth losses, completing the arguments sketched out in \Cref{sec:mdm_lb_body_arxiv}. Finally, \Cref{sec:ddm_ce_upper} proves a general upper bound on certainty equivalence decision making under a certain regularity assumption. Of note, \Cref{part:general} does not assume we are operating in the setting of a linear dynamical system---the results here apply to the more general martingale decision making setting.

\Cref{part:lin} covers our results in the setting of linear dynamical decision making setting. We begin in \Cref{sec:lds_notation} by introducing additional notation specific to linear dynamical systems we will use throughout. In \Cref{sec:lds_lb}, we apply the results of \Cref{part:general} to prove our lower bound on \emph{optimal} decision-making in linear dynamical systems. This section also shows that our restricted policy class, $\calU_{\gamma^2}^\mathrm{p}$, the set of periodic signals, contains a near-optimal policy. \Cref{sec:lds_ce_upper} shows that sequential open-loop policies meet our regularity assumption and we therefore obtain a corollary on efficient certainty equivalence decision making in linear dynamical systems. In addition, \Cref{sec:lds_ce_upper} provides rates at which the covariates of linear dynamical systems concentrate, an important piece in our analysis. Finally, \Cref{sec:lds_exp_design} proves the upper bound on the performance of \algname. Our proof relies on showing that certainty equivalence experiment design plays near-optimal inputs, and that \algname is itself a sequential open-loop policy, allowing us to apply our certainty equivalence bound proved in \Cref{sec:lds_ce_upper}.

Lastly, \Cref{part:examples} covers applications of our results. \Cref{sec:exp_design} shows that the \lqrx problem is an instance of our general smooth decision making setting, and that we can therefore apply all our results to this problem. \Cref{sec:examples} works out explicitly the rates obtained by \algname and other exploration approaches in several \lqrx examples. Finally, \Cref{sec:exp_details} provides additional details on our numerical experiments.

\subsection{Notation}
Below we present notation used throughout this work. We define our signal notation in more detail in \Cref{ssec:fourier_preliminaries} and additional details on our notation for linear dynamical systems is presented in \Cref{sec:lds_notation}. We will overload notation somewhat throughout, using $\theta$ to refer to a vector as well as the concatenation of matrices, $\theta = (A,B)$. In the latter case, $\| \theta \|_\op$ denotes the operator norm of the matrix $(A,B)$ but we let $\| \theta \|_M^2 = \vectorize(\theta)^\top M \vectorize(\theta)$.

\begin{center}
\begin{tabular}{ |c | l |  }
\hline
\textbf{Mathematical Notation} & \textbf{Definition} \\ 
\hline
$ \| \cdot \|_\op $ & Matrix operator norm \\
$ \| \cdot \|_F $ & Matrix Frobenius norm \\
$ \| \cdot \|_2 $ & Vector 2-norm \\
$ \| \cdot \|_M $ & Vector Mahalanobis norm, $\| x \|_M^2 = x^\top M x$ \\
$ \| \cdot \|_{\Hinf} $ & System $\mathcal{H}$-infinity norm \\
$ \imag $ & Imaginary number, $\sqrt{-1}$ \\
$ \mathbb{S}_{++}^d$ & Positive definite matrices of dimension $d \times d$ \\
$ \mathbb{S}_{+}^d$ & Positive semi-definite matrices of dimension $d \times d$ \\
$ \mathcal{S}^{d-1}$ & Unit ball in $\R^d$ \\
\hline
\textbf{Policy Notation} &  \\
\hline
$ \pi/\piexp$ & Exploration policy \\
$ \policyset $ & Policies with average expected power bounded by $\gamma^2$ \\
$ \policysetgood$ & \specialcell{Sequential open-loop policies with average \\ \qquad expected power bounded by $\gamma^2$} \\
$ \Pi_{\gamma^2}^{\mathrm{p}} $ & Periodic policies with average expected power bounded by $\gamma^2$ \\
$ \traj $ & Input-state trajectory, $\traj = (x_{1:T+1},u_{1:T})$ \\
$ \plan(\traj)$ & Decision rule \\
$ \ace(\traj)$ & Certainty-equivalence decision rule \\
\hline
\textbf{Complexity Notation} &  \\
\hline 
$ \Phi_T(\pi;\thetast)$  & Idealized risk, $\Phi_T(\pi;\thetast) := \tr(\taskhes(\thetast) \matGam_{T}(\pi;\thetast)^{-1})$ \\
$\Phiopt(\gamma^2;\thetast) $ & Optimal risk, $ \Phiopt(\gamma^2;\thetast) := \liminf_{T \to \infty} \inf_{\piexp \in \Pi_{\gamma^2}} \Phi_T(\piexp; \thetast)$ \\
$ \Phiss(\gamma^2;\thetast) $ & Steady-state analogue of $\Phiopt(\gamma^2;\thetast)$ \\
$ \mmax_{\piexp}(\calR; \calB)  $ & Local minimax risk \\
$\mmax_{\gamma^2}(\calR;\calB)$ & \specialcell{Exploration local minimax risk, lower bound on optimal \\ \qquad policy risk } \\
\hline
\ddmx \textbf{Notation} & \\
\hline
$ \theta \in \R^{\dimtheta} $ & Nominal instance \\
$ \fraka \in \R^{\dima} $ & Decision variable \\
$ \Jlqr_{\thetast}(\fraka) $ & Loss function \\
$\calR_{\theta}(\fraka)/\calR(\fraka;\theta)$ & Excess risk \\
$ \aopt(\theta)$ & Optimal decision for instance $\theta$ \\
$\taskhes(\thetast)$ & \specialcell{Hessian of certainty equivalence excess risk,  \\ \qquad $\taskhes(\thetast) = \nabla_{\theta}^2 \calR(\aopt(\theta);\thetast)|_{\theta = \thetast}$}  \\
$ L_{\calR i}, i = 1,2,3 $ & Upper bound on $\| \nabla_\fraka^{(i)} \calR(\fraka;\theta) \|_\op $ \\
$ L_{\fraka i}, i = 1,2,3 $ & Upper bound on $\| \nabla_\theta^{(i)} \aopt(\theta) \|_\op $ \\
$ \Lra$ & Lipschitz constant of $\nabla_\fraka^2 \calR(\fraka;\theta)$ in $\theta$ \\
$ \mu $ & Parameter for quadratic lower bound on $\calR(\fraka;\thetast)$ \\
$ \Lquad$ & $\frac{1}{6} (L_{\calR3} L_{\fraka1}^3 + 3 L_{\calR2} L_{\fraka2} L_{\fraka1} + L_{\calR1} L_{\fraka3})$ \\
 $L_\taskhes$ & $6 \Lquad  +  L_{\calR2} L_{\fraka1} + \Lra L_{\fraka1}^2$ \\
$\betataylor(\theta)$ & Radius in which gradient bounds hold \\
\hline
\end{tabular}
\end{center}

\begin{center}
\begin{tabular}{ |c | l |  }
\hline
\textbf{LDS Notation} & \textbf{Definition} \\ 
\hline
$\theta = (A,B)$ & System parameters \\  
$\dimx$ & State dimension \\
$ \dimu $ & Input dimension \\
$ d $ & $ \dimx + \dimu $ \\
$ \sigma_w^2 $ & Process noise variance \\
$ \xu_t$ & Portion of state driven by input \\
$ \xw_t$ & Portion of state drive by noise ($x_t = \xu_t + \xw_t$) \\
$ \rho(A) $ & Spectral radius of $A$ \\
$\thetatil = (\Atil,\Btil)$ & Lifted dynamical system, $\Atil = \begin{bmatrix} A & B \\ 0 & 0 \end{bmatrix}, \Btil = \begin{bmatrix} 0 \\ I \end{bmatrix}$ \\
$ \tau(A,\rho)$ & Scaling of upper bound on $\| A^k \|_\op$, $\| A^k \|_\op \le \tau(A,\rho) \rho^k$ \\
$ \rhool $ & $  \max \left \{ \frac{1}{2}, \frac{2 \| \Ast \|_{\Hinf} \| \Ast \|_\op}{1 + 2 \| \Ast \|_{\Hinf} \| \Ast \|_\op} \right\} $ \\
$ \tauol $ & $\tau(\Atilst,\rhool)$, scaling of upper bound on $\| \Atilst^k \|_\op$ \\
\hline 
 \textbf{Signal Notation} & \\
 \hline
 $ \Fourier ( \cdot ) $ & Discrete-time Fourier transform (DFT) \\
 $\Fourier^{-1} (\cdot )$ & Inverse discrete-time Fourier transform \\
 $\bmu$ & Vector signal, $(u_t)_{t=1}^k$, $u_t \in \C^{\dimu}$ \\
 $\bmucheck$ & Discrete-time Fourier transform of $\bmu$ \\
 $\ucheck_t$ & Element of $\bmucheck$,  DFT of $(u_s)_{s=1}^k $\\
 $\bmU$ & Matrix signal, $(U_t)_{t=1}^k$, $U_t \in \C^{\dimu \times \dimu}$, $U_t$ Hermitian \\
 $ U_t$ & Matrix signal element \\
 $\calU_{\gamma^2,k}$ & Set of length $k$ matrix signals with power bounded by $\gamma^2$ \\
\hline
 \textbf{Covariance Notation} & \\
 \hline
$\Sigma_T$ & Random covariates \\
$\matGam$ & Kronecker of covariates, $I_{\dimx} \otimes \Gamma$ \\
$ \Gamnoise_t(\theta,\Sigma_u)$ & Expected $t$-step noise covariance when $u_t \sim \calN(0,\Sigma_u)$ \\
$ \Gamnoise_t(\theta,\sigma_u)$ & $\Gamnoise_t(\theta,\sigma_u^2 I)$ \\
$ \Gamin_t(\theta,\bmu,x_0)$ & Covariance obtained on noiseless system playing $\bmu$ starting from $x_0$  \\
$ \Gamma_T(\theta,\bmu,\sigma_u,x_0)$ & \specialcell{Expected covariance on noisy system when playing $\bmu$, \\ \qquad input noise $\calN(0,\sigma^2 I)$, starting from $x_0$} \\
$ \Gamfreq_k(\theta,\bmU)$ & \specialcell{Frequency-domain steady-state covariance for length-$k$ matrix \\ \qquad input $\bmU$} \\
$ \Gamfreq_{t,k}(\theta,\bmU)$ & $\frac{t}{k} \Gamfreq_k(\theta,\bmU)$ \\
$ \Gamss_{t,k}(\theta,\bmU,\sigma_u)$ & \specialcell{Expected steady state covariance when playing length-$k$ input $\bmU$ \\ \qquad and input noise $\calN(0,\sigma_u^2 I)$} \\
$ \bar{\Gamma}_T$ & High probability upper bound on covariates \\
$ \gamup $ & \specialcell{High probability upper bound on covariates for sequential \\ \qquad open-loop policies} \\
$ \lamund $ & Lower bound on minimum eigenvalue of covariates \\
$\lamnoise(\sigma_u)$ & \specialcell{Minimum eigenvalue of noise Grammian, minimum excitation \\ \qquad due to noise} \\
$\lamnoise$ & $ \lamnoise(\gamma/\sqrt{2\dimu})$ \\
$ \alphast(\thetast,\gamma^2)$ & Smoothness of covariates with respect to $\theta$ \\
$ \rcov(\thetast) $ & Radius in which smoothness of covariates holds \\
 \hline
\end{tabular}
\end{center}

%% file: body/remarks_and_extensions.tex
%!TEX root = ../main.tex

\section{Remarks and Extensions}\label{sec:remarks}

\subsection{Randomized Decisions}
Our framework extends to possibly randomized decisions $\frakahat$; that is, decisions $\frakahat = \synth(\traj, \matxi)$, where again $\traj$ is the observed trajectory, and $\matxi$ is internal algorithmic randomness. Note that our upper bounds all hold for the \emph{deterministic} certainty equivalence decision rule. Our lower bounds, however, encompass these randomnized decision rules. This can be be seen by examining the proof of our lower bound \Cref{thm:gauss_assouad_M}, which proceeds by lower bounding the \emph{Bayes risk} over a distribution supported on a ball of a given radius. Hence, the performance of any randomized decision rule $\frakahat$ is no better than the performance of the deterministic decision rule which considers the value of the random seed $\matxi$ attaining the least Bayes risk over the distribution considered in the lower bound \Cref{thm:gauss_assouad_M}. 

\subsection{Non-Identity Noise Covariance \label{sec:noise_cov}}
A known non-identity noise covariance $\Sigma_w$  can be adressed by a change of basis to whiten the noise. Unknown noise covariances can be estimated. One can show that one need only estimate $\Sigma_w$ up to a constance accuracy, i.e. $|\hat{\Sigma}_w - \Sigma_w| \le c \lambda_{\min}(\Sigma_w)$ for a small constant $c > 0$, and use $\hat{\Sigma}_w$ either for a change of basis. One can show that this will still yield optimal rates up to constant factors (determined by the magnitude of $c$).

\subsection{Unstable Systems and State Costs in Experiment Design}
In many cases, one may wish to perform experiment design on systems that are either unstable, i.e. $\rho(\Ast) > 1$, or are systems which are marginally stable $\rho(\Ast) = 1$, or which have a large mixing time, $\rho(\Ast) \approx 1$.  This poses two challenges: 
\begin{itemize}
	\item To show example optimality, our analysis requires concentration of the empirical covariance matrix around its expectation. For either unstable or marginally stable systems, existing analysis suggests this may not be true \citep{simchowitz2018learning,sarkar2019near}. Moreover, estimation with unstable systems requires additional nondegeneracy conditions \citep{sarkar2019near}.
	\item Because the magnitude of the state, and thus eigenvalues of the covariance matrix grow rapidly in marginally stable and in unstable systems, they may constitute a somewhat unrealistic setting for experiment design: in practice, very large states/covariances are highly undesirable, whereas for estimation, they can be quite beneficial. 
\end{itemize}
To adress these concerns, we propose three settings which would yield meaningful extensions of experiment design to unstable/marginally unstable settings. 

\paragraph{Multiple Rollouts:} One can instead consider experiment design with, say, $n$ independent rollouts of finite horizon $H$. By forcing the system to reset, this is sufficient to ensure concentration of the relevant covariance matrices, and obviate consistency issues that may arise in the unstable setting. 

\paragraph{Stabilizing Controller:} Another approach is to assume the existence of a stabilizing controller $K_0$, and select inputs $u_t = K_0 x_t + \nu_t$, where $\nu_t$ is an additional input chosen to optimize the experiment design. We can then impose the total power constraint on the total square norm of the $\nu_t$ inputs. 

\newcommand{\Rbudg}{R_{\mathrm{budg}}}
\newcommand{\Qbudg}{Q_{\mathrm{budg}}}

\subsubsection{State Costs} Imposing total power constraints on the additional inputs $\nu_t$ in the above example may appear somewhat artificial. Instead, one may wish to explicitly encode the tradeoff between ensuring state magnitudes are small, and the rate of estimation (as determined by the eigenvalues of covariance matrix) is fast. To this end, we can consider control budgets of the form of LQR-like penalties
\begin{align}
\sum_{t=1}^T  x_t^\top \Qbudg x_t +  u_t^T \Rbudg u_t \le T \gamma^2.  \label{eq:state_cost}
\end{align}
We stress that the cost matrices $\Qbudg$ and $\Rbudg$ above pertain to the experiment design, and not to, say, an LQR synthesis task for which the experiment design is being considered. 

We further note that satisfying the constraint \eqref{eq:state_cost} may be infeasible: indeed, this occurs whenever both (a) the optimal infinite LQR cost for with cost matrices $(\Qbudg, \Rbudg)$ is strictly greater than $\gamma^2$, and (b) the the horizon $T$ is sufficiently large (so that the  finite horizon optimal costs approaches its limiting, infite-horizon value).

Finally, the design for budgets of the form \Cref{eq:state_cost} may be \emph{closed-loop}: that is, they may necessarily require a inputs $u_t$ which are functions of past states $\matx_{1:t}$. In constrast, when budget only constraints total \emph{input power}, we have shown that \emph{open-loop} inputs (i.e. those not dependending on past states) suffice for optimality, up to constant factors. This raises the question of how to conduct efficient experiment design over such closed loop policies. In the interest of brevity, we sketch a promising approach to this problem, and omit the details for future work:
\begin{itemize}
	\item Observe that the cost \eqref{eq:state_cost} is itself a linear form in the joint covariance matrices of the states and inputs. Hence, the experiment design roughly amoungs to optimizing a convex function of the form $\tr(\mathcal{H} \cdot \Gamma^{-1})$ over feasible state-input covariance matrices $\Gamma$, subject to a linear constrain of the form $\tr( \mathcal{C} \cdot \Gamma)$. This is a convex program in $\Gamma$.
	\item To characterize the set of feasible covariance matrices, we can observe that any feasible covariance matrix can be obtained by combining a linear feedback policy with an open loop policy (this can be verified using Gaussianity). One promising computation approach to perform this optimization is to use system level synthesis \citep{anderson2019system}, where the linear feedback term can be represented as a linear form in the noise variables $w_t$. Thus, the desired covariance matrices can be represnted as outer-products of open-loop inputs and linear forms. 
	\item The SLS representation then describes the set of feasible covariance matrices as outer-products of linear forms; this is not yet a convex representation. However, just as this paper operators on the convex hull of covariance matrices arising from open-loop inputs, a similar convex relaxation can yield a convex representation of covariance matrices with closed-loop feedback. This relaxation is not loose: the space of feasible expected covariance matrices over all policies is convex, since one can always interpolate between two covariance matrices via probabilistic interpolations between the policies which generated them (i.e. selected some policy $\pi_1$ with probability $p$, and another with probability $1-p$). This means that, given a feasible expected covariance matrix recovered from this relaxation, we can produce a policy to generate it. 
\end{itemize}

\subsection{Expectation v.s. High Probability}
Observe that our upper bounds are stated with high probability, whereas lower bounds are stated in expectation. This is because on lower bounds proceed (like most information theoretic lower bounds) via bounds on the Bayes-Risk, which regard expected performance; on the other hand, our upper bounds may not hold in expectation because, on a highly improbable failure event, the estimate may produce a decision which has \emph{infinite} cost (e.g. a controller returned for an LQR task which fails to stabilize the system).

To close the gap between the two, we can make the following modifications:
\begin{itemize}
	\item Our lower bounds on expected risk can be restated as lower bounds on a constant probability of error. To see this, we note that our lower bound holds over a localized set of instances, $\calB = \{\theta:\|\theta - \theta_0\|_2 \le r\}$. Hence, any decision rule / experiment design procedure can be modified to only return decisions which satisfy some minimum worst-case performance on $\calB$ (and, under the smoothness assumptions considered in this work, this can be done without harming the performance of the decision rule). Thus, the worst case suboptimality of the decision rule can be no more than a constant, and thus, the lower bound in expectation can be tranformed into a lower bound holding with constant probability.
	\item Similarly, if the learner is given side information (e.g. a convex set  $\calA$ of possible decisions known to contain an open ball around the optimal decision $\aopt(\thetast)$, and such that the cost $\sup_{\fraka \in \calA}\mathcal{J}_{\thetast}(\fraka) < \infty$), then the learner can achieve upper bounds in expectation by projecting their decision $\frakahat$ onto the set $\calA$, namely
	\begin{align*}
	\frakahat' = \mathrm{Proj}_{\calA}(\frakahat)
	\end{align*}
	Then, whenever $\frakahat'$ is sufficiently close to $\aopt(\thetast)$, $\frakahat' =\frakahat$ and the cost will be unaffected; however, on low-probability failure events, the projection step ensures the cost remains bounded.
\end{itemize}

\subsection{Subspaces and Parameteric Uncertainty}
In many applications, one considers linear dynamical systems $(A,B)$ where some coordinates, or more generally, subspaces of the dynamical matrices are known to the learner, and only some coordinates or subspaces must be learned. In this case, learning the matrices $(A,B)$ with unconstrained least squares may be suboptimal. 

However, the subspace-constrained learning setting can be easily re-written as an unconstrained learning problem restricted to an appropriate subspace, and this resulting structure obeys the general martingale least squares setting outlined in \Cref{part:general}. Hence, the same arguments given in that section demonstrate can be used to demonstrate optimality of certainty equivalence. The algorithm \algname can be similarly modified to optimize for the covariance matrix in the relevant restricted subspace. 

A more general constrained setting is where $(A,B) = (A(\theta),B(\theta))$ are smooth, possibly nonlinear functions of a hidden parameter $\theta$. In this case, we conjecture that one can achieve optimal rates by obtaining a course estimate $\thetahat$ of $\theta$,  applying constrained least squares in the subspace defined by the image of the Jacobian $\frac{\rmd }{\rmd \theta} A(\theta),B(\theta))$ at $\thetahat = \theta$. We leave  the details for future work.

\subsection{Parametric Nonlinear Systems}
Many of the results in this work can be extended to the parameteric non-linear systems considered in the recent literature \citep{mania2020active,kakade2020information}: 
\begin{align}
x_t = \langle \thetast, \phi(x_t,u_t) \rangle + w_t, \quad w_t \iidsim \calN(0,\sigma_w^2) \label{eq:nonlinear}
\end{align}
where $\thetast \in \R^d$ is a linear paramter describing the dynamics, and $\phi : \R^{\dimx + \dimu} \to \R^d$ is an embedding function \emph{known} to the learner. Despite the nonlinear, the dynamics \eqref{eq:nonlinear} satisfy the martingale regression setting considered in \Cref{sec:M_norm_regression}, and thus the \ddmx upper and lower bounds in that section extend (\Cref{thm:gauss_assouad_M} and \Cref{thm:Mnorm_est_bound}); similarly, the guarantees of \Cref{sec:general_decision} extend as well as long as the cost functional $\Jfunc_{\theta}(\fraka)$ satisfies the requisite regularity conditions. 

Unfortunately, attempts to extend these guarantees to optimal experiment design encounter a number of difficulties:
\begin{enumerate}
	\item The experiment-design lower bounds established for linear dynamical systems require verifying that we can consider, without loss of generality, exploration policies $\piexp$ which produce sufficiently ``regular'' periodic inputs (see \Cref{sec:lds_lb}); it is not clear how this argument would generalize to the nonlinear setting \Cref{eq:nonlinear}, where potentially highly pathological exploration policies may be preferrable. 
	\item The certainty-equivalence upper bounds for linear systems require demonstrating concentration of the empirical covariance matrix around its mean; for linear systems, this can be verified as long as the dynamical matrix $\Ast$ is stable. For nonlinear systems, further conditions need to be imposed. 
	\item The experiment design problem for nonlinear systems may be computationally intractable. In addition, the experiment design objective may be very sensitivie to errors in the estimate of the parameter $\thetast$, so that solving the certainty equivlanet experiment design objective (i.e. optimal design based on an estimate $\thetahat$) may be a poor proxy for the optimal design.
	\item For nonlinear systems, controlling how the error in parameter estimation error translates into suboptimality in the decision $\frakahat$ for the given task may be quite challenging. Even for LQR synthesis in linear systems, verifying the smoothness conditions in  \Cref{sec:general_decision} relies on subtle technical tools developed specifically for LQR \citep{simchowitz2020naive}.
\end{enumerate} 

\subsection{Suboptimality of Low Regret Algorithms \label{sec:insufficiency_low_reg}}
\newcommand{\Kerr}[1][K]{#1\text{-}\mathrm{err}}
\newcommand{\Keerr}{K_{e}\text{-}\mathrm{err}}

\newcommand{\bigohst}[1]{\mathcal{O}_{\star}(#1)}

Here, we state a formal lower bound about the suboptimality of low regret algorithms. Consider a nominal instance $(\Ast,\Bst)$. For simplicity, we consider a normalization where $\Rx \succeq I$ and $\Ru \succeq I$, which can be enforced by suitable renormalization. 

\begin{prop}[Formal statement of \Cref{prop:informal_lr}]\label{prop:formal_regret_lb} Fix a nominal instance $\thetast = (\Ast,\Bst)$ with optimal value function $\Pst$, costs $\Rx,\Ru \succeq I$, and select a regret lower bound bound 
\begin{align*}
R \ge \dimu^2 \dimx\mathrm{poly}(\|\Pst\|_{\op},\|\Bst\|_{\op}) + \|\Pst\|_{\op}^2\sqrt{\dimx T}/4. 
\end{align*}
Then, over the ball of instances $\calB = \calB(R) := \{\theta:\|\thetast - \theta\|_{\fro}^2 \le \frac{\dimu^2 \dimx}{16 \|\Pst\|_{\op} R} \}$, the following lower bound for any low-regret exploration policy $\pilr$ and controller $K_{\mathrm{lr}}$ synthesized from the trajectory it collects:
\begin{align}
\max_{\theta \in \calB}\Exp_{\theta,\pilr}[\mathrm{Reg}_T] \ge R, \quad \text{ or } \quad \max_{\theta \in \calB}\Exp_{\theta,\pilr}[\calR_{\lqr,\theta_\mate}(K_{\mathrm{lr}})] \ge  \frac{\dimu^2  }{320 R} \cdot \left(\max_{1\le m \le \dimx}m\cdot \sigma_{m}(\Ast + \Bst \Kst)^2\right).
\end{align}
where $\sigma_m(\cdot)$ denotes the $m$-th largest singular value.
\end{prop}
For many instances of interest, $\left(\max_{1\le m \le \dimx}m\sigma_{m}(\Ast + \Bst \Kst)^2\right) > 0$ is a constant bounded away from $0$, and even scales with dimension $\dimx$. Hence, we find a strong tradeoff between low regret and optimal estimation. The key intuition behind the proof is that low regret algorithms converge to inputs $\matu_t \approx \Kst \matx_t$ approaching the optimal control policy; in doing so, they under-explore directions \emph{perpendicular} to the hyperplane $\{(x,u): u = \Kst x\}$, which are necessary for indentifying the optimal control policy. This idea, as well as the rigorous proof, draws heavily on the regret lower bound due to \citep{simchowitz2020naive}. 

\begin{proof}[Proof of \Cref{prop:formal_regret_lb}]
Throughout, fix a low regret policy $\pi_{\mathrm{lr}}$.
The proof follows from the arguments of \cite{simchowitz2020naive}. Fix a nominal instance $\thetast = (\Ast,\Bst)$, with optimal controller $K_{\star}$. Let $m \in [\dimx]$, and adopt the shorthand $n = \dimu$. For binary vectors $e \in \{-1,+1\}^{nm}$ consider a packing
\begin{align*}
\theta_e := (A_e,B_e) = (\Ast - \Delta_e \Kst, \Bst + \Delta_e), \quad \Delta_e = \epsilon \sum_{i=1}^n \sum_{j=1}^m e_{i,j}v_iw_j^\top,
\end{align*}
where $\epsilon$ is a parameter to be chosen small than $\epsilon_0 := \frac{1}{nm \cdot \mathrm{poly}(\|\Pst\|_{\op})}$ for a larger enough polynomial $\mathrm{poly}$,  and $(v_i)$ and $(w_j)$ are appropriately selected orthonormal basis vectors. These instance are constructed so that $A_e + B_e \Kst$ are identical for all packing indices $e$; in other words, by selecting the optimal controller for the nominal instance $\thetast$, al the instances are indistinguishable.

Let $K_e := K_{\mathrm{opt}}(\theta_e)$ denote the optimal controller for these instances. We let $\Expop_{\mate}$ denote expectation under the uniform distribution over $\mate \sim \{-1,1\}^{nm}$ from the hypercube.  We consider the term $\Kerr$ from \cite{simchowitz2020naive}, modified to include all $T$ time steps (instead of $T/2$). For any controller $K \in \R^{\dimu \dimx}$, define
\begin{align*}
\Kerr[K] := \Exp_{\theta_e,\pi_{\mathrm{lr}}}[\sum_{t=1}^T \|u_t - K x_t\|_2^2].
\end{align*}
The reason for considering $T$ steps is because here we are concerned with the \emph{offline} learning problem, where the learner is allowed to use all data from the trajectory to synthesize a controller. 

The first claim lower bounds the regret by average deviation from the optimal control policy under the nominal instance:
\begin{claim}\label{claim:reg_lb} Let $\gamma_{\mathrm{err}} = \dimx \cdot \mathrm{poly}(\|\Pst\|_{\op},\|\Bst\|_{\op})$, where $\calO$ hides universal constants. Then, 
\begin{align}
 \Exp_{\theta_e,\pi_{\mathrm{lr}}}[\mathrm{Regret}_T] \ge \frac{1}{4}\Expop_{\mate}\Kerr_{\mate}[K_{\star}] - nmT\|\Pst\|_{\op}^4 \epsilon^2 - \gamma_{\mathrm{err}}
\end{align}
\end{claim}
\begin{proof}
A modification of \cite[Lemma 4.3]{simchowitz2020naive} use all $T$ steps (rather than $T/2$, and using a sum over the terms $\eta_t$ in that proof rather than a bound by the maximum) shows that 
\begin{align}
\mathrm{Reg}_e := \Exp_{\theta_e,\pi_{\mathrm{lr}}}[\mathrm{Regret}_T] \ge \frac{1}{2}\Kerr_{e}[K_e]  - \gamma_{\mathrm{err}}, \quad \text{where } \gamma_{\mathrm{err}} = \dimx \cdot \mathrm{poly}(\|\Pst\|_{\op},\|\Bst\|_{\op}), \label{eq:Reg_e}
\end{align} 
 From Lemma 4.7 in \cite{simchowitz2017simulator}, we also have
\begin{align*}
\Expop_{\mate}\Kerr_{\mate}[K_e] \le 2\Expop_{\mate}\Kerr_{\mate}[K_{\star}] + 4nmT\|\Pst\|_{\op}^4 \epsilon^2. 
\end{align*}
Combining the two displays gives the claim. 
\end{proof}

Next, since the instances $\thetast$ only differ along directions $(x,u) \in \R^{\dimx + \dimu}$ perpendicular to the hyperplane $u = \Kst x$, samples collected perpendicular to this hyperplane essential for disambiguating between the instances $\theta_e$. This leads to the following lower bound.
\begin{claim}\label{claim:info_lb} Set  $R = \frac{n}{48\epsilon^2}$. Then either
\begin{align}
\Expop_{\mate}[\Kerr_{\mate}[\Kst]] \ge 12R, \quad \text{ or } \quad \Expop_{\mate}\Exp_{\theta_{\mate},\pilr}[\calR_{\lqr,\theta_\mate}(K_{\mathrm{lr}})] \ge \frac{\sigma_{m}(\Ast + \Bst \Kst)^2 n^2 m }{480 R}
\end{align}
\end{claim}
\begin{proof}
Modifying Lemma 4.5 in \cite{simchowitz2020naive} shows that for any binary estimator $\hat{e}$,  either $\Kerr_{e}$ is small on $\mate$ drawn from the hypercube, or else $\hat{e}$ has large hamming error.
\begin{align*}
\text{either } \Expop_{\mate}[\Kerr_{\mate}[\Kst]] \ge \frac{n}{4\epsilon^2}, \text{ or } \Expop_{\mate}\Exp_{\theta_{\mate},\pilr}[\mathrm{d}_{\mathsf{ham}}(\hat{e},\mate)]  \ge \frac{nm}{4}.
\end{align*}
Combining with Lemma 4.6 in \cite{simchowitz2020naive}, it follows that 
\begin{align*}
\text{either } \Expop_{\mate}[\Kerr_{\mate}[\Kst]] \ge \frac{n}{4\epsilon^2}, \text{ or } \Expop_{\mate}\Exp_{\theta_{\mate},\pilr}[\|K_{\mathrm{lr}} - K_e\|_{\fro}^2]  \ge \frac{\sigma_{m}(\Ast + \Bst \Kst)^2 nm \epsilon^2}{10}.
\end{align*}
From \cite[Lemma 3]{mania2019certainty}, we can bound $\|K_{\mathrm{lr}} - K_e\|_{\fro}^2 \ge \frac{\calR_{\lqr,\theta_e}(K_{\mathrm{lr}})}{\sigma_{\min}(\Ru)}$. Thus, using $\Ru \succeq I $,
\begin{align*}
 \Expop_{\mate}[\Kerr_{\mate}[\Kst]] \ge \frac{n}{4\epsilon^2}, \quad \text{ or } \quad \Expop_{\mate }\Exp_{\theta_{\mate},\pilr}[\calR_{\lqr,\theta_\mate}(K_{\mathrm{lr}})] \ge \frac{\sigma_{m}(\Ast + \Bst \Kst)^2 nm \epsilon^2}{10}.
\end{align*}
Reparameterizing $R =\frac{n}{32\epsilon^2}$ gives either $\Expop_{\mate \sim \{-1,1\}}[\Kerr_{\mate}[\Kst]] \ge 12R$, or else $\Expop_{\mate \sim \{-1,1\}}\Exp_{\theta_{\mate},\pilr}[\calR_{\lqr,\theta_\mate}(K_{\mathrm{lr}})] \ge \frac{\sigma_{m}(\Ast + \Bst \Kst)^2 n^2 m }{480 R}$, as needed.
\end{proof}
Combining \Cref{claim:info_lb,claim:reg_lb} and taking $R = \frac{n}{32\epsilon^2}$ gives 
\begin{align}
\Expop_{\mate}\Exp_{\theta_\mate, \pilr}[\mathrm{Reg}_T] \ge 3R - \underbrace{nmT\|\Pst\|_{\op}^4 \epsilon^2}_{ = \frac{mT \|\Pst\|_{\op}^4}{32 R}} - \gamma_{\mathrm{err}}, \quad \text{ or } \quad \Expop_{\mate}\Exp_{\theta_{\mate},\pilr}[\calR_{\lqr,\theta_\mate}(K_{\mathrm{lr}})] \ge \frac{\sigma_{m}(\Ast + \Bst \Kst)^2 n^2 m }{320 R}
\end{align}
In particular, if we take $\epsilon$ so that $R \ge \max\{\gamma_{\mathrm{err}}, \frac{mT \|\Pst\|_{\op}^4}{32 R}\}$, then either 
\begin{align}
\Expop_{\mate}\Exp_{\theta_\mate, \pilr}[\mathrm{Reg}_T] \ge R, \quad \text{ or } \quad \Expop_{\mate}\Exp_{\theta_{\mate},\pilr}[\calR_{\lqr,\theta_\mate}(K_{\mathrm{lr}})] \ge \frac{\sigma_{m}(\Ast + \Bst \Kst)^2 n^2 m }{320 R}
\end{align}
Let us conclude by verifying the requisite ranges for conditions on $\epsilon$ and regret bound $R$ for the above to hold. We require that $R \ge \gamma_{\mathrm{err}} = d \mathrm{poly}(\|\Pst\|_{\op},\|\Bst\|_{\op})$. We also require $R^2 \ge \frac{mT \|\Pst\|_{\op}^4}{32}$, so that $R \ge \|\Pst\|_{\op}^2\sqrt{mT}/4$. Finally, we require $\epsilon^2 = \frac{n}{32 R} \le \frac{1}{nm \cdot \mathrm{poly}(\|\Pst\|_{\op})}$, so $R \ge n^2m \cdot \mathrm{poly}(\|\Pst\|_{\op})$ for a a possibily modified polynomial function. Concluding, and using $n = \dimu$ and $m \le \dimx$, it is enough to select
\begin{align*}
R \ge \mathrm{poly}(\|\Pst\|_{\op},\|\Bst\|_{\op})\dimu^2 \dimx + \|\Pst\|_{\op}^2\sqrt{\dimx T}/4. 
\end{align*}
Finally, we note that all the instances $\theta_e$ have
\begin{align*}
\|\theta_e - \theta_{\star}\|_{\fro}^2 = \|\Kst \Delta_e\|_{\fro}^2 + \|\Delta_e\|_{\fro}^2 \le (1+\|\Kst\|_{\op}^2) nm \epsilon^2 = \frac{n^2 m(1+\|\Kst\|_{\op}^2) }{32 R}.
\end{align*}
Concluding, we note that for $\Ru \succeq I$, one can bound $\|\Kst\|_{\op}^2 \le \|\Pst\|_{\op}$ (this follows since $\Pst \succeq \Rx + \Kst^\top \Ru \Kst$ by a standard computation). Hence, taking $n = \dimu$ and $m \le \dimx$ all instances lie in the ball $\calB = \{\theta:\|\thetast - \theta\|_{\fro}^2 \le \frac{\dimu^2 \dimx}{16 \|\Pst\|_{\op} R} \}$. The bound follows.

\end{proof}

%% file: body/synthesis_rates_arxiv_supp.tex
%!TEX root = ../main_arxiv.tex

\section{Optimal Rates for Martingale Regression in General Norms}
In this section we complete the proofs of \Cref{thm:gauss_assouad_M} and \Cref{thm:Mnorm_est_bound}.

\subsection{Proof of $M$-norm Regression Lower Bound Lemmas (Theorem \ref{thm:gauss_assouad_M}) \label{sec:thm_gauss_assouad_proof_lemmas}}

\begin{proof}[Proof of \Cref{lem:replica}]
By Fubini's theorem and Bayes' rule,
\begin{align*}
\Exp_{X \sim \Pr_{X}}\Exp_{Y \sim \Pr_{Y \mid X} (\cdot \mid X)}[f(X,Y)] & = \int \int f(x,y) \Pr_{Y \mid X} (y \mid x)  \Pr_X(x) dy dx \\
& = \int \int \int  f(x,y) \Pr_{X' \mid Y}(x' \mid y) \Pr_{Y \mid X} (y \mid x)  \Pr_X(x) dx' dy dx\\
& =  \int \int \int  f(x,y) \frac{\Pr_{Y|X'}(y \mid x') \Pr_{X'}(x')}{\Pr_{Y}(y)} \Pr_{Y \mid X} (y \mid x)  \Pr_X(x) dx' dy dx\\
& =  \int \int \int  f(x,y) \Pr_{X|Y}(x|y) \Pr_{Y \mid X} (y \mid x)  \Pr_{X'}(x') dx' dy dx\\
& =  \int \int \int  f(x,y) \Pr_{X|Y}(x|y) \Pr_{Y \mid X} (y \mid x)  \Pr_{X'}(x') dx dy dx' \\
& = \Exp_{X' \sim \Pr_{X'}}\Exp_{Y \sim \Pr_{Y \mid X}(\cdot \mid X')}\Exp_{X \sim \Pr_{X \mid Y}( \cdot \mid Y)}[f(X,Y)]
\end{align*}
Relabeling gives the result. 
\end{proof}

\begin{proof}[Proof of \Cref{lem:cond_exp_trunc_M} ]
	\begin{align*}
	 \Exp\left[\|Z - \Exp[Z \mid \calE]\|_M^2 \mid \calE\right] & \ge \Exp\left[\I\{\calE\} \cdot \|Z - \Exp[Z \mid \calE]\|_M^2 \right] \\
	&= \underbrace{\Exp\left\|Z - \Exp[Z\mid \calE]\right\|_M^2}_{(i)} - \underbrace{\Exp\left[\I\{\calE^c\} \cdot \|Z - \Exp[Z \mid \calE]\|_M^2 \right]}_{(ii)}
	\end{align*}
	Next, we lower bound $(i) = \Exp\left\|Z - \Exp[Z\mid \calE]\right\|_M^2  \ge \Exp \| Z - \Exp[Z] \|_M^2$. Thus, it remains to upper bound $(ii)$:
	\begin{align*}
	(ii) = \Exp\left[\I\{\calE^c\} \cdot \|Z - \Exp[Z \mid \calE]\|_M^2 \right] &\le 2\Exp\left[\I\{\calE^c\} \cdot \left(\|Z - \mu \|_M^2 + \|\mu - \Exp[Z \mid \calE] \|_M^2\right) \right]\\
	&\le 2\Exp\left[\I\{\calE^c\} \cdot \left(\|Z - \mu \|_M^2 + r^2 \| M \|_{\op} \right) \right]\\
	&\le 2\Exp\left[\I\{\calE^c\} \cdot \left(\|Z - \mu \|_2^2 \| M \|_{\op} + r^2 \| M \|_{\op} \right) \right],
	\end{align*}
	where in the second line, we use that, under $\calE$, $\|Z - \mu \|_2 \le r$. Moreover, under $\calE^c$, $\|Z - \mu \|_2^2 \ge r$, so that $\|Z - \mu \|_2^2 + r^2 \le 2\|Z - \mu\|_2^2$. Hence, $(ii) \le 4 \| M \|_{\op} \Exp[\I\{\calE^c\}\|Z - \mu \|^2_2]$. Thus, 
	\begin{align*}
	\Exp[\| Z - \Exp[ Z \mid \calE] \|_M^2 \mid \calE] 
	\ge \Exp\left\|Z - \Exp[Z]\right\|_M^2 - 4 \| M \|_{\op} \Exp[\I\{\calE^c\}\|Z - \mu\|^2_2],\end{align*} as needed.
\end{proof}

\begin{proof}[Proof of \Cref{lem:Gaussian_expectation_M}]

	Due to the fact that we have gaussian likelihoods, we we
	\begin{align*}
	\rmd\Pr(\frakD_T \mid \theta) &\propto \exp( -\frac{1}{2 \sigma_w^2}\sum_{t=1}^T ( y_t - \langle \theta, z_t \rangle)^2   ) \propto \exp( -\frac{1}{2 \sigma_w^2}\theta^\top \matSig_T \theta + \frac{1}{\sigma_w^2} \theta^\top \sum_{t=1}^T z_t y_t^\top   )
	\end{align*}
	On the other hand, for any given $\theta$,  $\rmd\Pr(\theta) \propto \exp( - \frac{1}{2} \theta^\top \Lambda^{-1} \theta)$. Hence, 
	\begin{align*}
	\rmd\Pr(\theta \mid \frakD_T) \propto \exp( -\frac{1}{2 \sigma_w^2}\theta^\top (\matSig_T + \Lambda^{-1}) \theta + \frac{1}{\sigma_w^2} \theta^\top \sum_{t=1}^T z_t y_t^\top   )
	\end{align*}
	Thus, $\theta \mid \frakD_T$ is conditionally Gaussian with covariance $\sigma_w^2 (\matSig_T + \Lambda^{-1})^{-1}$. It follows that:
	\begin{align*}
	& \Exp_{\bmtheta' \sim \Dfull(\frakD_T)} \left [  \|\bmtheta' - \Exp_{\bmtheta'' \sim \Dfull(\frakD_T)} [\bmtheta'' ] \|_{M}^2 \right ]  \\
	& = \tr \left ( M^{1/2} \Exp_{\bmtheta' \sim \Dfull(\frakD_T)} \left [ (\bmtheta' - \Exp_{\bmtheta'' \sim \Dfull(\frakD_T)} [\bmtheta'' ]) (\bmtheta' - \Exp_{\bmtheta'' \sim \Dfull(\frakD_T)} [\bmtheta'' ])^\top \right ] M^{1/2} \right ) \\
	& = \sigma_w^2 \tr \left ( M^{1/2} (\matSig_T + \Lambda^{-1})^{-1} M^{1/2} \right ) 
	\end{align*}
\end{proof}

	\begin{proof}[Proof of \Cref{lem:hansonwright_consequence} ]
	From Proposition 1.1. in \cite{hsu2012tail}, we have $t > 0$, it holds that $\Pr[\matg^\top \Lambda \matg > \tr(\Lambda) + 2\sqrt{t}\|\Lambda\|_{\fro} + 2\|\Lambda\|_{\op}t] \le e^{-t}$. In particular, if $\sqrt{t}\|\Lambda\|_{\op} \ge \|\Lambda\|_{\fro}$, then $\Pr[\matg^\top \Lambda \matg > \tr(\Lambda) + 4 t\|\Lambda\|_{\op}] \le e^{-t}$. Reparametrizing $u = 4t\|\Lambda\|_{\op}$, we have that if $\sqrt{u \|\Lambda\|_{\op}}/2 \ge \|\Lambda\|_{\fro}$, then$ \Pr[\matg^\top \Lambda \matg > \tr(\Lambda) + u] \le e^{-u/4\|\Lambda\|_{\op}}$. Lastly, the condition $\sqrt{u \|\Lambda\|_{\op}}/2 \le \|\Lambda\|_{\fro}$ is equivalent to $u \ge 4\|\Lambda\|_{\fro}^2/\|\Lambda\|_{\op}$. Since $\|\Lambda\|_{\fro}^2 \le \tr(\Lambda)\|\Lambda\|_{\op}$, it suffices that $u \ge 4\tr(\Lambda)$. This concludes the proof.
\end{proof}

\subsection{Proof of $M$-norm Regression Upper Bound (Theorem \ref{thm:Mnorm_est_bound}) \label{sec:proof_thm_Mnorm_est}}
\begin{proof}
Let $\beta$ be a parameter to be tuned, and set. Since $M$ may not be full rank, we consider a perturbation
\begin{align*}
N := M + \zeta I, \quad \zeta =  \tr(M \Gamma^{-1})/\tr(\Gamma^{-1}), 
\end{align*}
where $\zeta > 0$ is to be chosen. We further define
\begin{itemize}
	\item $\matZ \in \R^{T \times d}$ denote the matrix whose rows are $z_t^\top$
	\item $\matw \in \R^T$ as the vector whose entries are $w_t$. 
	\item $\thetals$ the least squares estimate of $\thetast$ defined in \Cref{eq:thetals} 
	\item We let $v_1,\ldots,v_d$ be the eigenvectors of $N^{1/2} \Gamma^{-1} N^{1/2}$ and $\lambda_j := \lambda_j(N^{1/2} \Gamma^{-1} N^{1/2})$, which we note are deterministic. 
\end{itemize}
The error of the least-squares estimate is then,
\begin{align*}
 \| \thetals - \thetast \|_M^2 = \| (\matZ^\top \matZ)^{-1} \matZ^\top \matw \|_M^2 
 \end{align*}

Since $\epsilon < \lambda_{\min}(\Gamma)/4$, we can apply \Cref{prop:mat_inverse_bound} to get
\begin{align*} \| (\matZ^\top \matZ)^{-1} - \Gamma^{-1} \|_\op \le \frac{\epsilon}{\lambda_{\min}(\Gamma) ( \lambda_{\min}(\Gamma) -  \epsilon)} < \frac{2 \epsilon}{ \lambda_{\min}(\Gamma)^2} 
\end{align*}
We now invoke the following lemma, controlling the relation of (weighted) squares of matrices in the PSD order:

\begin{lem}\label{prop:matrix_square_order}
Let $A, B, M \succeq 0$ and $C = A - B$. Then,
$$ A M A + 7C M C \succeq B M B/2 $$
\end{lem}
The lemma is proven at the end of this section.  Instantiating \Cref{prop:matrix_square_order} with $A = \Gamma^{-1}$ and $B = (\matZ^\top \matZ)^{-1}$ and $M = N$, we have
\begin{align*}
\Gamma^{-1} N \Gamma^{-1}  +  \frac{13 \| N \|_\op \epsilon^2}{ \lambda_{\min}(\Gamma)^4 } I  \succeq (\matZ^\top \matZ)^{-1} N (\matZ^\top \matZ)^{-1}
\end{align*}

Suppose that $\epsilon$ is chosen sufficiently small that, for a constant $\alpha$ to be specified
\begin{align}
\alpha\Gamma^{-1} N \Gamma^{-1}  &\succeq  \frac{13 \| N \|_\op \epsilon^2}{ \lambda_{\min}(\Gamma)^4 }; \label{eq:alpha_Line}
\end{align}
we shall revisit this point at the end of the proof. Then, 
\begin{align}
\| (\matZ^\top \matZ)^{-1} \matZ^\top \matw \|_{M}^2 &\le \| (\matZ^\top \matZ)^{-1} \matZ^\top \matw \|_{N}^2 \\
& = \matw^\top \matZ (\matZ^\top \matZ)^{-1} N (\matZ^\top \matZ)^{-1} \matZ^\top \matw \nonumber \\
& \le (1+\alpha) \matw^\top \matZ \Gamma^{-1} N \Gamma^{-1} \matZ^\top \matw \nonumber \\
& = (1+\alpha)\| N^{1/2} \Gamma^{-1} \matZ^\top \matw \|_2^2 \nonumber \\
& = (1+\alpha)\sum_{j=1}^{\dimtheta} (v_j^\top N^{1/2} \Gamma^{-1} N^{1/2} N^{-1/2} \matZ^\top \matw)^2 \nonumber \\
& = (1+\alpha)\sum_{j=1}^{\dimtheta} \lambda_j^2 (v_j^\top N^{-1/2} \matZ^\top \matw)^2  \nonumber\\
& = (1+\alpha)\sum_{j=1}^{\dimtheta} \lambda_j^2 (v_j^\top \matZ_j^\top \matw)^2 , \quad \text{ where } \matZ_j = \matZ N^{-1/2} v_j \in \R^T
\label{eq:ls_error_bound_one}.
\end{align}

 We specialize the self-normalized martingale concentration inequality:
\begin{lem}[Theorem 1 of \cite{abbasi2011improved}]\label{lem:self_norm_general}
Let $\{ \mate_t \}_{t \ge 1} \in \R^\N$ be a scalar, $\calF_t$-adapted sequence such that $\mate_t | \calF_{t-1}$ is $\sigmaw^2$ sub-Gaussian. Let $\{ \matx_t \}_{t \ge 1} \in (\R^{\dimtheta})^\N$ be a sequence of $\calF_t$-adapted vectors. Fix a matrix $V_0 \succeq 0$. Then, with probability $1-\delta$,
\begin{align*} \left \| \sum_{t=1}^T \matx_t \mate_t \right \|_{(V_0 + \sum_{t=1}^T \matx_t \matx_t^\top)^{-1}} \le 2 \sigmaw^2 \log \left ( \frac{1}{\delta} \det \left (V_0^{-1/2} \left (V_0 + \sum_{t=1}^T \matx_t \matx_t^\top \right )V_0^{-1/2} \right ) \right ).
\end{align*}
In particular, if $d = 1$, then selecting a scalar $V_0 = \tau$,
\begin{align*} \left \| \sum_{t=1}^T \matx_t \mate_t \right \| \le 2 \sigmaw^2 (\tau + \sum_{t=1}^T \|\matx_t\|^2) \log \left(\frac{\tau + \sum_{t=1}^T \|\matx_t\|^2}{\delta \tau}  \right ).
\end{align*}
\end{lem}
Applying \Cref{lem:self_norm_general} with a union bound over indices $j \in [d]$, it holds with probability $1 - \delta$ for all $j \in[d]$ simultaenously for any fixed $\tau > 0$
\begin{align}
(v_j^\top \matZ_j^\top \matw)^2 \le 2\sigma_w^2(\| \matZ_j \|_2^2 + \tau \lambda_j) \log  \frac{d(\| \matZ_j \|_2^2 + \tau \lambda_j)}{\tau \lambda_j \delta}
\end{align}
In addition, note that for any $\tau \ge \beta$, 
\begin{align*}
\|\matZ_j\|_2^2 &= v_j^\top N^{-1/2} \matZ^\top \matZ N^{-1/2} v_j \le v_j^\top N^{-1/2} (\Gamma + \epsilon I) N^{-1/2} v_j \\
&\le (1+\beta)v_j^\top N^{-1/2}\Gamma N^{-1/2} v_j = (1+\beta)\lambda_j^{-1} \le (1+\tau)\lambda_j^{-1}
\end{align*}
 Hence, with probability $1 - \delta$  the following holds for all $j \in [d]$ simultaenously
\begin{align*}
(v_j^\top \matz_j^\top \matw)^2 \le \frac{2(1+2\tau)\sigma_w^2}{\lambda_j} \log \frac{d(\tau^{-1} + 2)}{\delta}
\end{align*}
Hence, combining with \Cref{eq:ls_error_bound_one}, we have that with probability
\begin{align*}
\| (\matZ^\top \matZ)^{-1} \matZ^\top \matw \|_{M}^2 
&\le  2(1+\alpha)(1+\tau)\sigma_w^2\sum_{j=1}^{\dimtheta} \lambda_j  \log \frac{d(\tau^{-1} + 2)}{\delta}\\
&\le  2(1+\alpha)(1+\tau)\sigma_w^2\log \frac{d(\tau^{-1} + 2)}{\delta} \cdot \tr(N^{1/2} \Gamma^{-1} N^{1/2}).
\end{align*}
Finally, we can simplify $\tr(N^{1/2} \Gamma^{-1} N^{1/2}) = \tr(N \Gamma^{-1}) = \tr((M + \zeta I) \Gamma^{-1}) \le 2 \tr(M \Gamma^{-1})$ for our choice of $\zeta =  \tr(M \Gamma^{-1})/\tr(\Gamma^{-1})$; thus, choosing $\tau \ge 1/4 \ge \beta$ (recall the  assumption,$\beta \le 1/4$), we have
\begin{align}
\| (\matZ^\top \matZ)^{-1} \matZ^\top \matw \|_{M}^2 \le   5\sigma_w^2(1+\alpha)\log \frac{6 d}{\delta} \cdot \tr(M \Gamma^{-1}), \text{ w.p. } 1 - \delta.
\end{align}
To conclude, let us compute find a suitable constant $\alpha$ satisfying \Cref{eq:alpha_Line}. Recall that we wanat $
\alpha\Gamma^{-1} N \Gamma^{-1}  \succeq  \frac{13 \| N \|_\op \epsilon^2}{ \lambda_{\min}(\Gamma)^4 }$. Since $N \succeq \zeta$, and $\epsilon \le \beta \lambda_{\min}(\Gamma)$, we want 
\begin{align*} 
\alpha \zeta \ge 13 \beta^2 \|N\|_{\op} = 13 \beta^2 (\|M\|_{\op}+ \zeta)
\end{align*}
Recalling $\zeta = \tr(M \Gamma^{-1})/\tr(\Gamma^{-1}) \le \|M\|_{\op}$, we can chose $\alpha \ge \frac{26 \beta^2 \|M\|_{\op} \tr(\Gamma^{-1})}{ \tr(M \Gamma^{-1})}$. In particular, since $\tr(M \Gamma^{-1}) \ge \|M\|_{\op}\lambda_{\min}(\Gamma^{-1})$, we can take
\begin{align*}
\alpha = \frac{26 \beta^2 \|M\|_{\op} \tr(\Gamma^{-1})}{\|M\|_{\op}  \lambda_{\min}(\Gamma)^{-1}} = 26 \beta^2 \lambda_{\max}(\Gamma) \tr(\Gamma^{-1})
\end{align*}
\end{proof}

\begin{proof}[Proof of \Cref{prop:matrix_square_order}]
Clearly, $A M A = BMB + CMC + BMC + CMB$. Since $C M C, BMB \succeq 0$:
\begin{align*}
BMC + CMB & \succeq BMC + CMB - 4 CMC - BMB / 4 \\
& = -(B/2 - 2C)M(B/2-2C) \\
& \succeq -2(BMB/4 + 4CMC) \\
& = -BMB/2 - 8 CMC
\end{align*}
Thus, 
$$ AMA \succeq BMB + CMC - BMB/2 - 8CMC = BMB/2 - 7 CMC. $$
\end{proof}

%% file: body/general_decisions_arxiv.tex
%!TEX root = ../main_arxiv.tex

\section{Lower Bounds on Martingale Decision Making}\label{sec:general_decision}

Next, we provide the formal proofs of the results stated in \Cref{sec:mdm_lb_body_arxiv}, as well as the formal proof of \Cref{cor:simple_regret_lb2_nice}.

\subsection{Proof of General Decision Making Lower Bounds}
For the remainder of Section \ref{sec:general_decision}, unless otherwise stated we assume the expectation is taken with respect to $\theta$ and some fixed exploration policy $\piexp$. Hence, we write $\Exp[\cdot]$ in place of $ \Exp_{\theta,\piexp}[\cdot]$.

\subsubsection{Proof of Theorem \ref{thm:simple_regret_lb2} \label{sec:proof_thm_simpl_reg_lb2}}
For simplicity, we shall write $r^2 = 5 \dimtheta /(\lambda T)^{5/6}$. The result follows by instantiating \Cref{thm:gauss_assouad_M} to lower bound:
\begin{align*} \min_{\thetahat } \max_{\theta : \| \theta - \thetast \|_{2} \leq r} \Exp_{\theta,\piexp} [\| \thetahat - \theta \|_{M(\thetast) }^2],
\end{align*}
via the simplification provided by \Cref{thm:simple_regret_lb}. Apply \Cref{thm:gauss_assouad_M} with parameters $\Gamma = \lambda T \cdot I$, so that $\tr(\Gamma^{-1}) = \frac{\dimtheta}{\lambda T}$ and $\lambda_{\min}(\Gamma) = \lambda T$. Then, the remainder term $\Psi$ from that theorem is bounded by
\begin{align*}
\Psi(r;\Gamma,\taskhes(\thetast)) \leq \frac{32 \| \taskhes(\thetast) \|_\op}{\lambda T} \exp \left ( - \frac{1}{5} r^2 \lambda T \right )
\end{align*}
Selecting $r^2 = 5 \dimtheta /(\lambda T)^{5/6} $ yields
\begin{align*}
 \Psi(r;\Gamma,\taskhes(\thetast)) \le  \frac{32 \| \taskhes(\thetast) \|_\op}{\lambda T} \exp \left ( - \dimtheta (\lambda T)^{1/6}  \right ) 
 \end{align*}
Noting that $r^2 \geq 5\tr(\Gamma)$, Theorem \ref{thm:gauss_assouad_M} yields that
\begin{align*}
& \min_{\thetahat} \max_{\theta : \| \theta - \thetast \|_{2} \leq r}  \Exp_{\theta,\piexp} [\| \thetahat - \theta \|_{\taskhes(\thetast) }^2] \\
& \geq  \sigma_w^2 \min_{\theta : \| \theta - \thetast \|_{2} \leq r} \tr\left(\taskhes(\thetast) \left( \Exp_{\theta,\piexp}[\matSig_T] + \lambda T \cdot I \right)^{-1}  \right) - \frac{32 \| \taskhes(\thetast) \|_\op}{\lambda T} \exp \left ( - \dimtheta (\lambda T)^{1/6}  \right ) \\
& \geq  \sigma_w^2 \min_{\theta : \| \theta - \thetast \|_{2} \leq r} \tr\left(\taskhes(\thetast) \left( \Exp_{\theta,\piexp}[\matSig_T] + \lambda T \cdot I \right)^{-1}  \right) - \frac{32  L_{\fraka1}^2 L_{\calR2}} {\lambda T} \exp \left ( - \dimtheta (\lambda T)^{1/6}  \right ),
\end{align*}
where in the last line, we invoked \Cref{asm:smoothness} to obtain
\begin{align*}
\| \taskhes(\thetast) \|_\op = \| \Gfraka(\thetast)^\top \Mfraka(\thetast) \Gfraka(\thetast) \|_\op \le L_{\fraka1}^2 L_{\calR2}
\end{align*}

Now, observe that our condition on $\lambda T$, namely $\lambda T \ge \left ( 80 \dimtheta/\betataylor(\thetast)^2 \right )^{6/5}$, implies that $r = (5 \dimtheta /(\lambda T)^{5/6})^{1/2} \le \frac{1}{4}\betataylor(\thetast)$. Hence, we can apply \Cref{thm:simple_regret_lb} to obtain
\begin{align*}
& \min_{\frakahat} \max_{\theta : \| \theta - \thetast \|_{2} \leq r} \Exp [\calR(\frakahat; \theta)]   \geq \min \Bigg \{ 5 \mu L_{\fraka1}^2 r^2  ,   \sigma_w^2 \min_{\theta\| \theta - \thetast \|_{2} \leq r} \tr\left(\taskhes(\thetast) \left( \Exp_{\theta,\piexp}[\matSig_T]+ \lambda T  I \right)^{-1}  \right) \\
& \qquad \qquad \qquad  -  C_1 r^3 - \frac{32  L_{\fraka1}^2 L_{\calR2}} {\lambda T} \exp \left ( - \dimtheta (\lambda T)^{1/6}  \right )  \Bigg \}
\end{align*}
for $C_1$ as in Lemma \ref{thm:simple_regret_lb} To conclude, we consolidate
\begin{align*}
C_1 r^3 + \frac{32  L_{\fraka1}^2 L_{\calR2}} {\lambda T} \exp \left ( - \dimtheta (\lambda T)^{1/6}  \right ) = \frac{C_1 (5 \dimtheta)^{3/2} + 32  L_{\fraka1}^2 L_{\calR2} (\lambda T)^{1/4} \exp \left ( - \dimtheta (\lambda T)^{1/6}  \right )  }{(\lambda T)^{5/4}}
\end{align*}
Observing that $(\lambda T)^{1/4} \exp \left ( - \dimtheta (\lambda T)^{1/6}\right)$ is bounded above by a universal constant (since $\dimtheta \ge 1$, and for all $d \ge 1$, $\max_{x \ge 0} x e^{-dx} \le \max_{x \ge 0} x e^{-x}$ is bounded),  the above is at most
\begin{align*}
\calO\left(\frac{C_1 \dimtheta^{3/2} + L_{\fraka1}^2 L_{\calR2}}{(\lambda T)^{5/4}}\right) = \frac{C_2}{(\lambda T)^{5/4}},
\end{align*}
for $C_2$ as in the statement of the lemma. To conclude, it suffices to show that for our choice of $r$, we have 
\begin{align*}
5 \mu L_{\fraka1}^2 r^2  \ge  \sigma_w^2 \min_{\theta\| \theta - \thetast \|_{2} \leq r} \tr\left(\taskhes(\thetast) \left( \Exp_{\theta,\piexp}[\matSig_T]+ \lambda T  I \right)^{-1}\right).
\end{align*}
Lower bounding $\Exp_{\theta,\piexp}[\matSig_T]+ \lambda T  I  \succeq \lambda T I$, and upper bounding $\tr(\taskhes(\thetast)) \le \dimtheta  \| \taskhes(\thetast) \|_\op \le \dimtheta L_{\fraka 1}^2 L_{\calR 2}$, and substituting in the choice of $r^2$, it is enough that 
\begin{align*}
 \mu L_{\fraka1}^2 \cdot (5 \dimtheta /(\lambda T)^{5/6}) \ge \sigma_w^2 \frac{\dimtheta L_{\fraka 1}^2 L_{\calR 2}}{\lambda T}
\end{align*}
Rearranging requires that
\begin{align*}
(\lambda T)^{1/6} \ge \frac{\sigma_w^2 L_{\calR 2}}{5\mu},
\end{align*}
which is satisfied for our choice of $\lambda$.

\subsubsection{Proof of Lemma \ref{thm:simple_regret_lb}}\label{sec:lem_quad_approx_lb_pf}
Our strategy is to show that an action $\fraka$ with low excess risk can be used to produce an estimate of a parameter $\theta$ with low error in the task hessian norm $\|\cdot\|_{\taskhes(\theta)}^2$. Specifically, we define the perturbation term
\begin{align*}
\updelta_{\star}(\frakahat) &= \argmin_{\updelta} \| \Mfraka(\thetast)^{1/2} \left ( ( \frakahat -  \aopt(\thetast)) - \Gfraka(\thetast) \updelta \right ) \|_2 \\
&= (\Mfraka(\thetast)^{1/2} \Gfraka(\thetast))^\dagger \Mfraka(\thetast)^{1/2} ( \frakahat -  \aopt(\thetast))
\end{align*}
and define the induced estimate
\begin{align*}
\thetahat(\frakahat) &= \thetast + \updelta_{\star}(\frakahat)
\end{align*}

\paragraph{Ensuring $\frakahat$ close $\aopt(\thetast)$:} 
We first want to restrict the lower bound to being only over $\frakahat$ close $\aopt(\thetast)$. To this end, note that
\begin{align*}
 \min_{\frakahat} \max_{\theta : \| \theta - \thetast \|_{2 }^2 \leq r(T)} \Exp [\calR(\frakahat; \theta)] & = \min \Bigg \{  \min_{\frakahat: \| \frakahat - \aopt(\thetast) \|_2^2 \leq r_\fraka^2} \max_{\theta : \| \theta - \thetast \|_{2 }^2 \leq r^2} \Exp [\calR(\frakahat; \theta)],  \\
 & \qquad \qquad \min_{\frakahat: \| \frakahat - \aopt(\thetast) \|_2^2 > r_\fraka^2} \max_{\theta : \| \theta - \thetast \|_{2 }^2 \leq r^2} \Exp [\calR(\frakahat; \theta)] \Bigg \} 
\end{align*}
where we are free to choose $r_\fraka$ as we wish but our choice will satisfy $r_\fraka \geq L_{\fraka1} r$. Now:
\begin{align*}
\| \frakahat - \aopt(\thetast) \|_2 & \leq \| \frakahat - \aopt(\theta) \|_2 + \| \aopt(\theta) - \aopt(\thetast) \|_2 \\
& \leq  \| \frakahat - \aopt(\theta) \|_2 + L_{\fraka1} \| \theta - \thetast \|_2 \\
& \leq \| \frakahat - \aopt(\theta) \|_2 + L_{\fraka1} r
\end{align*}
By the argument above, $\| \frakahat - \aopt(\thetast) \|_2^2 > r_\fraka(T)$ then implies that:
\begin{align*}
\calR(\frakahat; \theta) \geq \frac{\mu}{2}  \| \frakahat - \aopt(\theta) \|_2^2 \geq \frac{\mu}{2} \left ( \| \frakahat - \aopt(\thetast) \|_2 - L_{\fraka1} r \right )^2 \geq \frac{\mu}{2} (r_\fraka - L_{\fraka1} r)^2
\end{align*}
Choosing $r_\fraka = (1 + \sqrt{2}) L_{\fraka1} r$, 
\begin{align*}
\min_{\frakahat: \| \frakahat - \aopt(\thetast) \|_2^2 > r_\fraka^2} \max_{\theta : \| \theta - \thetast \|_{2 }^2 \leq r^2} \Exp [\calR(\frakahat; \theta)] \geq \mu L_{\fraka1}^2 r^2
\end{align*}
Thus, definining the constant
\begin{align*} C_\fraka := (1+\sqrt{2}) L_{\fraka1} ,
\end{align*}
So ultimately we have:
\begin{align}
 \min_{\frakahat} \max_{\theta : \| \theta - \thetast \|_{2 }^2 \leq r(T)} \Exp [\calR(\frakahat; \theta)] &\ge \min \Bigg \{  \min_{\frakahat: \| \frakahat - \aopt(\thetast) \|_2^2 \leq C_\fraka^2 r^2} \max_{\theta : \| \theta - \thetast \|_{2 }^2 \leq r^2} \Exp [\calR(\frakahat; \theta)],  \mu L_{\fraka1}^2 r^2 \Bigg \}  \label{eq:truncation_of_a}
 \end{align}
We now proceed to lower bound the first term in the above expression. In particular, throughout we assume that
\begin{align}
\| \frakahat - \aopt(\thetast) \|_2 \leq C_\fraka r, \quad C_\fraka := (1+\sqrt{2}) L_{\fraka1} , \label{eq:Ca_condition}
\end{align}

\paragraph{Taylor expansions:} Fix a $t \in [0,1]$, parameter $\theta$, estimated action $\frakahat$, and define the interpolations
\begin{align*}
\frakahat_t = t \frakahat + (1-t) \aopt(\theta), \quad \theta_t := t \theta + (1-t)\theta_t
\end{align*}
Throughout, we will let $\frakahat'$ and $\theta'$ denote certain values of $\frakahat_t$ and $\theta_t$ for some interpolation parameters $t',t'' \in [0,1]$ chosen so as to satisfy the application of Taylor's theorem to follow. 

First, by Taylor's theorem,
\begin{align*}
\calR(\frakahat;\theta) & = \calR(\aopt(\theta);\theta) + \frac{d}{dt} \calR(\frakahat_t;\theta)|_{t=0} + \frac{1}{2} \frac{d^2}{dt^2} \calR(\frakahat_t;\theta)|_{t=0} + \frac{1}{6} \frac{d^3}{dt^3} \calR(\frakahat_t;\theta)|_{t=t'} \\
& = \calR(\aopt(\theta); \theta) + (\nabla_\fraka \calR(\fraka;\theta)|_{\fraka = \aopt(\theta)})^\top ( \frakahat - \aopt(\theta))\\
& \qquad \qquad + \frac{1}{2} (\aopt(\theta) - \frakahat)^\top (\nabla_\fraka^2 \calR(\fraka;\theta)|_{\fraka = \aopt(\theta)}) (\aopt(\theta) - \frakahat) \\
& \qquad \qquad + \frac{1}{6} \nabla_{\fraka}^3 \calR(\fraka;\theta)|_{\fraka=\fraka'}[\frakahat - \aopt(\theta),\frakahat - \aopt(\theta),\frakahat - \aopt(\theta)]
\end{align*}
where $t' \in [0,1], \fraka' = \frakahat_{t'}$. The second equality follows by the chain rule and since $\frac{d}{dt} \frakahat_t = \frakahat - \aopt(\theta)$. Since $\aopt(\theta)$ minimizes the excess risk, we have
\begin{align*}
\calR(\aopt(\theta); \theta) = 0, \quad \nabla_\fraka \calR(\fraka;\theta)|_{\fraka = \aopt(\theta)} = 0
\end{align*}
Thus, we may simplify
\begin{align*}
\calR(\frakahat;\theta) & = \frac{1}{2} (\aopt(\theta) - \frakahat)^\top (\nabla_\fraka^2 \calR(\fraka;\theta)|_{\fraka = \aopt(\theta)}) (\aopt(\theta) - \frakahat) \\
& \qquad \qquad + \frac{1}{6} \nabla_{\fraka}^3 \calR(\fraka;\theta)|_{\fraka=\fraka'}[\frakahat - \aopt(\theta),\frakahat - \aopt(\theta),\frakahat - \aopt(\theta)]  
\end{align*}
We can similarly Taylor expand $\aopt$ to get:
\begin{align*}
\aopt(\theta) = \aopt(\thetast) + \Gfraka(\thetast)(\theta - \thetast) + \nabla_\theta^2 \aopt(\theta)|_{\theta = \theta'}[\theta - \thetast, \theta - \thetast],
\end{align*}
where again we set $\theta' = t'' \theta + (1-t'') \thetast$ for some $t'' \in [0,1]$. Recall the definitions
\begin{align*}
\updelta_{\star} &= \argmin_{\updelta} \| \Mfraka(\thetast)^{1/2} \left ( ( \frakahat -  \aopt(\thetast)) - \Gfraka(\thetast) \updelta \right ) \|_2 \\
&= (\Mfraka(\thetast)^{1/2} \Gfraka(\thetast))^\dagger \Mfraka(\thetast)^{1/2} ( \frakahat -  \aopt(\thetast))\\
\thetahat(\frakahat) &= \thetast + \updelta_{\star}
\end{align*}
Writing $\frakahat = \aopt(\thetast) + \Gfraka(\thetast) \updelta_{\star} + \updelta_{\frakahat}$ for some $\updelta_{\frakahat}$ and denoting $\updelta_\theta = \theta - \thetast$, we then have:
\begin{align*}
\calR(\frakahat; \theta) & = \frac{1}{2} (\updelta_\theta - \updelta_{\star})^\top \Gfraka(\thetast)^\top \Mfraka(\thetast) \Gfraka(\thetast) (\updelta_\theta - \updelta_{\star}) \\
& \qquad + \underbrace{ \frac{1}{2} \updelta_{\frakahat}^\top \Mfraka(\thetast) \updelta_{\frakahat}}_{(a1)}  - \underbrace{\updelta_{\frakahat}^\top \Mfraka(\thetast) ( \nabla_\theta^2 \aopt(\theta)|_{\theta = \theta'}[\updelta_\theta, \updelta_\theta])}_{(a2)} \\
& \qquad + \underbrace{\frac{1}{2} ( \nabla_\theta^2 \aopt(\theta)|_{\theta = \theta'}[\updelta_\theta, \updelta_\theta] )^\top \Mfraka(\thetast) ( \nabla_\theta^2 \aopt(\theta)|_{\theta = \theta'}[\updelta_\theta, \updelta_\theta])}_{(a3)} \\
& \qquad + \underbrace{(\updelta_\theta - \updelta_{\star})^\top \Gfraka(\thetast)^\top \Mfraka(\thetast) ( \nabla_\theta^2 \aopt(\theta)|_{\theta = \theta'}[\updelta_\theta, \updelta_\theta] )}_{(a4)}  - \underbrace{(\updelta_\theta - \updelta_{\star})^\top \Gfraka(\thetast)^\top \Mfraka(\thetast) \updelta_{\frakahat}}_{(a5)} \\
& \qquad +  \underbrace{\frac{1}{2} (\aopt(\theta) - \frakahat)^\top \left (\nabla_\fraka^2 \calR(\fraka;\theta)|_{\fraka = \aopt(\theta)} - \Mfraka(\thetast) \right ) (\aopt(\theta) - \frakahat) }_{(a6)} \\
& \qquad + \underbrace{\frac{1}{6} \nabla_{\fraka}^3 \calR(\fraka;\theta)|_{\fraka=\fraka'}[\frakahat - \aopt(\theta),\frakahat - \aopt(\theta),\frakahat - \aopt(\theta)]}_{(a7)}
\end{align*}

\paragraph{Controlling the Taylor Expansion through norm bounds:} We verify thatwe are in the regime where Assumption \ref{asm:smoothness} holds.
\begin{claim}\label{claim:good_regime} For $\fraka$ satisfying \Cref{eq:Ca_condition}, it holds that
\begin{align}
&\| \theta' - \thetast \|_2 \le \| \theta - \thetast \|_2  \le \betataylor(\thetast) \label{eq:lbpf_thetaball}\\
&\max\{\|\frakahat - \aopt(\thetast)\|, \|\fraka' - \aopt(\thetast)\|_2, \|\aopt(\theta) - \aopt(\thetast)\|\} \le L_{\fraka 1}\betataylor(\thetast)\label{eq:lbpf_aball}
\end{align}
\end{claim}
\begin{proof}[Proof of \Cref{claim:good_regime}]
By assumption, we have,
\begin{equation}\label{eq:lbpf_betathetacond}
r \le \frac{1}{4}\betataylor(\thetast)
\end{equation}
Now recall that 
$\theta' = t' \theta + (1-t') \thetast$, for some $t' \in [0,1]$, so
\begin{align*}
\| \theta' - \thetast \|_2 & \le \| \theta - \thetast \|_2 \le   r
\end{align*}
From this and trivial manipulations of $\| \theta - \thetast \|_\op, \| \theta_0 - \thetast \|_\op$, it follows that \eqref{eq:lbpf_betathetacond} implies \eqref{eq:lbpf_thetaball}. 

To verify \eqref{eq:lbpf_aball}, recall that $\fraka' =t'' \frakahat + (1 - t'') \aopt(\theta)$ for some $ t'' \in [0,1]$. Hence, 
\begin{align*}
&\max\{\|\frakahat - \aopt(\thetast)\|, \|\fraka' - \aopt(\thetast)\|_2, \|\aopt(\theta) - \aopt(\thetast)\|\} \\
&\quad\le \|\frakahat - \aopt(\thetast)\| + \|\aopt(\theta) - \aopt(\thetast)\| 
\end{align*}
From \Cref{eq:Ca_condition}, it holds that $\|\frakahat - \aopt(\thetast)\| \le (1+\sqrt{2}) r L_{\fraka1}$;  moreover, since $\|\theta - \thetast\| \le \betataylor(\thetast)$, the smoothness condition, \Cref{asm:smoothness}, implies that that  $\|\aopt(\thetast) - \aopt(\theta)\| \le L_{\fraka1}\|\thetast - \theta\| \le rL_{\fraka 1}$. 
Hence, 
\begin{align*}
&\max\{\|\frakahat - \aopt(\thetast)\|, \|\fraka' - \aopt(\thetast)\|_2, \|\aopt(\theta) - \aopt(\thetast)\|\} \\
&\quad\le (2+\sqrt{2}) L_{\fraka 1} r \le 4 L_{\fraka 1} r
\end{align*}
Thus, \eqref{eq:lbpf_betathetacond} implies \eqref{eq:lbpf_aball} holds.

\end{proof}
 The following bounds will be useful. 
\begin{itemize}
\item By assumption: $\| \frakahat - \aopt(\thetast) \|_2^2 \le C_\fraka r$, 
\begin{align*}
\| \updelta_\theta\|_2^2 = \| \theta - \thetast \|_2^2
\end{align*}
\item We have that 
\begin{align*}
 \| \Mfraka(\thetast)^{1/2} \Gfraka(\thetast) \updelta_{\star} \|_2 \le \| \Mfraka(\thetast)^{1/2}(\frakahat - \aopt(\thetast)) \|_2 \le \| \Mfraka(\thetast)^{1/2}\|_\op C_\fraka r
 \end{align*} 
 This follows since, recalling the definition of $\updelta_{\star}$ and letting $U \Sigma V^\top =  \Mfraka(\thetast)^{1/2} \Gfraka(\thetast)$, we have $\| \Mfraka(\thetast)^{1/2} \Gfraka(\thetast) \updelta_{\star} \|_2 = \| \Sigma \Sigma^\dagger U^\top \Mfraka(\thetast)^{1/2} (\frakahat - \aopt(\thetast)) \|_2$ and since $\| \Sigma \Sigma^\dagger \|_\op \le 1$. 
\item $\| \Mfraka(\thetast)^{1/2} \updelta_{\frakahat} \|_2 \le \| \Mfraka(\thetast)^{1/2} (\frakahat - \aopt(\thetast)) \|_2 + \| \Mfraka(\thetast)^{1/2} \Gfraka(\thetast) \updelta_{\star} \|_2 \le 2 \| \Mfraka(\thetast)^{1/2} (\frakahat - \aopt(\thetast)) \|_2$
\item By Assumption \ref{asm:smoothness}, so long as \eqref{eq:lbpf_thetaball} holds: $\| \Gfraka(\thetast)\|_\op \le L_{\fraka 1}, \| \nabla_\theta^2 \aopt(\theta)|_{\theta = \theta'} \|_\op \le L_{\fraka2} $.
\item By Assumption \ref{asm:smoothness}, so long as \eqref{eq:lbpf_aball} holds: $\| \nabla_\fraka^3 \calR(\fraka;\theta)|_{\fraka = \fraka'} \|_\op \le L_{\calR3} $.
\item $\| \Mfraka(\thetast)^{1/2} \|_\op = \sqrt{\| \Mfraka(\thetast) \|_\op} = \sqrt{\|  \nabla_{\fraka}^2 \calR(\fraka;\thetast)|_{\fraka = \aopt(\thetast)} \|_\op} \le \sqrt{ L_{\calR2}}$. To see why the first equality holds, note that for any PSD $M = U\Sigma U^\top$, $\| M^{1/2} \|_\op = \| \Sigma^{1/2} \|_\op = \max_{i} \sqrt{\sigma_i} = \sqrt{\max_i \sigma_i} = \sqrt{\| M \|_\op}$.
\end{itemize}

\paragraph{Lower bounding the excess risk $\calR(\frakahat; \theta)$:} Throughout the remainder of the proof, we let $c$ denote a universal numerical constant which may change from line to line. From the above observations
\begin{align*}
 (a2) = c L_{\fraka2} L_{\calR 2} C_\fraka r^{3} 
\end{align*}
By the bounds given above:
\begin{align*}
(a4) = c  L_{\fraka2} (L_{\fraka 1} + C_\fraka) L_{\calR2} r^{3} 
\end{align*}
To bound $(a6)$, we can apply Proposition \ref{prop:gd_lipschitz} to get that, when \eqref{eq:lbpf_aball} holds,
\begin{align*}
\| \nabla_\fraka^2 \calR(\fraka;\theta) |_{\fraka=\aopt(\theta)} - \nabla_\fraka^2 \calR(\fraka;\theta) |_{\fraka=\aopt(\thetast)} \|_\op \le L_{\calR 3} \| \aopt(\theta) - \aopt(\thetast) \|_2 \le L_{\calR 3} L_{\fraka1} r
\end{align*}
Using that $\| \aopt(\theta) - \frakahat \|_2 \le \| \aopt(\theta) - \aopt(\thetast) \|_2 + \| \aopt(\thetast) - \frakahat \|_2 \le (L_{\fraka1} + C_\fraka r)$, we have: 
\begin{align*}
(a6) = c (L_{\fraka1} + C_\fraka)^2 (L_{\calR 3} L_{\fraka1} + \Lra) r^{3} 
\end{align*}
This same bound on $\| \aopt(\theta) - \frakahat \|_2$ gives:
\begin{align*} 
(a7) = c  L_{\calR3} (L_{\fraka1} + C_\fraka)^3 r^{3/2}
\end{align*}
It remains to bound $(a5)$. Recall that 
\begin{align*} \updelta_{\star} = \argmin_{\updelta} \| \Mfraka(\thetast)^{1/2} (\frakahat - \aopt(\thetast) - \Gfraka(\thetast)\updelta) \|_2
\end{align*}
so $\updelta_{\star}$ is the projection of $\Mfraka(\thetast)^{1/2} (\frakahat - \aopt(\thetast))$ onto the image of $\Mfraka(\thetast)^{1/2} \Gfraka(\thetast)$. It follows that:
\begin{align*}
\Mfraka(\thetast)^{1/2} (\frakahat - \aopt(\thetast) - \Gfraka(\thetast)\updelta_{\star}) = \Mfraka(\thetast)^{1/2} \updelta_{\frakahat} \quad \bot \quad \text{image}(\Mfraka(\thetast)^{1/2} \Gfraka(\thetast)) 
\end{align*}
which implies
\begin{align*}
(a5) =  -(\updelta_\theta - \updelta_{\star})^\top \Gfraka(\thetast)^\top \Mfraka(\thetast) \updelta_{\frakahat} = 0
\end{align*}
Combining everything, we've shown that:
\begin{align*}
\calR(\frakahat; \theta) & \ge \frac{1}{2} (\updelta_\theta - \updelta_{\star})^\top \Gfraka(\thetast)^\top \Mfraka(\thetast) \Gfraka(\thetast) (\updelta_\theta - \updelta_{\star}) + (a1) + (a3)  - \calO \Big (  C_1 r^{3}  \Big ) 
\end{align*}
for
\begin{align*} C_1 = 2 L_{\fraka1} L_{\fraka2} L_{\calR2}  + 8 L_{\fraka1}^3 L_{\calR 3} + 4 \Lra
\end{align*}
However, $\Mfraka(\thetast)$ is PSD so $(a1),(a3) \ge 0$, giving:
\begin{align*}
\calR(\frakahat; \theta) & \ge \frac{1}{2} (\updelta_\theta - \updelta_{\star})^\top \Gfraka(\thetast)^\top \Mfraka(\thetast) \Gfraka(\thetast) (\updelta_\theta - \updelta_{\star}) - c  C_1 r^3 
\end{align*}

\paragraph{Completing the proof:} By definition, $\updelta_\theta - \updelta_{\star} = \theta - \thetahat(\frakahat)$ and $\Gfraka(\thetast)^\top \Mfraka(\thetast) \Gfraka(\thetast) = \taskhes(\thetast)$, so
$$(\updelta_\theta - \updelta_{\star})^\top \Gfraka(\thetast)^\top \Mfraka(\thetast) \Gfraka(\thetast) (\updelta_\theta - \updelta_{\star}) =  \| \theta - \thetahat(\frakahat) \|_{\taskhes(\thetast)}^2 $$
Putting things together, we then have that
\begin{align*}
& \min_{\frakahat: \| \frakahat - \aopt(\thetast) \|_2 \leq C_\fraka r} \max_{\theta : \| \theta - \thetast \|_{2 } \leq r} \Exp [\calR(\frakahat; \theta)] \\
& \qquad \qquad \qquad \ge \min_{\frakahat: \| \frakahat - \aopt(\thetast) \|_2 \leq C_\fraka r} \max_{\theta : \| \theta - \thetast \|_{2 } \leq r} \Exp \left [ \frac{1}{2} \| \theta - \thetahat(\frakahat) \|_{\taskhes(\thetast)}^2 \right ]  - c C_1 r^3  
\end{align*}

Given knowledge of $\thetast$, $\thetahat(\frakahat)$ is simply an estimator of $\theta$, so it follows that from \Cref{eq:truncation_of_a} that
\begin{align*}
& \min_{\frakahat: \| \frakahat - \aopt(\thetast) \|_2^2 \leq C_\fraka r(T)} \max_{\theta : \| \theta - \thetast \|_{2 }^2 \leq r(T)} \Exp \left [ \frac{1}{2} \| \theta - \thetahat(\frakahat) \|_{\taskhes(\thetast)}^2 \right ]   \ge \min_{\thetahat} \max_{\theta : \| \theta - \thetast \|_{2 }^2 \leq r(T)} \Exp \left [ \frac{1}{2} \| \theta - \thetahat \|_{\taskhes(\thetast)}^2 \right ]  
\end{align*}
This concludes the proof.
\qed

\subsection{Proof of Theorem \ref{cor:simple_regret_lb2_nice}}\label{sec:smpl_regret_lb2_nice_pf}
We apply Theorem \ref{thm:simple_regret_lb2} with $\lambda = \lamund$, which is greater than 0 by Assumption \ref{asm:suff_excite}. Then,
\begin{align*}
    & \tr\left(\taskhes(\thetast) \left( \Exp_{\theta,\piexp}[\matSig_T] + \lamund T  I \right)^{-1}  \right)  \ge \tr\left(\taskhes(\thetast) \left( \Exp_{\theta,\piexp}[\matSig_T] +  \Exp_{\thetast,\piexp}[\matSig_T] \right)^{-1}  \right)
\end{align*}  
Under Assumption \ref{asm:smooth_covariates}, for any $\theta$ satisfying $\| \theta - \thetast \|_2^2 \le 5 \dimtheta/(\lamund T)^{5/6}$ and as long as $5 \dimtheta/(\lamund T)^{5/6} \le \betaexp(\thetast)^2$, we have
$$  \Exp_{\theta,\piexp}[\matSig_T] \preceq  \ccexp \Exp_{\thetast,\piexp}[\matSig_T] + \bigg (  \frac{\alphast(\thetast,\gamma^2) \sqrt{5 \dimtheta} T^{7/12} }{\lamund^{5/12}} + \Ccov T^{1-\alpha} \bigg ) \cdot I $$
Therefore, since our alternate instances, $\theta \in \calB_T$, do satisfy $\| \theta - \thetast \|_2^2 \le 5 \dimtheta/(\lamund T)^{5/6}$, if $T$ is large enough that
$$\frac{\alphast(\thetast,\gamma^2) \sqrt{5 \dimtheta} T^{7/12} }{\lamund^{5/12}} + \Ccov T^{1-\alpha} \le \ccexp T \lamund $$
we will have, for all $\theta \in \calB_T$,
\begin{align*}
\Exp_{\theta,\piexp}[\matSig_T] & \preceq \ccexp \Exp_{\thetast,\piexp}[\matSig_T] + \bigg ( \frac{\alphast(\thetast,\gamma^2) \sqrt{5 \dimtheta} T^{7/12} }{\lamund^{5/12}} + \Ccov T^{1-\alpha} \bigg ) \cdot I \\
& \preceq \ccexp \Exp_{\thetast,\piexp}[\matSig_T] + \ccexp T \lamund \cdot I \preceq 2\ccexp \Exp_{\thetast,\piexp}[\matSig_T]
\end{align*}
The result then follows from Theorem \ref{thm:simple_regret_lb2} and simple manipulations.
\qed

\subsection{Proof of Proposition \ref{prop:gd_lipschitz} and Proposition \ref{quad:certainty_equivalence}}
\begin{proof}[Proof of Proposition \ref{prop:gd_lipschitz}]
We prove this for a generic function $f: \R^n \rightarrow \R^m$. Fix some $x,y \in \R^n$ and let $x_t = t x + (1-t) y$. Then, by Taylor's Theorem,
$$ f(x) = f(y) + \frac{d}{dt} f(x_t)|_{t = t'} $$
for some $t' \in [0,1]$. By the chain rule, $\frac{d}{dt} f(x_t) = \nabla_x f(x)|_{x = x_t} \cdot \frac{d}{dt} x_t = \nabla_x f(x)|_{x = x_t} \cdot (x - y)$. So:
$$ \| f(x) - f(y) \|_\op \le \| \nabla_x f(x)|_{x=x_{t'}} \|_\op \cdot \| x - y \|_\op $$
The result follows in our setting using the norm bounds given in Assumption \ref{asm:smoothness}.
\end{proof}

\begin{proof}[Proof of Proposition \ref{quad:certainty_equivalence}]
Let $\theta_t = t \thetahat + (1-t) \thetast$. Note that for any $t$, by Proposition \ref{prop:gd_lipschitz},
\begin{align*}
\| \aopt(\theta_t) - \aopt(\thetast) \|_2 \le L_{\fraka 1} \| \theta_t - \thetast \|_2 \le L_{\fraka 1} \| \thetahat - \thetast \|_2 \le L_{\fraka 1}  \betataylor(\thetast) %\le 1/ \betaa(\thetast)
\end{align*}
where the last inequality follows by Assumption \ref{asm:smoothness}. We are therefore in the regime where the norm bounds given in Assumption \ref{asm:smoothness} hold, which we will make use of throughout the proof. By Taylor's Theorem:
\begin{align*}
\calR(\aopt(\thetahat);\thetast)  = \calR(\aopt(\theta_1);\thetast) & = \calR(\aopt(\theta_0);\thetast) + \frac{d}{dt} \calR(\aopt(\theta_t);\thetast)|_{t=0} + \frac{1}{2} \frac{d^2}{dt^2} \calR(\aopt(\theta_t);\thetast)|_{t=0} \\
& \qquad + \frac{1}{6} \frac{d^3}{dt^3} \calR(\aopt(\theta_t);\thetast)|_{t=t'} 
\end{align*}
where $t' \in [0,1]$. Assumption \ref{asm:smoothness} gives that $\calR(\aopt(\theta_0);\thetast) = \calR(\aopt(\thetast);\thetast) = 0$. Furthermore, $\frac{d}{dt} \calR(\aopt(\theta_t);\thetast)|_{t=0} = \nabla_\fraka \calR(\fraka; \thetast)|_{\fraka = \aopt(\theta_0)} \cdot \nabla_\theta \aopt(\theta)|_{\theta = \theta_0} \cdot \frac{d}{dt} \theta_t |_{t=0}$, but by Assumption \ref{asm:smoothness}, $\nabla_\fraka \calR(\fraka; \thetast)|_{\fraka = \aopt(\theta_0)} = 0$. Finally, by the chain rule and since $\frac{d}{dt} \theta_t = \thetahat - \thetast$:
$$ \frac{d^2}{dt^2} \calR(\aopt(\theta_t);\thetast)|_{t=0} = (\thetahat - \thetast)^\top \nabla_\theta^2 \calR(\aopt(\theta);\thetast)|_{\theta = \theta_0} (\thetahat - \thetast) = (\thetahat - \thetast)^\top \taskhes(\thetast) (\thetahat - \thetast) $$
$$ \frac{d^3}{dt^3} \calR(\aopt(\theta_t);\thetast) = \nabla_\theta^3 \calR(\aopt(\theta);\thetast)|_{\theta=\theta_{t'}}[\thetahat - \thetast,\thetahat-\thetast,\thetahat-\thetast] $$
It then follows that,
\begin{align*}
\left | \calR(\aopt(\thetahat);\thetast) - \frac{1}{2} \| \thetahat - \thetast \|_{\taskhes(\thetast)}^2 \right | & \le \frac{1}{6} \| \nabla_{\theta}^3 \calR(\aopt(\theta);\thetast)|_{\theta = \theta_{t'}}[\thetahat - \thetast,\thetahat -\thetast, \thetahat - \thetast] \|_\op 
\end{align*}
The chain rule gives,
\begin{align*}
&\nabla_{\theta}^3\calR(\aopt(\theta);\thetast)[\thetahat - \thetast,\thetahat -\thetast, \thetahat - \thetast]  \\
&= \nabla_\fraka^3 \calR(\aopt(\theta);\thetast)[\nabla_\theta \aopt(\theta)[\thetahat -\thetast], \nabla_\theta \aopt(\theta)[\thetahat -\thetast], \nabla_\theta \aopt(\theta)[\thetahat -\thetast]] \\
& \qquad + 3 \nabla_\fraka^2 \calR(\aopt(\theta);\thetast)[\nabla_\theta^2 \aopt(\theta)[\thetahat -\thetast, \thetahat -\thetast],\nabla_\theta \aopt(\theta)[\thetahat -\thetast]] \\
& \qquad + \nabla_\fraka \calR(\aopt(\theta);\thetast)[\nabla_\theta^3 \aopt(\theta)[\thetahat - \thetast,\thetahat -\thetast, \thetahat - \thetast]] 
\end{align*}
so,
$$  \| \nabla_{\theta}^3 \calR(\aopt(\theta);\thetast)|_{\theta = \theta_{t'}}[\thetahat - \thetast,\thetahat -\thetast, \thetahat - \thetast] \|_\op \le (L_{\calR3} L_{\fraka1}^3 + 3 L_{\calR2} L_{\fraka2} L_{\fraka1} + L_{\calR1} L_{\fraka3}) \| \thetahat - \thetast \|_\op^3$$
which proves the first inequality. For the second inequality, recall that by definition,
$$ \taskhes(\thetast) = \nabla_{\theta}^2 \calR(\aopt(\theta);\thetast)|_{\theta = \thetast} $$
so, by Taylor's Theorem,
$$ \taskhes(\thetast) = \nabla_{\theta}^2 \calR(\aopt(\theta);\thetast)|_{\theta = \thetahat} + \frac{d}{dt} \nabla_\theta^2 \calR(\aopt(\theta_t);\thetast)|_{t = t'} $$
for $t' \in [0,1]$. However,
$$ \frac{d}{dt} \nabla_\theta^2 \calR(\aopt(\theta_t);\thetast)|_{t = t'}  = \nabla_\theta^3 \calR(\aopt(\theta);\thetast)|_{\theta = \theta_{t'}} \cdot \frac{d}{dt} \theta_t $$
Thus, 
$$ \| \taskhes(\thetast) - \nabla_{\theta}^2 \calR(\aopt(\theta);\thetast)|_{\theta = \thetahat} \|_\op \le (L_{\calR3} L_{\fraka1}^3 + 3 L_{\calR2} L_{\fraka2} L_{\fraka1} + L_{\calR1} L_{\fraka3}) \| \thetahat - \thetast \|_2 $$
By the chain rule,
$$ \nabla_\theta^2 \calR(\aopt(\theta);\theta') = \nabla_\fraka^2 \calR(\fraka; \theta')|_{\fraka = \aopt(\theta)}[\nabla_\theta \aopt(\theta), \nabla_\theta \aopt(\theta)] + \nabla_\fraka \calR(\fraka; \theta')|_{\fraka = \aopt(\theta)} \cdot \nabla_\theta^2 \aopt(\theta) $$
So, by Definition  \ref{def:gen_dec_properties}, since $\nabla_\fraka \calR(\fraka; \theta')|_{\fraka = \aopt(\theta')}  = 0$, we have:
\begin{align*}
& \| \nabla_{\theta}^2 \calR(\aopt(\theta);\thetast)|_{\theta = \thetahat} - \nabla_{\theta}^2 \calR(\aopt(\theta);\thetahat)|_{\theta = \thetahat} \|_\op \\
& = \|   \nabla_\fraka^2 \calR(\fraka; \thetast)|_{\fraka = \aopt(\thetahat)}[\nabla_\theta \aopt(\theta)|_{\theta=\thetahat}, \nabla_\theta \aopt(\theta)|_{\theta=\thetahat}] + \nabla_\fraka \calR(\fraka; \thetast)|_{\fraka = \aopt(\thetahat)} \cdot \nabla_\theta^2 \aopt(\theta)|_{\theta=\thetahat} \\
& \qquad \qquad  - \nabla_\fraka^2 \calR(\fraka; \thetahat)|_{\fraka = \aopt(\thetahat)}[\nabla_\theta \aopt(\theta)|_{\theta=\thetahat}, \nabla_\theta \aopt(\theta)|_{\theta=\thetahat}] \|_\op \\
& \le \Lra \| \nabla_\theta \aopt(\theta)|_{\theta = \thetahat} \|_\op^2 \| \thetahat - \thetast \|_2 + \|  \nabla_\fraka \calR(\fraka; \thetast)|_{\fraka = \aopt(\thetahat)} \cdot \nabla_\theta^2 \aopt(\theta)|_{\theta=\thetahat} \|_\op \\
& \le \Lra L_{\fraka1}^2 \| \thetahat - \thetast \|_2 + L_{\fraka2} \| \nabla_\fraka \calR(\fraka; \thetast)|_{\fraka = \aopt(\thetahat)} \|_\op
\end{align*}
However, $\nabla_\fraka \calR(\fraka; \thetast)|_{\fraka = \aopt(\thetahat)}$ is Lipschitz continuous so, since $\nabla_\fraka \calR(\fraka; \thetast)|_{\fraka = \aopt(\thetast)} = 0$,
\begin{align*}
\| \nabla_\fraka \calR(\fraka; \thetast)|_{\fraka = \aopt(\thetahat)} \|_\op & = \| \nabla_\fraka \calR(\fraka; \thetast)|_{\fraka = \aopt(\thetahat)} - \nabla_\fraka \calR(\fraka; \thetast)|_{\fraka = \aopt(\thetast)} \|_\op \\
& \le L_{\calR2} \| \aopt(\thetahat) - \aopt(\thetast) \|_2 \\
& \le L_{\calR2} L_{\fraka1} \| \thetahat - \thetast \|_2
\end{align*}
Since $\taskhes(\thetahat) = \nabla_{\theta}^2 \calR(\aopt(\theta);\thetahat)|_{\theta = \thetahat}$, we've shown that:
$$ \| \taskhes(\thetast) - \taskhes(\thetahat) \|_\op \le (L_{\calR3} L_{\fraka1}^3 + 3 L_{\calR2} L_{\fraka2} L_{\fraka1} +  L_{\calR2} L_{\fraka1} + L_{\calR1} L_{\fraka3} + \Lra L_{\fraka1}^2 ) \| \thetast - \thetahat \|_2 $$
which proves the second inequality. 
\end{proof}

%% file: body/ce_upper_bound_ddm.tex
%!TEX root = ../main.tex

\section{Upper Bounds on Certainty Equivalence Decision Making}\label{sec:ddm_ce_upper}

\subsection{Certainty Equivalence Upper Bound}\label{sec:ce_upper_pf}

In this section we assume we are in the linear dynamical system setting of Section \ref{sec:overview_lds} and that we are playing an exploration policy $\piexp$.

\begin{proof}[Proof of Theorem \ref{thm:ce_upper_bound}]
We define the following events.
\begin{align*}
\calA & = \left \{  \calR(\aopt(\thetals); \thetast) \le 5 \sigmaw^2 \tr(\taskhes(\thetast) \Exp_{\thetast,\piexp}[\matSig_T]^{-1}) \log\frac{6\dimtheta}{\delta} + \frac{C_1}{T^{3/2}} + \frac{C_2}{T^{1+2\alpha}} \right \} \tag{Good event} \\
\calE_1 & = \{ \lammin(\Sigma_T) \ge \lamund T, \Sigma_T \preceq T \covup \} \tag{Sufficient excitation} \\
\calE_2 & =  \{ \| \thetals - \thetast \|_2 \le \betast(\thetast) \} \tag{Quadratic approximation regime} \\
\calE_{3} & = \{ \| \Sigma_T - \Exp_{\thetast,\piexp}[\Sigma_T] \|_\op \le \tfrac{\Ccon}{T^\alpha} \lammin(\Exp_{\thetast,\piexp}[\Sigma_T])  \} \tag{Concentration of covariates}
\end{align*}
We would like to show that $\calA$ holds with high probability. The following is trivial.
$$ \Pr[\calA^c] \le \Pr[ \calA^c \cap \calE_1 \cap \calE_2 \cap \calE_3] + \Pr[\calE_1^c] + \Pr[\calE_1 \cap \calE_2^c] + \Pr[\calE_3^c]$$

\paragraph{Events $\calE_{i}$ hold with high probability:} We now show that the events $\calE_1,\calE_2$, and $\calE_{3}$ hold with high probability. Since $\piexp$ satisfies Assumption \ref{asm:minimal_policy}, we will have $\Pr[\calE_1^c] \le \delta$ and $\Pr[\calE_3^c] \le \delta$ as long as 
\begin{align}\label{eq:ce_upper_burnin1}
    T \ge \Texpse(\piexp), \quad T \ge \Texpcon(\piexp)
\end{align}
By Lemma \ref{lem:general_op_bound}, on the event $\calE_1$, with probability at least $1-\delta$,
$$ \| \thetals - \thetast \|_2 \le C \sqrt{\frac{ \log(1/\delta) + \dimtheta +  \logdet(\covup/\lamund + I)}{\lamund T}}$$
So as long as
\begin{align}\label{eq:ce_upper_burnin2}
T \ge \frac{C  ( \log(1/\delta) + \dimtheta +  \logdet(\covup/\lamund + I))}{\lamund \betast(\thetast)^2}
\end{align}
we will have
$$ \| \thetals - \thetast \|_2 \le  C \sqrt{\frac{\log(1/\delta) + \dimx +  \logdet(\covup/\lamund + I)}{\lamund T}} \le \betast(\thetast) $$
Thus, $ \Pr[\calE_1 \cap \calE_2^c] \le \delta$.

\paragraph{Events $\calE_{i}$ imply good event holds:} We now consider the event $\calE_1 \cap \calE_2 \cap \calE_3$. By Proposition \ref{quad:certainty_equivalence}, since $\calR$ satisfies Assumption \ref{asm:smoothness}, on this event we have
\begin{align*}
\calR(\aopt(\thetals);\thetast) & \le \| \thetals - \thetast \|_{\taskhes(\thetast)}^2 + \Lquad \| \thetals - \thetast \|_2^3  \le \| \thetals - \thetast \|_{\taskhes(\thetast)}^2 + C_1/T^{3/2} 
\end{align*}
where the last inequality follows by the bound on $\| \thetals - \thetast \|_2$ shown above for 
\begin{align*}
    C_1 := C \Lquad \frac{( \log(1/\delta) + \dimtheta +  \logdet(\covup/\lamund + I))^{3/2}}{\lamund^{3/2}}
\end{align*}
By Theorem \ref{thm:Mnorm_est_bound}, on the event $\calE_3$ and if $T$ is large enough so that
\begin{align}\label{eq:ce_upper_burnin3}
    \Ccon/T^\alpha < 1/4
\end{align}
and since $\| \matSig_T - \Exp_{\thetast,\piexp}[\matSig_T] \|_\op $, with probability at least $1-\delta$,
\begin{align*}
   \| \thetals - \thetast \|_{\taskhes(\thetast)}^2 & \le 5 \sigmaw^2 \tr(\taskhes(\thetast) \Exp_{\thetast,\piexp}[\matSig_T]^{-1}) \log\frac{6\dimtheta}{\delta} \\
   & \qquad  + \frac{130   \sigmaw^2 \Ccon^2}{T^{2\alpha}}
   \lammin(\Exp_{\thetast,\piexp}[\matSig_T]) \tr(\Exp_{\thetast,\piexp}[\matSig_T]^{-1}) \tr(\taskhes(\thetast) \Exp_{\thetast,\piexp}[\matSig_T]^{-1}) \log\frac{6\dimtheta}{\delta} \\
   & \le 5 \sigmaw^2 \tr(\taskhes(\thetast) \Exp_{\thetast,\piexp}[\matSig_T]^{-1}) \log\frac{6\dimtheta}{\delta} + \frac{260 \sigmaw^2 \Ccon^2 \dimtheta \tr(\taskhes(\thetast))}{\lamund T^{1 + 2\alpha}} \log \frac{6\dimtheta}{\delta}
\end{align*}
where the final inequality follows since
$$ \Exp_{\thetast,\piexp}[\matSig_T] \succeq \Exp_{\thetast,\piexp}[\I \{ \calE_1 \} \matSig_T] \ge \Pr[\calE_1] \lamund T \cdot I \ge \frac{1}{2} \lamund T \cdot I $$
Thus, $\Pr[\calE^c \cap \calE_1 \cap \calE_2 \cap \calE_3] \le \delta$ with
\begin{align*}
    C_2 :=  \frac{260 \sigmaw^2 \Ccon^2 \dimtheta \tr(\taskhes(\thetast))}{\lamund } \log \frac{6\dimtheta}{\delta}
\end{align*}
so it follows that $\Pr[\calA^c] \le 4 \delta$. The final result then follows by rescaling $\delta$ and so long as $T$ is large enough that \eqref{eq:ce_upper_burnin1}, \eqref{eq:ce_upper_burnin2}, and \eqref{eq:ce_upper_burnin3} hold, which will be the case if \eqref{eq:ce_opt_burnin} holds.
\end{proof}

\subsection{Euclidean Norm Estimation}
\begin{lem}\label{lem:general_op_bound}
Assume our data is generated according to \Cref{eq:ddm_observations} and let
$$ \thetals := \min_{\theta} \sum_{t=1}^T \| y_{t} - \theta^\top z_t \|_2^2 $$
Then on the event
\begin{align*}
        \calE := \Big \{ \lammin(\matSig_T) \ge \lamund T, \matSig_T \preceq T \covup \Big \}
    \end{align*}
with probability at least $1-\delta$:
$$ \| \thetals - \thetast \|_2 \le C \sqrt{\frac{ \log(1/\delta) + \dimtheta + \logdet(\covup/\lamund + I)}{\lamund T}}. $$
\end{lem}
\begin{proof}
Define the following events:
\begin{align*}
\calA & = \left \{ \| \thetahat_i - \thetast \|_2 \le C \sqrt{\frac{ \log(1/\delta) + \dimtheta + \logdet(\covup/\lamund + I)}{\lamund T}} \right \} \\
\calE_1 & = \left \{ \left \| \left ( \sum_{ t = 1}^T z_t z_t^\top \right )^{-1/2} \sum_{t=1}^T z_t w_t^\top \right \|_\op \le c_2 \sigmaw \sqrt{ \log \frac{1}{\delta} + \dimtheta + \logdet(\covup/\lamund + I)} \right \} 
\end{align*}
Our goal is to show that $\Pr[\calA^c \cap \calE ] \le \delta$. The following is trivial.
\begin{align*}
\Pr[\calA^c \cap \calE ] & \le \Pr[\calA^c \cap \calE \cap \calE_1] + \Pr[\calE  \cap \calE_1^c]
\end{align*}
As $\thetals$ is the least squares estimate, we will have that $\thetals = (\sum_{t=1}^T z_t z_t^\top)^{-1} \sum_{t=1}^T z_t y_{t} = \thetast +  (\sum_{t=T-T_i}^T z_t z_t^\top)^{-1} \sum_{t=1}^T z_t w_t^\top$. Given this, the error can be decomposed as:
\begin{align*}
\| \thetals - \thetast \|_2 & = \left \|  \left (\sum_{t=1}^T z_t z_t^\top \right )^{-1} \sum_{t=1}^T z_t w_t^\top \right \|_2  \le \left \|  \left ( \sum_{t=1}^T z_t z_t^\top \right )^{-1/2} \right \|_\op \left \| \left (\sum_{t=1}^T z_t z_t^\top \right )^{-1/2} \sum_{t=1}^T z_t w_t^\top \right \|_2 \\
& = \left \| \left (\sum_{t=1}^T z_t z_t^\top \right )^{-1/2} \sum_{t=1}^T z_t w_t^\top \right \|_2 / \sqrt{\lambda_{\min} \left ( \sum_{t=1}^T z_t z_t^\top \right )}
\end{align*}
It follows that, on the event $\calE \cap \calE_1$, the error bound given in $\calA$ holds. Thus, $ \Pr[\calA^c \cap \calE \cap \calE_1]  = 0$. Lemma \ref{lem:self_norm_general} implies that $\Pr[\calE \cap \calE_1^c] \le \delta$, so $\Pr[\calA^c \cap \calE] \le \delta$. 
\end{proof}

%% file: appendix/lds_preliminaries.tex
%!TEX root = ../main.tex

\section{Notation for Linear Dynamical Systems}\label{sec:lds_notation}
We next introduce notation used throughout Part \ref{part:lin}. Throughout, we consider linear dynamical systems of the form
\begin{equation}\label{eq:lds_dyn_covpf}
x_{t+1} = \Ast x_t + \Bst u_t + w_t 
\end{equation}
where $\Ast \in \R^{\dimx \times \dimx}, \Bst \in \R^{\dimx \times \dimu}$, and $w_t \sim \calN(0,\sigmaw^2 I)$. We denote $\theta = (A,B)$ and $\thetast := (\Ast,\Bst)$. We will sometimes break up the state into the portion driven by the input, $\xu_t$, and the portion driven by the noise, $\xw_t$. In particular, we have
\begin{align*}
    \xu_{t+1} = \Ast \xu_t + \Bst u_t, \quad \xw_{t+1} = \Ast \xw_t + w_t
\end{align*}
Due to linearity, $x_t = \xu_t + \xw_t$.

\subsection{Covariance Notation}
At the center of our analysis are the covariance matrices that arise from excitation of the linear system with a certain input. For an input sequence $\bmu := (u_1,\dots,u_t) \in \R^{t\dimu}$, we define the open loop input covariance
\begin{align}
\Gamin_t(\theta,\bmu,x_0) := \sum_{s=0}^{t-1} \xbmu_s (\xbmu_s)^\top \quad \text{where} \quad \xbmu_{s+1} = A \xbmu_s + B u_s, \quad \xbmu_0 = x_0
\end{align}
We overload notation, so that the above is also defined when $\bmu = (u_s)_{s=1}^{t'}$ for $t' \ge t$, or even infinite sequences $\bmu = (u_s)_{s \ge 1}$. In addition, if $\bmu = (u_s)_{s=1}^{t'}$ for $t' < t$, we define $\Gamin_t(\theta,\bmu,x_0)$ to be the open loop covariance when playing $\bmu$ periodically: that is, the input $u_t = u_{\mathrm{mod}(t,t')}$. Recall that:
\begin{align*}
 \Gamnoise_t(\theta,\Lambda_u) := \sum_{s =0}^{t-1} A^s \Lambda_w (A^s)^\top +  \sum_{s=0}^{t-1} A^s B \Lambda_u B^\top (A^s)^\top
\end{align*}
and observe that we can equivalently define
\begin{align*}
\Gamnoise_t(\theta,\Lambda_u) = \Exp \left [ x_t  x_t^\top  \mid u_s \iidsim \calN(0,\Lambda_u), w_s \iidsim \calN(0,\Lambda_w), s \le t, x_0 = 0 \right]
\end{align*}
We also define the following, which corresponds to the total expected average covariates starting from some state $x_0$ and playing any input  $u_t = \util_t + u_t^w$, where $\bmu = (\util_t)_{t=1}^k$ and $u_t^w \iidsim \calN(0,\sigma_u^2 I)$:
\begin{align} 
\Gamma_T(\theta,\bmu,\sigma_u,x_0) := \frac{1}{T} \Gamin_T(\theta,\bmu,x_0) + \frac{1}{T} \sum_{t=1}^T \Gamnoise_t(\theta,\sigma_u) 
\end{align}
We also set:
\begin{equation*}
\Gamnoisetil_t :=  \begin{bmatrix} \sigmaw^2 \sum_{s = 0}^{t-1} \Ast^s (\Ast^s)^\top + \frac{\gamma^2}{2\dimu} \sum_{s=0}^{t-1} \Ast^s \Bst \Bst^\top (\Ast^s)^\top & 0 \\ 0 & \frac{\gamma^2}{2\dimu} I \end{bmatrix} 
\end{equation*}
We briefly recall the following definitions stated in Section \ref{sec:overview_lds}. We will consider the set of inputs
\begin{equation*}
 \calU_{\gamma^2,k} := \left \{\bmU = ( U_\ell )_{\ell=1}^{k} \ : \ U_\ell \in \Hermpsd[\dimu], \quad \bmU \text{ is symmetric}, \quad\sum_{\ell=1}^{k} \tr(U_\ell) \le k^2 \gamma^2 \right \} 
\end{equation*}
For some $\bmU \in \calU_{\gamma^2,k}$ we define
\begin{align*} 
\Gamfreq_k(\theta,\bmU) &= \frac{1}{k}  \sum_{\ell=1}^k (e^{\imag \frac{2\pi \ell}{k}} I - A)^{-1} B U_\ell B^\herm (e^{\imag \frac{2\pi \ell}{k}} I - A)^{-\herm}, \quad \Gamfreq_{t,k}(\theta,\bmU) = \frac{t}{k}\Gamfreq_k(\theta,\bmU) 
\end{align*}
which, as we noted in Section \ref{sec:overview_lds}, correspond to the steady-state covariates when the input $\bmU$ is played periodically. Finally, we set:
\begin{align*} 
\Gamss_{T,k}(\theta,\bmU,\sigma_u) = \frac{1}{k} \Gamfreq_k(\theta,\bmU) + \frac{1}{T} \sum_{t=1}^T \Gamnoise_t(\theta,\sigma_u) 
\end{align*}
which correspond to the expected steady-state  covariates of the noisy system when playing inputs $\bmU$.

\subsection{Lifted Dynamical System}
We will set
\begin{equation}\label{eq:lds_extended}
 \thetatil := (\Atil,\Btil), \quad \Atil := \begin{bmatrix} A & B \\ 0 & 0 \end{bmatrix}, \quad \Btil := \begin{bmatrix} 0 \\ I \end{bmatrix} 
\end{equation}
and in particular $\thetastbar := (\Atilst,\Btilst)$. Then consider the dynamical system:
\begin{equation}\label{eq:input_state_dyn}
z_{t+1} = \Atilst z_t + \Btilst  u_{t+1}  +  \begin{bmatrix} w_t \\ 0 \end{bmatrix} 
\end{equation}
We note that $z_t = [x_t; u_t]$, where $x_t$ is the state of \eqref{eq:lds_dyn_covpf}. It follows that
$$ \Sigma_T := \sum_{t=1}^T z_t z_t^\top = \sum_{t=1}^T \begin{bmatrix} x_t \\ u_t \end{bmatrix} \begin{bmatrix} x_t \\ u_t \end{bmatrix}^\top $$
so a bound on the covariates of the system \eqref{eq:input_state_dyn} can be directly applied to the state-input covariates from \eqref{eq:lds_dyn_covpf}. For subsequent results, we will use $z_t := [x_t;u_t]$.

\subsection{Linear Dynamical Systems as Vector Regression}\label{sec:lds_vec}

We can write the system \eqref{eq:lds_dyn_covpf} in the form
\begin{equation}\label{eq:dyn_vec_Mpf}
y_{s} = \langle \phist,  v_{s} \rangle + \eta_s
\end{equation}
To obtain this mapping, we reindex time: for a fixed $t$ of \eqref{eq:lds_dyn_covpf}, define $s = (\dimx+\dimu)t + i$ for some $i \in \{1,\ldots,\dimx\}$. Furthermore, we set $\phist = [A_{\star,1};B_{\star,1};\ldots;A_{\star,\dimx};B_{\star,\dimx}] \in \R^{\dimx^2 + \dimx \dimu}$, where $A_{\star,j},B_{\star,j}$ denote the $j$th row of $\Ast$ and $\Bst$, respectively, $\eta_s = [w_t]_i$, and $v_{s} =  [\mathbf{0},\ldots,\mathbf{0},x_t,u_{t},\mathbf{0},\ldots,\mathbf{0}]$, where $x_t$ starts at index $(\dimx + \dimu) (i-1) + 1$. With these definitions we will have $y_s = [x_{t+1}]_i$. It follows that if we run \eqref{eq:dyn_vec_Mpf} from time $s = 1$ to $s = (\dimx + \dimu)T + \dimx$ the set of observations obtained will be identical to those obtained from $x_{t+1} = \Ast x_t + \Bst u_t + w_t$. Thus, \eqref{eq:dyn_vec_Mpf} is simply a vectorization of $x_{t+1} = \Ast x_t + \Bst u_t + w_t$. It is easy to see that, if $\wh{\phi}$ denotes the least squares estimate of $\phist$ obtained from observations of \eqref{eq:dyn_vec_Mpf} and $\Ahat,\Bhat$ denote the least squares estimates of $\Ast,\Bst$ obtained from observations of \eqref{eq:lds_dyn_covpf}, we will have $\wh{\phi} = [\Ahat_1;\Bhat_1;\ldots;\Ahat_{\dimx};\Bhat_{\dimx}]$. Furthermore,\footnote{Note that we change notation slightly here. Previously $\matSig_T$ denoted the set of covariates in the general regression setting after $T$ steps, while here $\matSig_T$ is the set of covariates after $(\dimx + \dimu)T + \dimx$ steps in the general regression setting, but corresponds to running our linear dynamical system for $T$ steps. As subsequent results are concerned with the time scale of the linear dynamical system, this change in notation will simply further analysis.}
$$ \matSig_T := \sum_{s=1}^{(\dimx + \dimu)T + \dimx} z_{s} z_{s}^\top = I_{\dimx} \otimes \sum_{t=1}^T x_t x_t^\top  = I \otimes \Sigma_T$$
Thus, $\Exp \matSig_T = I_{\dimx} \otimes \Exp \sum_{t=1}^T x_t x_t^\top$, $\| \matSig_T - \Exp \matSig_T \|_\op = \| \sum_{t=1}^T x_t x_t^\top - \Exp \sum_{t=1}^T x_t x_t^\top \|_\op$, and $\lammin(\Exp \sum_{t=1}^T x_t x_t^\top) = \lammin(\Exp \matSig_T)$. 
This equivalence allows us to apply results from Section \ref{sec:M_norm_regression} and Section \ref{sec:general_decision} in the dynamical system setting.

\subsection{Key Parameters in the Analysis}
For any $\theta = (A,B)$, the $\Hinf$ norm of $\theta$ is defined as:
$$ \| \theta \|_{\Hinf} := \max_{\omega \in [0,2\pi]} \| (e^{\imag \omega} I - A)^{-1} B \|_\op $$
To control the transient behavior, let:
$$ \tau(A,\rho) := \sup \{ \| A^k \|_\op \rho^{-k} \ : \ k \ge 0 \} $$
$\tau(A,\rho)$ is the smallest value such that $\| A^k \|_\op \le \tau(A,\rho) \rho^k$ for all $k$. We will define 
$$\rhool := \max \left \{ \frac{1}{2}, \frac{2 \| \Ast \|_{\Hinf} \| \Ast \|_\op}{1 + 2 \| \Ast \|_{\Hinf} \| \Ast \|_\op} \right\}$$ 
and $\tauol := \tau(\Atilst,\rhool)$. The following result relates $\tauol, \frac{1}{1-\rhool}$, and $\| \thetastbar \|_{\Hinf}$ to $\| \Ast \|_{\Hinf}$ and $\| \Bst \|_\op$, which will aid in simplifying our results.

\begin{lem}  The following upper bounds hold:
\begin{align*}
\frac{1}{1-\rhool} \le 2 + 2 \| \Ast \|_{\Hinf}^2, \quad \tauol \le 2(1+2\| \Bst \|_\op) \| \Ast \|_{\Hinf}, \quad \| \thetastbar \|_{\Hinf} \le 1 + (1+\|\Bst\|_{\op})\|\Ast\|_{\Hinf}.
\end{align*}
\end{lem}\label{lem:hinf_upper_bound}

In addition, we can relate the value of $\tau$ for a lifted system $\thetatil$ to the original system $\theta$.
\begin{lem}\label{lem:tau_aug}
Let $\Atil$ be defined as in \eqref{eq:lds_extended}. Then $\tau(\Atil,\rho) \le (1 + \rho^{-1} \| B \|_\op) \tau(A,\rho)$.
\end{lem}

We introduce the following constants to control the smoothness of the covariates:
\begin{equation*}
\betaexplds(\thetast) := \min \left \{ \frac{1-\rhool}{2\tauol}, \frac{1}{2 \| \Ast \|_{\Hinf}}, 1 \right \}
\end{equation*}
\begin{align*}
\alphastlds(\thetast,\gamma^2) & :=  \frac{8  (\sigmaw^2 + \sigma_u^2 \| \Bst \|_\op^2) \tauol^3}{(1 - \rhool^2)^2} + \frac{4 \sigma_u^2 (\| \Bst \|_\op + 1) \tauol}{1 - \rhool^2}  + 34 \gamma^2  \| \Ast \|_{\Hinf}^3 (\| \Bst \|_\op + 1)^2 
\end{align*}
Lemma \ref{lem:smooth_covariates} implies that, if
$$ \| \theta - \thetast \|_\op \le  \betaexplds(\thetast)$$
then for any $u \in \calU_{\gamma^2,k}$, if $T_2$ is divisible by $k$, 
$$\| \matGamss_{T_1,T_2}(\theta,u,\sigma_u) -  \matGamss_{T_1,T_2}(\thetast,u,\sigma_u) \|_\op  \le \alphastlds(\thetast,\gamma^2) \cdot \| \theta - \thetast \|_\op.$$
This holds regardless of the loss $\calR$. 

Finally, in our analysis it will be convenient to work with a slightly different definition of the optimal risk, which we define as:
\begin{align*}
 \Phiss(\gamma^2;\thetast) := \liminf_{T \rightarrow \infty} \min_{u \in \calU_{\gamma^2,T}} \tr \Big ( \taskhes(\thetast) \matGamss_{T,T}(\thetastbar,u,0)^{-1} \Big ) 
 \end{align*}
As the following result shows, $\Phiopt$ and $\Phiss$ are equal up to absolute constants.

\begin{lem}\label{prop:phiss_phiopt}
$\Phiopt(\gamma^2;\thetast)$ and $\Phiss(\gamma^2;\thetast)$ are equal up to constants:
$$ \frac{1}{4} \Phiopt(\gamma^2; \thetast) \le \Phiss(\gamma^2;\thetast) \le 16 \Phiopt(\gamma^2;\thetast). $$
\end{lem}

\subsection{Linear Dynamical Systems Notation Proofs}
\begin{proof}[Proof of Lemma \ref{lem:hinf_upper_bound}]
We have that
\begin{align*}
 \| \thetastbar \|_{\Hinf} &= \max_{\omega \in [0,2\pi]}\left\|\left(e^{\imag \omega} I - \begin{bmatrix} \Ast & \Bst \\ 0 & 0 \end{bmatrix} \right)^{-1} \begin{bmatrix} 0 \\ I  \end{bmatrix}\right\|_{\op} \\
 &\le \max_{\omega \in [0,2\pi]}\left\|\left(e^{\imag \omega} I - \begin{bmatrix} \Ast & \Bst \\ 0 & 0 \end{bmatrix} \right)^{-1} \right\|_{\op} 
\end{align*}
For each $\omega$, set $A(\omega) := (e^{\imag \omega} I - \Ast)^{-1}$. Then, using the block matrix inverse formula,
\begin{align*}
\left(e^{\imag \omega} I - \begin{bmatrix} \Ast & \Bst \\ 0 & 0 \end{bmatrix} \right)^{-1} = \begin{bmatrix} A(\omega) & -A(\omega)e^{-\imag \omega} \Bst \\
0 & e^{-\imag \omega} I
\end{bmatrix}
\end{align*}
Thus, 
\begin{align*}
 \max_{\omega \in [0,2\pi]}\left\|\left(e^{\imag \omega} I - \begin{bmatrix} \Ast & \Bst \\ 0 & 0 \end{bmatrix} \right)^{-1} \right\|_{\op} &\le  \max_{\omega \in [0,2\pi]}\left\|\begin{bmatrix} A(\omega) & -A(\omega)e^{-\imag \omega} \Bst \\
0 & e^{-\imag \omega} I
\end{bmatrix}\right\| \\
&\le 1 + (1+\|\Bst\|_{\op})\max_{\omega \in[0,2\pi]}\|A(\omega)\|_{\op} \\
&= 1 + (1+\|\Bst\|_{\op})\|\Ast\|_{\Hinf}.
 \end{align*}
In the case of scalar $\Ast$, Lemma 4.1 \cite{tu2017non} shows that $\| \Ast^k \|_\op \le \| \frac{1}{\rho} \Ast \|_{\Hinf} \rho^k$. In the case when $\dimx > 1$, we can apply their proof to the sequence $u^\top \Ast^k v$ for some $u,v$ with $\| u \|_2 = \| v \|_2 = 1$. Doing so, we obtain
\begin{align*}
u^\top \Ast^k v \le \| \tfrac{1}{\rho} \Ast \|_{\Hinf} \rho^k
\end{align*}
As this holds for all $u$ and $v$, we have $\| \Ast^k \|_\op \le \| \tfrac{1}{\rho} \Ast \|_{\Hinf} \rho^k$. As $\tau(\Ast,\rho)$ is the smallest value satisfying $\| \Ast^k \|_\op \le \tau(\Ast,\rho) \rho^k$ for all $k$, it follows that $\tau(\Ast,\rho) \le \| \tfrac{1}{\rho} \Ast \|_{\Hinf}$. We next wish to upper bound $\| \tfrac{1}{\rho} \Ast \|_{\Hinf}$ by $\| \Ast \|_{\Hinf}$ for some choice of $\rho$. Lemma F.9 of \cite{wagenmaker2020active} gives that
\begin{align*}
\| \Ast - \tfrac{1}{\rho} \Ast \|_\op \le \frac{1}{2 \| \Ast \|_{\Hinf}} \quad \implies \quad \| \tfrac{1}{\rho} \Ast \|_\op \le 2 \| \Ast \|_{\Hinf}
\end{align*}
A sufficient condition to meet this is 
\begin{align*}
\rho \ge \frac{2 \| \Ast \|_{\Hinf} \| \Ast \|_\op}{1 + 2 \| \Ast \|_{\Hinf} \| \Ast \|_\op}
\end{align*}
As $\rhool$ satisfies this, it follows that $\tau(\Ast,\rhool) \le \| \frac{1}{\rhool} \Ast \|_{\Hinf} \le 2 \| \Ast \|_{\Hinf}$. Combining this with \Cref{lem:tau_aug}, we conclude that
\begin{align*}
\tauol \le (1 + \rhool^{-1} \| \Bst \|_\op) \tau(\Ast,\rhool) \le 2 (1 + \rhool^{-1} \| \Bst \|_\op) \| \Ast \|_{\Hinf} \le 2 (1 + 2 \| \Bst \|_\op) \| \Ast \|_{\Hinf}
\end{align*}
Finally, by definition of $\rhool$ it follows
\begin{align*}
\frac{1}{1-\rhool} \le \max\{1 + 2 \| \Ast \|_{\Hinf} \| \Ast \|_\op,2\} \le 2 + 2 \| \Ast \|_{\Hinf} \| \Ast \|_\op 
\end{align*}
We then upper bound $\| \Ast \|_\op  \le \| \Ast \|_{\Hinf}$ to obtain the final result.
\end{proof}

\begin{proof}[Proof of Lemma \ref{lem:tau_aug}]
Note that:
$$ \Atil^k = \begin{bmatrix} A^k & A^{k-1} B \\ 0 & 0 \end{bmatrix} $$
Thus,
$$ \| \Atil^k \|_\op = \sup_{v \in \calS^{d+p-1}} \left \| \begin{bmatrix} A^k & A^{k-1} B \\ 0 & 0 \end{bmatrix} v \right \|_\op = \sup_{v \in \calS^{d+p-1}} \left \| A^k v_1 + A^{k-1} B v_2 \right \|_\op \le \| A^k \|_\op + \| A^{k-1} B \|_\op$$
so, for any $\rho > 0$,
$$ \| \Atil^k \|_\op \rho^{-k} \le \| A^k \|_\op \rho^{-k} + \| B \|_\op \| A^{k-1} \|_\op \rho^{-k} \le \tau(A,\rho) + \rho^{-1} \| B \|_\op \| A^{k-1} \|_\op \rho^{-(k-1)} \le (1 + \rho^{-1} \| B \|_\op) \tau(A,\rho) $$ 
\end{proof}

\begin{proof}[Proof of Proposition \ref{prop:steady_state_inputs}]
1. Follows by Parseval's Theorem and simple manipulations. For 2., take some $\ell$ such that $\frac{\ell}{k'} = \frac{n}{k}$ for some integer $n$. Then,
$$\ucheck_\ell' = \sum_{s=1}^{k'} u_s e^{-\imag \frac{2\pi \ell s}{k'}} =  \sum_{s=1}^{k'} u_s e^{-\imag \frac{2\pi n s}{k}} = \sum_{s=1}^{k'} \util_{\mathrm{mod}(s,k)} e^{-\imag \frac{2\pi n \mathrm{mod}(s,k)}{k}} = \frac{k'}{k}  \sum_{s=1}^{k} \util_{s} e^{-\imag \frac{2\pi n s}{k}} = \frac{k'}{k} \ucheck_n $$
Furthermore, if $\frac{\ell}{k'} \neq \frac{n}{k}$ for all integers $n$, we will have 
$$\ucheck_\ell' = \sum_{s=1}^{k'} u_s e^{-\imag \frac{2\pi \ell s}{k'}} = \sum_{r=1}^k \util_r \sum_{s=0}^{k'/k-1} e^{-\imag \frac{2\pi \ell (ks+r)}{k'}} = \sum_{r=1}^k \util_r e^{-\imag \frac{2\pi \ell r}{k'}} \sum_{s=0}^{k'/k-1} e^{-\imag \frac{2\pi \ell ks}{k'}} = 0$$
Plugging this into the expression for $\Gamfreq_{k'}(\theta,\bmu')$, the conclusion follows. 
\end{proof}

\begin{proof}[Proof of Proposition \ref{prop:constructtimeinput}]
Fix some $m$ and $j \in [\dimu]$ and consider the segment of $\bmu_m$, $\bmu_m^j := (u_t)_{t= (j-1) km+1}^{jkm}$. By construction, this is a signal with period $k$. Assume we play this input starting from some state $x_0$ not necessarily equal to 0. Let $x_t^{\bmu_m^j}$ denote the response generated on the noiseless system. By Parseval's Theorem and \Cref{prop:steady_state_inputs}, it follows that
\begin{align*}
\frac{1}{k} \Gamfreq_k(\theta, \bmu_1^j)  = \lim_{m \rightarrow \infty} \frac{1}{km} \sum_{t=0}^{km} x_t^{\bmu_m^j} (x_t^{\bmu_m^j})^\top
\end{align*}
Furthermore, by the construction of $\bmu_1^j$ given in \texttt{ConstructTimeInput}, we have
\begin{align*}
 \frac{1}{k} \Gamfreq_k(\theta, \bmu_1^j)  = \frac{\dimu}{k^2} \sum_{\ell=1}^k \lambda_{\ell,j} (e^{\imag \frac{2\pi\ell}{k}} I - A)^{-1} B v_{\ell,j} v_{\ell,j}^\herm B^\herm (e^{\imag \frac{2\pi\ell}{k}} I - A)^{-\herm}
\end{align*}
Now note that, if we play the entire sequence of inputs $\bmu_m$, we will have
\begin{align*}
\sum_{t=0}^{\dimu k m} x_t^{\bmu_m} (x_t^{\bmu_m})^\top = \sum_{j=1}^{\dimu}  \sum_{t= (j-1)km +1}^{jkm} x_t^{\bmu_m^j}(x_t^{\bmu_m^j})^\top
\end{align*}
where the starting state, $x_{(j-1)km+1}^{\bmu_m^j}$, is equal to the final state produced when playing the previous input, $x_{(j-1)km}^{\bmu_m^{j-1}}$. Note that, as we assume the system is stable and the input has bounded energy and is of period $k$, the norm of $x_{(j-1)km}^{\bmu_m^{j-1}}$ will scale sublinearly $m$ (see \Cref{sec:state_norm_bounds}). It follows that, 
\begin{align*}
\lim_{m \rightarrow \infty} \frac{1}{\dimu m k} \sum_{t=0}^{\dimu mk  } x_t^{\bmu_{m}}  (x_t^{\bmu_m})^\top & = \lim_{m \rightarrow \infty} \frac{1}{\dimu m k}  \sum_{j=1}^{\dimu}  \sum_{t= (j-1)km +1}^{jkm} x_t^{\bmu_m^j} (x_t^{\bmu_m^j})^\top \\
& = \frac{1}{\dimu} \sum_{j=1}^{\dimu} \lim_{m \rightarrow \infty} \frac{1}{mk}  \sum_{t= (j-1)km +1}^{jkm} x_t^{\bmu_m^j} (x_t^{\bmu_m^j})^\top \\
& = \frac{1}{\dimu} \sum_{j=1}^{\dimu}  \frac{1}{k} \Gamfreq_k(\theta, \bmu_1^j) \\
& = \frac{1}{k^2}  \sum_{\ell = 1}^{k} \sum_{j=1}^{\dimu} (e^{\imag \frac{2\pi\ell}{k}} I - A)^{-1} B ( \lambda_{\ell,j} v_{\ell,j} v_{\ell,j}^\herm) B^\herm (e^{\imag \frac{2\pi\ell}{k}} I - A)^{-\herm} \\
& = \frac{1}{k^2} \sum_{\ell = 1}^k (e^{\imag \frac{2\pi\ell}{k}} I - A)^{-1} B U_\ell B^\herm (e^{\imag \frac{2\pi\ell}{k}} I - A)^{-\herm} \\
& = \frac{1}{k} \Gamfreq_k(\theta,\bmU)
\end{align*}
where the second to last inequality follows by the definition of $\lambda_{\ell,j}, v_{\ell,j}$ given in \texttt{ConstructTimeInput}. To see that the power constraint holds, note that, by Parseval's Theorem and the construction of the input,
\begin{align*}
\sum_{t=1}^{\dimu mk} u_t^\top u_t & = \sum_{j=1}^{\dimu} \sum_{t=(j-1)km + 1}^{jkm} u_t^\top u_t = \sum_{j=1}^{\dimu} \frac{m}{k} \sum_{\ell=1}^{k} \dimu \lambda_{j,\ell} v_{j,\ell}^\top v_{j,\ell} = \sum_{j=1}^{\dimu} \frac{m}{k} \sum_{\ell = 1}^k \dimu \tr(\lambda_{j,\ell} v_{j,\ell} v_{j,\ell}^\top) \\
& = \frac{m \dimu}{k} \sum_{\ell=1}^{k} \tr(U_\ell) \le m \dimu k \gamma^2
\end{align*}
where the final inequality holds since $\bmU \in \calU_{\gamma^2,k}$.
\end{proof}

\begin{proof}[Proof of \Cref{prop:phiss_phiopt}]
Fix some $T$. In the proof of Theorem \ref{thm:lds_lb} we showed that, for some $\theta_0$ satisfying $\|\thetast - \theta_0\|^2_{F} \le 5 (\dimx^2 + \dimx \dimu) /(\lamstinf T^{5/6})$,
$$  \min_{\pi \in \policyset} \tr\left(\taskhes(\theta_0) \left( \Exp_{\thetast,\pi}[\matSig_{T}] + \lamstinf T \cdot I \right)^{-1}  \right) \ge \frac{1}{16T} \Phiss(\gamma^2;\thetast) - \calO \left ( \frac{1}{T^{17/12}} \right ) $$
Following the proof of Theorem \ref{thm:lds_lb}, we can use Proposition \ref{quad:certainty_equivalence} to show that
\begin{align*}
\min_{\pi \in \policyset} & \tr\left(\taskhes(\theta_0) \left( \Exp_{\thetast,\pi}[\matSig_{T}] + \lamstinf T \cdot I \right)^{-1}  \right) \\
& \le \min_{\pi \in \policyset} \tr\left(\taskhes(\thetast) \left( \Exp_{\thetast,\pi}[\matSig_{T}] + \lamstinf T \cdot I \right)^{-1}  \right) + \calO \left ( \frac{1}{T^{17/12}} \right ) \\
& \le \min_{\pi \in \policyset} \tr\left(\taskhes(\thetast) \left( \Exp_{\thetast,\pi}[\matSig_{T}]  \right)^{-1}  \right) + \calO \left ( \frac{1}{T^{17/12}} \right ) 
\end{align*}
Renormalizing by $T$, it follows that for any $T$
$$ \inf_{\pi \in \policyset} \Phi_T(\thetast; \pi) = \min_{\pi \in \policyset} \tr\left(\taskhes(\thetast) \left( \Exp_{\thetast,\pi}[\matSig_{T}/T]  \right)^{-1}  \right) \ge \frac{1}{16} \Phiss(\gamma^2;\thetast) - \calO \left ( \frac{1}{T^{5/12}} \right ) $$
Taking $\liminf_{T \rightarrow \infty}$ of both sides proves the first inequality. 

For the second inequality, some trivial manipulations of \eqref{eq:opt_inputs_upper} in the proof of Lemma \ref{lem:global_optimal_inputs} shows that, for sufficiently large $T$,
$$ \min_{\pi \in \policyset} \tr\left(\taskhes(\theta_0) \left( \Exp_{\thetast,\pi}[\matSig_{T}] \right)^{-1}  \right) \le \frac{4}{T} \min_{u \in \calU_{\gamma^2,T}} \tr \left ( \taskhes(\thetast) \Gamss_{T,T}(\thetastbar,u,0)^{-1} \right ) $$ 
Renormalizing by $T$ and taking $\liminf_{T \rightarrow \infty}$ of both sides gives the result.
\end{proof}

%% file: body/lds_lower_bound.tex
%!TEX root = ../main.tex

\section{Lower Bounds in Linear Dynamical Systems}\label{sec:lds_lb}

\subsection{Regular Policies in Linear Dynamical Systems}\label{sec:reg_policy_lds}

The following result shows that Assumption \ref{asm:suff_excite} and \ref{asm:smooth_covariates} are met if we assume the dynamics are linear, and that $\piexp$ satisfies a certain regularity condition. This implies that \Cref{cor:simple_regret_lb2_nice} holds for a fairly general set of policies in linear dynamical systems.

\begin{lem}\label{lem:lds_periodic_ce_lb}
Assume that $\piexp \in \policyset$ plays input $u_t \sim \calN(\util_t, \sigma_u^2 I)$ where $\util_t$ and $\sigma_u$ are chosen deterministically at time 0, and $\util_t$ is periodic with period $k$. Then, if our dynamics are
\begin{align*}
x_{t+1} = \Ast x_t + \Bst u_t + w_t
\end{align*}
$\piexp$ and $\thetast$ satisfy Assumption \ref{asm:suff_excite} and \ref{asm:smooth_covariates} with $\alpha = 1$, $\ccexp = 1$, and
\begin{align*}
& \lamund = \min \{ \sigma_w^2, \sigma_u^2 \}, \quad \alphast(\thetast;\gamma^2) =   \Csys (\sigma_w^2 +  \gamma^2) , \quad \betaexplds(\thetast) =\Csys^{-1}, \quad \Ccov =\Csys \gamma^2 ( \sqrt{T} k + k^2)
\end{align*}
for some constant $\Csys = \poly(\| \Bst \|_\op, \| \Ast \|_{\Hinf})$ and with $\dimtheta = \dimx^2 + \dimx \dimu$.
\end{lem}
\begin{proof}
That this policy satisfies Assumption \ref{asm:suff_excite} with $\lamund = \min \{ \sigma_w^2, \sigma_u^2 \}$ is trivial.

To see that Assumption \ref{asm:smooth_covariates} is satisfied, fix some $\theta$ and denote $\bmutil = (\util_t)_{t=1}^k$. Then, by definition,
\begin{align*}
\frac{1}{T} \Exp_{\theta,\piexp}[\Sigma_T] = \frac{1}{T} \Gamin_T(\thetatil,\bmutil,0) + \frac{1}{T} \sum_{t=1}^T \Gamnoise_t(\thetatil,\sigma_u)
\end{align*}
Note that since $\piexp \in \policyset$ and $\util_t$ is periodic we have $\sum_{t=0}^{k-1} \util_t^\top \util_t \le k \gamma^2$. By Lemma \ref{lem:lds_lb_cov_ss_exp}, it follows that
\begin{align*}
\Big \| \frac{1}{T} \Gamin_T(\thetatil,\bmutil,0) - \frac{1}{k} \Gamfreq_{k}(\thetatil,\bmutil) \Big \|_\op \le \frac{\tau(\Atil,\rho) \| \thetatil \|_{\Hinf}^2  \sqrt{T+1} k \gamma^2}{(1-\rho^k) T} + \frac{\tau(\Atil,\rho)^2 \| \thetatil \|_{\Hinf}^2  k^2 \gamma^2}{(1-\rho^k)^2 T} =: \Delta(\theta)
\end{align*}
which implies that
\begin{align*}
\Big \| \frac{1}{T} \Exp_{\theta,\piexp}[\Sigma_T]  - \Gamss_{T,k}(\thetatil,\bmutil,\sigma_u) \Big \|_\op \le \Delta(\theta)
\end{align*}
Applying this same bound to $\thetast$, the triangle inequality gives
\begin{align*}
& \Big \| \frac{1}{T} \Exp_{\theta,\piexp}[\Sigma_T] - \frac{1}{T} \Exp_{\thetast,\piexp}[\Sigma_T] \Big \|_\op \le  \| \Gamss_{T,k}(\thetatil,\bmutil,\sigma_u) - \Gamss_{T,k}(\thetastbar,\bmutil,\sigma_u) \|_\op + \Delta(\theta) + \Delta(\thetast)
\end{align*}
Lemma \ref{lem:smooth_covariates} gives that, as long as $\| \theta - \thetast \|_\op \le \betaexplds(\thetastbar)/2$,
\begin{align*}
\| \Gamss_{T,k}(\thetatil,\bmutil,\Lambda_u) - \Gamss_{T,k}(\thetastbar,\bmutil,\Lambda_u) \|_\op \le 2\alphast(\thetastbar,\gamma^2) \| \theta - \thetast \|_\op
\end{align*}
where $\betaexplds(\thetastbar)$ and $\alphast(\thetastbar)$ are defined here as in Lemma \ref{lem:smooth_covariates}. Note here that we use that $\| \theta - \thetast \|_\op$ and $\| \thetatil - \thetastbar \|_\op$ are within a factor of 2 of each other since
\begin{align*}
\max \{ \| A - \Ast \|_\op, \| B - \Bst \|_\op \} \le \| \theta - \thetast \|_\op, \| \thetatil - \thetastbar \|_\op \le \| A - \Ast \|_\op + \| B - \Bst \|_\op
\end{align*}
It remains to simplify $\Delta(\theta)$. By Lemma F.9 of \cite{wagenmaker2020active}, as long as $\| \thetatil - \thetastbar \|_\op \le c/\| \thetastbar \|_{\Hinf}$, we will have that $\| \thetatil \|_{\Hinf}$ and $\| \thetastbar \|_{\Hinf}$ are within a constant factor of each other. Next, note that Lemma \ref{prop:matrix_binomial} implies that, so long as $\| \thetatil - \thetastbar \|_\op \le \epsilon$, $ \| \Atil^k \|_\op \le \tau(\Atilst,\rho) ( \rho + \tau(\Atilst,\rho) \epsilon)^k$. This implies that
$$ \tau(\Atil, \rho + \tau(\Atilst,\rho) \epsilon) = \sup_k \| \Atil^k \|_\op (\rho + \tau(\Atilst,\rho) \epsilon)^{-k} \le \tau(\Atilst,\rho) $$
As long as $\epsilon < (1-\rhool)/(2\tau(\Atilst,\rho))$ we can then choose $\rho = \rhool + \tau(\Atilst,\rhool) \epsilon$ which will allow us to upper bound
$$ \frac{\tau(\Atil,\rho)}{(1-\rho^k)} \le \frac{c \tau(\Atilst,\rhool)}{(1-\rhool^k)} $$
It follows that $\Delta(\theta) \le c \Delta(\thetast)$. By Lemma \ref{lem:hinf_upper_bound} and some algebra, the assumptions then hold with
\begin{align*}
& \lamund = \min \{ \sigma_w^2, \sigma_u^2 \}, \quad \alphast(\thetast;\gamma^2) =   \frac{ c_1 (\sigma_w^2 + \sigma_u^2) \tauol^3}{(1 - \rhool)^2}  +  c_2 \gamma^2  (1 + \| \Bst \|_\op)^3 \| \Ast \|_{\Hinf}^3   , \\
& \betaexplds(\thetast) = c_3 \min \left \{ \frac{1-\rhool}{\tauol}, \frac{1/2}{1 + (1+\| \Bst \|_\op) \| \Ast \|_{\Hinf} }, 1 \right \}, \\
& \Ccov = \frac{c_4 \tauol (1 + \| \Bst \|_\op)^2 \| \Ast \|_{\Hinf}^2 \sqrt{T} k \gamma^2}{1-\rhool^k} + \frac{c_5 \tauol^2 (1 + \| \Bst \|_\op)^2 \| \Ast \|_{\Hinf}^2 k^2 \gamma^2}{(1-\rhool^k)^2}, \quad \alpha = 1, \quad \ccexp = 1
\end{align*}
For the final statement we simplify all expressions involving problem-dependent constants by simply upper bounding them by constants that are $\poly(\| \Ast \|_{\Hinf}, \| \Bst \|_\op)$, and noting that $\sigma_u^2 \le \gamma^2$. 
\end{proof}

\subsection{Proof of Theorem \ref{thm:lds_lb}}\label{sec:proof_lds_lb_opt}

\begin{proof}
The outline of the proof is as follows.
\begin{enumerate}
    \item Apply Theorem \ref{thm:simple_regret_lb2} to show that
    \begin{align*}
        \min_{\frakahat} \max_{\theta \in \mathcal{B}_T}\Exp_{\theta,\piexp}[ \calR(\frakahat_T;\theta)] \ge \min_{\theta:\|\theta - \theta_0\|^2_{F} \le 5 (\dimx^2 + \dimx \dimu) /(\lambda T^{5/6})} \Exp\left[\tr\left(\taskhes(\thetast) \left( \Exp_{\theta,\piexp}[\matSig_{T}] + \lambda T \cdot I \right)^{-1}  \right)\right]
    \end{align*}
    for a particular choice of $\lambda$.
    
    \item Apply Lemma \ref{lem:lds_lb_input_power} to show that, for any policy $\piexp$ and any $\theta$, there exists a \textit{periodic} policy $\piexp'$ such that 
    $$ \Exp_{\theta,\piexp}[\matSig_{T}] \preceq \Exp_{\theta,\piexp'}[\matSig_{c_1T}] + c_2 $$
    
    \item Given that $\piexp'$ is periodic, apply Lemma \ref{lem:lds_lb_cov_ss_exp} to show that we can upper bound the expected covariates by the expected \textit{steady-state} covariates:
    $$ \Exp_{\theta,\piexp'}[\matSig_{c_1T}] \preceq \Exp_{\theta,\piexp'}[\Gamfreq_{c_1 T}(\thetatil,\bmU)] + c_3$$
    
    \item Use the frequency-domain representation to show that there exists a non-random input $\bmU'$ that meets the power constraint and achieves the same steady state covariates:
    \begin{align*}
    \Exp_{\theta,\piexp'}[\Gamfreq_{c_1 T}(\thetatil,\bmU)] = \Gamfreq_{c_1 T}(\thetatil,\bmU')
    \end{align*}
    
    \item Apply the perturbation bound for the steady state covariates given in Lemma \ref{lem:smooth_covariates} to show that, for any $\theta$ in our set, we can upper bound the covariates on $\theta$ by the covariates on $\thetast$:
    $$ \Gamfreq_{c_1 T}(\thetatil,\bmU') \preceq \Gamfreq_{c_1 T}(\thetastbar,\bmU') + c_4 $$
    
    \item Finally, we conclude the proof by optimizing over $\piexp'$ to obtain a lower bound scaling as $\Phiopt(\gamma^2;\thetast)$.
\end{enumerate}

Throughout the proof, we assume expectations are taken with respect to $\theta$ and $\piexp$, and therefore write $\Exp[\cdot]$ in place of $ \Exp_{\theta,\piexp}[\cdot]$. As stated in Section \ref{sec:lds_vec}, linear dynamical systems our simply an instance of vector regression and we can therefore apply the results of Section \ref{sec:general_decision} in this setting. 

\paragraph{Applying Theorem \ref{thm:simple_regret_lb2}:} Define
$$\lamstinf := \limsup_{T' \rightarrow \infty} \max_{\bmU \in \calU_{\gamma^2,T'}} \frac{1}{10T'} \lammin ( \Gamss_{T'}(\thetastbar,\bmU,0)) $$
Under Assumption \ref{asm:lds_lb_suff_exci} and by Lemma \ref{lem:lds_lb_regularizer}, we will have that $\lamstinf > 0$. Then the first conclusion of Theorem \ref{thm:simple_regret_lb2} holds with $\lambda = \lamstinf$. That is, if $T$ is large enough that the burn-in of Theorem \ref{thm:simple_regret_lb2} is met, we will have
\begin{align*}
        \min_{\frakahat} \max_{\theta \in \mathcal{B}_T}\Exp_{\theta,\piexp}[ \calR(\frakahat_T;\theta)] \ge \min_{\theta:\|\theta - \theta_0\|^2_{F} \le 5 (\dimx^2 + \dimx \dimu) /(\lamstinf T^{5/6})} \Exp\left[\tr\left(\taskhes(\thetast) \left( \Exp[\matSig_{T}] + \lamstinf T \cdot I \right)^{-1}  \right)\right]
\end{align*}

\paragraph{Sufficiency of periodic policies:} 
Our goal is to lower bound
$$  \min_{\theta:\|\theta - \theta_0\|^2_{F} \le 5 (\dimx^2 + \dimx \dimu) /(\lamstinf T^{5/6})} \Exp\left[\tr\left(\taskhes(\thetast) \left( \Exp[\matSig_{T}] + \lamstinf T \cdot I \right)^{-1}  \right)\right] $$
Fix some $\theta$ such that $\| \theta_0 - \theta \|_F^2 \le 5 (\dimx^2 + \dimx \dimu)/(\lamstinf T^{5/6})$, and consider the extended system $\thetatil$, as defined in \eqref{eq:lds_extended}. Let $\zu_t$ denote the component of the state of $\thetatil$ driven by both the random and deterministic components of the input and $\zw_t$ the component driven by the process noise. Then $z_t = \zu_t + \zw_t$, so 
$$ \sum_{t=1}^T z_t z_t^\top = \sum_{t=1}^T (\zu_t + \zw_t)(\zu_t + \zw_t)^\top \preceq 2\sum_{t=1}^T \left ( \zu_t (\zu_t)^\top + \zw_t (\zw_t)^\top \right ) $$
Therefore,
$$ \Exp[\Sigma_T] \preceq 2 \Exp[\tsum_{t=1}^T \zu_t (\zu_t)^\top] + 2 \tsum_{t=1}^T \Gamnoise_t(\thetatil,0) $$
By the power constraint on $u_t$ and Lemma \ref{lem:lds_lb_input_power},
$$ \Exp \sum_{t=1}^T \zu_t (\zu_t)^\top \preceq \Exp \sum_{t=1}^{2T+T_\epsilon} z^{\bmutil}_t (z^{\bmutil}_t)^\top + 5 \epsilon I \preceq \Exp \sum_{t=1}^{4T} z^{\bmutil}_t (z^{\bmutil}_t)^\top + 5 \epsilon  I $$
for some input $\bmutil$ with period $\kutil := T_\epsilon = 2H((\dimx^2+\dimx)/2+1)$ satisfying $\Exp[\sum_{t=0}^{\kutil - 1} \util_t^\top \util_t] \le \kutil \gamma^2$, and 
$$H = \left \lceil \log \left (\frac{\epsilon(1-\rho^2)}{8 \tau(\Atil,\rho)^3 \gamma^2 T^2} \right )/\log \rho \right \rceil$$
The final inequality follows by upper bounding $T_\epsilon \le 2T$, which will hold by our definition of $H$ and assumption on the size of $T$. Choosing $\epsilon = \lamstinf T/5$, we can upper bound 
$$ \Exp[\tsum_{t=1}^T \zu_t (\zu_t)^\top] +  \tsum_{t=1}^T \Gamnoise_t(\thetatil,0) \preceq \Exp \sum_{t=1}^{4T} z^{\bmutil}_t (z^{\bmutil}_t)^\top + \tsum_{t=1}^T \Gamnoise_t(\thetatil,0) + \lamstinf T \cdot I$$

\paragraph{From time domain to frequency domain:} The conditions of Lemma \ref{lem:lds_lb_cov_ss_exp} are met for this $\bmutil$, so it follows that 
\begin{align*}
& \Exp \sum_{t=1}^{4T} z^{\bmutil}_t (z^{\bmutil}_t)^\top \preceq \Exp \Gamfreq_{4T}(\thetatil,\bmutil) \\
& \qquad + \left ( \frac{\tau(\Atil,\rho)  (2H((\dimx^2+\dimx)/2+1))  \sqrt{4T+1} }{1-\rho^k} + \frac{\tau(\Atil,\rho)^2 (2H((\dimx^2+\dimx)/2+1))^2 }{(1-\rho^k)^2} \right )\| \thetatil \|_{\Hinf}^2 \gamma^2 \cdot I 
\end{align*}
Then if
\begin{equation}\label{eq:lds_lb_init1}
T \ge \frac{1}{\lamstinf} \left ( \frac{\tau(\Atil,\rho)  (2H((\dimx^2+\dimx)/2+1))  \sqrt{4T+1} }{1-\rho^k} + \frac{\tau(\Atil,\rho)^2 (2H((\dimx^2+\dimx)/2+1))^2 }{(1-\rho^k)^2} \right )\| \thetatil \|_{\Hinf}^2 \gamma^2 
\end{equation}
we can upper bound
$$  \Exp \sum_{t=1}^{4T} z^{\bmutil}_t (z^{\bmutil}_t)^\top + \tsum_{t=1}^T \Gamnoise_t(\thetatil,0) + \lamstinf T \cdot I \preceq  \Exp \Gamfreq_{4T}(\thetatil,\bmutil) + \tsum_{t=1}^T \Gamnoise_t(\thetatil,0) + 2 \lamstinf T \cdot I $$

\paragraph{Sufficiency of deterministic inputs:} Let $\calUtil_{\gamma^2,\kutil}$ denote the set of inputs with average expected power bounded by $\gamma^2$ and period $\kutil$. Then we have shown that
\begin{align*}
\tr & \Big ( \taskhes(\thetast) (\Exp[\matSig_T] + \lamstinf T \cdot I )^{-1} \Big ) \\
& \ge \frac{1}{2} \tr \Big ( \taskhes(\thetast) \left ( \Exp \matGamfreq_{4T}(\thetatil,\bmutil) + \tsum_{t=1}^T \matGamnoise_t(\thetatil,0) + 3 \lamstinf T \cdot I  \right )^{-1} \Big ) \\
& \ge \frac{1}{2} \min_{\bmu \in \calUtil_{\gamma^2,\kutil}} \tr \Big ( \taskhes(\thetast) \left ( \Exp \matGamfreq_{4T}(\thetatil,\bmu) + \tsum_{t=1}^T \matGamnoise_t(\thetatil,0) + 3 \lamstinf T \cdot I  \right )^{-1} \Big )
\end{align*}
By definition of $\Gamfreq$ and for any $\bmu \in \calUtil_{\gamma^2,\kutil}$, using that $\bmucheck = \Fourier(\bmu)$,
\begin{align*}
\Exp \Gamfreq_{4T}(\thetatil,\bmu) & = \Exp \frac{4T}{\kutil} \frac{1}{\kutil} \sum_{t=1}^{\kutil} (e^{\imag \frac{2\pi t}{\kutil}} I - \Atil)^{-1} \Btil \ucheck_t \ucheck_t^{\herm} \Btil^{\herm} (e^{\imag \frac{2\pi t}{\kutil}} I - \Atil)^{-\herm} \\
& =  \frac{4T}{\kutil} \frac{1}{\kutil} \sum_{t=1}^{\kutil} (e^{\imag \frac{2\pi t}{\kutil}} I - \Atil)^{-1} \Btil \Exp[\ucheck_t \ucheck_t^{\herm}] \Btil^{\herm} (e^{\imag \frac{2\pi t}{\kutil}} I - \Atil)^{-\herm} 
\end{align*}
Define $U_t := \Exp[\ucheck_t \ucheck_t^{\herm}]$. By Parseval's Theorem, and the power constraint on $\bmu$, we have
$$\tsum_{t=1}^{\kutil} \tr(U_t) =  \Exp[\tsum_{t=1}^{\kutil} \ucheck_t \ucheck_t^{\herm}] = \Exp[\kutil \tsum_{t=0}^{\kutil} u_t^\top u_t]  \le \kutil^2 \gamma^2 $$
Thus, optimizing over over the (possibly random) input $\bmu$, is equivalent to optimizing over PSD matrices $U_t$ that satisfy this trace constraint. Therefore,
\begin{align*}
&\frac{1}{2} \min_{\bmu \in \calUtil_{\gamma^2,\kutil}} \tr \Big ( \taskhes(\thetast) \left ( \Exp \matGamfreq_{4T}(\thetatil,\bmu) + \tsum_{t=1}^T \matGamnoise_t(\thetatil,0) + 3 \lamstinf T \cdot I  \right )^{-1} \Big ) \\
& \qquad \qquad  \ge \frac{1}{2} \min_{\bmu \in \calU_{\gamma^2,\kutil}} \tr \Big ( \taskhes(\thetast) \left ( \matGamfreq_{4T}(\thetatil,\bmu) + \tsum_{t=1}^T \matGamnoise_t(\thetatil,0) + 3 \lamstinf T \cdot I  \right )^{-1} \Big ) \\
& \qquad \qquad  \ge \frac{1}{8T} \min_{\bmu \in \calU_{\gamma^2,\kutil}} \tr \Big ( \taskhes(\thetast) \left ( \matGamss_{4T}(\thetatil,\bmu,0) + 3 \lamstinf  \cdot I  \right )^{-1} \Big ) \\
& \qquad \qquad  \ge \frac{1}{8T} \min_{\bmU \in \calU_{\gamma^2,4T}} \tr \Big ( \taskhes(\thetast) \left ( \matGamss_{4T}(\thetatil,\bmU,0) + 3 \lamstinf  \cdot I  \right )^{-1} \Big ) 
\end{align*}
where the constraint set in the second minimization is simply the set defined in \eqref{eq:input_set}, and we can thus drop the expectation. 

\paragraph{From $\theta$ to $\thetast$:} By assumption 
$$\| \theta - \thetast \|_\op \le \| \theta - \thetast \|_F \le \| \theta - \theta_0 \|_F + \| \theta_0 - \thetast \|_F \le 2\sqrt{5(\dimx^2 + \dimx \dimu)}/(\sqrt{\lamstinf} T^{5/12})$$
so if
\begin{equation}\label{eq:lds_lb_init1_1}
2\sqrt{5(\dimx^2 + \dimx \dimu)}/(\sqrt{\lamstinf} T^{5/12}) \le \betaexplds(\thetast)
\end{equation}
we are in the domain of Lemma \ref{lem:smooth_covariates} and
\begin{align*}
\|  \matGamss_{4T}(\thetatil,\bmU,0) -  \matGamss_{4T}(\thetastbar,\bmU,0) \|_\op & \le \alphastlds(\thetast,\gamma^2) \cdot \| \thetatil - \thetastbar \|_\op = \alphastlds(\thetast,\gamma^2) \cdot \| \theta - \thetast \|_\op \\
& \le  \frac{\alphastlds(\thetast,\gamma^2) 2\sqrt{5(\dimx^2 + \dimx \dimu)}}{\sqrt{\lamstinf} T^{5/12}}
\end{align*}
It follows that as long as 
\begin{equation}\label{eq:lds_lb_init2}
 \frac{\alphastlds(\thetast,\gamma^2) 2\sqrt{5(\dimx^2 + \dimx \dimu)}}{\sqrt{\lamstinf} T^{5/12}} \le \lamstinf  
 \end{equation}
then
$$ \matGamss_{4T}(\thetatil,\bmU,0) \preceq \matGamss_{4T}(\thetastbar,\bmU,0) + \lamstinf  \cdot I $$
and thus,
\begin{align*}
& \frac{1}{8T} \min_{\bmU \in \calU_{\gamma^2,4T}} \tr \Big ( \taskhes(\thetast) \left ( \matGamss_{4T}(\thetatil,\bmU,0) + 3 \lamstinf  \cdot I  \right )^{-1} \Big ) \\
& \qquad  \ge \frac{1}{8T} \min_{\bmU \in \calU_{\gamma^2,4T}} \tr \Big ( \taskhes(\thetast) \left ( \matGamss_{4T}(\thetastbar,\bmU,0) + 4 \lamstinf  \cdot I  \right )^{-1} \Big )
\end{align*}

\paragraph{Concluding the lower bound:} Next, by Lemma \ref{lem:opt_k_large_approx}, so long as
\begin{equation}\label{eq:lds_lb_init3}
4T \geq \max \left \{ \frac{8 \pi \| \thetastbar \|_{\Hinf} \gamma^2}{ \lamstinf }, \frac{\pi}{2 \| \thetastbar \|_{\Hinf}} \right \} \left ( \max_{\omega \in [0,2\pi]} \| (e^{\imag\omega} I - \Atilst)^{-2} \Btilst \|_\op \right )
\end{equation}
then for any $T' \ge 4T$ and $\bmU^\star \in \calU_{\gamma^2,4T}$, there exists a $\bmU' \in \calU_{\gamma^2,T'}$ such that
$$ \left \| \frac{1}{4T} \Gamfreq_{4T}(\thetastbar,\bmU') - \frac{1}{T'} \Gamfreq_{T'}(\thetastbar,\bmU^\star) \right \|_\op \le \frac{1}{2}  \lamstinf $$
Furthermore, by Lemma \ref{lem:opt_k_large_approx_noise}, if
\begin{equation}\label{eq:lds_lb_init4}
4T \ge \max \left \{  \frac{16 \tau(\Atilst,\rho)^2 (\sigmaw^2 + \gamma^2/\dimu ) }{(1-\rho^2)^2  \lamstinf}  , \log \left (  \frac{(1-\rho^2)^2  \lamstinf }{16\tau(\Atilst,\rho)^2 (\sigmaw^2 + \gamma^2/\dimu ) }\right ) \frac{1}{2 \log \rho} \right \}
\end{equation}
then, for any $T' \ge 4T$,
$$ \left \| \frac{1}{4T} \sum_{t=1}^{4T} \Gamnoise_t(\thetastbar,\gamma/\sqrt{\dimu}) - \frac{1}{T'} \sum_{t=1}^{T'} \Gamnoise_t(\thetastbar,\gamma/\sqrt{\dimu}) \right \|_\op \le \frac{1}{2}  \lamstinf $$
By what we've just shown, for any $T' \ge 4T$
\begin{align*}
 \matGamss_{4T}(\thetastbar,\bmU',0) & \preceq \matGamss_{T'}(\thetastbar,\bmU^\star,0) + \lamstinf \cdot I  
\end{align*}
Thus,
\begin{align*}
& \frac{1}{8T} \min_{\bmU \in \calU_{\gamma^2,4T}} \tr \left ( \taskhes(\thetast) \Big (  \matGamss_{4T}(\thetastbar,\bmU,0)  + 4 \lamstinf  \cdot I \Big )^{-1} \right ) \\
 & \qquad \qquad  \ge \liminf_{T' \rightarrow \infty} \frac{1}{8T} \min_{\bmU \in \calU_{\gamma^2,T'}} \tr \left ( \taskhes(\thetast) \Big (  \matGamss_{T'}(\thetastbar,\bmU,0)  + 5 \lamstinf  \cdot I \Big )^{-1} \right ) 
\end{align*}
Let
$$ \lamst := \inf_{T' \ge T} \max_{\bmU \in \calU_{\gamma^2,T'}} \frac{1}{10T'} \lammin ( \Gamfreq_{T'}(\thetastbar,\bmU)) $$
Note that by Lemma \ref{lem:opt_k_large_approx}, so long as
\begin{equation}\label{eq:lds_lb_init5}
 T \geq \max \left \{ \frac{8 \pi \| \thetastbar \|_{\Hinf} \gamma^2}{ \lamstinf }, \frac{\pi}{2 \| \thetastbar \|_{\Hinf}} \right \} \left ( \max_{\omega \in [0,2\pi]} \| (e^{\imag \omega} I - \Atilst)^{-2} \Btilst \|_\op \right )
 \end{equation}
then for any $T' \ge T$ and $\bmU^\star \in \calU_{\gamma^2,T}$, there exists a $\bmU' \in \calU_{\gamma^2,T'}$ such that
$$ \left \| \frac{1}{T} \Gamfreq_{T}(\thetastbar,\bmU') - \frac{1}{T'} \Gamfreq_{T'}(\thetastbar,\bmU^\star) \right \|_\op \le \frac{1}{2}  \lamstinf $$
This implies that so long as $T$ satisfies \eqref{eq:lds_lb_init5}, we will have $\lamst \ge \frac{1}{2} \lamstinf $. By definition of $\lamst$, for any $T' \ge T$ there exists some input $\bmU'' \in \calU_{\gamma^2,T'}$ such that $ \lammin(\Gamss_{T'}(\thetastbar,\bmU'',0)) \ge 10 \lamst$. It follows that for any $T' \ge T$, 
\begin{align*}
\min_{\bmU \in \calU_{\gamma^2,T'}} & \tr \left ( \taskhes(\thetast) \Big (  \matGamss_{T'}(\thetastbar,\bmU,0)  + 5 \lamstinf  \cdot I \Big )^{-1} \right ) \\
& \ge \min_{\bmU \in \calU_{\gamma^2,T'}} \tr \left ( \taskhes(\thetast) \Big (  \matGamss_{T'}(\thetastbar,\bmU,0)  + 10 \lamst  \cdot I \Big )^{-1} \right ) \\
  & \ge \min_{\bmU \in \calU_{\gamma^2,T'}} \tr \left ( \taskhes(\thetast) \Big (  \matGamss_{T'}(\thetastbar,\bmU,0)  +  \matGamss_{T'}(\thetastbar,\bmU'',0) \Big )^{-1} \right ) \\
& \ge \min_{\bmU \in \calU_{2 \gamma^2,T'}} \tr \left ( \taskhes(\thetast) \Big (  \matGamss_{T'}(\thetastbar,\bmU,0)  \Big )^{-1} \right ) \\
& \ge \frac{1}{2} \min_{\bmU \in \calU_{ \gamma^2,T'}} \tr \left ( \taskhes(\thetast) \Big (  \matGamss_{T'}(\thetastbar,\bmU,0)  \Big )^{-1} \right ) 
\end{align*}
This implies that
\begin{align*}
& \liminf_{T' \rightarrow \infty} \frac{1}{8T} \min_{\bmU \in \calU_{\gamma^2,T'}} \tr \left ( \taskhes(\thetast) \Big (  \matGamss_{T'}(\thetastbar,\bmU,0)  + 5 \lamstinf  \cdot I \Big )^{-1} \right ) \\
& \qquad \qquad \ge \liminf_{T' \rightarrow \infty} \frac{1}{16T} \min_{\bmU \in \calU_{\gamma^2,T'}} \tr \left ( \taskhes(\thetast) \Big (  \matGamss_{T'}(\thetastbar,\bmU,0)  \Big )^{-1} \right )  = \frac{1}{16T} \Phiss(\gamma^2;\thetast).
\end{align*}
Putting everything together, Theorem \ref{thm:simple_regret_lb2} and what we have shown imply that as long as $T$ is large enough so that the burn-in of Theorem \ref{thm:simple_regret_lb2} is met, $T \ge H((\dimx^2+\dimx)/2+1)$ and \eqref{eq:lds_lb_init1}, \eqref{eq:lds_lb_init1_1}, \eqref{eq:lds_lb_init2}, \eqref{eq:lds_lb_init3}, \eqref{eq:lds_lb_init4}, and \eqref{eq:lds_lb_init5} hold, we will have
\begin{align*}
& \min_{\frakahat} \max_{\theta : \| \theta - \theta_0 \|_{2}^2 \leq 5 (\dimx^2+\dimx \dimu) /(\lamstinf T^{5/6})} \Exp [\calR(\frakahat; \theta)]   \geq  \frac{\sigmaw^2}{16T} \Phiss(\gamma^2;\thetast)  - \frac{ C_1  }{(\lamstinf T)^{5/4}}  
\end{align*}
where $ C_2 = \calO \Big (  (  L_{\fraka1}  L_{\fraka2} L_{\calR2} + L_{\fraka1}^3 L_{\calR 3} +  \Lra ) (\dimx^2+\dimx \dimu)^{3/2} + L_{\fraka 1}^2 L_{\calR 2} \Big ) $. Finally we can lower bound $\Phiss(\gamma^2;\thetast)$ with $\Phiopt(\gamma^2;\thetast)/4$ by \Cref{prop:phiss_phiopt}.

\paragraph{Simplifying the Burn-In Time:} It remains to simplify the bound. First, note that by Lemma F.9 of \cite{wagenmaker2020active}, as long as $\| \thetatil - \thetastbar \|_\op \le c/\| \thetastbar \|_{\Hinf}$, we will have that $\| \thetatil \|_{\Hinf}$ and $\| \thetastbar \|_{\Hinf}$ are within a constant factor of each other. Next, note that Lemma \ref{prop:matrix_binomial} implies that, so long as $\| \thetatil - \thetastbar \|_\op \le \epsilon$, $ \| \Atil^k \|_\op \le \tau(\Atilst,\rho) ( \rho + \tau(\Atilst,\rho) \epsilon)^k$. This implies that
$$ \tau(\Atil, \rho + \tau(\Atilst,\rho) \epsilon) = \sup_k \| \Atil^k \|_\op (\rho + \tau(\Atilst,\rho) \epsilon)^{-k} \le \tau(\Atilst,\rho) $$
As long as $\epsilon < (1-\rhool)/(2\tau(\Atilst,\rho))$ we can then choose $\rho = \rhool + \tau(\Atilst,\rhool) \epsilon$ which will allow us to upper bound
$$ \frac{\tau(\Atil,\rho)^n}{(1-\rho^m)^p} \le \frac{c \tau(\Atilst,\rhool)^n}{(1-\rhool^m)^p} $$
As we have assumed $\| \theta - \theta_0 \|_F, \| \thetast - \theta_0 \|_F \le \sqrt{5(\dimx^2 + \dimx \dimu)}/(\sqrt{\lamstinf} T^{5/12})$, we can upper bound $ \| \thetatil - \thetastbar \|_\op \le \| \theta - \thetast \|_F \le 2\sqrt{5(\dimx^2 + \dimx \dimu)}/(\sqrt{\lamstinf} T^{5/12})$. Some algebra, Lemma \ref{lem:tau_aug}, and the definition of $\alphastlds(\thetast, \gamma^2)$ and $\betaexplds(\thetast)$ then gives that as long as
$$ T \ge \poly \left ( \frac{1}{1-\rhool}, \tauol, \| \Bst \|_\op, \dimx, \dimu, \| \thetastbar \|_{\Hinf}, \gamma^2, \sigmaw^2, \frac{1}{\lamstinf}, \log T \right )$$
these bounds on $\| \thetatil \|_{\Hinf}$ and $\tau(\Atil,\rho)$ will hold, $T \ge H((\dimx^2+\dimx)/2+1)$ and \eqref{eq:lds_lb_init1}, \eqref{eq:lds_lb_init1_1}, \eqref{eq:lds_lb_init2}, \eqref{eq:lds_lb_init3}, \eqref{eq:lds_lb_init4}, and \eqref{eq:lds_lb_init5} hold. Finally, we use Lemma \ref{lem:lds_lb_regularizer} to replace $\lamstinf$ with $\lamnoise$, and Lemma \ref{lem:hinf_upper_bound} to upper bound $\tauol, \frac{1}{1-\rhool}$ and $\| \thetastbar \|_{\Hinf}$ by $\poly(\| \Bst \|_\op, \| \Ast \|_{\Hinf})$.
\end{proof}

\begin{lem}\label{lem:lds_lb_regularizer}
$$ \min \left \{ \lammin \left ( \sigmaw^2 \sum_{t=0}^{\dimx} \Ast^t (\Ast^t)^\top + \sigma_u^2 \sum_{t=0}^{\dimx} \Ast^t \Bst \Bst^\top (\Ast^t)^\top \right ), \sigma_u^2 \right \} \le \limsup_{T \rightarrow \infty} \sup_{\bmU \in \calU_{\gamma^2,T}} 2 \lammin(\Gamss_T(\thetastbar,\bmU)) $$
\end{lem}
\begin{proof}
Fix $T$ and consider playing the input $u_t \sim \cN(0,\sigma_u^2 \cdot I)$. By definition,
$$ \sum_{t=1}^T \Gamnoise_t(\thetastbar,\sigma_u) = \sum_{t=1}^T \Gamnoise_t(\thetastbar,0) + \Exp \left [ \sum_{t=1}^T \xu_t (\xu_t)^\top \right ]$$
By Lemma \ref{lem:lds_lb_input_power}, it follows that there exists some input $\bmutil = ( \util_t )_{t=0}^{k-1}$ with average expected power bounded by $\gamma^2$ such that  
$$\Exp \left [ \sum_{t=1}^T \xu_t (\xu_t)^\top \right ] \preceq \Exp \left [ \sum_{t=1}^{2T + k} x^{\bmutil}_t (x^{\bmutil}_t)^\top \right ] + 5 I$$
where $k := T_1 = 2H((\dimx^2+\dimx)/2+1)$ and
$$H = \left \lceil \log \left (\frac{1-\rho^2}{8 \tau(\Atilst,\rho)^3 \gamma^2 T^2} \right )/\log \rho \right \rceil = \calO(\log T)$$
By Lemma \ref{lem:lds_lb_cov_ss_exp},
\begin{align*}
\Exp \left [ \sum_{t=1}^{2T+k} x_t^{\bmutil} (x_t^{\bmutil})^\top \right ] & \preceq \Exp \Gamfreq_{2T+k}(\thetastbar,\bmutil)   + \left ( \frac{\tau(\Atilst,\rho) \| \thetastbar \|_{\Hinf}^2 \sqrt{2T+k+1} k \gamma^2}{1-\rho^k} + \frac{\tau(\Atilst,\rho)^2 \| \thetastbar \|_{\Hinf}^2 k^2 \gamma^2}{(1-\rho^k)^2} \right ) \cdot I 
\end{align*}
By definition of $\Gamfreq$ and for any $\bmU \in \calUtil_{\gamma^2,k}$,
\begin{align*}
\Exp \Gamfreq_{2T+k}(\thetatil,\bmU) & = \Exp \frac{4T}{k} \frac{1}{k} \sum_{t=1}^{k} (e^{\imag \frac{2\pi t}{k}} I - \Atil)^{-1} \Btil \ucheck_t \ucheck_t^\herm \Btil^\herm (e^{\imag \frac{2\pi t}{k}} I - \Atil)^{-\herm} \\
& =  \frac{4T}{k} \frac{1}{k} \sum_{t=1}^{k} (e^{\imag \frac{2\pi t}{k}} I - \Atil)^{-1} \Btil \Exp[\ucheck_t \ucheck_t^\herm ] \Btil^\herm (e^{\imag \frac{2\pi t}{k}} I - \Atil)^{-\herm} 
\end{align*}
Define $U_t := \Exp[\ucheck_t \ucheck_t^\herm ]$. By Parseval's Theorem, and the power constraint on $\util$, we have
$$\tsum_{t=1}^{k} \tr(U_t) =  \Exp[\tsum_{t=1}^{k} \ucheck_t^\herm \ucheck_t ] = \Exp[k \tsum_{t=0}^{k-1} u_t^\top u_t]  \le k^2 \gamma^2 $$
\newcommand{\bmUtil}{\wt{\bmU}}
It follows that there exists some $\bmUtil \in \calU_{\gamma^2,k}$ such that $\Exp \Gamfreq_{2T+k}(\thetastbar,\bmutil) = \Gamfreq_{2T+k}(\thetastbar,\bmUtil)$. Putting this together, we have that
\begin{align*}
    \lammin \left ( \sum_{t=1}^T \Gamnoise_t(\thetastbar,\gamma/\sqrt{\dimu}) \right ) & \le \lammin \left ( (2T+k) \Gamss_{2T+k}(\thetastbar,\bmUtil) \right ) + \calO(\sqrt{T} \log T + \poly \log T ) \\
    & \le \sup_{\bmU \in \calU_{\gamma^2, 2T+k}} \lammin \left ( (2T+k) \Gamss_{2T+k}(\thetastbar,\bmU) \right ) + \calO(\sqrt{T} \log T + \poly \log T )
\end{align*}
Dividing through by $T$ and taking the $\limsup_{T\rightarrow \infty}$, we have
$$ \limsup_{T \rightarrow \infty} \frac{1}{T} \lammin \left ( \sum_{t=1}^T \Gamnoise_t(\thetastbar,\sigma_u) \right ) \le \limsup_{T \rightarrow \infty} \sup_{\bmU \in \calU_{\gamma^2,T}} 2 \lammin(\Gamss_T(\thetastbar,\bmU)) $$
Finally, we see that by definition and some algebra that  
\begin{align*}
    \limsup_{T \rightarrow \infty} \frac{1}{T} \lammin \left ( \sum_{t=1}^T \Gamnoise_t(\thetastbar,\sigma_u) \right ) & = \lammin \left ( \sigmaw^2 \sum_{t=0}^\infty \Atilst^t (\Atilst^t)^\top + \sigma_u^2 \sum_{t=0}^\infty \Atilst^t \Btilst \Btilst^\top (\Atilst^t)^\top \right ) \\
& = \min \left \{ \lammin \left ( \sigmaw^2 \sum_{t=0}^\infty \Ast^t (\Ast^t)^\top + \sigma_u^2 \sum_{t=0}^\infty \Ast^t \Bst \Bst^\top (\Ast^t)^\top \right ), \sigma_u^2 \right \}
\end{align*} 
Noting that
\begin{align*}
    \lammin \left ( \sigmaw^2 \sum_{t=0}^\infty \Ast^t (\Ast^t)^\top + \sigma_u^2 \sum_{t=0}^\infty \Ast^t \Bst \Bst^\top (\Ast^t)^\top \right ) \ge \lammin \left ( \sigmaw^2 \sum_{t=0}^{\dimx} \Ast^t (\Ast^t)^\top + \sigma_u^2 \sum_{t=0}^{\dimx} \Ast^t \Bst \Bst^\top (\Ast^t)^\top \right )
\end{align*}
completes the proof.
\end{proof}

\subsection{Periodicity of Optimal Inputs}
In what follows, consider an arbitrary system $(A,B)$, with $A$ stable. Let $\rho \ge \rho(A)$ be less than $1$, and recall $\tau(A,\rho) := \sup_{n \ge 0} \rho^{-n} \|A^n\|_{\op}$, finally for any error parameter $\epsilon > 0$,
\begin{align}
H_{\epsilon} := \left \lceil \log \left (\frac{\epsilon(1-\rho^2)}{8 \| B \|_\op^2 \tau(A,\rho)^3 \gamma^2 T^2} \right )/\log \rho \right \rceil.
\end{align}
and define the effective time horizon
\begin{align}
T_{\epsilon} := 2H_{\epsilon}((\dimx^2 + \dimx)/2+1).
\end{align}

\begin{lem}\label{lem:lds_lb_input_power}
Consider some input $\{ u_t \}_{t=0}^{T-1}$ satisfying $\Exp[\sum_{t=0}^{T-1} u_t^\top u_t] \le T \gamma^2$. Then there exists an input $\{ \util_t \}_{t=0}^{T_{\epsilon}-1}$ such that 
\begin{align*}
\Exp\left[\sum_{t=0}^{T_{\epsilon}-1} \util_t^\top \util_t\right] \le T_{\epsilon} \gamma^2
\end{align*}
and extending to times $t \ge T_{\epsilon} -1$ via a periodic signal $\util_t = \util_{\mathrm{mod}(t,T_{\epsilon})}$, where equality here holds almost surely, satisfies
\begin{align*}
\Exp \bigg [ \sum_{t=1}^{T} \xu_t (\xu_t)^\top \bigg ] \preceq  \Exp  \bigg [ \sum_{t=1}^{2T+3T_{\epsilon}/2} x^{\bmutil}_t (x^{\bmutil}_t)^\top \bigg ] + 5 \epsilon I,
\end{align*}
 where above, $\xu$ are the states under the initial inputs $(u_t)$, and $x^{\bmutil}$ are the iterates under $(\tilde{u}_t)$, and where  we take $\xu_0 = 0$ in both.
\end{lem}
\begin{proof}
In what follows, we regard $\epsilon > 0$ as fixed, and write $H \gets H_{\epsilon}$. 

We consider the response on the system with no process noise starting from $\xu_0 = 0$. Given some input $\{ u_t \}_{t=0}^{T-1} \in \inputset$, the state evolves as
\begin{align*}
 \xu_t = \sum_{s=0}^{t-1} A^{t-s-1} B u_s 
 \end{align*}
We will use $G_t$ to denote the Markov parameters for this system:
\begin{align*}
 G_t := [B, AB, \ldots, A^{t-2} B, A^{t-1} B]
 \end{align*}
and will define the extended input, $\fraku_t$, and truncated extended input, $\fraku_{t;H}$ as:
\begin{align*}
\fraku_t := [u_t; u_{t-1}; \ldots; u_1; u_0], \quad \fraku_{t;H} := [u_t; u_{t-1}; \ldots ; u_{t-H+2}; u_{t-H+1}] 
\end{align*}
If $t < H - 1$, we define $u_{-s} = 0$ for all $s > 0$. Then the state can be written as $\xu_t = G_t \fraku_{t-1}$. We can approximate the state using the last $H$ inputs as $\xu_{t;H} = G_H \fraku_{t-1;H}$. The following result bounds the error in such an approximation.

\begin{claim}\label{lem:lds_lb_truncation_error}
Fix some input $u_t$ with $\Exp[\sum_{t=0}^{T-1} u_t^\top u_t] \le T \gamma^2$ and assume we start from $\xu_0 = 0$. Then, under our choice of $H \ge \log (\frac{\epsilon(1-\rho^2)}{2 \| B \|_\op^2 \tau(A,\rho)^3 \gamma^2 T})/\log \rho$, we will have
\begin{align*} \Exp \| \xu_t (\xu_t)^\top - \xu_{t;H} (\xu_{t;H})^\top \|_\op \le \epsilon
\end{align*}
\end{claim}
\begin{proof}[Proof of \Cref{lem:lds_lb_truncation_error}]
We first bound the state difference: 
\begin{align*}
\| \xu_{t;H} - \xu_t \|_2 & = \Big \| \sum_{s=0}^{t-H-1} A^{t-s-1} B u_s \Big \|_2 \le \| A^H \|_\op \Big \| \sum_{s=0}^{t-H-1} A^{t-H-s-1} B u_s \Big \|_2  \le \tau(A,\rho) \rho^H \| \xu_{t-H} \|_2
\end{align*}
By Jensen's inequality, we can bound $\Exp \| \xu_t \|_2$ as
\begin{align*}
\Exp \| \xu_t \|_2 & = \Exp \left \| \sum_{s = 0}^{t-1} A^{t-s-1} B u_s \right \|_2   \le \| B \|_\op \tau(A,\rho) \Exp \left ( \sum_{s = 0}^{t-1} \rho^{t-s-1} \| u_s \|_2 \right ) \\
& \le \| B \|_\op \tau(A,\rho)  \sqrt{ \sum_{s = 0}^{t-1} \rho^{2(t-s-1)} } \Exp \sqrt{  \sum_{s=0}^{t-1} \| u_s \|_2^2} \\
& \le \| B \|_\op \tau(A,\rho)  \sqrt{ \sum_{s = 0}^{t-1} \rho^{2(t-s-1)} }  \sqrt{  \Exp \sum_{s=0}^{T-1} \| u_s \|_2^2} \\
& \le \frac{\| B \|_\op \tau(A,\rho) \gamma \sqrt{T}}{\sqrt{1 - \rho^2}}
\end{align*}
Thus, $ \Exp \| \xu_{t;H} - \xu_t \|_2 \le \frac{\| B \|_\op \tau(A,\rho)^2 \gamma \sqrt{2T}}{\sqrt{1-\rho^2}} \rho^H $. Thus, by the triangle inequality and what we have just shown, 
$$ \Exp \| \xu_t (\xu_t)^\top - \xu_{t;H} (\xu_{t;H})^\top \|_\op \le \Exp[(\| \xu_t \|_2 + \| \xu_{t;H} \|_2)\| \xu_t - \xu_{t;H} \|_2] \le \frac{2 \| B \|_\op^2 \tau(A,\rho)^3 \gamma^2 T}{1-\rho^2} \rho^H $$
where the last inequality follows by noting that our above argument also applies to bounding $ \Exp \| \xu_{t;H} \|_2$. The conclusion follows by some algebra. 
\end{proof}

Fix $H = \lceil \log (\frac{\epsilon(1-\rho^2)}{8 \| B \|_\op^2 \tau(A,\rho)^3 \gamma^2 T^2})/\log \rho \rceil$, then, by \Cref{lem:lds_lb_truncation_error},
\begin{align*}
\Exp \sum_{t=1}^T \xu_t (\xu_t)^\top \preceq \Exp \sum_{j=1}^{\lceil T/H \rceil} \sum_{t=H(j-1)+1}^{Hj} \xu_t (\xu_t)^\top \preceq   \Exp \sum_{j=1}^{\lceil T/H \rceil} \sum_{t=H(j-1)+1}^{Hj} \xu_{t;H} (\xu_{t;H})^\top  + \epsilon I 
\end{align*}
Now consider a  realization of $(u_t)$ in our probability space.
\newcommand{\outerU}{\bm{\mathsf{U}}}
\newcommand{\outerUtil}{\tilde{\outerU}}

 Defining
\begin{equation}\label{eq:mat_inputs}
\outerU_{j;H} := \sum_{t=H(j-1)+1}^{Hj} \fraku_{t-1;H} \fraku_{t-1;H}^\top 
\end{equation}
we can rewrite the covariates in terms of the Markov parameters as 

\begin{align*}
\sum_{j=1}^{\lceil T/H \rceil} \sum_{t=H(j-1)+1}^{Hj} \xu_{t;H} (\xu_{t;H})^\top = \sum_{j=1}^{\lceil T/H \rceil}  G_H \left ( \sum_{t=H(j-1)+1}^{Hj} \fraku_{t-1;H} \fraku_{t-1;H}^\top \right ) G_H^\top = \sum_{j=1}^{\lceil T/H \rceil}  G_H \outerU_{j;H} G_H^\top 
\end{align*}
We will define the set of normalized covariance matrices as
\begin{align*}
\calM_H := \left \{ G_H \outerU G_H^\top \ : \ \outerU \text{ of form \eqref{eq:mat_inputs} for input } \{ \bar{u}_t \}_{t=1}^{2H}, \sum_{t=1}^{2H} \bar{u}_t^\top \bar{u}_t = 1 \right \} 
\end{align*}
The following result will allow us to express this in a more convenient form.

\begin{claim}\label{lem:lds_lb_caratheodory}
Consider any $n$ and some $\{ \outerU_{j;H}' \}_{j=1}^n$, $\outerU_{j;H}' \in \calM_H$. Let $p_j \in [0,1], \sum_{j=1}^n p_j = 1$. Then there exists some set of matrix inputs $\{ \outerU_{j;H}'' \}_{j=1}^{(\dimx^2+\dimx)/2 + 1}$, $\outerU_{j;H}'' \in \calM_H$, and some set of weights $q_j \in [0,1], \sum_{j=1}^{(\dimx^2+\dimx)/2 + 1} q_j = 1$ such that
\begin{align*}
 \sum_{j=1}^n p_j G_H \outerU_{j;H}' G_H^\top = \sum_{j=1}^{\dimu (\dimu + 1)/2+1} q_j G_H \outerU_{j;H}'' G_H^\top. 
 \end{align*}
\end{claim}
\begin{proof}
    This is a direct consequence of Caratheodory's Theorem. By definition, $\calM \subseteq \calS_{+}^{\dimx}$. The dimension of $\calS_{+}^{\dimx}$ is $(\dimx^2+\dimx)/2$ so the points in $\calM$ can be thought of as living in a $(\dimx^2+\dimx)/2$-dimensional space. Caratheodory's Theorem then gives that, for any point, $x$, that is a convex combination of elements of $\calM$, $x$ can also be written as a convex combination of at most $\dim(\calM)+1$ points in $\calM$. Taking $x = \sum_{j=1}^n p_j G_H \outerU_{j;H}' G_H^\top$, it follows that there exists $(\dimx^2+\dimx)/2 + 1$ points $\outerU_{j;H}'' \in \calM$ and set of weights $q_j \in [0,1], \sum_{j=1}^m q_j = 1$ such that $x = \sum_{j=1}^{(\dimx^2+\dimx)/2 + 1} q_j G_H \outerU_{j;H}'' G_H^\top$.
\end{proof}

We shall use the following definition going forward:
\begin{defn}
For a given $\outerU_{j;H}$, let $\gamma^2[\outerU_{j;H}]$  denote the power of the input corresponding to $\outerU_{j;H}$. That is, if $\outerU_{j;H}$ is formed according to \eqref{eq:mat_inputs}, 
\begin{align}
\gamma^2[\outerU_{j;H}] := \sum_{t=H(j-2)+1}^{Hj-1} u_t^\top u_t
\end{align}
\end{defn}
 Note then that $\outerU_{j;H} = \gamma^2[\outerU_{j;H}] \cdot \outerU_{j;H}'$ for some $\outerU_{j;H}' \in \calM_H$. 

Instantiating  \Cref{lem:lds_lb_caratheodory} with 
\begin{align*}
\outerU_{j;H}' \gets \outerU_{j;H}/\gamma^2[\outerU_{j;H}], \quad \text{and} \quad p_j = \gamma^2[\outerU_{j;H}] / (\sum_{i=1}^{\lceil T/H \rceil} \gamma^2[\outerU_{i;H}]),
\end{align*} 
we have that there exists some set of matrices $\{ \outerUtil_{j;H} \}_{j=1}^{(\dimx^2+\dimx)/2+1} \subseteq \calM_H$ and some set of weights $q_j$ such that
\begin{align}
\sum_{j=1}^{\lceil T/H \rceil}  G_H \outerU_{j;H} G_H^\top = \left ( \sum_{i=1}^{\lceil T/H \rceil} \gamma^2[\outerU_{i;H}] \right ) \sum_{j=1}^{(\dimx^2+\dimx)/2+1} q_j G_H \outerUtil_{j;H} G_H^\top  \label{eq:q_j}
\end{align}

For future reference, we denote 
\begin{align*}\gamtil := \sum_{i=1}^{\lceil T/H \rceil} \gamma^2[\outerU_{i;H}]
\end{align*}
Note that while $\outerUtil_{j;H} \in \calM_H$, the associate covariates may not be realizable in  only $H((\dimx^2+\dimx)/2+1)$ steps, because $\bar{u}_t$ for a given $t$ will be present in both blocks $\outerU_{j;H}$ and $\outerU_{j+1;H}$---these blocks cannot be chosen independently. However, this response can be realized in $2H((\dimx^2+\dimx)/2+1)$ steps, which we recall is precisely our definition of $T_{\epsilon}$. 

For a given $j \in \{ 1, \ldots, (\dimx^2 + \dimx)/2+1\}$, let $\{ \util_{t;j} \}_{t=0}^{2H-1}$ be the set of inputs for which \eqref{eq:mat_inputs} is satisfied for $\outerUtil_{j;H}$, and such that $\sum_{t=0}^{2H-1} \util_{t;j}^\top \util_{t;j} = 1$. Let $\{ \util_t \}_{t=0}^{T_{\epsilon}}$ denote the sequence of inputs formed by concatenating $\{ \sqrt{\frac{1}{T} \gamtil q_j T_\epsilon }  \util_{t;j} \}_{t=0}^{2H-1}$ for all $j$. That is, set
\begin{align*}
\util_t = \sqrt{\tfrac{1}{T} \gamtil q_j T_\epsilon} \util_{t',j'} \quad \text{where} \quad j' = \lfloor t/j \rfloor + 1, t' = t - (j' - 1) 2H
\end{align*}
Finally, extend $\util_t$ to all $t$ via $\util_t = \util_{\mathrm{mod}(t,T_{\epsilon})}$, where the equality holds almost surely. Then, 
\begin{align}
\frac{1}{T} \frac{\gamtil q_j T_{\epsilon}}{2}  G_H \outerUtil_{j;H} G_H^\top = \sum_{t=2Hj-H+1}^{2Hj } x_{t;H}^{\bmutil} (x_{t;H}^{\bmutil})^\top
\end{align}
 so $\sum_{j=1}^{(\dimx^2+\dimx)/2+1} \frac{T_{\epsilon}}{2T } \gamtil q_j   G_H \outerUtil_{j;H} G_H^\top$
 corresponds to half of the input response, and thus
\begin{align*}
\gamtil \sum_{j=1}^{(\dimx^2+\dimx)/2+1} q_j G_H \outerUtil_{j;H} G_H^\top & =  \frac{T}{T_\epsilon/2}  \sum_{j=1}^{(\dimx^2+\dimx)/2+1} \sum_{t=2Hj-H+1}^{2Hj } x_{t;H}^{\bmutil} (x_{t;H}^{\bmutil})^\top \\
& \preceq \left \lceil  \frac{T}{T_{\epsilon}/2} \right \rceil  \sum_{j=1}^{(\dimx^2+\dimx)/2+1} \sum_{t=2Hj-H+1}^{2Hj } x_{t;H}^{\bmutil} (x_{t;H}^{\bmutil})^\top \\
& \overset{(a)}{=}   \sum_{j=1}^{\left \lceil  \frac{T}{T_\epsilon/2} \right \rceil((\dimx^2+\dimx)/2+1)} \sum_{t=2Hj-H+1}^{2Hj } x_{t;H}^{\bmutil} (x_{t;H}^{\bmutil})^\top \\
& \overset{(b)}{\preceq} \sum_{t=1}^{2H\left \lceil  \frac{T}{T_\epsilon/2} \right \rceil((\dimx^2+\dimx)/2+1)} x_{t;H}^{\bmutil} (x_{t;H}^{\bmutil})^\top \\
& \preceq  \sum_{t=1}^{2T + T_\epsilon} x_{t;H}^{\bmutil} (x_{t;H}^{\bmutil})^\top
\end{align*}
where $(a)$ holds because, by construction, we will have $x_{t;H}^{\bmutil} = x_{t+j2 H ((\dimx^2+\dimx)/2+1);H}^{\bmutil}$ for any $j$, since the input is $T_\epsilon$-periodic, and $(b)$ follows as we are simply including more PSD terms in the sum. As this holds pointwise in our probability space, it follows that
\begin{align*}
\Exp \sum_{j=1}^{\lceil T/H \rceil} \sum_{t=H(j-1)+1}^{Hj} \xu_{t;H} (\xu_{t;H})^\top = \Exp \gamtil \sum_{j=1}^{(\dimx^2+\dimx)/2+1} q_j G_H \outerUtil_{j;H} G_H^\top & \preceq \Exp  \sum_{t=1}^{2T + T_\epsilon} x_{t;H}^{\bmutil} (x_{t;H}^{\bmutil})^\top
\end{align*}
The input sequence  $\{ \util_t \}_{t=0}^{T_\epsilon}$ satisfies
\begin{align*}
\Exp \sum_{t=0}^{T_\epsilon - 1}  \util_t^\top \util_t & =  \Exp \frac{T_\epsilon/2 }{T} \cdot \gamtil \cdot \sum_{j=1}^{(\dimx^2+\dimx)/2+1}  q_j \sum_{t=0}^{2H-1} \util_{t;j}^\top \util_{t;j} \\
& \overset{(a)}{=} \Exp \frac{T_\epsilon/2}{T}\cdot \gamtil \cdot \sum_{j=1}^{(\dimx^2+\dimx)/2+1}  q_j \\
& = \frac{T_\epsilon/2}{T} \Exp \gamtil \\
& \overset{(b)}{\le} T_\epsilon \gamma^2
\end{align*}
where $(a)$ follows since $\sum_{t=0}^{2H-1} \util_{t;j}^\top \util_{t;j} = 1$ almost surely, and $(b)$ follows since
$$\Exp \gamtil =  \Exp \sum_{i=1}^{\lceil T/H \rceil} \gamma^2(\outerU_{i;H}) = \Exp \sum_{i=1}^{\lceil T/H \rceil} \sum_{t=H(i-2)+1}^{Hi-1} u_t^\top u_t \le 2 T \gamma^2$$
Thus, $\util_t$ satisfies the power constraint $\Exp \sum_{t=0}^{T_\epsilon-1} \util_t^\top \util_t \le T_\epsilon \gamma^2$, which implies it also satisfies the constraint $\Exp \sum_{t=1}^T \util_t^\top \util_t \le T \gamma^2$. By Lemma \ref{lem:lds_lb_truncation_error}, given our choice of $H$ and this power constraint,
$$ \Exp \sum_{t=1}^{2T + T_\epsilon} x_{t;H}^{\bmutil} (x_{t;H}^{\bmutil})^\top \preceq \Exp \sum_{t=1}^{2T + T_\epsilon} x_{t}^{\bmutil} (x_{t}^{\bmutil})^\top + 4\epsilon I $$
Finally, note that for any $s$, we can bound
\begin{align*}
  \Exp \sum_{t=s}^{T_{\epsilon}+s-1} \util_t^\top \util_t \le 2T_\epsilon \gamma^2
  \end{align*}
since the sum can be contained by at most two periods of the input. The conclusion follows. 
\end{proof}

\subsection{Frequency Domain Approximation}

\begin{lem}\label{lem:lds_lb_cov_ss_exp}
Let $\{ u_t \}_{t=0}^{k-1}$ be a signal with $\Exp[\sum_{t=0}^{k-1} u_t^\top u_t] \le k \gamma^2$. Consider playing $u_t$ periodically for $T$ steps on system $\theta = (A,B)$ with no process noise, where we assume $x_0 = 0$ and set $u_t = u_{\mathrm{mod}(t,k)}$ almost surely. Then,
\begin{align*}
\Big \| \Exp \sum_{t=0}^T x_t x_t^\top & - \Exp \frac{1}{T} \sum_{t=1}^T (e^{\imag \frac{2\pi t}{T}} I - A)^{-1} B \ucheck_t \ucheck_t^\herm B^\herm (e^{\imag \frac{2\pi t}{T}} I - A)^{-\herm} \Big \|_\op \\
& \qquad \qquad  \le  \frac{\tau(A,\rho) \| \theta \|_{\Hinf}^2 \| B \|_\op \sqrt{T+1} k \gamma^2}{1-\rho^k} + \frac{\tau(A,\rho)^2 \| \theta \|_{\Hinf}^2 \| B \|_\op^2 k^2 \gamma^2}{(1-\rho^k)^2}. 
\end{align*}
where $(\ucheck_t)_{t=1}^T = \Fourier^{-1}((u_t)_{t=1}^T)$.
\end{lem}
\begin{proof}
Define $G(e^{\imag \omega}) := (e^{\imag \omega} I - A)^{-1} B$ and let $(\xcheck_t)_{t=1}^T = \Fourier^{-1}((x_t)_{t=1}^T)$ denote the $T$ point DFT of $x_t$. Then, by Parseval's Theorem,
\begin{align*}
& \left \| \sum_{t=0}^T x_t x_t^\top - \frac{1}{T} \sum_{t=1}^T G(e^{\imag\frac{2\pi t}{T}}) \ucheck_t \ucheck_t^\herm G(e^{\imag \frac{2\pi t}{T}})^\herm \right \|_\op  = \left \| \frac{1}{T} \sum_{t=1}^T \xcheck_t \xcheck_t^\herm - \frac{1}{T} \sum_{t=1}^T G(e^{\imag \frac{2\pi t}{T}}) \ucheck_t \ucheck_t^\herm G(e^{\imag \frac{2\pi t}{T}})^\herm \right \|_\op \\
& \qquad \qquad \le \frac{1}{T} \sum_{t=1}^T \| G(e^{\imag \frac{2\pi t}{T}}) \ucheck_t - \xcheck_t \|_2 (\| \xcheck_t \|_2 + \| G(e^{\imag \frac{2\pi t}{T}}) \ucheck_t \|_2)
\end{align*}
By Taylor expanding,
$$ G(e^{\imag\omega}) = \sum_{s=0}^\infty e^{-\imag\omega(s+1)}A^s B$$
By definition of a DFT, and since $x_0 = 0$,
$$ \xcheck_\ell = \sum_{t=0}^{T-1} e^{-\imag \frac{2\pi \ell}{T} t} x_t = \sum_{t=1}^{T-1} \sum_{s=0}^{t-1} e^{-\imag \frac{2\pi \ell}{T} t} A^{t-s-1} B u_s = \sum_{s=0}^{T-1} \left ( \sum_{t=0}^{T-s-2} e^{-\imag \frac{2\pi \ell}{T} (t+1)} A^t \right ) e^{-\imag \frac{2\pi \ell}{T} s} B u_s$$
$$ \ucheck_s = \sum_{t=0}^{T-1} e^{-\imag \frac{2\pi \ell}{T} t} u_t $$
Therefore,
\begin{align*}
G(e^{\imag\frac{2\pi \ell}{T}}) \ucheck_\ell - \xcheck_\ell  & =  \sum_{s=0}^{T-1} \left ( \sum_{t=0}^\infty e^{-\imag\frac{2\pi \ell}{T}(t+1)} A^t \right ) e^{-\imag\frac{2\pi \ell}{T} s} B u_s - \sum_{s=0}^{T-1} \left ( \sum_{t=0}^{T-s-2} e^{-\imag\frac{2\pi \ell}{T}(t+1)} A^t \right ) e^{-\imag\frac{2\pi \ell}{T} s} B u_s\\
& = \sum_{s=0}^{T-1} \left ( \sum_{t=T-s-1}^\infty e^{-\imag\frac{2\pi \ell}{T}(t+1)} A^t \right ) e^{-\imag\frac{2\pi \ell}{T} s} B u_s \\
& = \sum_{s=0}^{T-1} e^{-\imag\frac{2\pi \ell}{T}(T-s-1)} A^{T-s-1} \left ( \sum_{t=0}^\infty e^{-\imag\frac{2\pi \ell}{T}(t+1)} A^t \right ) e^{-\imag\frac{2\pi \ell}{T} s} B u_s \\
& = e^{-\imag\frac{2\pi \ell}{T}(T-1)} \sum_{s=0}^{T-1} A^{T-s-1} G(e^{\imag\frac{2\pi \ell}{T}}) B u_s
\end{align*}
Thus, since $u_s = u_{s+k}$ by assumption,
\begin{align*}
\| G(e^{\imag\frac{2\pi\ell}{T}}) \ucheck_\ell - \xcheck_\ell \|_2 & \le \tau(A,\rho)  \| \theta \|_{\Hinf} \| B \|_\op \sum_{s=0}^{T-1} \rho^{T-s-1} \| u_s \|_2 \\
& \le \tau(A,\rho)  \| \theta \|_{\Hinf} \| B \|_\op \sum_{j = 0}^{\lceil T/k \rceil} \rho^{kj} \sum_{s=0}^{k-1} \| u_s \|_2 \\
& \le \tau(A,\rho)  \| \theta \|_{\Hinf} \| B \|_\op \sqrt{k} \sqrt{\sum_{s=0}^{k-1} \| u_s \|_2^2} \sum_{j = 0}^{\lceil T/k \rceil} \rho^{kj} \\
& \le \frac{\tau(A,\rho)  \| \theta \|_{\Hinf} \| B \|_\op \sqrt{k}}{1-\rho^k} \sqrt{\sum_{s=0}^{k-1} \| u_s \|_2^2} 
\end{align*}
By Parseval's Theorem, and again since $u_s = u_{s+k}$, 
$$\sum_{t=1}^T \| \ucheck_t \|_2^2 = T \sum_{t=0}^{T-1} \| u_t \|_2^2 \le T \lceil T/k \rceil \sum_{t=0}^{k-1} \| u_t \|_2^2$$
So,
\begin{align*}
& \frac{1}{T} \sum_{t=1}^T \| G(e^{\imag \frac{2\pi t}{T}}) \ucheck_t - \xcheck_t \|_2 \| G(e^{\imag \frac{2\pi t}{T}}) \ucheck_t \|_2  \le \frac{\| \theta \|_{\Hinf} }{T} \sqrt{\sum_{t=1}^T \| G(e^{\imag \frac{2\pi t}{T}}) \ucheck_t - \xcheck_t \|_2^2} \sqrt{\sum_{t=1}^T \| \ucheck_t \|_2^2} \\
& \qquad \qquad \le \frac{\| \theta \|_{\Hinf} }{T} \sqrt{T \frac{\tau(A,\rho)^2  \| \theta \|_{\Hinf}^2 \| B \|_\op^2 k}{(1-\rho^k)^2} \sum_{s=0}^{k-1} \| u_s \|_2^2} \sqrt{T \lceil T/k \rceil \sum_{t=0}^{k-1} \| u_t \|_2^2} \\
& \qquad \qquad = \frac{\tau(A,\rho) \| \theta \|_{\Hinf}^2 \| B \|_\op \sqrt{k \lceil T/k \rceil}}{1-\rho^k} \sum_{t=0}^{k-1} \| u_t \|_2^2 
\end{align*}
Again by Parseval's theorem, and by the same calculation as was performed above,
\begin{align*}
\sum_{t=1}^T \| \xcheck_t \|_2^2 & = T \sum_{t=1}^{T-1} \| x_t \|_2^2  = T \sum_{t=1}^{T-1} \left \| \sum_{s=0}^{t-1} A^{t-s-1} B u_s \right \|_2^2  \le T \sum_{t=1}^{T-1} \tau(A,\rho)^2 \| B \|_\op^2 \left ( \sum_{s=0}^{t-1} \rho^{t-s-1} \| u_s \|_2 \right )^2 \\
& \le \frac{T \tau(A,\rho)^2 \| B \|_\op^2 k}{(1-\rho^k)^2} \sum_{s=0}^{k-1} \| u_s \|_2^2
\end{align*}
So,
\begin{align*}
& \frac{1}{T} \sum_{t=1}^T \| G(e^{\imag \frac{2\pi t}{T}}) \ucheck_t - \xcheck_t \|_2 \| X(e^{\imag \frac{2\pi t}{T}}) \|_2 \le \frac{1}{T} \sqrt{\sum_{t=1}^T \| G(e^{\imag \frac{2\pi t}{T}}) \ucheck_t - \xcheck_t \|_2^2} \sqrt{\sum_{t=1}^T \| \xcheck_t \|_2^2} \\
& \qquad \qquad \le \frac{\| \theta \|_{\Hinf} }{T} \sqrt{T \frac{\tau(A,\rho)^2  \| \theta \|_{\Hinf}^2 \| B \|_\op^2 k}{(1-\rho^k)^2} \sum_{s=0}^{k-1} \| u_s \|_2^2} \sqrt{\frac{T \tau(A,\rho)^2 \| B \|_\op^2 k}{(1-\rho^k)^2} \sum_{s=0}^{k-1} \| u_s \|_2^2} \\
& \qquad \qquad = \frac{\tau(A,\rho)^2 \| \theta \|_{\Hinf}^2 \| B \|_\op^2 k}{(1-\rho^k)^2} \sum_{t=0}^{k-1} \| u_t \|_2^2 
\end{align*}
It follows that
\begin{align*}
& \Exp \left \| \sum_{t=0}^T x_t x_t^\top - \frac{1}{T} \sum_{t=1}^T G(e^{\imag\frac{2\pi t}{T}}) \ucheck_t \ucheck_t^{\herm} G(e^{\imag\frac{2\pi t}{T}})^{\herm} \right \|_\op \\
& \qquad \qquad \le \frac{\tau(A,\rho) \| \theta \|_{\Hinf}^2 \| B \|_\op \sqrt{k \lceil T/k \rceil}}{1-\rho^k} \Exp \sum_{t=0}^{k-1} \| u_t \|_2^2 + \frac{\tau(A,\rho)^2 \| \theta \|_{\Hinf}^2 \| B \|_\op^2 k}{(1-\rho^k)^2} \Exp \sum_{t=0}^{k-1} \| u_t \|_2^2 \\
& \qquad \qquad \le \frac{\tau(A,\rho) \| \theta \|_{\Hinf}^2 \| B \|_\op \sqrt{T+1} k \gamma^2}{1-\rho^k} + \frac{\tau(A,\rho)^2 \| \theta \|_{\Hinf}^2 \| B \|_\op^2 k^2 \gamma^2}{(1-\rho^k)^2} 
\end{align*}
and the conclusion follows.
\end{proof}

\subsection{Smoothness of Covariates}
\begin{lem}\label{lem:smooth_covariates}
For all $\bmU \in \calU_{\gamma^2,k}$ and all $\theta$ with 
\begin{equation}\label{eq:cov_perturb_init}
\| \theta - \thetast \|_\op \le \min \left \{ \frac{1-\rho}{2\tau(\Ast,\rho)}, \frac{1}{2 \| \Ast \|_{\Hinf} }, 1 \right \} =: \betaexplds(\thetast)
\end{equation}
if $T_2$ is divisible by $k$, 
\begin{align}\label{eq:alphast_lds}
\begin{split} 
 \| \matGamss_{T_1,T_2}(\theta,\bmU,\sigma_u) - & \matGamss_{T_1,T_2}(\thetast,\bmU,\sigma_u) \|_\op  \le \Bigg ( \frac{8  (\sigmaw^2 + \sigma_u^2 \| \Bst \|_\op^2) \tau(\Ast,\rho)^3}{(1 - \rho^2)^2} + \frac{4 \sigma_u^2 (\| \Bst \|_\op + 1) \tau(\Ast,\rho)^2}{1 - \rho^2}  \\
& \qquad +  34 \gamma^2  \| \Ast \|_{\Hinf}^3 (\| \Bst \|_\op + 1)^2  \Bigg ) \| \theta - \thetast \|_\op \\
& \qquad =: \alphastlds(\thetast,\gamma^2) \cdot \| \theta - \thetast \|_\op. 
\end{split}
\end{align}
\end{lem}
\begin{proof}
For convenience denote $\epsilon = \| \theta - \thetast \|_\op$. As $\matGamss_{T_1,T_2}(\theta,\bmU,\sigma_u) = I_{\dimx} \otimes \Gamss_{T_1,T_2}(\theta,\bmU,\sigma_u)$,
$$ \| \matGamss_{T_1,T_2}(\theta,\bmU,\sigma_u) - \matGamss_{T_1,T_2}(\thetast,\bmU,\sigma_u) \|_\op = \| \Gamss_{T_1,T_2}(\theta,\bmU,\sigma_u) - \Gamss_{T_1,T_2}(\thetast,\bmU,\sigma_u) \|_\op $$
By definition, when $T_2$ is divisible by $k$,
$$ \Gamss_{T_1,T_2}(\theta,\bmU,\sigma_u) = \frac{1}{T_2} \Gamfreq_{T_2}(\theta,\bmU) + \frac{1}{T_1} \sum_{t=1}^{T_1} \Gamnoise_t(\theta,\sigma_u) $$
$$ \Gamfreq_{T_2}(\theta,\bmU) = \frac{T_2}{k} \frac{1}{k}  \sum_{\ell=1}^k (e^{\imag \frac{2\pi \ell}{k}} I - A)^{-1} B U_\ell B^{\herm} (e^{\imag \frac{2\pi \ell}{k}} I - A)^{-\herm}  $$
$$ \Gamnoise_t(\theta,\sigma_u) = \sigmaw^2 \sum_{s =0}^{t-1} A^s (A^s)^\top + \sigma_u^2 \sum_{s=0}^{t-1} A^s B B^\top (A^s)^\top $$
Thus,
\begin{align*}
\| \Gamss_{T_1,T_2}(\theta,\bmU,\sigma_u) & - \Gamss_{T_1,T_2}(\thetast,\bmU,\sigma_u) \|_\op \\
 & \le \underbrace{\frac{1}{T_2} \| \Gamfreq_{T_2}(\theta,\bmU) - \Gamfreq_{T_2}(\thetast,\bmU) \|_\op}_{(a)} + \underbrace{\frac{1}{T_1} \sum_{t=1}^{T_1} \| \Gamnoise_t(\theta,\sigma_u) - \Gamnoise_t(\thetast,\sigma_u) \|_\op}_{(b)}
\end{align*}
Lemma \ref{lem:cov_perturbation} gives that, when \eqref{eq:cov_perturb_init} holds,
\begin{align*}
(a)  & \le \Big ( \max_{\omega \in [0,2\pi]}  \gamma^2  \| (e^{\imag\omega} I - \Ast)^{-1} \|_\op^2 \| \Bst \|_\op  \left ( 16 \| (e^{\imag\omega} I - \Ast)^{-1} \|_\op \| \Bst \|_\op+ 2  \right ) \Big ) \epsilon \\
& \qquad \qquad \qquad + \Big ( \max_{\omega \in [0,2\pi]}  \gamma^2 \| (e^{\imag\omega} I - \Ast)^{-1} \|_\op^2 (32 \| (e^{\imag\omega} I - \Ast)^{-1} \|_\op \| \Bst \|_\op + 2) \Big ) \epsilon^2 \\
& \qquad \qquad \qquad + \Big ( \max_{\omega \in [0,2\pi]} 16  \gamma^2 \| (e^{\imag\omega} I - \Ast)^{-1} \|_\op^3 \Big ) \epsilon^3 \\
& \le \Big ( \max_{\omega \in [0,2\pi]}  34 \gamma^2  \| (e^{\imag\omega} I - \Ast)^{-1} \|_\op^3 (\| \Bst \|_\op + 1)^2  \Big ) \epsilon
\end{align*}
while Lemma \ref{lem:cov_noise_perturbation} gives that, when \eqref{eq:cov_perturb_init} holds,
\begin{align*}
(b) & \le \left ( \frac{8  (\sigmaw^2 + \sigma_u^2 \| \Bst \|_\op^2) \tau(\Ast,\rho)^3}{(1 - \rho^2)^2} + \frac{4 \sigma_u^2 \| \Bst \|_\op \tau(\Ast,\rho)^2}{1 - \rho^2} \right ) \epsilon + \frac{2 \sigma_u^2 \tau(\Ast,\rho)^2 }{1-\rho^2} \epsilon^2 \\
& \le \left ( \frac{8  (\sigmaw^2 + \sigma_u^2 \| \Bst \|_\op^2) \tau(\Ast,\rho)^3}{(1 - \rho^2)^2} + \frac{4 \sigma_u^2 (\| \Bst \|_\op + 1) \tau(\Ast,\rho)^2}{1 - \rho^2} \right ) \epsilon
\end{align*}
The result follows.
\end{proof}

\begin{lem}\label{lem:cov_noise_perturbation}
Assume that $\| [A,B] - [\Ahat,\Bhat] \|_\op \le \epsilon$ and that $\epsilon \le \frac{1-\rho}{2\tau(A,\rho)}$, then
\begin{align*}
& \left \| \sum_{k=0}^{t} \left ( \sigmaw^2 A^k (A^k)^\top + \sigma_u^2 A^k B B^\top (A^k)^\top \right ) -  \sum_{k=0}^{t} \left ( \sigmaw^2 \Ahat^k (\Ahat^k)^\top + \sigma_u^2 \Ahat^k \Bhat \Bhat^\top (\Ahat^k)^\top \right ) \right \|_\op \\
& \qquad \qquad \qquad \qquad \le \left ( \frac{8  (\sigmaw^2 + \sigma_u^2 \| B \|_\op^2) \tau(A,\rho)^3}{(1 - \rho^2)^2} + \frac{4 \sigma_u^2 \| B \|_\op \tau(A,\rho)^2}{1 - \rho^2} \right ) \epsilon + \frac{2 \sigma_u^2 \tau(A,\rho)^2 }{1-\rho^2} \epsilon^2.
\end{align*}
\end{lem}
\begin{proof}
First note that $\| [A,B] - [\Ahat,\Bhat] \|_\op \le \epsilon$ implies $\| A - \Ahat \|_\op \le \epsilon, \| B - \Bhat \|_\op \le \epsilon$. We will denote $\Ahat = A + \DelA$, where $\| \DelA \|_\op \le \epsilon$. We can upper bound
\begin{align*}
&  \left \| \sum_{k=0}^{t} \left ( \sigmaw^2 A^k (A^k)^\top + \sigma_u^2 A^k B B^\top (A^k)^\top \right ) -  \sum_{k=0}^{t} \left ( \sigmaw^2 \Ahat^k (\Ahat^k)^\top + \sigma_u^2 \Ahat^k \Bhat \Bhat^\top (\Ahat^k)^\top \right ) \right \|_\op \\
& \le \sigmaw^2 \sum_{k=0}^t \| A^k (A^k)^\top - (A + \DelA)^k ((A + \DelA)^k)^\top \|_\op \\
& \qquad + \sigma_u^2 \sum_{k=0}^t  \| A^k B B^\top (A^k)^\top - (A + \DelA)^k \Bhat \Bhat^\top ((A + \DelA)^k)^\top \|_\op
\end{align*}
By the triangle inequality,
\begin{align*}
\| A^k (A^k)^\top - & (A+\DelA)^k ((A+\DelA)^k)^\top \|_\op \\
& \le \| A^k (A^k)^\top - A^k ((A + \DelA)^k)^\top + A^k ((A + \DelA)^k)^\top - (A+\DelA)^k ((A+\DelA)^k)^\top \|_\op \\
&  \le (\| A^k \|_\op + \| (A+ \DelA)^k \|_\op) \| A^k - (A + \DelA)^k \|_\op 
\end{align*}
By Proposition \ref{prop:matrix_binomial},
\begin{align*}
\| (A+ \DelA)^k \|_\op & \le  \tau(A,\rho) (\rho + \tau(A,\rho) \epsilon)^k
\end{align*}
and 
\begin{align*}
\| (A + \DelA)^k - A^k \|_\op  \le k \tau(A,\rho)^2(\rho + \tau(A,\rho)\epsilon)^{k-1}  \epsilon
\end{align*}
Combining all of this we have 
$$ \| A^k (A^k)^\top - (A+\DelA)^k ((A+\DelA)^k)^\top \|_\op \le k \tau(A,\rho)^3 (\rho^k + (\rho + \tau(A,\rho) \epsilon)^k) (\rho + \tau(A,\rho) \epsilon)^{k-1} \epsilon $$
Denote $\rho_2 := \rho + \tau(A,\rho) \epsilon$. Since we have assumed that $\epsilon \le \frac{1-\rho}{2\tau(A,\rho)}$, $\rho_2 \le \frac{1}{2} + \frac{1}{2} \rho$. Then it follows:
\begin{align*}
\sigmaw^2 \sum_{k=0}^t  \| A^k (A^k)^\top - (A+\DelA)^k ((A+\DelA)^k)^\top \|_\op & \le \frac{\sigmaw^2 \tau(A,\rho)^3 \epsilon}{\rho_2} \sum_{k=0}^t k ((\rho \rho_2)^{k} + \rho_2^{2k}) \\
& \le \sigmaw^2 \tau(A,\rho)^3 \epsilon \left ( \frac{\rho}{(1 - \rho \rho_2)^2} + \frac{\rho_2}{(1-\rho_2^2)^2} \right ) \\
& \le \frac{2 \sigmaw^2 \tau(A,\rho)^3 \epsilon}{(1 - \rho_2^2)^2} \\
& \le \frac{8 \sigmaw^2 \tau(A,\rho)^3 \epsilon}{(1 - \rho^2)^2}
\end{align*}
where the last inequality follows since $1-\rho_2^2 \ge \frac{1}{2} (1 - \rho^2)$. Denoting $\Bhat = B + \DelB$, and using what we have already shown, we have
\begin{align*}
& \| A^k B B^\top (A^k)^\top - (A + \DelA)^k \Bhat \Bhat^\top ((A + \DelA)^k)^\top \|_\op \\
& = \| A^k B B^\top (A^k)^\top - (A + \DelA)^k B B^\top  ((A + \DelA)^k)^\top - (A + \DelA)^k \DelB B^\top  ((A + \DelA)^k)^\top \\
& \qquad \qquad - (A + \DelA)^k B \DelB^\top  ((A + \DelA)^k)^\top - (A + \DelA)^k \DelB \DelB^\top  ((A + \DelA)^k)^\top\|_\op \\
& \le \| A^k B B^\top (A^k)^\top - A^k B B^\top ((A + \DelA)^k)^\top + A^k B B^\top ((A + \DelA)^k)^\top - (A + \DelA)^k B B^\top  ((A + \DelA)^k)^\top \|_\op \\
& \qquad \qquad 2 \| B \|_\op \tau(A,\rho)^2 (\rho + \tau(A,\rho) \epsilon)^{2k} \epsilon + \tau(A,\rho)^2 (\rho + \tau(A,\rho) \epsilon)^{2k} \epsilon^2 \\
& \le \| B \|_\op^2 (\| A^k \|_\op + \| (A+ \DelA)^k \|_\op) \| A^k - (A + \DelA)^k \|_\op + 2 \| B \|_\op \tau(A,\rho)^2 (\rho + \tau(A,\rho) \epsilon)^{2k} \epsilon \\
& \qquad \qquad + \tau(A,\rho)^2 (\rho + \tau(A,\rho) \epsilon)^{2k} \epsilon^2 \\
& \le \| B \|_\op^2 k \tau(A,\rho)^3 (\rho^k + \rho_2^k) \rho_2^{k-1} \epsilon + 2 \| B \|_\op \tau(A,\rho)^2 \rho_2^{2k} \epsilon  + \tau(A,\rho)^2 \rho_2^{2k} \epsilon^2 \\
\end{align*}
Thus,
\begin{align*}
 \sigma_u^2 \sum_{k=0}^t & \| A^k B B^\top (A^k)^\top - (A + \DelA)^k \Bhat \Bhat^\top ((A + \DelA)^k)^\top \|_\op \\
& \le \sigma_u^2 \| B \|_\op^2  \tau(A,\rho)^3 \epsilon \sum_{k=0}^t k (\rho^k + \rho_2^k) \rho_2^{k-1}  + 2 \sigma_u^2 \| B \|_\op \tau(A,\rho)^2 \epsilon \sum_{k=0}^{t} \rho_2^{2k}   + \sigma_u^2 \tau(A,\rho)^2 \epsilon^2 \sum_{k=0}^t \rho_2^{2k} \\
& \le \frac{8 \sigma_u^2 \| B \|_\op^2 \tau(A,\rho)^3 \epsilon}{(1 - \rho^2)^2} + \frac{4 \sigma_u^2 \| B \|_\op \tau(A,\rho)^2 \epsilon}{1 - \rho^2} + \frac{2 \sigma_u^2 \tau(A,\rho)^2 \epsilon^2}{1-\rho^2}
\end{align*}
The conclusion follows.
\end{proof}

\begin{lem}\label{lem:cov_perturbation}
Assume that $\| [\wh{A}, \wh{B}] - [A,B] \|_\op \leq \epsilon$, $\epsilon \leq (\max_{\omega \in [0,2\pi]} 2 \| (e^{\imag\omega} I - A)^{-1} \|_\op)^{-1}$, and $U_\ell \succeq 0$, $\sum_{\ell=1}^k \tr(U_\ell) \le k^2 \gamma^2$, then:
\begin{align*} & \frac{1}{k} \left \| \sum_{\ell=1}^k (e^{\imag \omega_\ell} I - \wh{A})^{-1} \wh{B} U_\ell \wh{B}^{\herm} (e^{\imag \omega_\ell} I - \wh{A})^{-\herm} - \sum_{\ell=1}^k (e^{\imag \omega_\ell} I - A)^{-1} B U_\ell B^{\herm} (e^{\imag \omega_\ell} I - A)^{-\herm} \right \|_\op \\
& \qquad \qquad \le \Big ( \max_{\omega \in [0,2\pi]} k \gamma^2  \| (e^{\imag\omega} I - A)^{-1} \|_\op^2 \| B \|_\op  \left ( 16 \| (e^{\imag\omega} I - A)^{-1} \|_\op \| B \|_\op+ 2  \right ) \Big ) \epsilon \\
& \qquad \qquad \qquad + \Big ( \max_{\omega \in [0,2\pi]} k \gamma^2 \| (e^{\imag\omega} I - A)^{-1} \|_\op^2 (32 \| (e^{\imag\omega} I - A)^{-1} \|_\op \| B \|_\op + 2) \Big ) \epsilon^2 \\
& \qquad \qquad \qquad + \Big ( \max_{\omega \in [0,2\pi]} 16 k \gamma^2 \| (e^{\imag\omega} I - A)^{-1} \|_\op^3 \Big ) \epsilon^3
\end{align*}
\end{lem}
\begin{proof}
Note first that $\| [\wh{A}, \wh{B}] - [A,B] \|_\op \leq \epsilon$ implies $\| \wh{A} - A \|_\op \leq \epsilon, \| \Bhat - B \|_\op \leq \epsilon$ since:
$$ \| [\wh{A}, \wh{B}] - [A,B] \|_\op = \max_{u \in \calS^{2d},v \in \calS^{d+p}} u^\top ([\wh{A}, \wh{B}] - [A,B]) v \ge \max_{\substack{u \in \calS^{2d}, u_{d+1:2d}=0 \\ v \in \calS^{d+p},v_{d+1:d+p}=0}} u^\top ([\wh{A}, \wh{B}] - [A,B]) v = \| \Ahat - A \|_\op $$
If we denote $\wh{A} = A + \Delta_A, \wh{B} = B + \Delta_B$, then: 
\begin{align*}
& \left \| \sum_{\ell=1}^k (e^{\imag \omega_\ell} I - \wh{A})^{-1} \wh{B} U_\ell  \wh{B}^{\herm} (e^{\imag \omega_\ell} I - \wh{A})^{-\herm} - \sum_{\ell=1}^k (e^{\imag \omega_\ell} I - A)^{-1} B U_\ell  B^{\herm} (e^{\imag \omega_\ell} I - A)^{-\herm} \right \|_\op \\
& \leq \underbrace{\left \| \sum_{\ell=1}^k (e^{\imag \omega_\ell} I - \wh{A})^{-1} \wh{B} U_\ell  \wh{B}^{\herm} (e^{\imag \omega_\ell} I - \wh{A})^{-\herm} - \sum_{\ell=1}^k (e^{\imag \omega_\ell} I - A)^{-1} \wh{B} U_\ell  \wh{B}^{\herm} (e^{\imag \omega_\ell} I - A)^{-\herm} \right \|_\op}_{(i)} \\
& \qquad \qquad + \underbrace{\left \| \sum_{\ell=1}^k (e^{\imag \omega_\ell} I -A)^{-1} \wh{B} U_\ell \wh{B}^{\herm} (e^{\imag \omega_\ell} I - A)^{-\herm} - \sum_{\ell=1}^k (e^{\imag \omega_\ell} I - A)^{-1} B U_\ell  B^{\herm} (e^{\imag \omega_\ell} I - A)^{-\herm} \right \|_\op}_{(ii)}
\end{align*}
By Lemma F.4 and F.7 of \cite{wagenmaker2020active}:
\begin{align*}
(i) & \le \max_{\omega \in [0,2\pi]} 16k^2 \gamma^2 \| (e^{\imag\omega} I - A)^{-1} \|_\op \| (e^{\imag\omega} I - A)^{-1} \Bhat \|_\op^2 \epsilon \\
& \le \max_{\omega \in [0,2\pi]} 16k^2 \gamma^2 \| (e^{\imag\omega} I - A)^{-1} \|_\op^3 (\| B \|_\op + \epsilon)^2 \epsilon
\end{align*}
Note that these lemmas assume that $U_\ell$ are rank 1, but a trivial modification extends the results to arbitrary $U_\ell \succeq 0$ satisfying $\sum_{\ell=1}^k \tr(U_\ell) \le k^2 \gamma^2$. Further:
\begin{align*}
(ii) & \le 2 \left \| \sum_{\ell=1}^k (e^{\imag \omega_\ell} I -A)^{-1} \DelB U_\ell B^{\herm} (e^{\imag \omega_\ell} I - A)^{-\herm} \right \|_\op + \left \| \sum_{\ell=1}^k (e^{\imag \omega_\ell} I -A)^{-1} \DelB U_\ell \DelB^{\herm} (e^{\imag \omega_\ell} I - A)^{-\herm} \right \|_\op \\
& \le (2 \| \DelB \|_\op \| B \|_\op + \| \DelB \|_\op^2) \left ( \max_{\omega \in [0,2\pi]} \| (e^{\imag\omega} I - A)^{-1} \|_\op^2 \right ) \sum_{\ell=1}^k \| U_\ell  \|_\op \\
& \le k^2 \gamma^2 (2 \epsilon \| B \|_\op + \epsilon^2) \left ( \max_{\omega \in [0,2\pi]} \| (e^{\imag\omega} I - A)^{-1} \|_\op^2 \right )
\end{align*}
The result is then immediate. 
\end{proof}

\begin{prop}\label{prop:matrix_binomial}
Assume that $\| A^k \|_\op \le \tau \rho^k$ for all $k \ge 0$ and that $\| \Delta \|_\op \le \epsilon$. Then, for $k \ge 1$,
$$ \| (A + \Delta)^k \|_\op \le \tau(\rho + \tau \epsilon)^k, \quad \| (A + \Delta)^k - A^k \| \le \tau(\rho + \tau \epsilon)^k - \tau \rho^k \le k \tau^2 (\rho + \tau \epsilon)^{k-1} \epsilon$$
\end{prop}
\begin{proof}
If $A$ and $\Delta$ were scalars, the Binomial Theorem would give:
$$ (A + \Delta)^k = \sum_{s=0}^k \binom{k}{s} A^{k-s} \Delta^s $$
As matrix multiplication does not commute, we cannot simply apply the Binomial Theorem. However, we note that, for a fixed $s$, we will have $\binom{k}{s}$ terms of the form
$$ A^{n_1} \Delta^{m_1} A^{n_2} \Delta^{m_2} \ldots A^{n_s} \Delta^{m_s} A^{n_{s+1}} $$
where $\sum_{i=1}^{s+1} n_i = k - s, \sum_{i=1}^s m_i = s$. Critically, there will be at most $s$ $\Delta$ terms in this product. Then, using our assumption on $\| A^k \|_\op$, we have
\begin{align*}
\| A^{n_1} \Delta^{m_1} A^{n_2} \Delta^{m_2} \ldots A^{n_s} \Delta^{m_s} A^{n_{s+1}} \|_\op & \le \left ( \prod_{i=1}^{s+1} \| A^{n_i} \|_\op \right ) \left ( \prod_{i=1}^{s} \| \Delta^{m_i} \|_\op \right )  \\
& \le \tau^{s+1} \left ( \prod_{i=1}^{s+1} \rho^{n_i} \right ) \left ( \prod_{i=1}^{s} \epsilon^{m_i} \right ) \\
& = \tau^{s+1} \rho^{k-s} \epsilon^s
\end{align*}
As this bound does not depend on the specific values of $n_i,m_i$, we will have
\begin{align*}
\| (A + \Delta)^k \|_\op & \le \sum_{s=0}^k \binom{k}{s} \tau^{s+1} \rho^{k-s} \epsilon^s  = \tau ( \rho + \tau \epsilon)^k
\end{align*}
where the final equality holds by the Binomial Theorem. Similarly, note that $\| (A + \Delta)^k - A^k \|_\op$ will behave identically, except that the $A^k$ term will be removed from the expansion of $(A + \Delta)^k$. Thus,
\begin{align*}
\| (A + \Delta)^k - A^k \|_\op & \le \sum_{s=1}^k \binom{k}{s} \tau^{s+1} \rho^{k-s} \epsilon^s  = \tau ( \rho + \tau \epsilon)^k - \tau \rho^k
\end{align*}
To show the final conclusion, note that, for $k \ge 1$, the derivative of $(a + x)^k$ is $\frac{d}{dx} (a + x)^k = k(a + x)^{k-1}$. By the Mean Value Theorem,
$$ | (a + x)^k - (a + y)^k | \le \left ( \max_{z \in [x,y]} k |(a + z)^{k-1}| \right ) | x - y | $$
Applying this observation in our setting gives that
$$ \tau ( \rho + \tau \epsilon)^k - \tau \rho^k \le \tau \left ( \max_{z \in [0,\tau \epsilon]} k (\rho + z)^{k-1} \right ) \tau \epsilon \le k \tau^2 (\rho + \tau \epsilon)^{k-1} \epsilon $$
\end{proof}

\subsection{Infinite-Horizon Approximation}

\begin{lem}\label{lem:opt_k_large_approx}
Fix any $\bar{k}$ and input $\bmU^\star \in \calU_{\gamma^2,\bar{k}}$. Then for any:
$$ k \geq \max \left \{ \frac{4 \pi \| \theta \|_{\Hinf} \gamma^2}{\epsilon}, \frac{\pi}{2 \| \theta \|_{\Hinf}} \right \} \left ( \max_{\omega \in [0,2\pi]} \| (e^{\imag\omega} I - A)^{-2} B \|_\op \right ) $$
there exists an input $\bmU \in \calU_{\gamma^2,k}$ such that:
$$ \left \| \frac{1}{k} \Gamfreq_k(\theta,\bmU) - \frac{1}{\bar{k}} \Gamfreq_{\bar{k}}(\theta,\bmU^\star) \right \|_\op \leq \epsilon. $$
\end{lem}
\begin{proof}
For simplicity denote $G_{k,\ell} = (e^{\imag2\pi \ell/k} I - A)^{-1} B$. Consider some $k$ and, given $\ell \in [1,\bar{k}]$, let $\ell_k(\ell) \in [1,k]$ be the index such that $|\ell_k(\ell)/k - \ell/\bar{k}|$ is minimized. Let $\ell^{-1}_k(\ell) : \{1,\ldots,k\} \rightarrow 2^{\{1,\ldots,\bar{k}\}}$ return the set of indices that map to $\ell$. Then:
\begin{align*}
 \frac{1}{\bar{k}} \Gamfreq_{\bar{k}}(\theta,\bmU^\star)&= \frac{1}{\bar{k}^2} \sum_{\ell =1 }^{\bar{k}} G_{\bar{k},\ell} U_{\bar{k},\ell}^\star G_{\bar{k},\ell}^{\herm}  = \frac{1}{\bar{k}^2} \sum_{\ell =1 }^{\bar{k}} (G_{k,\ell_k(\ell)} + \Delta_\ell) U_{\bar{k},\ell}^\star (G_{k,\ell_k(\ell)} + \Delta_\ell)^{\herm} \\
& = \sum_{\ell = 1}^k G_{k,\ell} \bigg ( \frac{1}{\bar{k}^2} \sum_{\ell' \in \ell_k^{-1}(\ell)} U_{\bar{k},\ell'}^\star \bigg ) G_{k,\ell}^{\herm} + \frac{1}{\bar{k}^2} \sum_{\ell =1 }^{\bar{k}} \left ( G_{k,\ell_k(\ell)}  U_{\bar{k},\ell}^\star \Delta_\ell^{\herm} + \Delta_\ell U_{\bar{k},\ell}^\star G_{k,\ell_k(\ell)}^{\herm} +  \Delta_\ell U_{\bar{k},\ell}^\star  \Delta_\ell^{\herm} \right ) 
\end{align*}
Set $U_{k,\ell} = \frac{1}{\bar{k}^2} \sum_{\ell' \in \ell_k^{-1}(\ell)} U_{\bar{k},\ell'}^\star $, denote $\delta := \max_{\ell \in [1,...,\bar{k}]} \| \Delta_\ell \|_2$, and note that $\| \theta \|_{\Hinf} \ge \| G_{k,\ell}$ for all $k,\ell$. Then:
\begin{align*}
\left \| \frac{1}{k} \Gamfreq_k(\theta,\bmU) - \frac{1}{\bar{k}} \Gamfreq_{\bar{k}}(\theta,\bmU^\star) \right \|_\op & \leq \left \| \frac{1}{\bar{k}^2} \sum_{\ell =1 }^{\bar{k}} \left ( G_{k,\ell_k(\ell)}  U_{\bar{k},\ell}^\star \Delta_\ell^{\herm} + \Delta_\ell U_{\bar{k},\ell}^\star G_{k,\ell_k(\ell)}^{\herm} +  \Delta_\ell U_{\bar{k},\ell}^\star  \Delta_\ell^{\herm} \right )  \right \|_\op \\
& \leq \frac{2 \delta \| \theta \|_{\Hinf} + \delta^2}{\bar{k}^2} \sum_{\ell = 1}^{\bar{k}} \| U_{\bar{k},\ell}^\star \|_\op \\
& \leq \frac{2 \delta \| \theta \|_{\Hinf} + \delta^2}{\bar{k}^2} \sum_{\ell = 1}^{\bar{k}} \tr ( U_{\bar{k},\ell}^\star ) \\
& \leq (2 \delta \| \theta \|_{\Hinf} + \delta^2)  \gamma^2
\end{align*}
where the final equality holds since $\wt{U}_{\bar{k},\ell}^\star$ is a feasible input and so must meet the power constraint. By Lemma H.1 of \cite{wagenmaker2020active}:
$$ \| G_{k,\ell} - G_{k',\ell'} \|_2 \leq  2 \pi  | \ell / k - \ell' / k' | \left ( \max_{\omega \in [0,2\pi]} \| (e^{\imag\omega} I - A)^{-2} B \|_\op \right ) $$
Using this we can bound:
$$ \| G_{\bar{k},\ell} - G_{k,\ell_k(\ell)} \|_2 \leq 2\pi | \ell / \bar{k} - \ell_k(\ell) / k | \left ( \max_{\omega \in [0,2\pi]} \| (e^{\imag\omega} I - A)^{-2} B \|_\op \right ) \leq \pi /k \left ( \max_{\omega \in [0,2\pi]} \| (e^{\imag\omega} I - A)^{-2} B \|_\op \right ) $$ 
since any $x \in [0,1]$ is at most $1/(2k)$ from the nearest fraction $i/k$. This implies 
$$\delta \leq  \pi/k \left ( \max_{\omega \in [0,2\pi]} \| (e^{\imag\omega} I - A)^{-2} B \|_\op \right )$$ 
so:
\begin{align*}
\left \| \frac{1}{k} \Gamfreq_k(\theta,\bmU) - \frac{1}{\bar{k}} \Gamfreq_{\bar{k}}(\theta,\bmU^\star) \right \|_\op & \leq \frac{2 \pi \| \theta \|_{\Hinf} \gamma^2}{k} \left ( \max_{\omega \in [0,2\pi]} \| (e^{\imag\omega} I - A)^{-2} B \|_\op \right ) \\
& \qquad \qquad + \frac{\pi^2 \gamma^2}{k^2} \left ( \max_{\omega \in [0,2\pi]} \| (e^{\imag\omega} I - A)^{-2} B \|_\op \right )^2 \\
& \le \frac{4\pi \| \theta \|_{\Hinf} \gamma^2}{k} \left ( \max_{\omega \in [0,2\pi]} \| (e^{\imag\omega} I - A)^{-2} B \|_\op \right )
\end{align*}
where the last inequality holds so long as $k \geq \frac{\pi}{2 \| \theta\|_{\Hinf}} \left ( \max_{\omega \in [0,2\pi]} \| (e^{\imag\omega} I - A)^{-2} B \|_\op \right )$. To make this less than $\epsilon$, we must choose:
$$ k \geq \frac{4\pi \| \theta \|_{\Hinf} \gamma^2}{\epsilon} \left ( \max_{\omega \in [0,2\pi]} \| (e^{\imag\omega} I - A)^{-2} B \|_\op \right ) $$
Finally, note that the input $\wt{U}_{k,\ell}$ is feasible since:
$$ \sum_{\ell = 1}^k \tr(U_{k,\ell}) = \frac{1}{\bar{k}^2} \sum_{\ell = 1}^k \sum_{\ell' \in \ell_k^{-1}(\ell)} \tr ( U_{\bar{k},\ell'}^\star) \leq \frac{1}{\bar{k}^2} \sum_{\ell = 1}^{\bar{k}} \tr ( U_{\bar{k},\ell}^\star) \leq \gamma^2 $$
The conclusion follows immediately. 
\end{proof}

\begin{lem}\label{lem:opt_k_large_approx_noise}
If 
$$ T \ge \max \left \{  \frac{8\tau(A,\rho)^2 (\sigmaw^2 + \sigma_u^2 \| B \|_\op^2) }{(1-\rho^2)^2}  \frac{1}{\epsilon}, \log \left (  \frac{(1-\rho^2)^2 \epsilon }{8\tau(A,\rho)^2 (\sigmaw^2 + \sigma_u^2 \| B \|_\op^2) }\right ) \frac{1}{2 \log \rho} \right \}$$
then, for any $T' \ge T$,
$$ \left \| \frac{1}{T} \sum_{t=1}^T \Gamnoise_t(\theta,\sigma_u) - \frac{1}{T'} \sum_{t=1}^{T'} \Gamnoise_t(\theta,\sigma_u) \right \|_\op \le \epsilon. $$
\end{lem}
\begin{proof}
By definition of $\Gamnoise_t$,
\begin{align*}
& \left \| \frac{1}{T} \sum_{t=1}^T \Gamnoise_t(\theta,\sigma_u) - \frac{1}{T'} \sum_{t=1}^{T'} \Gamnoise_t(\theta,\sigma_u) \right \|_\op \\
& \qquad = \left \| \frac{1}{T} \sum_{t=1}^T \sum_{s=0}^{t-1} \left ( \sigmaw^2 A^s (A^s)^\top + \sigma_u^2 A^s B B^\top (A^s)^\top \right ) - \frac{1}{T'} \sum_{t=1}^{T'} \sum_{s=0}^{t-1} \left ( \sigmaw^2 A^s (A^s)^\top + \sigma_u^2 A^s B B^\top (A^s)^\top \right )  \right \|_\op \\
& \qquad = \Bigg \| \frac{1}{T} \sum_{s=0}^{T-1} \left ( \sigmaw^2 (T-s) A^s (A^s)^\top + \sigma_u^2 (T-s) A^s B B^\top (A^s)^\top \right ) \\
& \qquad \qquad \qquad \qquad -  \frac{1}{T'} \sum_{s=0}^{T'-1} \left ( \sigmaw^2 (T'-s) A^s (A^s)^\top + \sigma_u^2 (T'-s) A^s B B^\top (A^s)^\top \right ) \Bigg  \|_\op \\
& \qquad = \Bigg \| \sum_{s=0}^{T-1} \left ( \sigmaw^2 (1 - \frac{s}{T}) A^s (A^s)^\top + \sigma_u^2 (1 - \frac{s}{T}) A^s B B^\top (A^s)^\top \right ) \\
& \qquad \qquad \qquad \qquad - \sum_{s=0}^{T'-1} \left ( \sigmaw^2 (1 - \frac{s}{T'}) A^s (A^s)^\top + \sigma_u^2 (1 - \frac{s}{T'}) A^s B B^\top (A^s)^\top \right ) \Bigg \|_\op \\
& \qquad = \Bigg \| \sum_{s=0}^{T-1} \left (  \left ( \frac{s}{T'} - \frac{s}{T} \right ) (\sigmaw^2 A^s (A^s)^\top + \sigma_u^2 A^s B B^\top (A^s)^\top) \right ) \\
& \qquad \qquad \qquad \qquad  - \sum_{s=T}^{T'-1}  \left ( \sigmaw^2 (1 - \frac{s}{T'}) A^s (A^s)^\top + \sigma_u^2 (1 - \frac{s}{T'}) A^s B B^\top (A^s)^\top \right ) \Bigg \|_\op \\
& \qquad \le | {T'}^{-1} - T^{-1}| \tau(A,\rho)^2 (\sigmaw^2 + \sigma_u^2 \| B \|_\op^2) \sum_{s=0}^{T-1} s \rho^{2s} + (\sigmaw^2 + \sigma_u^2 \| B \|_\op^2) \tau(A,\rho)^2 \sum_{s=T}^{T'-1} (1-s/T') \rho^{2s} \\
& \qquad \le | {T'}^{-1} - T^{-1}| \tau(A,\rho)^2 (\sigmaw^2 + \sigma_u^2 \| B \|_\op^2) \frac{\rho^2  + \rho^{2T + 2} T}{(1 - \rho^2)^2} \\
& \qquad \qquad \qquad + \tau(A,\rho)^2 (\sigmaw^2 + \sigma_u^2 \| B \|_\op^2)  \frac{\rho^{2T'+2} + \rho^{2T+2} T + \rho^{2T} T'}{(1-\rho^2)^2 T'} \\
& \qquad  \le \frac{4\tau(A,\rho)^2 (\sigmaw^2 + \sigma_u^2 \| B \|_\op^2) }{(1-\rho^2)^2} \left (  T^{-1}   + \rho^{2T} \right ) 
\end{align*}
If we set $T$ large enough such that 
$$  \frac{4\tau(A,\rho)^2 (\sigmaw^2 + \sigma_u^2 \| B \|_\op^2) }{(1-\rho^2)^2}  T^{-1}  \le \epsilon/2, \qquad  \frac{4\tau(A,\rho)^2 (\sigmaw^2 + \sigma_u^2 \| B \|_\op^2) }{(1-\rho^2)^2} \rho^{2T} \le \epsilon/2 $$ 
the desired bound will hold. Rearranging these gives the result. 
\end{proof}

%% file: body/ce_upper_bound.tex
%!TEX root = ../main.tex

\section{Certainty Equivalence Decision Making with Sequential Open-Loop Policies}\label{sec:lds_ce_upper}

In this section we assume we are in the linear dynamical system setting of Section \ref{sec:overview_lds} and that we are playing an exploration policy $\piexp$.

\begin{proof}[Proof of Corollary \ref{cor:ce_upper_bound_nice}]
This is a direct consequence of Theorem \ref{thm:ce_upper_bound} and Lemma \ref{lem:exp_design_regular}. The stated results follows from some algebra to simplify terms and using Lemma \ref{lem:hinf_upper_bound}, to upper bound $\poly (\tauol, \frac{1}{1-\rhool} )$ terms by $\poly(\| \Bst \|_\op, \| \Ast \|_{\Hinf})$, and setting $\dimtheta = \dimx^2 + \dimx \dimu$, the dimensionality of $(A,B)$.
\end{proof}

\subsection{Sequential Open-Loop Policies Satisfy Assumption \ref{asm:minimal_policy}}\label{sec:sol_regular}

\begin{lem}\label{lem:exp_design_regular}
Any policy $\piexp \in \policysetgood$ meets Assumption \ref{asm:minimal_policy} on some system $\thetast$ with
\begin{align*}
    & \Texpse(\piexp) = c_1 \dimx \Big ( (\dimx + \dimu) \log(\gamup/\lammin(\Gamnoise_{\dimx}(\thetastbar,\sigma_u)) + 1) + \log \frac{n}{\delta} \Big ) \\
    & \lamund = c_2 \lammin(\Gamnoise_{\dimx}(\thetastbar,\sigma_u)) \\
    & \covup = \gamup \cdot I \\
    & \Texpcon(\piexp) = \max \left \{ \Texpse(\piexp), (n + \sqrt{\dimx} \sigmaw^2/\gamma^2) \log \frac{2(n+1)^2}{\delta} \right \} \\
    & \Ccon = n \frac{c_3 \tauol^3 (\sigmaw^2 + \gamma^2 ) }{(1-\rhool)^{5/2} \lammin(\Gamnoise_{\dimx}(\thetastbar,\sigma_u))} \sqrt{ \log \frac{n}{\delta} + \dimx + \dimu} \\
    & \alpha = 1/2
\end{align*}
for universal constants $c_1,c_2,c_3$.
\end{lem}
\begin{proof}[Proof of Theorem \ref{thm:ce_upper_bound}]
Let:
$$ \Sigma_{i} := \sum_{t=\tbreak_{i-1}}^{\tbreak_{i} - 1} z_t z_t^\top, i = 1,\ldots,n, \qquad \Sigma_{t:t'} := \sum_{s=t}^{t'} z_s z_s^\top  $$

\newcommand{\jund}{\underline{j}}
\newcommand{\jup}{\overline{j}}
\paragraph{Sufficient Excitation:} We first show that the sufficient excitation condition is met by $\piexp$. First, note that by Lemma \ref{lem:cov_up} and some algebra, for any time $t$, $\Sigma_t \preceq T \gamup I $ with probability at least $1-\delta/n$. Fix a time $t \ge \Texpse(\piexp)$ where $\Texpse(\piexp)$ is defined as above. Since $\piexp \in \policysetgood$, the low-switching condition implies that there exists some set of epochs $\{i_1,\ldots,i_m\} \subseteq [n]$, such that for $j \in [m]$, $\tbreak_{i_j +1} - \tbreak_{i_j} \ge \frac{1}{2} \Texpse(\piexp)$, $\tbreak_{i_j+1} \le t$, and 
\begin{align*}
    \sum_{j=1}^m (\tbreak_{i_j + 1} - \tbreak_{i_j}) \ge \frac{1}{2} t
\end{align*}
This follows directly from the fact that, for any $t_0$, there exists some epoch $i \in [n]$ such that $|\{t_0,t_0 + \Texpse(\piexp)-1\} \cap \{\tbreak_i,\ldots,\tbreak_{i+1}-1\}| \ge \frac{1}{2} \Texpse(\piexp)$. Now consider some $j \in [m]$. By Lemma \ref{lem:cov_lb_noise}, if
\begin{align}\label{eq:seq_good_policy_interval}
    \tbreak_{i_j+1}  - \tbreak_{i_j} \ge c_1 \dimx \Big ( (\dimx + \dimu) \log(\gamup/\lammin(\Gamnoise_{\dimx}(\thetastbar,\Lambda_{u,i_j}) + 1) + \log \frac{n}{\delta} \Big )
\end{align}
we will have that, with probability at least $1-\delta/n$, for some $c_2$,
\begin{align*}
    \Sigma_{\tbreak_{i_j}:\tbreak_{i_j+1}} \succeq c_2 (\tbreak_{i_j+1} - \tbreak_{i_j}) \lammin(\Gamnoise_{\dimx}(\thetastbar,\Lambda_{u,i_j})))
\end{align*}
However, by definition of $\Gamnoise$ and since by assumption $\lammin(\Lambda_{u,i_j}) \ge \sigma_u^2$, we will have $\lammin(\Gamnoise_{\dimx}(\thetastbar,\Lambda_{u,i_j}) \ge \lamnoise(\sigma_u)$, which implies that $\log(\gamup/\lamnoise(\sigma_u) + 1) \ge \log(\gamup/\lammin(\Gamnoise_{\dimx}(\thetastbar,\Lambda_{u,i_j})) + 1)$. As we know that $t_{i_j+1} - t_{i_j} \ge \frac{1}{2} \Texpse(\piexp)$, it follows that \eqref{eq:seq_good_policy_interval} is met, so we conclude that with probability at least $1-\delta/n$,
\begin{align*}
    \Sigma_{\tbreak_{i_j}:\tbreak_{i_j+1}} \succeq  c_2(\tbreak_{i_j+1} - \tbreak_{i_j}) \lamnoise(\sigma_u)
\end{align*}
Union bounding over this event holding for each $j \in [m]$, it follows that with probability at least $1-\delta$,
\begin{align*}
    \Sigma_t \succeq \sum_{j=1}^m \Sigma_{\tbreak_{i_j}:\tbreak_{i_j+1}} \succeq  \sum_{j=1}^m c_2(\tbreak_{i_j+1} - \tbreak_{i_j}) \lamnoise(\sigma_u)  \ge \frac{c_2}{2} t \lamnoise(\sigma_u)
\end{align*}
By proper choice of constants, we will then have that the sufficient excitation condition of Assumption \ref{asm:minimal_policy} is met with 
\begin{align*}
    & \Texpse(\piexp) = c_1 \dimx \Big ( (\dimx + \dimu) \log(\gamup/\lammin(\Gamnoise_{\dimx}(\thetastbar,\sigma_u)) + 1) + \log \frac{n}{\delta} \Big ) \\
    & \lamund = c_2 \lammin(\Gamnoise_{\dimx}(\thetastbar,\sigma_u)).
\end{align*}

\paragraph{Concentration of Covariates:} 
We define the following events for some $\epsilon_i$ to be specified.
\begin{align*}
\calE & = \{ \| \matSig_T - \Exp_{\thetast,\piexp}[\matSig_T] \|_\op \le \tfrac{\Ccon}{T^{\alpha}} \lammin(\Exp_{\thetast,\piexp}[\matSig_T]) \} \tag{Good event} \\
\calE_1 & = \{ \lammin(\matSig_T) \ge c_2 \lammin(\Gamnoise_{\dimx}(\thetastbar,\sigma_u)) T \} \tag{Sufficient excitation} \\
\calE_{2,i} & = \{ \| \matSig_i - \Exp_{\thetast,\piexp}[\matSig_i] \|_\op \le \epsilon_i \} \tag{Concentration of covariates}\\
\calE_{3,i} & = \left \{ \| z_{\tbreak_{i-1}} \|_2 \le   \sqrt{\tfrac{c \tauol^2 \gamma^2 T}{1-\rhool^2}} \right \} \tag{Bounded states}
\end{align*}
We would like to show that $\calE$ holds with high probability. The following is trivial.
$$ \Pr[\calE^c] \le \Pr[ \calE^c \cap \calE_1  \cap (\cap_{i=1}^{n} \calE_{2,i}) \cap (\cap_{i=1}^{n} \calE_{3,i})] + \Pr[\calE_1^c] + \sum_{i=1}^{n} \Pr[\calE_{3,i} \cap \calE_{2,i}^c] + \sum_{i=1}^{n} \Pr[\calE_{3,i}^c]$$
We first show that $\calE_1,\calE_{2,i},\calE_{3,i}$ hold with high probability.

Note first that, by what we have just shown, $\Pr[\calE_1^c] \le \delta$ as long as $T \ge \Texpse(\piexp)$.
By Lemma \ref{lem:ce_upper_state_bound}, since $\piexp \in \policysetgood$ \ref{asm:good_policy}, we will have that $\Pr[\calE_{3,i}^c] \le \delta$ as long as 
\begin{align}\label{eq:eq_policy_burnin1}
    T \ge (n + \sqrt{\dimx} \sigmaw^2/\gamma^2) \log \frac{2(n+1)}{\delta}
\end{align}
Next, we show that $\Pr[\calE_{3,i} \cap \calE_{2,i}^c] \le \delta$. Setting 
\begin{align*}
\epsilon_i =  &  \left ( \frac{c_3 \tauol \| \Gamnoise_T(\thetastbar,\Lambda_{u,i}) \|_\op }{1-\rhool} \sqrt{\tbreak_{i} - \tbreak_{i-1}} + \frac{c_4 \tauol^2  (\sqrt{T} \gamma  + \sqrt{(c \tauol^2 \gamma^2 T)(1-\rho^2)})}{(1-\rhool)^2} \max\{ \sigmaw, \sqrt{\| \Lambda_u \|_\op} \} \right )\\
& \qquad \qquad \cdot \sqrt{\log \frac{1}{\delta} + \dimx + \dimu }  + \frac{c_5 \max\{ \sigmaw^2, \| \Lambda_u \|_\op \} \tauol^2}{(1-\rhool)^2} (\log \frac{1}{\delta} + \dimx+\dimu) 
\end{align*}
Lemma \ref{lem:cov_op_concentration} with $\rho = \rhool$ implies directly that  $\Pr[\calE_{3,i} \cap \calE_{2,i}^c] \le \delta$. For future convenience, we can upper bound $\epsilon_i$ by
$$ \frac{c_3 \tauol^3 (\sigmaw^2 + \gamma^2 ) \sqrt{T} }{(1-\rhool)^{5/2}} \sqrt{ \log \frac{1}{\delta} + \dimx + \dimu} $$
as long as
\begin{align}\label{eq:eq_policy_burnin2}
    T \ge \log \frac{1}{\delta} + \dimx + \dimu
\end{align}
On the event $\calE_1  \cap (\cap_{i=1}^{n} \calE_{2,i}) \cap (\cap_{i=1}^{n} \calE_{3,i})$, we have
\begin{align*}
    \| \matSig_T - \Exp_{\thetast,\piexp}[\matSig_T] \|_\op & = \Big \| \sum_{i=1}^n \matSig_i - \sum_{i=1}^n \Exp_{\thetast,\piexp}[\matSig_i] \Big \|_\op \le \sum_{i=1}^n \| \matSig_i - \Exp_{\thetast,\piexp}[\matSig_i] \|_\op \le \sum_{i=1}^n \epsilon_i \\
    & \le n \frac{c_3 \tauol^3 (\sigmaw^2 + \gamma^2 ) \sqrt{T}}{(1-\rhool)^{5/2}} \sqrt{ \log \frac{1}{\delta} + \dimx + \dimu} 
\end{align*}
and
$$ \Exp_{\thetast,\piexp}[\matSig_T] \succeq \Exp_{\thetast,\piexp}[\I \{ \calE_1 \} \matSig_T] \ge \Pr[\calE_1] c_2 \lammin(\Gamnoise_{\dimx}(\thetastbar,\sigma_u)) T \cdot I \ge \frac{1}{2} \lammin(\Gamnoise_{\dimx}(\thetastbar,\sigma_u))  T \cdot I $$
Thus, $\Pr[\calE^c \cap \calE_1  \cap (\cap_{i=1}^{n} \calE_{2,i}) \cap (\cap_{i=1}^{n} \calE_{3,i})] = 0$ with $\alpha = 1/2$ and
\begin{align*}
    \Ccon := n \frac{c_3 \tauol^3 (\sigmaw^2 + \gamma^2 ) }{(1-\rhool)^{5/2} \lammin(\Gamnoise_{\dimx}(\thetastbar,\sigma_u))} \sqrt{ \log \frac{1}{\delta} + \dimx + \dimu}
\end{align*}
Thus, $\Pr[\calE^c] \le 2n+1$. Rescaling $\delta$ and setting $\Texpcon(\piexp)$ to guarantee \eqref{eq:eq_policy_burnin1} and \eqref{eq:eq_policy_burnin2} hold gives the result.

\end{proof}

\subsection{Concentration of Covariates}

\begin{lem}[Lemma E.3 of \cite{wagenmaker2020active}]\label{lem:cov_lb_noise}
Assume that our system $\theta$ is driven by some input $u_t = \util_t + \unoise_t$ where $\util_t$ is deterministic and $\unoise_t \sim \mathcal{N}(0, \Lambda_u)$. Then on the event that $ \sum_{t=1}^T z_t z_t^\top \preceq T \bar{\Gamma}_T$, for some $\bar{\Gamma}_T$, choosing $k$ so that:
\begin{equation}\label{eq:cov_lb_noise_burnin}
T \geq \frac{25600}{27} k \left (  2 (\dimx + \dimu) \log \frac{200}{3} + \log \det ( \bar{\Gamma}_T  (\Gamnoise_k)^{-1} ) + \log \frac{1}{\delta} \right )
\end{equation}
we will have with probability less than $\delta$:
$$ \sum_{t=1}^T z_t z_t^\top \not\succeq \frac{27}{25600} T \Gamnoise_k(\theta,\Lambda_u). $$
\end{lem}

\begin{lem}\label{lem:cov_up}
Assume that we are playing a policy $\piexp \in \policysetgood$ and that $\delta \in (0,1/3)$. Then with probability at least $1-\delta$, and assuming we start from some state $x_0=0$,
\begin{align*}
    \sum_{t=1}^T z_t z_t^\top & \preceq T \frac{c \tau(\Atilst,\rho)^2  }{1-\rho} \Big ( \sqrt{\dimx} \log \frac{T}{\delta} + \frac{ \gamma^2}{1-\rho} \Big ) \cdot I \\
    & \preceq c T \Big ( (1 + \| \Bst \|_\op^2) \| \Ast \|_{\Hinf}^4 \Big ) \Big ( \sqrt{\dimx} \log \frac{T}{\delta} + \gamma^2 \| \Ast \|_{\Hinf}^2 \Big ) \cdot I.
\end{align*}
\end{lem}
\begin{proof}
Note that we can break the state into the portion driven by the conditionally non-random input, $\util_t$, and the process noise and random input. We denote these components as $\zu_t$ and $\zw_t$. Then
\begin{align*}
    \Big \| \sum_{t=1}^T z_t z_t^\top \Big \|_\op \le \sum_{t=1}^T \| z_t \|_2^2 \le 2 \sum_{t=1}^T ( \| \zu_t \|_2^2 + \| \zw_t \|_2^2)
\end{align*}
Note that the input $\util_t$ will almost surely satisfy $\sum_{t=0}^{T-1} \util_t^\top \util_t \le T \gamma^2$. We can then apply Lemma \ref{lem:xtu_bound} to get that
\begin{align*}
   \sum_{t=1}^T \| \zu_t \|_2^2 \le  \frac{4 \tau(\Atilst,\rho)^2 \gamma^2 T}{(1 - \rho)^2}
\end{align*}
To bound the component $\sum_{t=1}^T \| \zw_t \|_2^2$, we apply Lemma \ref{lem:ce_upper_state_bound} with $\gamma^2 = 0$, union bounding over all $T$ steps. We simplify the bound by upper bounding $n$ by $T$ and using that $\delta \in (0,1/3)$ implies $\log T/\delta \ge 1$. The second bound follows by \Cref{lem:hinf_upper_bound}.
\end{proof}

\begin{lem}\label{lem:cov_op_concentration}
Consider the system
$$ x_{t+1} = A x_t + B u_t + \Lambda_w^{1/2} w_t $$
where $A \in \R^{d \times d}, w_t \sim \calN(0,I)$ and $u_t$ is deterministic and satisfies $\sum_{t=1}^T u_t^\top u_t \le T \gamma^2$. Assume that we start from some state $x_0$. Then we will have that, with probability at least $1-\delta$
\begin{align*}
\Big \| \sum_{t=0}^T x_t x_t^\top & - \Exp \sum_{t=0}^T x_t x_t^\top \Big \|_\op  \le \frac{c_3 \| \Lambda_w \|_\op \tau(A,\rho)^2}{(1-\rho)^2} (\log \frac{1}{\delta} + d) \\
& + \left ( \frac{c_1 \tau(A,\rho) \| \Gamnoise_T(\theta,0) \|_\op }{1-\rho} \sqrt{T} + \frac{c_2 \tau(A,\rho)^2 ( \sqrt{T} \gamma \| B \|_\op + \| x_0 \|_2)}{(1-\rho)^2} \sqrt{\| \Lambda_w \|_\op} \right ) \sqrt{\log \frac{1}{\delta} + d } 
\end{align*}
for universal constants $c_1,c_2,c_3$.
\end{lem}
\begin{proof}
Consider the systems
$$ \xu_{t+1} = A \xu_t + B u_t, \quad \xw_{t+1} = A \xw_t + \Lambda_w^{1/2} w_t $$
and note that $x_t = \xu_t + \xw_t$. Therefore,
$$ \sum_{t = 0}^T x_t x_t^\top = \sum_{t = 0}^T  \xu_t {\xu_t}^\top + \sum_{t = 0}^T \left ( \xu_t {\xw_t}^\top + \xw_t {\xu_t}^\top \right )  + \sum_{t=0}^T \xw_t {\xw_t}^\top, \quad  \Exp \sum_{t = 0}^T x_t x_t^\top = \sum_{t = 0}^T  \xu_t {\xu_t}^\top  + \Exp \sum_{t=0}^T \xw_t {\xw_t}^\top$$
The second equality is true as $\xu_t$ is deterministic and $\xw_t$ is mean 0. Fix some $v \in \calS^{d - 1}$. By Lemma \ref{lem:cov_mean_concentration} and Lemma \ref{lem:cov_mean_cross}, we'll have, simultaneously with probability $1-\delta$:
$$ \left | \sum_{t = 1}^T (v^\top \xw_t)^2 - \Exp \sum_{t = 1}^T (v^\top \xw_t)^2  \right | \le \frac{2 \tau(A,\rho) \| \Gamnoise_T(\theta,0) \|_\op }{1-\rho^2} \sqrt{T \log \frac{4}{\delta}} + \frac{8 \| \Lambda_w \|_\op \tau(A,\rho)^2}{(1-\rho)^2} \log \frac{4}{\delta}  $$
$$ \left | \sum_{t=1}^T v^\top \xu_t {\xw_t}^\top v \right | \le \frac{\tau(A,\rho)^2(4 \sqrt{T} \gamma \| B \|_\op + \| \xu_0 \|_2)}{(1 - \rho)^2} \sqrt{2 \| \Lambda_w \|_\op \log \frac{4}{\delta}}$$
Which implies that
\begin{align*}
\left | \sum_{t=0}^T (v^\top x_t)^2 - \Exp \sum_{t=0}^T (v^\top x_t)^2 \right | & \le \left | \sum_{t = 1}^T (v^\top \xw_t)^2 - \Exp \sum_{t = 1}^T (v^\top \xw_t)^2  \right | + 2 \left | \sum_{t=1}^T v^\top \xu_t {\xw_t}^\top v \right  | \\
& \le  \frac{2 \tau(A,\rho) \| \Gamnoise_T(\theta,0) \|_\op }{1-\rho^2} \sqrt{T \log \frac{4}{\delta}} + \frac{8 \| \Lambda_w \|_\op \tau(A,\rho)^2}{(1-\rho)^2} \log \frac{4}{\delta}  \\
& \qquad \qquad + \frac{2\tau(A,\rho)^2(4 \sqrt{T} \gamma \| B \|_\op + \| x_0 \|_2)}{(1 - \rho)^2} \sqrt{2 \| \Lambda_w \|_\op \log \frac{4}{\delta}}
\end{align*}
Note that if $M$ is symmetric $ \| M \|_\op = \sup_{v \in \calS^{d - 1}} | v^\top M v | $. Fix $v$ to be a vector for which this equality is attained. Let $\calT$ be an $\epsilon$-net of $\calS^{d-1}$.  Then we can then find some $v_0 \in \calT$ such that $\| v - v_0 \|_2 \le \epsilon$, and thus, 
$$ | v^\top M v - v_0^\top M v_0|  \le   | v^\top M v - v_0^\top M v| +  | v_0^\top M v - v_0^\top M v_0|  \le 2 \| M \|_\op \| v_0 - v \|_2 \le 2 \epsilon \| M \|_\op$$
Therefore,
$$ | v_0^\top M v_0| \ge | v^\top M v | - | v^\top M v - v_0^\top M v_0|  \le (1 - 2 \epsilon) \| M \|_\op $$
so $ \| M \|_\op \le \frac{1}{1-2\epsilon} \max_{v \in \calT } | v^\top M v|$. Applying this in our setting and choosing $\epsilon = 1/2$, gives
$$ \left \| \sum_{t=0}^T \xw_t {\xw_t}^\top - \Exp \sum_{t=0}^T \xw_t {\xw_t}^\top \right \|_\op \le 2 \max_{v \in \calT}  \left | \sum_{t = 1}^T (v^\top \xw_t)^2 - \Exp \sum_{t = 1}^T (v^\top \xw_t)^2  \right | $$
By Corollary 4.2.13 of \cite{vershynin2018high}, we will have $|\calT | \le 5^{d}$. Using our high probability bound on 

\noindent $\left | \sum_{t = 1}^T (v^\top x_t^w)^2 - \Exp \sum_{t = 1}^T (v^\top x_t^w)^2  \right |$ given above, and union bounding over $\calT$, we conclude that, with probability at least $1 - \delta$
\begin{align*}
& \left \| \sum_{t=0}^T x_t x_t^\top - \Exp \sum_{t=0}^T x_t x_t^\top \right \|_\op  \le  \frac{2 \tau(A,\rho) \| \Gamnoise_T(\theta,0) \|_\op }{1-\rho^2} \sqrt{T (\log \frac{4}{\delta} + d \log 5)} \\
& \qquad  + \frac{8 \| \Lambda_w \|_\op \tau(A,\rho)^2}{(1-\rho)^2} (\log \frac{4}{\delta} + d \log 5)  + \frac{2\tau(A,\rho)^2(4 \sqrt{T} \gamma \| B \|_\op + \| x_0 \|_2)}{(1 - \rho)^2} \sqrt{2 \| \Lambda_w \|_\op (\log \frac{4}{\delta} + d \log 5)}.
\end{align*}
\end{proof}

\begin{lem}\label{lem:cov_mean_concentration}
Consider the system
$$ x_{t+1} = A x_t + \Lambda_w^{1/2} w_t $$
where $A \in \R^{d \times d}, w_t \sim \calN(0, I)$, and assume $x_0 = 0$. Then, for any $v \in \calS^{d - 1}$, we'll have that, with probability at least $1 - \delta$
$$ \left | \sum_{t = 1}^T (v^\top x_t)^2 - \Exp \sum_{t = 1}^T (v^\top x_t)^2  \right | \le  \frac{2 \tau(A,\rho) \| \Gamnoise_T(\theta,0) \|_\op }{1-\rho^2} \sqrt{T \log \frac{2}{\delta}} + \frac{8 \| \Lambda_w \|_\op \tau(A,\rho)^2}{(1-\rho)^2} \log \frac{2}{\delta}.  $$
\end{lem}
\begin{proof}
This is a direct consequence of the Hanson-Wright Inequality. Note that, 
$$ x_t = \sum_{s = 0}^{t-1} A^{t-s-1} \Lambda_w^{1/2} w_s = G_t \frakw_t$$
where we have defined $G_t := [A^{t-1} \Lambda_w^{1/2},A^{t-2} \Lambda_w^{1/2},\ldots,A \Lambda_w^{1/2},\Lambda_w^{1/2},0,\ldots,0] \in \R^{d \times d T}$, $\frakw_T := [w_0^\top, w_1^\top,\ldots,w_{T-2}^\top,w_{T-1}^\top]^\top$. So,
$$ (v^\top x_t)^2 = \frakw_T^\top G_t^\top v v^\top G_t \frakw_T$$
and,
$$ \sum_{t=1}^T (v^\top x_t)^2 = \sum_{t = 1}^T  \frakw_T^\top G_t^\top v v^\top G_t \frakw_T = \frakw_T^\top \matG_T \frakw_T$$
where $\matG_T := \sum_{t=1}^T G_t^\top v v^\top G_t$. The Hanson-Wright inequality then immediately gives that,
$$ \Pr \left [ \left | \sum_{t = 1}^T (v^\top x_t)^2 - \Exp \sum_{t = 1}^T (v^\top x_t)^2  \right | \ge \epsilon \right ] \le 2 \exp \left ( - c \min \left \{ \frac{\epsilon^2}{ \| \matG_T \|_F^2}, \frac{\epsilon}{ \| \matG_T \|_\op} \right \} \right ) $$
For a fixed $\delta$, rearranging gives
$$ \Pr \left [ \left | \sum_{t = 1}^T (v^\top x_t)^2 - \Exp \sum_{t = 1}^T (v^\top x_t)^2  \right | \ge 2 \sqrt{\| \matG_T \|_F^2 \log(2/\delta)} + 2\| \matG_T \|_\op \log(2/\delta) \right ] \le \delta $$
We proceed to bound $\| \matG_T \|_\op$ and $\| \matG_T \|_F^2$. Consider some $u \in \calS^{d T-1}$ and note that, if we write $u = [u_0^\top, u_1^\top, \ldots, u_{T-2}^\top, u_{T-1}^\top]^\top$, where $u_i \in \R^{d}$, using the definition of $G_t$ given above, we have:
$$ G_t u  =  \sum_{s=0}^{t-1}  A^{t-s-1} \Lambda_w^{1/2} u_s =  \xu_t$$
where $x_t^u$ is the state of the system with matrix $A$ when the input $u$ is played and there is no noise. Thus,
$$ u^\top \matG_T u = \sum_{t = 1}^T (v^\top G_t u)^2 = \sum_{t =1}^T (v^\top \xu_t)^2 \le \lambda_{\max} \left ( \sum_{t=1}^T \xu_t {\xu_t}^\top \right ) $$
Then, invoking Lemma \ref{lem:xtu_bound} with $\gamma^2 = \| \Lambda_w \|_\op$ and $B = I$, we can bound,
$$ \lambda_{\max} \left ( \sum_{t=1}^T \xu_t {\xu_t}^\top \right )  \le \sum_{t = 1}^T \| \xu_t \|_2^2 \le \frac{4 \tau(A,\rho)^2  \| \Lambda_w \|_\op}{(1-\rho)^2} $$
As this does not depend on $u$, it is a valid bound on $\| \matG_T \|_\op$:
\begin{equation}\label{eq:matAtil_bound}
\| \matG_T \|_\op \le \frac{4 \tau(A,\rho)^2  \| \Lambda_w \|_\op}{(1-\rho)^2} 
\end{equation}
To bound $\| \matG_T \|_F^2$, we can write,
\begin{align*}
\| \matG_T \|_F^2  = \tr(\matG_T^\top \matG_T) & =  \sum_{t=1}^T \sum_{s=1}^T \tr( G_t^\top vv^\top G_t G_s^\top vv^\top G_s)  = \sum_{t=1}^T \sum_{s=1}^T (v^\top G_t G_s^\top v)^2 \le \sum_{t = 1}^T \sum_{s = 1}^T \| G_t G_s^\top \|_\op^2
\end{align*}
From the definition of $G$, we have,
$$ G_t G_s^\top = A^{\max \{ t - s, 0 \}} \left ( \sum_{k = 0}^{\min \{t,s \}} A^k \Lambda_w {A^k}^\top \right ) {A^{\max \{ s - t, 0 \}}}^\top =  A^{\max \{ t - s, 0 \}} \Gamnoise_{\min\{t,s\}}(\theta,0) {A^{\max \{ s - t, 0 \}}}^\top$$
so,
$$ \| G_t G_s^\top \|_\op \le \| A^{\max \{ t - s, 0 \}}  \|_\op \| A^{\max \{ s - t, 0 \}} \|_\op  \| \Gamnoise_{\min\{t,s\}}(\theta,0)  \|_\op \le \tau(A,\rho) \rho^{|t-s|}  \| \Gamnoise_{\min\{t,s\}}(\theta,0)  \|_\op $$
which implies
\begin{align*}
\sum_{t = 1}^T \sum_{s=1}^T \| G_t G_s^\top \|_\op^2 & \le \tau(A,\rho)^2 \| \Gamnoise_T(\theta,0) \|_\op^2  \sum_{t=1}^T \sum_{s = 1}^T \rho^{2|t - s|} \\
& = \tau(A,\rho)^2 \| \Gamnoise_T(\theta,0) \|_\op^2 \frac{(1-\rho^4) T + 2 \rho^{2(T+1)} - 2 \rho^2}{(1-\rho^2)^2} \le \frac{\tau(A,\rho)^2 \| \Gamnoise_T(\theta,0) \|_\op^2 T}{(1 - \rho^2)^2}
\end{align*}
Combining everything, we have shown that, for any $v \in \calS^{d-1}$,
$$ \Pr \left [ \left | \sum_{t = 1}^T (v^\top x_t)^2 - \Exp \sum_{t = 1}^T (v^\top x_t)^2  \right | \ge  \frac{2 \tau(A,\rho) \| \Gamnoise_T(\theta,0) \|_\op }{1-\rho^2} \sqrt{T \log \frac{2}{\delta}} + \frac{8 \| \Lambda_w \|_\op \tau(A,\rho)^2}{(1-\rho)^2} \log \frac{2}{\delta}    \right ] \le \delta $$
\end{proof}

\begin{lem}\label{lem:cov_mean_cross}
Consider the systems
$$ \xu_{t+1} = A \xu_t + B u_t, \quad \xw_{t+1} = A \xw_t + \Lambda_w^{1/2} w_t $$
where $A \in \R^{d \times d}, w_t \sim \calN(0,I)$ and $u_t$ a deterministic signal with $\sum_{t=1}^T u_t^\top u_t \le T \gamma^2$. Assume that $\xw_0 = 0$. Then, for any $v \in \calS^{d - 1}$, we will have that, with probability at least $1 - \delta$
$$ \left | \sum_{t=1}^T v^\top \xu_t {\xw_t}^\top v \right | \le \frac{\tau(A,\rho)^2(4 \sqrt{T} \gamma \| B \|_\op + \| \xu_0 \|_2)}{(1 - \rho)^2} \sqrt{2 \| \Lambda_w \|_\op \log \frac{2}{\delta}}.$$
\end{lem}
\begin{proof}
\newcommand{\Gu}{G^u}
We adopt the same notation as in the proof of Lemma \ref{lem:cov_mean_concentration}. Defining
$$\Gu_t := [A^{t-1} B, A^{t-2} B, \ldots, AB, B, 0, \ldots, 0] \in \R^{d \times \dimu T}, \quad \fraku_T := [u_0^\top, u_1^\top, \ldots, u_{T-2}^\top, u_{T-1}^\top]^\top \in \R^{\dimu T}$$
we have
$$ \xu_t = G_t^u \fraku_T + A^t \xu_0 , \quad \xw_t = G_t \frakw_T $$
which implies
$$ \sum_{t=1}^T v^\top \xu_t {\xw_t}^\top v = \left ( \fraku_T^\top  \sum_{t=1}^T (G_t^u)^\top v v^\top G_t + {x_0^u}^\top \sum_{t = 1}^T {A^t}^\top v v^\top G_t \right ) \frakw_T  =: \matg^\top \frakw_T \sim \calN  ( 0, \matg^\top  \matg) $$
By standard Gaussian concentration results, we then have that
$$ \Pr \left [ \left | \sum_{t=1}^T v^\top \xu_t {\xw_t}^\top v \right | \ge \sqrt{2 \matg^\top  \matg \log \frac{2}{\delta}} \right ] \le \delta $$
It remains to bound $\matg^\top  \matg$. To this end, note that
$$  \matg^\top  \matg \le \left ( \left \| \fraku_T^\top  \sum_{t=1}^T (G_t^u)^\top v v^\top G_t \right \|_2 + \left \| {\xu_0}^\top \sum_{t = 1}^T {A^t}^\top v v^\top G_t \right \|_2 \right )^2 $$
We can bound $\| A^t \|_\op \le \tau(A,\rho) \rho^t$ and,
$$ \| G_t \|_\op \le \| \Lambda_w^{1/2} \|_\op \sum_{s = 0}^{t-1} \| A^s \|_\op \le \| \Lambda_w^{1/2} \|_\op \tau(A,\rho) \sum_{s = 0}^{t-1} \rho^s \le \frac{\| \Lambda_w^{1/2} \|_\op \tau(A,\rho)}{1-\rho} $$
so,
$$ \left \| {\xu_0}^\top \sum_{t = 1}^T {A^t}^\top v v^\top G_t \right \|_2 \le \frac{\| \xu_0 \|_2 \| \Lambda_w^{1/2} \|_\op \tau(A,\rho)^2}{1-\rho} \sum_{t = 1}^T \rho^t \le \frac{\| \xu_0 \|_2 \| \Lambda_w^{1/2} \|_\op \tau(A,\rho)^2}{(1-\rho)^2}$$
Furthermore, letting $G_t' = [A^{t-1} ,A^{t-2} ,\ldots,A ,I,0,\ldots,0]$, we have
$$ \left \| \fraku_T^\top  \sum_{t=1}^T (G_t^u)^\top v v^\top G_t \right \|_2 \le  \| \fraku_T \|_2 \| B \|_\op \| \Lambda_w^{1/2} \|_\op \left \| \sum_{t=1}^T (G_t')^\top v v^\top G_t' \right \|_2 \le \| \fraku_T \|_2 \| B \|_\op \| \Lambda_w^{1/2} \|_\op \| \matG_T' \|_\op$$
where $\matG_T' = \sum_{t=1}^T (G_t')^\top v v^\top G_t'$. By \eqref{eq:matAtil_bound}, $\| \matG_T' \|_\op \le 4\tau(A,\rho)^2/(1-\rho)^2$. Since we have assumed that  $\sum_{t=1}^T u_t^\top u_t \le T \gamma^2$, we also have $\| \fraku_T \|_2 \le \sqrt{T} \gamma$. Combining everything, we have shown that
$$ \matg^\top \matg  \le \| \Lambda_w \|_\op \left ( \frac{4 \sqrt{T} \gamma \| B \|_\op \tau(A,\rho)^2}{(1 - \rho)^2} + \frac{\| \xu_0 \|_2 \tau(A,\rho)^2}{(1-\rho)^2} \right )^2 $$
Thus,
$$ \Pr \left [ \left | \sum_{t=1}^T v^\top \xu_t {\xw_t}^\top v \right | \ge \frac{\tau(A,\rho)^2(4 \sqrt{T} \gamma \| B \|_\op + \| \xu_0 \|_2)}{(1 - \rho)^2} \sqrt{2 \| \Lambda_w \|_\op \log \frac{2}{\delta}} \right ] \le \delta $$
\end{proof}

\subsection{State Norm Bounds}\label{sec:state_norm_bounds}

\begin{lem}\label{lem:xtu_bound}
Consider the system
$$ x_{t+1} = A x_t + B u_t $$
and assume that we start at state $x_0 = 0$. Then if $\sum_{t = 0}^{T-1} u_t^\top u_t \le T \gamma^2$, we will have
$$ \sum_{t=1}^T \| x_t \|_2^2 \le \frac{4 \tau(A,\rho)^2 \| B \|_\op^2 \gamma^2 T}{(1 - \rho)^2}. $$
\end{lem}
\begin{proof}
By definition $ x_T = \sum_{s = 0}^{T-1} A^{T-s-1} B u_s $, so
\begin{align*}
\sum_{t=1}^T \| x_t \|_2^2 & \le \sum_{t=1}^T \left ( \sum_{s=0}^{t-1}  \| A^{t-s-1} \|_\op  \| B \|_\op \| u_s \|_2 \right )^2  \le \tau(A,\rho)^2 \| B \|_\op^2 \sum_{t=1}^T \left ( \sum_{s = 0}^{t-1} \rho^{t-s-1} \| u_s \|_2 \right )^2
\end{align*}
Letting $\rho_t = \rho^t$, $v_t = \| u_t \|_2$, we define $y_t = \sum_{s = 0}^{t-1} \rho^{t-s-1} \| u_s \|_2 = (\rho * v)[t]$, where $*$ denotes convolution. By Parseval's Theorem,
$$ \sum_{t=1}^T \left ( \sum_{s = 0}^{t-1} \rho^{t-s-1} \| u_s \|_2 \right )^2 = \sum_{t = 1}^T y_t^2 = \frac{1}{T} \sum_{k=1}^T |Y_k|^2$$
where $Y_k$ denotes the DFT of $y_t$. As convolution in the time domain is multiplication in the frequency domain, we will have $Y_k = P_k V_k$ where $P_k$ is the DFT of $\rho_t$ and $V_k$ is the DFT of $v_t$. We can explicitly calculate $P_k$ as:
$$ P_k = \sum_{t = 0}^{T-1} \rho^t e^{-\imag \frac{2\pi k t}{T}} = \frac{1 - e^{-\imag 2 \pi k} \rho^T}{1 - e^{-\imag \frac{2\pi k}{T}} \rho}$$
Thus, 
$$ \frac{1}{T} \sum_{k=1}^T |Y_k|^2 = \frac{1}{T} \sum_{k = 1}^T \left | \frac{1 - e^{-\imag 2 \pi k} \rho^T}{1 - e^{-\imag \frac{2\pi k}{T}} \rho} \right |^2 | V_k |^2 $$
Note that, also by Parseval's Theorem, the constraint $\sum_{t = 1}^T \| u_t \|_2^2 \le \gamma^2 $ translates to $\frac{1}{T} \sum_{k=1}^T | V_k |^2 \le \gamma^2$. So,
\begin{align*}
    \frac{1}{T} \sum_{k = 1}^T \left | \frac{1 - e^{-\imag 2 \pi k} \rho^T}{1 - e^{-\imag \frac{2\pi k}{T}} \rho} \right |^2 | V_k |^2 & \le \max_{z : \| z \|_1 \le T^2 \gamma^2} \frac{1}{T} \sum_{k = 1}^T \left | \frac{1 - e^{-\imag 2 \pi k} \rho^T}{1 - e^{-\imag \frac{2\pi k}{T}} \rho} \right |^2 z_k \\
    & = \gamma^2 T \max_{k \in \{1,\ldots,T\}} \left | \frac{1 - e^{-\imag 2 \pi k} \rho^T}{1 - e^{-\imag \frac{2\pi k}{T}} \rho} \right |^2 \le \frac{4\gamma^2 T}{(1-\rho)^2}
\end{align*}
The conclusion follows.
\end{proof}

\begin{lem}\label{lem:ce_upper_state_bound}
Assume that we are playing a policy $\piexp \in \policysetgood$. Then with probability at least $1-\delta$, assuming $z_0 = 0$,
$$ \|z_t \|_2^2 \le \frac{c \tau(\Atilst,\rho)^2}{1-\rho^2} \Big (  \gamma^2 T + (\sqrt{\dimx} \sigmaw^2 + \sqrt{n T} \gamma^2) \sqrt{ \log \frac{2(n+1)}{\delta}} + (\sigmaw^2 + n \gamma^2) \log \frac{2(n+1)}{\delta} \Big )$$
and if
$$ T \ge (n + \sqrt{\dimx} \sigmaw^2/\gamma^2) \log \frac{2(n+1)}{\delta} $$
this bound can be simplified to
$$ \|z_t \|_2^2 \le \frac{c \tau(\Atilst,\rho)^2 \gamma^2 T}{1- \rho^2}.   $$
\end{lem}
\begin{proof}
By Assumption \ref{asm:good_policy}, for $s$ in epoch $j$, we can always write the input $u_s$ as $u_s = \util_s + \unoise_s$, where $\util_s$ is $\calF_{\tbreak_j}$ measurable and $\unoise_s \sim \calN(0,\Lambda_{u,j})$. Given this, we break the state up into the component driven by $\util_s$, which we denote as $\zu_t$, and the component driven by the process noise and $\unoise_s$, which we denote as $\zw_t$. By linearity, we will have that $z_t = \zu_t + \zw_t$, so 
$$ \| z_t \|_2^2 \le 2 \| \zu_t \|_2^2 + 2 \| \zw_t \|_2^2 $$
We can easily bound $\| \zu_t \|_2^2$ as:
\begin{align*}
    \| \zu_t \|_2^2 & = \Big \| \sum_{s=0}^{t-1} \Atilst^{t-s-1} \Btilst u_s \Big \|_2^2 \le \tau(\Atilst,\rho)^2 \Big ( \sum_{s=0}^{t-1} \rho^{t-s-1} \| u_s \|_2 \Big )^2 \\
    & \le \tau(\Atilst,\rho)^2 \Big ( \sum_{s=0}^{t-1} \rho^{2(t-s-1)} \Big ) \Big ( \sum_{s=0}^{t-1} \| u_s \|_2^2 \Big ) \le \frac{\tau(\Atilst,\rho)^2 \gamma^2 T}{1 - \rho^2}
\end{align*}
where the final inequality follows since, by assumption, $\sum_{s=0}^{t-1} \| u_s \|_2^2 \le T \gamma^2$ almost surely. We now bound $\| \zw_t \|_2^2$. Note that due to the possible correlations between $\Lambda_{u,j}$ and previous epochs, we cannot naively apply Gaussian concentration. We first upper bound $\| \zw_t \|_2^2$ as
\begin{align*}
    \| \zw_t \|_2^2 & = \Big \| \sum_{s=0}^{t-1} \Atilst^{t-s-1}(\Btilst u^w_s + w_s) \Big \|_2^2 \le 2 \Big \| \sum_{s=0}^{t-1} \Atilst^{t-s-1} \Btilst u^w_s  \Big \|_2^2 + 2 \Big \| \sum_{s=0}^{t-1} \Atilst^{t-s-1} w_s \Big \|_2^2 \\
    & \le 2 \tau(\Atilst,\rho)^2 \Big ( \sum_{s=0}^{t-1} \rho^{t-s-1} \| u_s^w \|_2 \Big )^2  + 2 \Big \| \sum_{s=0}^{t-1} \Atilst^{t-s-1} w_s \Big \|_2^2   \\
    & \le 2 \tau(\Atilst,\rho)^2  \Big ( \sum_{s=0}^{t-1} \rho^{2(t-s-1)} \Big ) \Big ( \sum_{s=0}^{t-1} \| u_s^w \|_2^2 \Big ) + 2 \Big \| \sum_{s=0}^{t-1} \Atilst^{t-s-1} w_s \Big \|_2^2
\end{align*}
We note that $ \| \sum_{s=0}^{t-1} \Atilst^{t-s-1} w_s \|_2^2$ is simply the norm of the the state of a dynamical system driven by noise $w_s$. We can therefore apply Lemma \ref{lem:xT_norm_bound} to get that with probability at least $1-\delta$,
\begin{align*}
    \Big \| \sum_{s=0}^{t-1} \Atilst^{t-s-1} w_s \Big \|_2^2 \le 4 \sqrt{\| \Gamnoise_t(\thetastbar,0) \|_F^2 \log \frac{2}{\delta}} + 4 \| \Gamnoise_t(\thetastbar,0) \|_\op \log \frac{2}{\delta}
\end{align*}
We can upper bound
\begin{align*}
    & \| \Gamnoise_t(\thetastbar,0) \|_\op = \sigmaw^2 \Big \| \sum_{s=0}^{t-1} \Atilst^t (\Atilst^t)^{\top} \Big \|_\op \le \sigmaw^2 \tau(\Atilst,\rho)^2 \sum_{s=0}^{t-1} \rho^{2t} \le \frac{\sigmaw^2 \tau(\Atilst,\rho)^2}{1-\rho^2} \\
    & \| \Gamnoise_t(\thetastbar,0) \|_F^2 \le \dimx \| \Gamnoise_t(\thetastbar,0) \|_\op^2
\end{align*}
To bound $\sum_{s=0}^{t-1} \| u_s^w \|_2^2$ we can apply Hanson-Wright to some epoch $j$ to get that
\begin{align*}
    \sum_{s=\tbreak_j}^{\tbreak_{j+1}-1} \| u_s^w \|_2^2 & \le \Exp \sum_{s=\tbreak_j}^{\tbreak_{j+1}-1} \| u_s^w \|_2^2 + 2 \sqrt{ (\tbreak_{j+1}-\tbreak_j) \| \Lambda_{u,j} \|_F^2 \log \frac{2}{\delta}} + 2 \| \Lambda_{u,j} \|_\op \log \frac{2}{\delta} \\
    & = \sum_{s=\tbreak_j}^{\tbreak_{j+1}-1} \tr(\Lambda_{u,j}) + 2 \sqrt{ (\tbreak_{j+1}-\tbreak_j) \| \Lambda_{u,j} \|_F^2 \log \frac{2}{\delta}} + 2 \| \Lambda_{u,j} \|_\op \log \frac{2}{\delta} \\
    & \le (\tbreak_{j+1}-\tbreak_j) \gamma^2 + 2 \sqrt{(\tbreak_{j+1}-\tbreak_j) \gamma^4 \log \frac{2}{\delta}} + 2 \gamma^2 \log \frac{2}{\delta}
\end{align*}
Assume that $t$ occurs in epoch $i$. Then if this bound holds for all epoch $j \le i$, we can bound
\begin{align*}
    \sum_{s=0}^{t-1} \| u_s^w \|_2^2 \le t \gamma^2 + \sum_{j=0}^i \sqrt{(\tbreak_{j+1}-\tbreak_j) \gamma^4 \log \frac{2}{\delta}} + 2 i \gamma^2 \log \frac{2}{\delta} \le t \gamma^2 + \sqrt{i t \gamma^4 \log \frac{2}{\delta}} + 2 i \gamma^2 \log \frac{2}{\delta}
\end{align*}
Union bounding over all $n$ epochs and the bound on the process noise, we then have that with probability at least $1-\delta$,
\begin{align*}
    \| \zw_t \|_2^2 \le \frac{c \tau(\Atilst,\rho)^2}{1-\rho^2} \Big (  \gamma^2 t + (\sqrt{\dimx} \sigmaw^2 + \sqrt{i t} \gamma^2) \sqrt{ \log \frac{2(n+1)}{\delta}} + (\sigmaw^2 + i \gamma^2) \log \frac{2(n+1)}{\delta} \Big )
\end{align*}
The result the follows by combining this with our bound on $\| \zu_t \|_2^2$, and upper bounding $t$ by $T$ and $i$ by $n$. The simplified bound holds by noting that for large enough $T$, we can upper bound the two lower order terms in the bound on $\| \zw_t \|_2^2$ by $(\dimx \sigmaw^2 + \gamma)T$.
\end{proof}

\begin{lem}\label{lem:xT_norm_bound}
Consider the system
$$ x_{t+1} = A x_t  + \Lambda_w^{1/2} w_t $$
where $w_t \sim \calN(0,I)$ and assume that we start at state $x_0 = 0$. Then, with probability at least $1-\delta$
$$ \| x_T \|_2^2 \le  2 \sqrt{ \| \Gamnoise_T(\theta,0) \|_F^2 \log \frac{2}{\delta}} + 2 \| \Gamnoise_T(\theta,0) \|_\op \log \frac{2}{\delta}. $$
\end{lem}
\begin{proof}
Using the same notation as in the proof of Lemma \ref{lem:cov_mean_concentration} and \ref{lem:cov_mean_cross}, we will have that $ x_T = G_T \frakw_T$. Applying Hanson-Wright then gives that, with probability at least $1-\delta$,
$$ \left | {x_T}^\top x_T - \tr(G_T^\top G_T) \right | \le 2\sqrt{\| G_T^\top G_T \|_F^2 \log \frac{2}{\delta}} + 2 \| G_T^\top G_T \|_\op \log \frac{2}{\delta} $$
By definition of $G_T$ we have that
$$ G_T^\top G_T = \sum_{s=0}^{T-1} \Lambda_w^{1/2} (A^s)^\top A^s \Lambda_w^{1/2}, \quad  G_T G_T^\top = \sum_{s = 0}^{T-1} A^s \Lambda_w (A^s)^\top = \Gamnoise_T(\theta,0)$$
So,
$$ \| G_T^\top G_T \|_\op = \|  G_T G_T^\top  \|_\op = \| \Gamnoise_T(\theta,0) \|_\op, \quad \| G_T^\top G_T \|_F^2 = \tr ( G_T G_T^\top G_T G_T^\top ) = \| \Gamnoise_T(\theta,0) \|_F^2$$
This concludes the proof.
\end{proof}

\begin{thm}[Hanson-Wright Inequality, \cite{vershynin2018high}]
Let $X \in \R^d$ be a random vector with independent, mean-zero, sub-Gaussian coordinates. Let $A \in \R^{d \times d}$. Then, for every $\epsilon \ge 0$, we have
$$ \Pr \left [ | X^\top A X - \Exp X^\top A X| \ge \epsilon \right ] \le 2 \exp \left ( -c \min \left \{ \frac{\epsilon^2}{K^4 \| A \|_F^2}, \frac{\epsilon}{K^2 \| A \|_\op} \right \} \right ) $$
where $K = \max_i \| X_i \|_{\psi_2}$.
\end{thm}
Recall that, if $X_i$ is gaussian with variance $\sigmaw^2$, $\| X_i \|_{\psi_2} \le C \sigmaw$.

%% file: body/algorithm.tex
%!TEX root = ../main.tex

\section{Experiment Design in Linear Dynamical Systems}\label{sec:lds_exp_design}

\subsection{Proof of Theorem \ref{thm:tople_upper_formal}}\label{sec:tople_up_pf}

\begin{proof}
Fix an epoch $i$ and let $T = \sum_{j=0}^i T_j$. Note that, by the definition of $T_i$, we will have $T_i = \frac{1}{2} (T + T_0)$. Similarly, $T_{i-1} =  \frac{1}{4} (T + T_0)$. Define the following events.
\begin{align*}
    & \calE_1 = \left \{ \| \thetahat_{i-1} - \thetast \|_\op \le C \sqrt{\frac{ \log(1/\delta) + (\dimx + \dimu) \log(\gamup/\lamnoise + 1)}{T_{i-1} \lamnoise}} =: \epsop \right \} \\
    & \calE_2 = \left \{ \calR(\aopt(\thetahat_i);\thetast) \le 5 \sigma_w^2 \left ( \taskhes(\thetast) \Exp_{\thetast}[\matSig_T]^{-1} \right ) \log \frac{24(\dimx^2 + \dimx \dimu)}{\delta} + \frac{C_1}{T^{3/2}} + \frac{C_2}{T^2} \right \} \\
    & \calE_3  = \left \{ \| z_{T-T_i} \|_2^2 \le \frac{4 \tauol^2 \gamma^2 k_i^2}{1 - \rhool^{k_i}} + 8 \sqrt{ \| \Gamnoisetil_T \|_F^2 \log \frac{2}{\delta}} + 8 \| \Gamnoisetil_T \|_\op \log \frac{2}{\delta} \right \}
\end{align*}
for $C_1,C_2$ as defined in Corollary \ref{cor:ce_upper_bound_nice}.

\paragraph{Events $\calE_i$ hold:} By Lemma \ref{lem:tople_regular}, we know that \algname $\in \policysetgood$. By Lemma \ref{lem:exp_design_regular}, this implies that \algname satisfies Assumption \ref{asm:minimal_policy} with 
\begin{align*}
    & \Texpse(\algname) = c_1 \dimx \Big ( (\dimx + \dimu) \log(\gamup/\lamnoise + 1) + \log \frac{n}{\delta} \Big ), \quad \lamund = c_2 \lamnoise, \quad \covup = \gamup \cdot I
\end{align*}
Thus, as long as
\begin{align}\label{eq:tople_upper_burnin1}
T \ge \Texpse(\algname)
\end{align}
we will have that with probability at least $1-\delta$, $\lammin(\Sigma_T) \ge c_2 \lamnoise T$ and $\Sigma_T \preceq T \gamup I$. We can therefore apply Lemma \ref{lem:general_op_bound2}, our operator norm estimation bound\footnote{Note that we could have instead employed \Cref{lem:general_op_bound} to upper bound $\| \thetahat_{i-1} - \thetast \|_F$. By exploiting the matrix structure of $\thetast$ and using an operator norm bound instead, we are able to save a factor of dimensionality in the burn-in time.}, to get that $\Pr[\calE_1^c] \le \delta$. Furthermore, by Corollary \ref{cor:ce_upper_bound_nice}, we will have, as long as $T$ is large enough for the burn-in, \eqref{eq:ce_opt_burnin_gsed} to be met, that $\Pr[\calE_2^c] \le \delta$. To show that $\calE_3$ occurs with high probability, we break up the state into two components: $\zu_t$, the portion of the state driven by $\util_t$, and $\zw_t$, the portion of the state driven by the input noise and process noise. As the structure of \algname is identical to that of the algorithm considered in \cite{wagenmaker2020active}, Lemma D.7 of \cite{wagenmaker2020active} gives that
\begin{align*}
    \| \zu_{T-T_i} \|_2^2 \le \frac{4 \tau(\Atilst,\rho)^2 k_i^2 \gamma^2}{(1-\rho^{k_i})^2}
\end{align*}
and we choose $\rho = \rhool$. Note that while this result is stated as a high-probability bound, since we are only considering the non-random portion of the input, it will hold deterministically. Crucially for subsequent steps, this scales as $k_i^2$ instead of $T$, which is the scaling we would obtain applying Lemma \ref{lem:ce_upper_state_bound} would scale. Next, applying Lemma \ref{lem:xT_norm_bound}  gives that, with probability $1-\delta$,
\begin{align*}
    \| \zw_{T-T_i} \|_2^2 \le 4 \sqrt{ \| \Gamnoisetil_T \|_F^2 \log \frac{2}{\delta}} + 4 \| \Gamnoisetil_T \|_\op \log \frac{2}{\delta}
\end{align*}
Note that we can apply Lemma \ref{lem:xT_norm_bound} since the input noise variance is deterministically fixed for all epochs, and by upper bounding the state bound for epochs $i \ge 1$ by the state bound that would hold if we always set the input noise to have variance $\gamma^2/\dimu$. This implies $\Pr[\calE_3^c] \le \delta$. Altogether then, we have that $\Pr[\calE_1 \cap \calE_2 \cap \calE_3] \ge 1 - 3 \delta$.

\paragraph{Events $\calE_i$ imply optimal inputs:}
We now assume that $\calE_1 \cap \calE_2 \cap \calE_3$ holds. Assume that $T_i$ is large enough that
\begin{align}\label{eq:tople_upper_burnin2}
    \epsop \le \min \{ \betaexplds(\thetast), \betast(\thetast)/\sqrt{\dimx}, \lamnoise/(4\alphastlds(\thetast,\gamma^2))  \}
\end{align}
Then, as long as,
\begin{align}
    & T_i \ge \poly \left ( \frac{1}{1-\rhool}, \tauol, \| \thetastbar \|_{\Hinf} \right ) \frac{\dimu \| z_{T-T_i} \|_2^2 + \dimu^{3/2} \gamma^2 k_i \sqrt{T_i} + \dimu^2 \gamma^2 k_i^2 + \sigmaw^2}{\lamnoise} \label{eq:tople_upper_tempburnin3} \\
    & k_i \geq \max \left \{ \frac{80 \pi \| \thetastbar \|_{\Hinf} \gamma^2}{ \lamnoise}, \frac{\pi}{2 \| \thetastbar \|_{\Hinf}} \right \} \left ( \max_{\omega \in [0,2\pi]} \| (e^{\imag\omega} I - \Atilst)^{-2} \Btil \|_\op \right ) \label{eq:tople_upper_tempburnin4} 
\end{align}
we can apply Lemma \ref{lem:global_optimal_inputs}, which gives that the performance achieved by $\bmU_i$ is nearly optimal. That is, for any $T' \ge T_i$, 
$$ \tr(\taskhes(\thetast) (\Exp_{\thetast}[\matSig_T])^{-1}) \le \tr(\taskhes(\thetast) (\Exp_{\thetast}[I_{\dimx} \otimes \tsum_{t=T-T_i}^T z_t z_t^\top])^{-1}) \le \min_{\bmu \in \calU_{\gamma^2,T'}} \frac{3 \tr \left (\taskhes(\thetast) \matGamss_{T',T'}(\thetastbar,\bmu,0)^{-1} \right )}{T_i} + 2 \Cexp $$
where
\begin{align*}
\Cexp & = \left ( \frac{4\sqrt{\dimx}(\dimx^2+\dimx \dimu) \Lra }{\lamnoise}+ \frac{8\alphastlds(\thetast,\gamma^2) \tr(\taskhes(\thetast))}{(\lamnoise)^2} \right ) \frac{ \epsop}{T_i}   + \frac{8 (\dimx^2+\dimx \dimu)  \alphastlds(\thetast,\gamma^2) \Lra}{(\lamnoise)^2} \frac{ \epsop^2}{T_i}. 
\end{align*}
and we have chosen $\epsilon$ such that $1/(1-\epsilon)^3 = 3/2$. Recall that Algorithm \ref{alg:lqr_simple_regret} uses $T_i =  \Cinit \dimu 2^i$ and $k_i =  \Cinit 2^{\lfloor i/4 \rfloor}$. We then have that,
$$ \frac{\Cinit^{3/4}}{\dimu^{1/4}} T_i^{1/4} \ge \Cinit 2^{i/4} \ge k_i \ge \Cinit 2^{i/4-1} = \frac{\Cinit^{3/4}}{2 \dimu^{1/4}} T_i^{1/4}$$
On event $\calE_3$, which upper bounds $\| z_{T-T_i} \|_2^2$, it follows that \eqref{eq:tople_upper_tempburnin3} and \eqref{eq:tople_upper_tempburnin4} hold as long as
\begin{align}
& T_i \ge \frac{\poly \left ( \frac{1}{1-\rhool}, \tauol, \| \thetastbar \|_{\Hinf} \right )}{\lamnoise} \bigg ( \dimu^{5/4} \gamma^2 \Cinit^{3/4} T_i^{3/4} + \dimu^{3/2} \gamma^2\Cinit^{3/2} \sqrt{T_i}\nonumber \\
& \qquad \qquad \qquad \qquad \qquad \qquad  + \sigmaw^2 + \sqrt{ \| \Gamnoise_{T-T_i} \|_F^2 \log \frac{2}{\delta}} + \| \Gamnoise_{T-T_i} \|_\op \log \frac{2}{\delta} \bigg ) \label{eq:tople_upper_burnin3} \\
& T_i^{1/4} \geq \max \left \{ \frac{80 \pi \| \thetastbar \|_{\Hinf} \gamma^2}{ \lamnoise}, \frac{\pi}{2 \| \thetastbar \|_{\Hinf}} \right \} \frac{2 \dimu^{1/4} \left ( \max_{\omega \in [0,2\pi]} \| (e^{\imag\omega} I - \Atilst)^{-2} \Btil \|_\op \right )}{\Cinit^{3/4}} \label{eq:tople_upper_burnin4}
\end{align}
We have then shown that, on the event $\calE_1 \cap \calE_2 \cap \calE_3$ and assuming $T$ is large enough to meet the burn-ins stated above, we have, for any $T'$,
\begin{align*}
    \calR(\aopt(\thetahat_i);\thetast) & \le 5 \sigma_w^2 \left ( \taskhes(\thetast) \Exp_{\thetast}[\matSig_T]^{-1} \right ) \log \frac{18(\dimx^2 + \dimx \dimu)}{\delta} + \frac{C_1}{T^{3/2}} + \frac{C_2}{T^2} \\
    & \le 15 \sigma_w^2 \min_{\bmu \in \calU_{\gamma^2,T'}} \frac{\tr \left (\taskhes(\thetast) \matGamss_{T',T'}(\thetastbar,\bmu,0)^{-1} \right )}{T_i} \log \frac{24(\dimx^2 + \dimx \dimu)}{\delta} + \frac{C_1}{T^{3/2}} + \frac{C_2}{T^2} \\
    & \qquad \qquad \qquad \qquad + 10 \sigma_w^2 \Cexp \log \frac{24(\dimx^2 + \dimx \dimu)}{\delta}
\end{align*}
As this bounds hold for any $T' \ge T_i$, we take $\liminf_{T' \rightarrow \infty}$, to obtain
\begin{align*}
    \calR(\aopt(\thetahat_i);\thetast) & \le  \frac{15 \sigma_w^2 \Phiss(\gamma^2;\thetast)}{T_i} \log \frac{24(\dimx^2 + \dimx \dimu)}{\delta} + \frac{C_1}{T^{3/2}} + \frac{C_2}{T^2} + 10 \sigma_w^2 \Cexp \log \frac{24(\dimx^2 + \dimx \dimu)}{\delta}
\end{align*}
We can also upper bound $\Phiss(\gamma^2;\thetast)$ by $16\Phiopt(\gamma^2;\thetast)$ via \Cref{prop:phiss_phiopt}. On $\calE_1 \cap \calE_2 \cap \calE_4$, it is easy to see $\Cexp = C_4/T^{3/2} + C_5/T^{2}$ for some $C_4,C_5$. The conclusion then follows by rescaling $\delta$ by a factor of 3, and since  $T_i \ge T/2$. The fact that the average expected power of the inputs is bounded by $\gamma^2$ follows by Lemma \ref{lem:tople_power_bound}. Finally, some algebra shows that the burn-in times stated above are all met as long as Assumption \ref{asm:upper_sufficient_T} holds.
\end{proof}

\begin{lem}\label{lem:tople_regular}
\algname $\in \policysetgood$ with $\sigma_u = \frac{\gamma}{\sqrt{2\dimu}}$ and $n = \calO(\log T)$.
\end{lem}
\begin{proof}
This follows directly by the formal definition of \algname, Algorithm \ref{alg:lqr_simple_regret}. In particular, we see that at each epoch $i$, \algname plays open-loop inputs $\util_t$ that are $\calF_{\tbreak_i}$ measurable. Furthermore, $\tbreak_i$ and $\Lambda_{u,i}$ are deterministically specified at the start of the algorithm, $\tr(\Lambda_{u,i}) \le \gamma^2$, $\lammin(\Lambda_{u,i}) \ge \gamma^2/(2\dimu)$, and Lemma \ref{lem:tople_power_bound} gives $\sum_{t=0}^{T-1} \util_t^\top \util_t \le \gamma^2$ deterministically. The fact that $n = \calO(\log T)$ follows since we increase the epoch length exponentially. Finally, the low-switching condition follows since the length of the epochs increase exponentially---once $T$ is large enough that $T \ge \Texpse(\piexp)$, we will have that at least half the initial interval is contained in the final epoch. Then for subsequent epochs, any interval of length $\Texpse(\piexp)$ will contain at most one epoch boundary. 
\end{proof}

\begin{lem}\label{lem:tople_power_bound}
Running Algorithm \ref{alg:lqr_simple_regret}, we will have $\frac{1}{T} \sum_{t=1}^T \Exp[u_t^\top u_t] \le \gamma^2$.
\end{lem}
\begin{proof}
By \Cref{prop:constructtimeinput}, we have that $\sum_{t=1}^{T_i} (\util_t^i)^\top \util_t^i \le T_i \gamma^2 / 2$. Thus,
$$ \sum_{t=1}^{T_i} \Exp[u_t^\top u_t] \le T_i \gamma^2/2 + \sum_{t=1}^{T_i} \Exp[(\unoise_t)^\top \unoise_t] = T_i \gamma^2/2 + \sum_{t=1}^{T_i} \gamma^2/2 = T_i \gamma^2$$
Thus, the average expected input power for a given epoch is bounded by $\gamma^2$. It follows then that, after running for $i$ epochs,
$$ \frac{1}{T} \sum_{t=1}^T \Exp[u_t^\top u_t] = \frac{1}{T} \sum_{j=1}^i \sum_{t=1}^{T_j} \Exp[u_{t,j}^\top u_{t,j}] \le \frac{1}{T} \sum_{j=1}^i T_j \gamma^2 = \gamma^2 $$
where the last equality follows since, by definition, $T = \sum_{j=1}^i T_j$. 
\end{proof}

\subsection{Certainty Equivalence Experiment Design}
\begin{lem}\label{lem:gd_ce_expdesign}
Fix a nominal instance $\thetast$ and let $\thetahat$ be some instance such that $\| \thetast - \thetahat \|_\circ \le \epsilon_\circ$, for $\circ \in \{\op,2\}$. Let $\Gamma(\theta,\bmU) \in \calS_+^{\dimtheta}$ be a map that satisfies, for all $\bmU \in \inputset$ and all $\theta$ with $\| \theta - \thetast \|_\circ \le \betaexp(\thetast)$,
\begin{equation}\label{eq:ce_expdesign_asm}
\left \| \Gamma(\theta,\bmU) - \Gamma(\thetast,\bmU) \right \|_\op \le \alphast(\thetast,\gamma^2) \cdot \| \theta - \thetast \|_\circ
\end{equation}
Assume that for all $\bmU \in \calU_{\gamma^2,k}$, $\lambda_{\min}(\Gamma(\thetast,\bmU)) \ge \underline{\lambda}$ and that $\epsilon_\circ < \min \{ \betaexp(\thetast), \underline{\lambda}/(2 \alphast(\thetast,\gamma^2)) \}, \epsilon_2 \le \betast(\thetast)$. Let:
$$\wh{\bmU} = \argmin_{\bmU \in \calU_{\gamma^2,k}} \tr \left ( \taskhes(\thetahat) \Gamma(\thetahat,\bmU)^{-1} \right ), \quad \bmU^\star = \argmin_{\bmU \in \calU_{\gamma^2,k}} \tr \left ( \taskhes(\thetast) \Gamma(\thetast,\bmU)^{-1} \right )  $$
Then, under Assumption \ref{asm:smoothness}, we have:
\begin{align*}
\left | \tr \left ( \taskhes(\thetast) \Gamma(\thetast,\bmU^\star)^{-1}  \right ) - \tr \left ( \taskhes(\thetast) \Gamma(\thetast,\wh{\bmU})^{-1} \right )  \right | & \le \frac{2 \dimtheta \Lra }{\underline{\lambda}} \epsilon_2 + \frac{4 \alphast(\thetast,\gamma^2) \tr(\taskhes(\thetast))}{\underline{\lambda}^2} \epsilon_\circ \\
& \qquad + \frac{4 \dimtheta \alphast(\thetast,\gamma^2) \Lra}{\underline{\lambda}^2} \epsilon_\circ \epsilon_2.
\end{align*}
\end{lem}
\begin{proof}
By Proposition \ref{quad:certainty_equivalence}, under Assumption \ref{asm:smoothness} and since $\| \thetast - \thetahat \|_2 \le \betast(\thetast)$, we have
$$  \| \taskhes(\thetast) - \taskhes(\thetahat) \|_\op \leq \Lra \|\thetast - \thetahat\|_{2} $$
Furthermore, by  \Cref{prop:mat_inverse_bound}, \eqref{eq:ce_expdesign_asm}, and since $\epsilon_\circ < \underline{\lambda}/(2\alphast(\thetast,\gamma^2))$, 
\begin{align*}
\| \Gamma(\thetahat,\bmU)^{-1} - \Gamma(\thetast,\bmU)^{-1} \|_\op & \le \frac{\| \Gamma(\thetahat,\bmU) - \Gamma(\thetast,\bmU) \|_\op}{\underline{\lambda}(\underline{\lambda} - \alphast(\thetast,\gamma^2) \epsilon_\circ) }  \le \frac{2 \alphast(\thetast,\gamma^2) \epsilon_\circ}{\underline{\lambda}^2}
\end{align*}
Thus, denoting $\Delta_\taskhes = \taskhes(\thetahat) - \taskhes(\thetast)$ and $\Delta_{\Sigma^{-1}} = \Gamma(\thetast,\bmU)^{-1} - \Gamma(\thetahat,\bmU)^{-1}$, the above bounds and Von Neumann's trace inequality imply:
\begin{align*}
& \left | \tr \left ( \taskhes(\thetahat) \Gamma(\thetahat,\bmU)^{-1} \right ) - \tr \left ( \taskhes(\thetast) \Gamma(\thetast,\bmU)^{-1} \right )\right |  \le  |\tr \left ( \Delta_\taskhes \Gamma(\thetast,u)^{-1} \right ) | + |\tr(\taskhes(\thetast) \Delta_{\Sigma^{-1}})| + |\tr(\Delta_\taskhes \Delta_{\Sigma^{-1}}) | \\
& \qquad \qquad \le \Lra \tr \left (  \Gamma(\thetast,\bmU)^{-1} \right ) \epsilon_2 + \frac{2 \alphast(\thetast,\gamma^2) \tr(\taskhes(\thetast))}{\underline{\lambda}^2} \epsilon_\circ + \frac{2 \dimtheta \alphast(\thetast,\gamma^2) \Lra}{\underline{\lambda}^2} \epsilon_\circ \epsilon_2 \\
& \qquad \qquad \le \frac{ \dimtheta \Lra }{\underline{\lambda}} \epsilon_2 + \frac{2 \alphast(\thetast,\gamma^2) \tr(\taskhes(\thetast))}{\underline{\lambda}^2} \epsilon_\circ + \frac{2 \dimtheta \alphast(\thetast,\gamma^2) \Lra}{\underline{\lambda}^2} \epsilon_\circ \epsilon_2 \\
& \qquad \qquad =: f(\epsilon)
\end{align*}
Assume that $\tr \left ( \taskhes(\thetahat) \Gamma(\thetahat,\wh{\bmU})^{-1} \right ) > \tr \left ( \taskhes(\thetast) \Gamma(\thetast,\bmU^\star)^{-1} \right ) $, then:
\begin{align*}
& \left | \tr \left ( \taskhes(\thetast) \Gamma(\thetast,\bmU^\star)^{-1}  \right ) - \tr \left ( \taskhes(\thetast) \Gamma(\thetast,\wh{\bmU})^{-1} \right )  \right | \\
& \quad \le \left | \tr \left ( \taskhes(\thetast) \Gamma(\thetast,\bmU^\star)^{-1}  \right ) - \tr \left ( \taskhes(\thetahat) \Gamma(\thetahat,\wh{\bmU})^{-1} \right ) \right |  + \left | \tr \left ( \taskhes(\thetahat) \Gamma(\thetahat,\wh{\bmU})^{-1} \right ) - \tr \left ( \taskhes(\thetast) \Gamma(\thetast,\wh{\bmU})^{-1} \right )  \right | \\
& \quad \le \left | \tr \left ( \taskhes(\thetast) \Gamma(\thetast,\bmU^\star)^{-1}  \right ) - \tr \left ( \taskhes(\thetahat) \Gamma(\thetahat,\bmU^\star)^{-1} \right ) \right |  + \left | \tr \left ( \taskhes(\thetahat) \Gamma(\thetahat,\wh{\bmU})^{-1} \right ) - \tr \left ( \taskhes(\thetast) \Gamma(\thetast,\wh{\bmU})^{-1} \right )  \right | \\
& \quad \le 2 f(\epsilon)
\end{align*}
where the second inequality holds because $\wh{\bmU}$ is the minimizer of $\tr \left ( \taskhes(\thetahat) \Gamma(\thetahat,\bmU)^{-1} \right )$. If instead $\tr \left ( \taskhes(\thetahat) \Gamma(\thetahat,\wh{\bmU})^{-1} \right ) \le \tr \left ( \taskhes(\thetast) \Gamma(\thetast,\bmU^\star)^{-1} \right ) $, we can replace $\tr \left ( \taskhes(\thetast) \Gamma(\thetast,\bmU^\star)^{-1}  \right )$ with $\tr \left ( \taskhes(\thetast) \Gamma(\thetast,\wh{\bmU})^{-1}  \right )$ in the above calculation to get the same result. The conclusion follows.

\end{proof}

\begin{lem}[Matrix Perturbation Bound]\label{prop:mat_inverse_bound}
Assume $A,B \in \PD$, $\| A - B \|_\op \leq \epsilon$, and $\epsilon < \lambda_{\min}(B)$. Then 
\begin{align*} \| A^{-1} - B^{-1} \|_\op \leq \frac{\epsilon}{\lambda_{\min}(B)(\lambda_{\min}(B) - \epsilon)}
\end{align*}
\end{lem}
\begin{proof}
Denote $\Delta = A - B$. By the matrix inversion lemma:
$$ A^{-1} = (B + \Delta)^{-1} = B^{-1} - B^{-1} (B^{-1} + \Delta^{-1})^{-1} B^{-1} $$
so:
\begin{align*}
\| A^{-1} - B^{-1} \|_\op & = \| B^{-1} (B^{-1} + \Delta^{-1})^{-1} B^{-1} \|_\op \\
& \leq \| B^{-1} \|_\op^2 \| (B^{-1} + \Delta^{-1})^{-1} \|_\op \\
& = \frac{1}{\lambda_{\min}(B)^2 \sigma_d(B^{-1} + \Delta^{-1})}
\end{align*}
However:
$$ \sigma_d(B^{-1} + \Delta^{-1}) \geq \sigma_d(\Delta^{-1}) - \sigma_1(B^{-1}) = \frac{1}{\| \Delta \|_\op} - \frac{1}{\lambda_{\min}(B)} = \frac{\lambda_{\min}(B) - \| \Delta \|_\op}{\lambda_{\min}(B) \| \Delta \|_\op}$$
Since we have assumed $\epsilon < \lambda_{\min}(B)$ and since $\| \Delta \|_\op \leq \epsilon$, this lower bound on $ \sigma_d(B^{-1} + \Delta^{-1})$ will be positive, so:
$$ \frac{1}{\lambda_{\min}(B)^2 \sigma_d(B^{-1} + \Delta^{-1})} \leq \frac{\lambda_{\min}(B) \| \Delta \|_\op}{\lambda_{\min}(B)^2 (\lambda_{\min}(B) - \| \Delta \|_\op)} \le \frac{ \| \Delta \|_\op}{\lambda_{\min}(B) (\lambda_{\min}(B) - \| \Delta \|_\op)}$$
The result follows since $\| \Delta \|_\op \leq \epsilon$.
\end{proof}

\subsection{Operator Norm Estimation}

\begin{lem}\label{lem:general_op_bound2}
Let
$$ \thetals = \min_{A,B} \sum_{t=1}^T \| x_{t+1} - A x_t - B u_t \|_2^2 $$
Then on the event
\begin{align*}
        \calE := \Big \{ \lammin(\Sigma_T) \ge \lamund T, \Sigma_T \preceq T \covup \Big \}
    \end{align*}
with probability at least $1-\delta$:
$$ \| \thetals - \thetast \|_\op \le C \sqrt{\frac{ \log(1/\delta) + \dimx + \logdet(\covup/\lamund + I)}{\lamund T}}. $$
\end{lem}
\begin{proof}
Define the following events:
\begin{align*}
\calA & = \left \{ \| \thetahat_i - \thetast \|_\op \le C \sqrt{\frac{ \log(1/\delta) + \dimx + \logdet(\covup/\lamund + I)}{\lamund T}} \right \} \\
\calE_1 & = \left \{ \left \| \left ( \sum_{ t = 1}^T z_t z_t^\top \right )^{-1/2} \sum_{t=1}^T z_t w_t^\top \right \|_\op \le c_2 \sigmaw \sqrt{ \log \frac{1}{\delta} + \dimx + \logdet(\covup/\lamund + I)} \right \} 
\end{align*}
Our goal is to show that $\Pr[\calA^c \cap \calE ] \le \delta$. The following is trivial.
\begin{align*}
\Pr[\calA^c \cap \calE ] & \le \Pr[\calA^c \cap \calE \cap \calE_1] + \Pr[\calE  \cap \calE_1^c]
\end{align*}
As $\thetals$ is the least squares estimate, we will have that $\thetals^\top = (\sum_{t=1}^T z_t z_t^\top)^{-1} \sum_{t=1}^T z_t x_{t+1}^\top = \thetast^\top +  (\sum_{t=T-T_i}^T z_t z_t^\top)^{-1} \sum_{t=1}^T z_t w_t^\top$. Given this, the error can be decomposed as:
\begin{align*}
\| \thetals - \thetast \|_\op & = \left \|  \left (\sum_{t=1}^T z_t z_t^\top \right )^{-1} \sum_{t=1}^T z_t w_t^\top \right \|_\op  \le \left \|  \left ( \sum_{t=1}^T z_t z_t^\top \right )^{-1/2} \right \|_\op \left \| \left (\sum_{t=1}^T z_t z_t^\top \right )^{-1/2} \sum_{t=1}^T z_t w_t^\top \right \|_\op \\
& = \left \| \left (\sum_{t=1}^T z_t z_t^\top \right )^{-1/2} \sum_{t=1}^T z_t w_t^\top \right \|_\op / \sqrt{\lambda_{\min} \left ( \sum_{t=1}^T z_t z_t^\top \right )}
\end{align*}
It follows that, on the event $\calE \cap \calE_1$, the error bound given in $\calA$ holds. Thus, $ \Pr[\calA^c \cap \calE \cap \calE_1]  = 0$. Lemma \ref{lem:self_normalized_past_data2} implies that $\Pr[\calE \cap \calE_1^c] \le \delta$, so $\Pr[\calA^c \cap \calE] \le \delta$. 
\end{proof}

\begin{lem}[Lemma E.6 of \cite{wagenmaker2020active}, see also \cite{abbasi2011improved}]\label{lem:self_normalized_past_data2}
Assume that $z_t$ is generated by \eqref{eq:input_state_dyn} with $w_t \sim \calN(0,\sigmaw^2 I)$ and input $u_t = \util_t + \unoise_t$, where $\util_t$ is $\mathcal{F}_{t-1}$ measurable and $\unoise_t \sim \mathcal{N}(0, \Lambda_u)$. On the event that $V_+ \succeq \sum_{t=1}^T z_t z_t^\top \succeq V_-$, we will have that, with probability less than $\delta$:
\begin{equation*}
\left \| \left ( \sum_{t=1}^T z_t z_t^\top \right )^{-1/2} \sum_{t=1}^T z_t w_t^\top \right \|_\op > \sigmaw \sqrt{16 \log \frac{1}{\delta} + 8 \log \det (V_+ V_{-}^{-1} + I) + 16 (\dimx + \dimu) \log 5} .
\end{equation*}
\end{lem}

\subsection{Optimality of Inputs}

\begin{lem}\label{lem:global_optimal_inputs}
Fix an epoch $i$ of Algorithm \ref{alg:lqr_simple_regret} and let $k_i$ denote the discretization level of the input set at that epoch, $\calU_{\gamma^2/2,k_i}$, and assume that $T_i/(\dimu k_i)$ is an integer. Let $k$ be any value satisfying $\dimx \le k \le T_i/2$ and take $\epsilon \in (0,1)$. Then, as long as
\begin{align*}
T_i  & \ge \max \Bigg \{ \frac{2 \dimu \tauol^3 (2 k_i \gamma + \| z_{T-T_i} \|_2) ((4 +  2 \tauol)  k_i \gamma + \tauol \| z_{T-T_i} \|_2))}{(1 - \rhool^{k_i})^2(1-\rhool) \lamnoise \epsilon }, \frac{16\tauol^2 (\sigmaw^2 + \gamma^2/(2\dimu) ) }{(1-\rhool^2)^2 \lamnoise \epsilon} \\
& \qquad \qquad \qquad \left (  \frac{ \tauol^2 k_i^2 \dimu \gamma^2}{(1-\rhool^{k_i})^2} + \frac{\tauol k_i \sqrt{\dimu} \gamma^2 \sqrt{T_i}}{1-\rhool^{k_i}} \right )  \frac{16 \dimu \| \thetastbar \|_{\Hinf}^2}{\lamnoise \epsilon }   , \log \left (  \frac{(1-\rhool^2)^2  \lamnoise \epsilon}{16\tauol^2 (\sigmaw^2 + \gamma^2/(2\dimu) ) }\right ) \frac{1}{2 \log \rhool}  \Bigg \}, 
\end{align*}
$$ k_i \geq \max \left \{ \frac{8 \pi \| \thetastbar \|_{\Hinf} \gamma^2}{ \lamnoise \epsilon}, \frac{\pi}{2 \| \thetastbar \|_{\Hinf}} \right \} \left ( \max_{\omega \in [0,2\pi]} \| (e^{\imag\omega} I - \Atilst)^{-2} \Btil \|_\op \right ),$$
$$ \| \thetahat_{i-1} - \thetast \|_\op \le \epsop  \le \min \left \{ \betaexplds(\thetast), \frac{\betast(\thetast)}{\sqrt{\dimx}}, \frac{\lamnoise}{4 \alphast(\thetast,\gamma^2)} \right \} $$
we have, for any $T' \ge T_i$,
$$ \tr \Big (\taskhes(\thetast) (\Exp_{\thetast,\bmU_{i},T_i}[I_{\dimx} \otimes \tsum_{t=T - T_i}^T z_t z_t^\top ])^{-1} \Big ) \le \min_{\bmU \in \calU_{\gamma^2,T'}} \frac{2 \tr \left (\taskhes(\thetast) \matGamss_{T',T'}(\thetastbar,\bmU,0)^{-1} \right )}{(1-\epsilon)^3 T_i} + \frac{\Cexp}{(1-\epsilon)^2} $$
where
\begin{align*}
\Cexp & = \left ( \frac{4\sqrt{\dimx}(\dimx^2+\dimx \dimu) \Lra }{\lamnoise}+ \frac{8 \alphastlds(\thetast,\gamma^2) \tr(\taskhes(\thetast))}{(\lamnoise)^2} \right ) \frac{\epsop}{T_i}  + \frac{8 (\dimx^2+\dimx \dimu) \alphastlds(\thetast,\gamma^2) \Lra}{(\lamnoise)^2} \frac{ \epsop^2}{T_i}. 
\end{align*}
\end{lem}
\begin{proof}
We will show that the following set of inequalities hold for large enough $T$ and arbitrary $T' \ge T$:
\begin{align}
\tr \Big (\taskhes(& \thetast)  (\Exp_{\thetast,\bmU_{i},T_i}[I_{\dimx} \otimes \tsum_{t=T - T_i}^T z_t z_t^\top ])^{-1} \Big ) \label{eq:opt_inputs_upper} \\
& \overset{(a)}{\le} \frac{1}{(1-\epsilon)^2 T_i} \tr \left (\taskhes(\thetast) \matGamss_{T_i,T_i/\dimu}(\thetastbar,\bmU_{i},\gamma/\sqrt{2\dimu})^{-1} \right )  \tag{Steady-state} \\
& \overset{(b)}{\le} \min_{\bmU \in \calU_{\gamma^2/2,k_i}} \frac{1}{(1-\epsilon)^2 T_i} \tr \left (\taskhes(\thetast) \matGamss_{T_i,T_i/\dimu}(\thetastbar,\bmU,\gamma/\sqrt{2\dimu})^{-1} \right )  + \frac{\Cexp}{(1-\epsilon)^2}  \tag{Optimal inputs} \\ 
& \overset{(c)}{\le} \min_{\bmU \in \calU_{\gamma^2/2,T'}} \frac{1}{(1-\epsilon)^3 T_i} \tr \left (\taskhes(\thetast) \matGamss_{T',T'}(\thetastbar,\bmU,\gamma/\sqrt{2\dimu})^{-1} \right )  + \frac{\Cexp}{(1-\epsilon)^2} \tag{Infinite horizon} \\
& \overset{(d)}{\le} \min_{\bmU \in \calU_{\gamma^2,T'}} \frac{2}{(1-\epsilon)^3 T_i} \tr \left (\taskhes(\thetast) \matGamss_{T',T'}(\thetastbar,\bmU,0)^{-1} \right ) + \frac{\Cexp}{(1-\epsilon)^2} \tag{Noiseless inputs}
\end{align}

\paragraph{Steady-state:} By definition,
$$ \Exp_{\thetast,\bmU_{i},T_i}[I_{\dimx} \otimes \tsum_{t=T - T_i}^T z_t z_t^\top ] = T_i \matGam_{T_i}(\thetastbar,\bmU_{i},\gamma/\sqrt{2\dimu},z_{T - T_i}) $$
Note that the inputs played by Algorithm \ref{alg:lqr_simple_regret} are constructed as in Lemma \ref{lem:opt_freq_approx_mat}. It follows then that, by Lemma \ref{lem:opt_freq_approx_mat}, as long as, 
\begin{align*}
T_i & \ge \max \Bigg \{ \frac{2 \dimu \tau(\Atilst,\rho)^3 (2 k_i \gamma + \| z_{T-T_i} \|_2) ((4 +  2 \tau(\Atilst,\rho))  k_i \gamma + \tau(\Atilst,\rho)  \| z_{T-T_i} \|_2))}{(1 - \rho^{k_i})^2(1-\rho) \lammin(\Gamnoise_{T_i/2}(\thetastbar,\gamma/\sqrt{2\dimu})) \epsilon }, \\
& \qquad \qquad \qquad \left (  \frac{ \tau(\Atilst,\rho)^2 k_i^2 \dimu \gamma^2}{(1-\rho^{k_i})^2} + \frac{\tau(\Atilst,\rho) k_i \dimu \gamma^2 \sqrt{T_i/\dimu}}{1-\rho^{k_i}} \right )  \frac{16 \dimu \| \thetastbar \|_{\Hinf}^2 }{\lammin(\Gamnoise_{T_i/2}(\thetast,\gamma/\sqrt{2\dimu})) \epsilon } \Bigg \} 
\end{align*}
we will have
$$ \Gamma_{T_i}(\thetastbar,\bmU_{i},\gamma/\sqrt{2\dimu},z_{T-T_i}) \succeq (1-\epsilon)^2 \Gamss_{T_i,T_i/\dimu}(\thetastbar,\bmU_{i},\gamma/\sqrt{2\dimu})$$
This implies that 
$$ \matGam_{T_i}(\thetastbar,\bmU_{i},\gamma/\sqrt{2\dimu},z_{T-T_i}) \succeq (1-\epsilon)^2 \matGamss_{T_i,T_i/\dimu}(\thetastbar,\bmU_{i},\gamma/\sqrt{2\dimu})$$
from which $(a)$ follows.

\paragraph{Optimal inputs:} We next apply Lemma \ref{lem:gd_ce_expdesign}, which bounds the suboptimality of certainty equivalent experiment design, to show $(b)$. We instantiate Lemma \ref{lem:gd_ce_expdesign} with 
$$ \underline{\lambda} =  \lammin(\Gamnoise_{k}(\thetastbar,\gamma/\sqrt{2\dimu}))/2, \quad \Gamma(\theta,\bmU) = \matGamss_{T_i,T_i/\dimu}(\theta,\bmU,\gamma/\sqrt{2\dimu}) $$
Note that Lemma \ref{lem:smooth_covariates} gives that the smoothness condition \eqref{eq:ce_expdesign_asm} holds for $\alphastlds(\thetast,\gamma^2)$ as defined in \eqref{eq:alphast_lds}. Furthermore, it is clear that, as long as $T_i \ge 2k$, we will have $\lammin(\Gamma(\thetastbar,\bmU)) \ge \underline{\lambda}$ for all $\bmU$. To apply Lemma \ref{lem:gd_ce_expdesign}, we need 
$$\| \thetahat_{i-1} - \thetast \|_\op  \le \min \{ \betaexplds(\thetast), \lammin(\Gamnoise_{k}(\thetastbar,\gamma/\sqrt{2\dimu}))/(4 \alphast(\thetast,\gamma^2)) \}, \quad \| \thetahat_{i-1} - \thetast \|_F \le \betast(\thetast)$$
where we choose to instantiate Lemma \ref{lem:gd_ce_expdesign} in the operator norm, and since the matrix Frobenius norm coincides with the vector 2-norm. The condition on $\| \thetahat_{i-1} - \thetast \|_\op$ will hold as long as our assumption on $\epsop$ holds. Since $\thetahat_i - \thetast$ is at most rank $\dimx$, we have 
 $$\| \thetahat_{i-1} - \thetast \|_F \le \sqrt{\dimx} \| \thetahat_{i-1} - \thetast \|_\op \le  \betast(\thetast) $$ 
where the last inequality again holds so long as our assumption on $\epsop$ holds. Then, since we design the input $\bmU_{i}$ on the estimate $\thetahat_{i-1}$, the conditions of Lemma \ref{lem:gd_ce_expdesign} are met for $\bmU_{i}$, so 
\begin{align*}
& \frac{1}{T_i} \left | \tr \left (\taskhes(\thetast) \matGamss_{T_i,T_i/\dimu}(\thetastbar,\bmU_{i},\gamma/\sqrt{2\dimu})^{-1} \right ) -  \min_{\bmU \in \calU_{\gamma^2/2,k_i}} \tr \left (\taskhes(\thetast) \matGamss_{T_i,T_i/\dimu}(\thetastbar,\bmU,\gamma/\sqrt{2\dimu})^{-1} \right ) \right | \\
& \qquad \le \left ( \frac{4\sqrt{\dimx}(\dimx^2+\dimx \dimu) \Lra \lammin(\Gamnoise_{k}(\thetastbar,\gamma/\sqrt{2\dimu})+ 8 \alphastlds(\thetast,\gamma^2) \tr(\taskhes(\thetast))}{\lammin(\Gamnoise_{k}(\thetastbar,\gamma/\sqrt{2\dimu})^2} \right ) \frac{\epsop}{T_i} \\
& \qquad \qquad \qquad \qquad + \frac{8 (\dimx^2+\dimx \dimu) \alphastlds(\thetast,\gamma^2) \Lra}{\lammin(\Gamnoise_{k}(\thetastbar,\gamma/\sqrt{2\dimu})^2} \frac{\epsop^2}{T_i} \\
& \qquad =: \Cexp
\end{align*}
and thus $(b)$ holds.

\paragraph{Infinite horizon:} Now,
$$ \Gamss_{T_i,T_i/\dimu}(\thetastbar,\bmU,\gamma/\sqrt{2\dimu}) = \frac{\dimu}{T_i} \Gamfreq_{T_i/\dimu}(\thetastbar,\bmU) + \frac{1}{T_i} \sum_{t=1}^{T_i} \Gamnoise_t(\thetastbar,\gamma/\sqrt{2\dimu}) $$
If $\bmU \in \calU_{\gamma^2/2,k_i}$, since $T_i/(\dimu k_i)$ is an integer, 
$$ \Gamfreq_{T_i/\dimu}(\thetastbar,\bmU) = \frac{T_i}{\dimu k_i^2} \sum_{\ell=1}^{k_i} (e^{\imag \frac{2\pi \ell}{k_i}} I - \Atilst)^{-1} \Btilst U_\ell \Btilst^\top (e^{\imag \frac{2\pi \ell}{k_i}} I - \Atilst)^{-\herm} $$
Then, by Lemma \ref{lem:opt_k_large_approx}, as long as
$$ k_i \geq \max \left \{ \frac{8 \pi \| \thetastbar \|_{\Hinf} \gamma^2}{ \lammin(\Gamnoise_{k}(\thetastbar,\gamma/\sqrt{2\dimu})) \epsilon}, \frac{\pi}{2 \| \thetastbar \|_{\Hinf}} \right \} \left ( \max_{\omega \in [0,2\pi]} \| (e^{\imag\omega} I - \Atilst)^{-2} \Btil \|_\op \right )$$
then for any $T' \ge k_i$ and $\bmU^\star \in \calU_{\gamma^2/2,T'}$, there exists a feasible $\bmU' \in \calU_{\gamma^2/2,k_i}$ such that
$$ \left \| \frac{\dimu}{T_i} \Gamfreq_{T_i/\dimu}(\thetastbar,\bmU') - \frac{1}{T'} \Gamfreq_{T'}(\thetastbar,\bmU^\star) \right \|_\op \le \frac{\epsilon}{2}  \lammin(\Gamnoise_{k}(\thetastbar,\gamma/\sqrt{2\dimu})) $$
Furthermore, by Lemma \ref{lem:opt_k_large_approx_noise}, if
$$ T_i \ge \max \left \{  \frac{16\tau(\Atilst,\rho)^2 (\sigmaw^2 + \gamma^2/(2\dimu) ) }{(1-\rho^2)^2  \lammin(\Gamnoise_{k}(\thetastbar,\gamma/\sqrt{2\dimu})) \epsilon}  , \log \left (  \frac{(1-\rho^2)^2  \lammin(\Gamnoise_{k}(\thetastbar,\gamma/\sqrt{2\dimu})) \epsilon }{16\tau(\Atilst,\rho)^2 (\sigmaw^2 + \gamma^2/(2\dimu) ) }\right ) \frac{1}{2 \log \rho} \right \}$$
then, for any $T' \ge T_i$,
$$ \left \| \frac{1}{T_i} \sum_{t=1}^{T_i} \Gamnoise_t(\thetastbar,\gamma/\sqrt{2\dimu}) - \frac{1}{T'} \sum_{t=1}^{T'} \Gamnoise_t(\thetastbar,\gamma/\sqrt{2\dimu}) \right \|_\op \le \frac{\epsilon}{2} \lammin(\Gamnoise_{k}(\thetastbar,\gamma/\sqrt{2\dimu})) $$
Now note that, by definition, 
$$ \min_{\bmU \in \calU_{\gamma^2/2,k_i}} \tr \left (\taskhes(\thetast) \matGamss_{T_i,T_i/\dimu}(\thetastbar,\bmU,\gamma/\sqrt{2\dimu})^{-1} \right ) \le \tr \left (\taskhes(\thetast) \matGamss_{T_i,T_i/\dimu}(\thetastbar,\bmU',\gamma/\sqrt{2\dimu})^{-1} \right ) $$
By what we've just shown, for any $T' \ge T_i$ 
\begin{align*}
& \matGamss_{T_i,T_i/\dimu}(\thetastbar,\bmU',\gamma/\sqrt{2\dimu}) \\
& \succeq \matGamss_{T',T'}(\thetastbar,\bmU^\star,\gamma/\sqrt{2\dimu}) - \frac{\epsilon}{2} \lammin(\Gamnoise_{k}(\thetastbar,\gamma/\sqrt{2\dimu})) \cdot I  \\
& \succeq (1-\epsilon) \matGamss_{T',T'}(\thetastbar,\bmU^\star,\gamma/\sqrt{2\dimu}) + \epsilon \frac{1}{2T'} \sum_{t=1}^{T'} \Gamnoise_t(\thetastbar,\gamma/\sqrt{2\dimu}) - \frac{\epsilon}{2}  \lammin(\Gamnoise_{k}(\thetastbar,\gamma/\sqrt{2\dimu})) \cdot I \\
& \succeq  (1-\epsilon) \matGamss_{T',T'}(\thetastbar,\bmU^\star,\gamma/\sqrt{2\dimu}) + \frac{\epsilon}{2} \lammin(\Gamnoise_{T'/2}(\thetastbar,\gamma/\sqrt{2\dimu})) \cdot I - \frac{\epsilon}{2}  \lammin(\Gamnoise_{k}(\thetastbar,\gamma/\sqrt{2\dimu})) \cdot I \\
& \succeq (1-\epsilon) \matGamss_{T',T'}(\thetastbar,\bmU^\star,\gamma/\sqrt{2\dimu})
\end{align*}
From this $(c)$ follows directly.

\paragraph{Noiseless inputs:} Finally, since
$$ \matGamss_{T',T'}(\thetastbar,\bmU^\star,\gamma/\sqrt{2\dimu}) \succeq \matGamss_{T',T'}(\thetastbar,\bmU^\star,0) $$
we can bound
\begin{align*}
\min_{\bmU \in \calU_{\gamma^2/2,T'}} \tr \left (\taskhes(\thetast) \matGamss_{T',T'}(\thetastbar,\bmU,\gamma/\sqrt{2\dimu})^{-1} \right ) & \le \min_{\bmU \in \calU_{\gamma^2/2,T'}} \tr \left (\taskhes(\thetast) \matGamss_{T',T'}(\thetastbar,\bmU,0)^{-1} \right )   \\
& \le 2 \min_{\bmU \in \calU_{\gamma^2,T'}} \tr \left (\taskhes(\thetast) \matGamss_{T',T'}(\thetastbar,\bmU,0)^{-1} \right )
\end{align*}
which proves $(d)$. Finally, to simplify the bound, we note that $\lammin(\Gamnoise_k(\thetastbar,\gamma/\sqrt{2\dimu})) \ge \lamnoise$, and we choose $\rho = \rhool$.
\end{proof}

\begin{lem}\label{lem:opt_freq_approx_mat}
Fix an input $\bmU \in \calU_{\gamma^2,k}$ and consider a system $\theta = (A,B)$. If we start from some state $x_0$ and play the time domain input $u_t =$ \texttt{ConstructTimeInput}$(\bmU,T/\dimu,k)$ where \texttt{ConstructTimeInput} is defined in Algorithm \ref{alg:construct_time_input} then, so long as $T/\dimu$ is divisible by $k$, and
\begin{align*}
T & \ge \max \Bigg \{ \frac{2 \dimu \tau(A,\rho)^3 (2 \| B \|_\op k \gamma + \| x_0 \|_2) ((4 +  2 \tau(A,\rho))  \| B \|_\op k \gamma + \tau(A,\rho)  \| x_0 \|_2))}{(1 - \rho^k)^2(1-\rho) \lammin(\Gamnoise_{T/2}(\theta,\gamma/\sqrt{2\dimu})) \epsilon}, \\
& \qquad \qquad \qquad \left (  \frac{ \tau(A,\rho)^2 k^2 \dimu \gamma^2}{(1-\rho^k)^2} + \frac{\tau(A,\rho) k \dimu \gamma^2 \sqrt{T/\dimu}}{1-\rho^k} \right )  \frac{16 \dimu  \| \theta \|_{\Hinf}^2 }{\lammin(\Gamnoise_{T/2}(\theta,\gamma/\sqrt{2\dimu}))  \epsilon} \Bigg \} 
\end{align*}
we will have
$$ \Gamma_T(\theta,\bmU,\gamma/\sqrt{2\dimu},x_0) \succeq (1-\epsilon)^2 \Gamss_{T,T/\dimu}(\theta,\bmU,\gamma/\sqrt{2\dimu}).$$
\end{lem}
\begin{proof}
Let $\ucheck_{\ell,j}$ be defined as in Algorithm \ref{alg:construct_time_input} for this $\bmU$, and denote $\bmU_j = (\ucheck_{\ell,j} \ucheck_{\ell,j}^\herm)_{\ell=1}^k$. Let $\util_{j,t}$ denote the time domain version of $\{ \ucheck_{\ell,j} \}_{\ell=1}^k$, as is specified in Algorithm \ref{alg:construct_time_input}. Since $\bmU \in \calU_{\gamma^2,k}$, some algebra shows that
$$ \frac{1}{k} \sum_{t=1}^k \wt{u}_{j,t+s}^\top \wt{u}_{j,t+s} \le 2 \dimu \gamma^2 $$
for any $s \ge 0$. We break up the sum of the response based on which input is being played:
$$ \sum_{t=1}^T \xu_t (\xu_t)^\top = \sum_{j=1}^{\dimu} \sum_{t=(j-1)T/\dimu + 1}^{jT/\dimu} \xu_t (\xu_t)^\top $$
which gives:
\begin{align*}
\Gamma_T(\theta,\bmU,\gamma/\sqrt{2\dimu},0) & = \frac{1}{T}  \sum_{j=1}^{\dimu} \sum_{t=(j-1)T/\dimu + 1}^{jT/\dimu} \xu_t (\xu_t)^\top + \frac{1}{T} \sum_{t=1}^T \Gamnoise_t(\theta,\gamma/\sqrt{2\dimu}) \\
& = \frac{1}{T} \sum_{j=1}^{\dimu} \left ( \Gamin_{T/\dimu}(\theta,\bmU_j,\xu_{(j-1)T/\dimu}) + \frac{1}{\dimu} \sum_{t=1}^T \Gamnoise_t(\theta,\gamma/\sqrt{2\dimu}) \right )
\end{align*}
If we start from some initial state $x'$,
$$ \xu_t = A^t x' +  \sum_{s=0}^{t-1} A^{t-s-1} B u_s =: \xui_t + \xutil_t$$
then,
\begin{align*} 
\Gamin_{T/\dimu}(\theta,\bmU_j,\xu_{(j-1)T/\dimu}) & = \sum_{t=0}^T x_t^u (x_t^u)^\top   = \sum_{t=0}^T \left [ \xutil_t (\xutil_t)^\top + \xutil_t (\xui_t)^\top + \xui_t (\xutil_t)^\top + \xui_t (\xui_t)^\top \right ]
\end{align*}
Now,
\begin{align*}
 \left \|  \sum_{t=0}^T  \xui_t (\xutil_t)^\top \right \|_\op & \le \| \xu_{(j-1)T/\dimu} \|_2 \sum_{t=0}^T \| A^t \|_\op \| \xutil_t \|_2  \le \| \xu_{(j-1)T/\dimu} \|_2 \tau(A,\rho) \sum_{t=0}^T \rho^t \| \xutil_t \|_2 \\
 &  \le  \frac{2\tau(A,\rho)^2 \| \xu_{(j-1)T/\dimu} \|_2 \| B \|_\op k \gamma}{(1 - \rho^k)(1-\rho)}
 \end{align*}
where the last inequality follows by Lemma D.7 of \cite{wagenmaker2020active}, which gives: 
$$ \| \xutil_t \|_2 \le \frac{2\tau(A,\rho)\| B \|_\op k \gamma}{1 - \rho^k}$$
Furthermore,
\begin{align*}
\left \| \sum_{t=0}^T \xui_t (\xui_t)^\top \right \|_\op  \le \| \xu_{(j-1)T/\dimu} \|_2^2 \sum_{t=0}^T \| A^t \|_\op^2 \le \frac{\| \xu_{(j-1)T/\dimu} \|_2^2 \tau(A,\rho)^2}{1 - \rho}
\end{align*}
Therefore,
\begin{align*}
& \Gamin_{T/\dimu}(\theta,\bmU_j,\xu_{(j-1)T/\dimu}) + \frac{1}{\dimu} \sum_{t=1}^T \Gamnoise_t(\theta,\gamma/\sqrt{2\dimu}) \\
& \qquad \succeq \Gamin_{T/\dimu}(\theta,\bmU_j,0) + \frac{1}{\dimu} \sum_{t=1}^T \Gamnoise_t(\theta,\gamma/\sqrt{2\dimu}) - \frac{\tau(A,\rho)^2 \| \xu_{(j-1)T/\dimu} \|_2 (4\| B \|_\op k \gamma + (1-\rho^k) \| \xu_{(j-1)T/\dimu} \|_2)}{(1 - \rho^k)(1-\rho)} \\
& \qquad \succeq (1-\epsilon) \Gamin_{T/\dimu}(\theta,\bmU_j,0) + \frac{1 - \epsilon}{\dimu} \sum_{t=1}^T \Gamnoise_t(\theta,\gamma/\sqrt{2\dimu}) + \frac{\epsilon T}{2 \dimu} \lammin(\Gamnoise_{T/2}(\theta,\gamma/\sqrt{2\dimu})) \cdot I \\
& \qquad \qquad \qquad   - \frac{\tau(A,\rho)^2 \| \xu_{(j-1)T/\dimu} \|_2 (4\| B \|_\op k \gamma + (1-\rho^k) \| \xu_{(j-1)T/\dimu} \|_2)}{(1 - \rho^k)(1-\rho)} \\
& \qquad \succeq (1-\epsilon) \Gamin_{T/\dimu}(\theta,\bmU_j,0) + \frac{1 - \epsilon}{\dimu} \sum_{t=1}^T \Gamnoise_t(\theta,\gamma/\sqrt{2\dimu}) 
\end{align*}
where the last inequality holds as long as
\begin{equation}\label{eq:Tinit_mat_opt1}
T \ge \frac{2 \dimu \tau(A,\rho)^2 \| \xu_{(j-1)T/\dimu} \|_2 (4\| B \|_\op k \gamma + (1-\rho^k) \| \xu_{(j-1)T/\dimu} \|_2)}{(1 - \rho^k)(1-\rho) \lammin(\Gamnoise_{T/2}(\theta,\gamma/\sqrt{2\dimu})) \epsilon}
\end{equation}
By Proposition \ref{prop:steady_state_inputs}, since $T/\dimu$ is divisible by $k$, we will have that
$$ \Gamfreq_{T/\dimu}(\theta,\bmU_j) = \frac{\dimu}{T} \sum_{\ell=1}^T (e^{\imag \frac{2\pi \ell}{T}} I - A)^{-1} B U_{j,\ell} U_{j,\ell}^{\herm} B^{\herm} (e^{\imag \frac{2\pi \ell}{T}} I - A)^{-\herm} $$
where $\{ U_{j,\ell} \}_{\ell=1}^T =  \Fourier(\util_{j,0},\ldots,\util_{j,T/\dimu})$. Then, by Lemma \ref{lem:lds_lb_cov_ss_exp}, we will have, 
\begin{align*} 
\| \Gamin_{T/\dimu}(\theta,\bmU_j,0) & - \Gamfreq_{T/\dimu}(\theta,\bmU_j) \|_\op  \\
& \le \left (  \frac{8 \tau(A,\rho)^2 k^2 \dimu \gamma^2}{(1-\rho^k)^2} + \frac{8 \tau(A,\rho) k \dimu \gamma^2 \sqrt{T/\dimu}}{1-\rho^k} \right ) \left ( \max_{\omega \in [0,2\pi]} \| (e^{-\imag \omega} I - A)^{-1} B \|_\op^2 \right ) \\
& \le \frac{\epsilon T}{2 \dimu} \lammin(\Gamnoise_{T/2}(\theta,\gamma/\sqrt{2\dimu})) 
\end{align*}
where the last inequality is true so long as
$$T \ge \left (  \frac{8 \tau(A,\rho)^2 k^2 \dimu \gamma^2}{(1-\rho^k)^2} + \frac{8 \tau(A,\rho) k \dimu \gamma^2 \sqrt{T/\dimu}}{1-\rho^k} \right ) \left ( \max_{\omega \in [0,2\pi]} \| (e^{\imag \omega} I - A)^{-1} B \|_\op^2 \right ) \frac{2 \dimu}{\lammin(\Gamnoise_{T/2}(\theta,\gamma/\sqrt{2\dimu}))  \epsilon}$$ 
Thus,
\begin{align*}
& \Gamin_{T/\dimu}(\theta,\bmU_j,0) + \frac{1}{\dimu} \sum_{t=1}^T \Gamnoise_t(\theta,\gamma/\sqrt{2\dimu})  \\
& \qquad  \succeq  \Gamfreq_{T/\dimu}(\theta,\bmU_j) + \frac{1 }{\dimu} \sum_{t=1}^T \Gamnoise_t(\theta,\gamma/\sqrt{2\dimu}) - \frac{\epsilon T}{2 \dimu} \lammin(\Gamnoise_{T/2}(\theta,\gamma/\sqrt{2\dimu}))  \cdot I \\
& \qquad \succeq   (1-\epsilon) \Gamfreq_{T/\dimu}(\theta,\bmU_j) + \frac{1 - \epsilon}{\dimu} \sum_{t=1}^T \Gamnoise_t(\theta,\gamma/\sqrt{2\dimu}) 
\end{align*}
It follows that if \eqref{eq:Tinit_mat_opt1} holds for each $j$,
\begin{align*}
\Gamma_T(\theta,\bmU,\gamma/\sqrt{2\dimu},0) & \succeq \frac{(1-\epsilon)^2}{T} \sum_{j=1}^{\dimu} \Gamfreq_{T/\dimu}(\theta,\bmU_j) + \frac{(1-\epsilon)^2}{T} \sum_{t=1}^T \Gamnoise_t(\theta,\gamma/\sqrt{2\dimu}) \\
& = \frac{\dimu (1-\epsilon)^2}{T} \Gamfreq_{T/\dimu}(\theta,\bmU) + \frac{(1-\epsilon)^2}{T} \sum_{t=1}^T \Gamnoise_t(\theta,\gamma/\sqrt{2\dimu}) \\
& = (1-\epsilon)^2 \Gamss_{T,T/\dimu}(\theta,\bmU,\gamma/\sqrt{2\dimu})
\end{align*}
It remains to ensure that \eqref{eq:Tinit_mat_opt1}  holds by bounding $\| \xu_{(j-1)T/\dimu} \|_2$. Again by Lemma D.7 of \cite{wagenmaker2020active}, we have
$$ \| \xu_{(j-1)T/\dimu} \|_2 \le \frac{2 \tau(A,\rho) \| B \|_\op k \gamma}{1-\rho^k} \I\{j > 1\} + \tau(A,\rho) \rho^{(j-1)T/\dimu} \| x_0 \|_2 \le \frac{\tau(A,\rho)(2 \| B \|_\op k \gamma + \| x_0 \|_2)}{1-\rho^k}$$
Some algebra gives the result.
\end{proof}

%% file: body/lqr.tex
%!TEX root = ../main.tex

\newcommand{\Ktheta}{K_{\theta}}
\newcommand{\Ptheta}{P_{\theta}}
\newcommand{\wtK}{\wt{K}}
\newcommand{\DelKa}{\Delta_{K1}}
\newcommand{\DelKb}{\Delta_{K2}}
\newcommand{\DelKc}{\Delta_{K3}}
\newcommand{\Deltheta}{\Delta_{\theta}}
\newcommand{\Pthetatil}{P_{\thetatil}}
\newcommand{\Kthetatil}{K_{\thetatil}}
\newcommand{\Pnorm}{\Psi_{\Pst}}
\newcommand{\Bnorm}{\Psi_{\Bst}}
\newcommand{\Runorm}{\Psi_{\Ru}}

\section{LQR as Linear Dynamical Decision Making}\label{sec:exp_design}

\subsection{LQR is an Instance of \lddm}
Throughout this section we will assume that $\Rx,\Ru \succeq I$ and that $\thetast$ is stabilizable. Note that, by Lemma 3.1 of \cite{simchowitz2020naive}, the assumption that $\thetast$ is stabilizable implies that in a neighborhood of $\thetast$, $P_{\infty}(\theta)$ and $\Kopt(\theta)$ are infinitely differentiable. We will make use of this fact throughout this section, freely taking derivatives of both quantities. Define:
$$ \Pnorm := \| \Pst \|_\op, \quad \Bnorm := \| \Bst \|_\op, \quad \Runorm := \| \Ru \|_\op . $$

\begin{thm}
If $\thetast$ is stabilizable and $\Ru,\Rx \succeq I$, Assumption \ref{asm:smoothness} is satisfied for $\Rlqrb$ with:
\begin{enumerate}
\item $\mu = 2$.
\item $ \betast(\thetast) = \min \left \{ \frac{1}{150 \Pnorm^5}, \frac{1}{240\Bnorm \Pnorm^{5}}, \frac{\Bnorm}{2} \right \}.$
\item $ L_{\calR 1} = c_1 \dimx \Bnorm \sqrt{\Pnorm} + \frac{c_2 \dimx \Runorm (1 + 1/\Bnorm)}{\sqrt{\Pnorm}}, \quad L_{\calR 2} = c_3 \dimx \Big ( \Bnorm^2 \Pnorm^2 + \Runorm (1 + \Bnorm) \Pnorm  \Big )$, $L_{\calR 3}  = c_4 \dimx \Big ( \Bnorm^3 \Pnorm^{7/2} +  \Runorm \Bnorm ( 1 + \Bnorm) \Pnorm^{5/2} \Big ) $.
\item $L_{\fraka 1} =  8 \Pnorm^{7/2}, \quad L_{\fraka 2} = \poly(\Pnorm), \quad L_{\fraka 3} =  \poly(\Pnorm, \Bnorm)$.
\item $\Lra =  \dimx \poly( \Pnorm, \Runorm, \Bnorm, 1/\Bnorm)  + \dimx^2 \Big ( \Bnorm^2 \Pnorm^4 + (1+\Bnorm) \Runorm \Pnorm^3 + \frac{\Runorm \Pnorm}{\Bnorm} \Big ) $.
\end{enumerate}
for universal constants $c_1,c_2,c_3,c_4$.
\end{thm}
\begin{proof}
From Lemma B.9 of \cite{simchowitz2020naive}, we have that:
$$ \Rlqrb(K; \thetast) = \tr \left ( \dlyap \left (\Ast + \Bst K, (K - \Kst)^\top (\Ru + \Bst^\top \Pst \Bst) (K - \Kst) \right ) \right ) $$
From this and the definition of $\dlyap$ it follows that $\Rlqrb(\Kst;\thetast) = 0$. Furthermore, if we define $K(t) = \Kst + t \DelK$ for some $\DelK$, by the chain rule we have that
$$ \frac{d}{dt} \Rlqrb(K(t); \thetast) = \nabla_K \Rlqrb(K;\thetast)|_{K = K(t)} [\DelK]$$
Using the expression for $\frac{d}{dt} \Rlqrb(K(t); \thetast)$ given in the proof of Lemma \ref{lem:lqr_lr1_bound}, we see that $\frac{d}{dt} \Rlqrb(K(t); \thetast)|_{t=0} = 0$ for $\DelK$, from which it follows that $\nabla_K \Rlqrb(K;\thetast)|_{K = \Kst} = 0$. Under the assumption that $\Ast + \Bst K$ stable, Lemma B.5 of \cite{simchowitz2020naive} gives that:
$$ \dlyap \left (\Ast + \Bst K, (K - \Kst)^\top (\Ru + \Bst^\top \Pst \Bst) (K - \Kst) \right ) \succeq (K - \Kst)^\top (\Ru + \Bst^\top \Pst \Bst) (K - \Kst)  $$ 
If $\Ru \succeq I$, then $ (K - \Kst)^\top (\Ru + \Bst^\top \Pst \Bst) (K - \Kst)  \succeq (K - \Kst)^\top (K - \Kst)$, and thus, under these conditions, we have
$$ \Rlqrb(K; \Thetast) \geq \tr \left ( (K - \Kst)^\top (K - \Kst) \right ) = \| K - \Kst \|_F^2 $$
If $\Ast + \Bst K$ is not stable but $(\Ast,\Bst)$ is a stabilizable system, then the LQR cost is infinite but the optimal LQR cost is finite so $\Rlqr(K; \Thetast) = \infty \geq \| K - \Kst \|_F^2$. Thus, we can choose $\mu = 2$. The gradient norm bounds follow directly from Lemmas \ref{lem:lqr_kgrad_bound}, \ref{lem:lqr_lr1_bound}, \ref{lem:lqr_lr23_bound}, and \ref{lem:lqr_smooth_hessian}. Note that these bounds hold in the domain
\begin{align}\label{eq:lqr_domain}
\| \theta - \thetast \|_\op \le \min \{ 1/(150 \Pnorm^5), \Bnorm/2 \}, \quad \| K - \Kst \|_\op \le 1/(30 \Bnorm \Pnorm^{3/2})
\end{align}
and that $L_{\fraka 1} = 8 \Pnorm^{7/2}$. Since 
\begin{align*}
\| \theta - \thetast \|_\op \le  \| \theta - \thetast \|_F, \quad \| K - \Kst \|_\op \le  \| K - \Kst \|_F
\end{align*}
choosing 
\begin{align*}
\betast(\thetast) = \min \left \{ \frac{1}{150 \Pnorm^5}, \frac{1}{240\Bnorm \Pnorm^{5}}, \frac{\Bnorm}{2} \right \}.
\end{align*}
we will have that any $\theta$ satisfying $\| \theta - \thetast \|_F \le \betast(\thetast)$ also satisfies \eqref{eq:lqr_domain} and that any $K$ satisfying $\| K - \Kst \|_F \le L_{\fraka 1}\betast(\thetast)$ also satisfies \eqref{eq:lqr_domain}.
\end{proof}

\subsection{Norm Bounds on Gradients}

\begin{lem}\label{lem:lqr_kgrad_bound}
Assume that $\thetast$ is stabilizable and $\Rx, \Ru \succeq I$. Consider some alternate instance $\theta_0 = (A_0,B_0)$ with $\| \theta_0 - \thetast \|_\op \le \min \{ 1/(150 \Pnorm^5), \Bnorm/2 \}$. Then, for any $\delta$ with $\| \delta \|_\op = 1$,
$$ \| \nabla_\theta \Kopt(\theta) |_{\theta = \theta_0} \|_\op \le 8 \Pnorm^{7/2}$$
$$ \| \nabla_\theta^2 \Kopt(\theta)|_{\theta = \theta_0} \|_\op \le \poly(\Pnorm)$$
$$ \| \nabla_\theta^3 \Kopt(\theta)|_{\theta = \theta_0}[\delta,\delta,\delta] \|_\op \le \poly(\Pnorm, \Bnorm)$$
where $c$ is a universal constant.
\end{lem}
\begin{proof}
Fix some $\Deltheta$ with $\| \Deltheta \|_\op = 1$ and let $\thetatil(s) = \theta_0 + s \Deltheta$. By the chain rule,
$$ \frac{d}{ds} \Kopt(\thetatil(s)) = \nabla_\theta \Kopt(\theta)|_{\theta = \thetatil(s)}[\Deltheta] $$
so to bound $\| \nabla_\theta \Kopt(\theta)|_{\theta = \theta_0} \|_\op$, it suffices to bound $\| \frac{d}{ds} \Kopt(\thetatil(s))|_{s = 0}\|_\op$ for all unit norm $\Deltheta$. Lemma 3.2 of \cite{simchowitz2020naive} gives that, for $s$ where $\thetatil(s)$ is stabilizable, and any unit norm $\Deltheta$,
$$ \| \frac{d}{ds} \Kopt(\thetatil(s)) \|_\op \le 7 \| P(s) \|_\op^{7/2} $$
By Lemma \ref{lem:lqr_random_norm_bounds}, $\thetatil(0) = \theta_0$ will be stabilizable, and $\| P(0) \|_\op \le 5\sqrt{\frac{3}{71}} \Pnorm$. Immediately, then, we have
$$  \| \nabla_\theta \Kopt(\theta) |_{\theta = \theta_0} \|_\op \le 8 \Pnorm^{7/2} $$
For the second bound, we note that
$$ \frac{d^2}{ds^2} \Kopt(\thetatil(s)) = \nabla_\theta^2 \Kopt(\theta)|_{\theta = \thetatil(s)}[\Deltheta,\Deltheta] $$
and, since the Hessian is symmetric, to obtain a bound on $\| \nabla_\theta^2 \Kopt(\theta)|_{\theta = \theta_0} \|_\op$ we can simply bound $\| \frac{d^2}{ds^2} \Kopt(\thetatil(s))|_{s=0} \|_\op$ for all unit norm $\Deltheta$. However, Lemma B.3 of \cite{simchowitz2020naive}, and the argument made above give that
$$ \| \frac{d^2}{ds^2} \Kopt(\thetatil(s))|_{s=0} \|_\op \le \poly(\Pnorm) $$
from which the second conclusion follows. Finally, for the third result, note that
$$ \frac{d^3}{ds^3} \Kopt(\thetatil(s)) = \nabla_\theta^3 \Kopt(\theta)|_{\theta = \thetatil(s)}[\Deltheta,\Deltheta,\Deltheta] $$
As before, it is sufficient to simply bound $\| \frac{d^3}{ds^3} \Kopt(\thetatil(s))|_{s = 0} \|_\op$. Since $\thetatil(0)$ is stabilizable by Lemma \ref{lem:lqr_random_norm_bounds}, Lemma \ref{lem:lqr_k3d_bound} gives
$$\| \frac{d^3}{ds^3} \Kopt(\thetatil(s))|_{s = 0} \|_\op \le \poly(\| P(0) \|_\op, \| B(0) \|_\op, \| \Acl(0) \|_\op) $$
By Lemma \ref{lem:lqr_random_norm_bounds}, and Lemma B.8 of \cite{simchowitz2020naive}, we can upper bound this by $\poly(\Pnorm, \Bnorm)$, which gives the final conclusion.
\end{proof}

\begin{lem}\label{lem:lqr_lr1_bound}
Assume that $\thetast$ is stabilizable and $\Rx, \Ru \succeq I$. For any $\theta$ satisfying $\| \theta - \thetast \|_\op \le  \min \{ 1/(150 \Pnorm^5), \Bnorm / 2 \}$ and $K_0$ satisfying $\| K_0 - \Kst \|_\op \le 1/(30 \Bnorm \Pnorm^{3/2})$:
$$\| \nabla_K \calR(K;\theta)|_{K=K_0} \|_\op \le c_1 \dimx \Bnorm \sqrt{\Pnorm} + \frac{c_2 \dimx \Runorm (1+ 1/\Bnorm)}{\sqrt{\Pnorm}} $$
for universal constant $c_1,c_2$.
\end{lem}
\begin{proof}
Throughout this proof we will assume that $\Pst \succeq I, \Ptheta \succeq I$, which holds by Lemma 4.2 of \cite{simchowitz2020naive} so long as $\Rx \succeq I$. Fix $\theta$ and let $\wtK(t) = K_0 + t \DelK$ for $\DelK$ satisfying $\| \DelK \|_\op = 1$. By the chain rule
$$ \frac{d}{dt} \calR(\wtK(t);\theta)|_{t=0} = \nabla_K \calR(K;\theta)|_{K = K_0}[\DelK] $$
so to bound $\| \nabla_K \calR(K;\theta)|_{K=K_0} \|_\op$, it is sufficient to bound $\frac{d}{dt} \calR(\wtK(t);\theta)|_{t = 0}$ over all unit norm $\DelK$.

For a given $\theta$, we'll denote $\Ktheta := \Kopt(\theta)$, $\Ptheta := P_\infty(\theta)$, and use $A,B$ to refer to the system matrices associated with $\theta$. Then by Lemma B.9 of \cite{simchowitz2020naive}, we have
$$ \calR(K;\theta) = \tr \left ( \dlyap \left ( A + B K, (K - \Ktheta)^\top(\Ru + B^\top \Ptheta B)(K - \Ktheta) \right ) \right ) $$
Define $ Q(t) :=  \dlyap \left ( A + B \wtK(t), (\wtK(t) - \Ktheta)^\top(\Ru + B^\top \Ptheta B)(\wtK(t) - \Ktheta) \right )$. It follows that $\frac{d}{dt} \calR(\wtK(t);\theta) = \tr(\frac{d}{dt} Q(t))$. By definition of $\dlyap$,
$$ Q(t) = (A + B\wtK(t))^\top Q(t) (A + B\wtK(t)) + (\wtK(t) - \Ktheta)^\top (\Ru + B^\top \Ptheta B)(\wtK(t) - \Ktheta) $$
Differentiating $Q(t)$ (and hiding $t$ dependence for simplicity), and since $\wtK' = \DelK$, we have
\begin{align*}
Q' & = (A + B\wtK)^\top Q' (A+B\wtK) + (B\DelK)^\top Q (A+B\wtK) + (A+B\wtK)^\top Q (B\DelK) \\
& \qquad \qquad \DelK^\top (\Ru + B^\top \Ptheta B)(\wtK - \Ktheta) + (\wtK - \Ktheta)^\top (\Ru + B^\top \Ptheta B) \DelK \\
& = \dlyap \Big ( A + B\wtK, (B\DelK)^\top Q (A+B\wtK) + (A+B\wtK)^\top Q (B\DelK) \\
& \qquad \qquad \qquad \DelK^\top (\Ru + B^\top \Ptheta B)(\wtK - \Ktheta) + (\wtK - \Ktheta)^\top (\Ru + B^\top \Ptheta B) \DelK \Big )
\end{align*}
By Lemma B.5 of \cite{simchowitz2020naive}, we can upper bound this $\dlyap$ expression as,
\begin{align*}
\| Q'(0) \|_\op \le 2 \| \dlyap(A+ & B\wtK(0),I) \|_\op \Big ( \| (A+B\wtK(0))^\top Q(0) B \DelK \|_\op \\
& + \| (\wtK(0) - \Ktheta)^\top (\Ru + B^\top \Ptheta B) \DelK \|_\op \Big ) 
\end{align*}
Since $\| \theta - \thetast \|_\op \le \min \min \{ 1/(150 \Pnorm^5), \Bnorm / 2 \}$ and $\| \wtK(0) - \Kst \|_\op \le 1/(30 \Bnorm \Pnorm^{3/2})$, we can apply the norm bounds in Lemma \ref{lem:lqr_random_norm_bounds} to upper bound this as 
$$ \| Q'(0) \|_\op \le c \Bnorm \sqrt{\Pnorm} + \frac{c \Runorm (1+ 1/\Bnorm)}{\sqrt{\Pnorm}} $$
As this holds independent of $\DelK$ and since $| \frac{d}{dt} \calR(K(t);\theta)|_{t = 0}| \le \dimx \| Q'(0) \|$, it follows that
$$ \| \nabla_K \calR(K;\theta)|_{K = K_0}  \|_\op \le c \dimx \Bnorm \sqrt{\Pnorm} + \frac{c \dimx \Runorm (1+ 1/\Bnorm)}{\sqrt{\Pnorm}}. $$
\end{proof}

\begin{lem}\label{lem:lqr_lr23_bound}
Assume that $\thetast$ is stabilizable and $\Rx, \Ru \succeq I$. For any $\theta$ satisfying $\| \theta - \thetast \|_\op \le  \min \{ 1/(150 \Pnorm^5), \Bnorm / 2 \}$ and $K_0$ satisfying $\| K_0 - \Kst \|_\op \le 1/(30 \Bnorm \Pnorm^{3/2})$:
$$\| \nabla_K^2 \calR(K;\theta)|_{K=K_0} \|_\op \le c_1 \dimx \Big ( \Bnorm^2 \Pnorm^2 + \Runorm (1 + \Bnorm) \Pnorm  \Big )$$
$$ \| \nabla_K^3 \calR(K;\theta)|_{K=K_0} \|_\op  \le c_2 \dimx \Big ( \Bnorm^3 \Pnorm^{7/2} +  \Runorm \Bnorm ( 1 + \Bnorm) \Pnorm^{5/2} \Big ) $$
for some universal constants $c_1,c_2$.
\end{lem}
\begin{proof}
Throughout, unless otherwise specified, we adopt the same notation as is used in the proof of Lemma \ref{lem:lqr_lr1_bound}. Let $\wtK(t_1,t_2) = K_0 + t_1 \DelKa + t_2 \DelKb$. By the chain rule, and since $\frac{d}{dt_1} \wtK(t_1,t_2) = \DelKa, \frac{d}{dt_2} \wtK(t_1,t_2) = \DelKb$,
$$ \frac{d}{dt_2} \frac{d}{dt_1} \calR(\wtK(t_1,t_2);\theta)|_{t_1=t_2=0} = \nabla_K^2 \calR(K;\theta)|_{K=K_0}[\DelKa,\DelKb]$$
To bound $\| \nabla_K^2 \calR(K;\theta)|_{K=K_0} \|_\op$, it then suffices to bound $\| \frac{d}{dt_2} \frac{d}{dt_1} \calR(\wtK(t_1,t_2);\theta)|_{t_1=t_2=0} \|_\op$ over all unit norm $\DelKa,\DelKb$. From the proof of Lemma \ref{lem:lqr_lr1_bound}, we have that
$$ \frac{d}{dt_2} \frac{d}{dt_1} \calR(\wtK(t_1,t_2);\theta) = \tr \left ( \frac{d}{dt_2} \frac{d}{dt_1} Q(t_1,t_2) \right ) $$
where 
$$ Q(t_1,t_2) = (A + B\wtK(t_1,t_2))^\top Q(t_1,t_2) (A + B\wtK(t_1,t_2)) + (\wtK(t_1,t_2) - \Ktheta)^\top (\Ru + B^\top \Ptheta B)(\wtK(t_1,t_2) - \Ktheta) $$
Using our expression for the first derivate of $Q$ from the proof of Lemma \ref{lem:lqr_lr1_bound}, dropping the explicit $t_1,t_2$ dependence, and adopting the notation $Q_{t_i} = \frac{d}{dt_i} Q$,
\begin{align*}
Q_{t_1} & = (A + B\wtK)^\top (Q_{t_1}) (A+B\wtK) + (B\DelKa)^\top Q (A+B\wtK) + (A+B\wtK)^\top Q (B\DelKa) \\
& \qquad \qquad \DelKa^\top (\Ru + B^\top \Ptheta B)(\wtK - \Ktheta) + (\wtK - \Ktheta)^\top (\Ru + B^\top \Ptheta B) \DelKa \\
& = \dlyap \Big ( A+B\wtK,  (B\DelKa)^\top Q (A+B\wtK) + (A+B\wtK)^\top Q (B\DelKa) \\
& \qquad \qquad \DelKa^\top (\Ru + B^\top \Ptheta B)(\wtK - \Ktheta) + (\wtK - \Ktheta)^\top (\Ru + B^\top \Ptheta B) \DelK \Big )
\end{align*}
then taking the derivative of this with respect to $t_2$ gives
\begin{align*}
Q_{t_1,t_2} & = (A + B\wtK)^\top (Q_{t_1,t_2}) (A+B\wtK) + (B \DelKb)^\top (Q_{t_1}) (A+B\wtK) + (A + B\wtK)^\top (Q_{t_1})(B \DelKb) \\
& \qquad  + (B \DelKa)^\top (Q_{t_2}) (A + B \wtK) + (A + B \wtK)^\top (Q_{t_2}) (B \DelKa) + (B\DelKa)^\top Q (B \DelKb) \\
& \qquad + (B\DelKb)^\top Q (B\DelKa)  +  \DelKa^\top (\Ru + B^\top \Ptheta B) \DelKb + \DelKb^\top (\Ru + B^\top \Ptheta B) \DelKa \\
& = \dlyap \Big ( A + B\wtK, (B \DelKb)^\top (Q_{t_1}) (A+B\wtK)  + (A + B\wtK)^\top (Q_{t_1})(B \DelKb) + (B \DelKa)^\top (Q_{t_2}) (A + B \wtK) \\
& \qquad + (A + B \wtK)^\top (Q_{t_2}) (B \DelKa) + (B\DelKa)^\top Q (B \DelKb) + (B\DelKb)^\top Q (B\DelKa) \\
& \qquad +  \DelKa^\top (\Ru + B^\top \Ptheta B) \DelKb + \DelKb^\top (\Ru + B^\top \Ptheta B) \DelKa \Big )
\end{align*}
We would like to bound the operator norm of $Q_{t_1,t_2}(0)$. Note that the bound on $\| Q'(0) \|_\op$ given in Lemma \ref{lem:lqr_lr1_bound} still applies in this setting due to our restriction that $\| \wtK(0,0) - \Kst \|_\op \le 1/(30\Bnorm \Pnorm^{3/2})$, so 
$$ \| Q(0) \|_\op \le \frac{c}{\Pnorm} + \frac{c \Runorm }{\Bnorm^2 \Pnorm^2} $$
$$ \| Q_{t_1}(0) \|_\op, \| Q_{t_2}(0) \|_\op \le c_2 \Bnorm \sqrt{\Pnorm} + \frac{c_3 \Runorm ( 1 + 1/\Bnorm)}{\sqrt{\Pnorm}} $$
Furthermore, we are in the domain where Lemma \ref{lem:lqr_random_norm_bounds} holds so,
\begin{align*}
\| Q_{t_1,t_2}(0) \|_\op & \le \| \dlyap(A+B\wtK(0),I) \|_\op \Big ( 2 \| B \|_\op \| A + B \wtK(0) \|_\op (\| Q_{t_1}(0) \|_\op + \| Q_{t_2}(0) \|_\op)  \\
& \qquad \qquad + 2 \| B \|_\op^2 \| Q(0) \|_\op + 2 \| \Ru + B^\top \Ptheta B \|_\op \Big ) \\
& \le c \Bnorm^2 \Pnorm^2 + c\Runorm (1 + \Bnorm) \Pnorm 
\end{align*}
Since $| \frac{d}{dt_1} \frac{d}{dt_2} \calR(\wtK(t_1,t_2);\theta)| \le \dimx \| Q_{t_1,t_2} \|_\op$, the first bound follows.

To bound $ \| \nabla_K^3 \calR(K;\theta) \|_\op$, we define $\wtK(t_1,t_2,t_3) = K_0 + t_1 \DelKa + t_2 \DelKb + t_3 \DelKc$, for 
$$\| \DelKa \|_\op, \| \DelKb \|_\op, \| \DelKc \|_\op = 1$$
and note that by the chain rule
$$ \frac{d}{dt_3} \frac{d}{dt_2} \frac{d}{dt_1} \calR(\wtK(t_1,t_2,t_3);\theta)|_{t_1=t_2=t_3=0} = \nabla_K^3 \calR(K;\theta)|_{K=K_0}[\DelKa,\DelKb,\DelKc]$$
so, as before, it suffices to bound $\| \frac{d}{dt_3} \frac{d}{dt_2} \frac{d}{dt_1} \calR(\wtK(t_1,t_2,t_3);\theta)|_{t_1=t_2=t_3=0}\|_\op$ over all unit norm $\DelKa,\DelKb,\DelKc$. Again we have that
$$ \frac{d}{dt_3} \frac{d}{dt_2} \frac{d}{dt_1} \calR(\wtK(t_1,t_2,t_3);\theta) = \tr \left ( \frac{d}{dt_3} \frac{d}{dt_2} \frac{d}{dt_1} Q(t_1,t_2,t_3) \right ) $$
To bound this, we can differentiate the expression for $Q_{t_1,t_2}$ given above with respect to $t_3$:
\begin{align*}
Q_{t_1,t_2,t_3} & = (A + B\wtK)^\top (Q_{t_1,t_2,t_3}) (A+B\wtK) + (B\DelKc)^\top (Q_{t_1,t_2}) (A+B\wtK) + (A + B\wtK)^\top (Q_{t_1,t_2}) (B \DelKc) \\
& \qquad + (B \DelKb)^\top (Q_{t_1,t_3}) (A+B\wtK) + (B \DelKb)^\top (Q_{t_1}) (B\DelKc)  + (A + B\wtK)^\top (Q_{t_1,t_3})(B \DelKb)  \\
& \qquad +  (B\DelKc)^\top (Q_{t_1})(B \DelKb) + (B \DelKa)^\top (Q_{t_2,t_3}) (A + B \wtK) + (B \DelKa)^\top (Q_{t_2}) (B \DelKc)  \\
& \qquad   +  (A + B \wtK)^\top (Q_{t_2,3}) (B \DelKa) +  (B \DelKc)^\top (Q_{t_2}) (B \DelKa) + (B\DelKa)^\top Q_{t_3} (B \DelKb) \\
& \qquad + (B\DelKb)^\top Q_{t_3} (B\DelKa) \\
& = \dlyap \Big ( A + B \wtK, (B\DelKc)^\top (Q_{t_1,t_2}) (A+B\wtK) + (A + B\wtK)^\top (Q_{t_1,t_2}) (B \DelKc) \\
& \qquad + (B \DelKb)^\top (Q_{t_1,t_3}) (A+B\wtK) + (B \DelKb)^\top (Q_{t_1}) (B\DelKc)  + (A + B\wtK)^\top (Q_{t_1,t_3})(B \DelKb)  \\
& \qquad +  (B\DelKc)^\top (Q_{t_1})(B \DelKb) + (B \DelKa)^\top (Q_{t_2,t_3}) (A + B \wtK) + (B \DelKa)^\top (Q_{t_2}) (B \DelKc)  \\
& \qquad   +  (A + B \wtK)^\top (Q_{t_2,3}) (B \DelKa) +  (B \DelKc)^\top (Q_{t_2}) (B \DelKa) + (B\DelKa)^\top Q_{t_3} (B \DelKb) \\
& \qquad + (B\DelKb)^\top Q_{t_3} (B\DelKa) \Big ) 
\end{align*}
Thus,
\begin{align*}
\| Q_{t_1,t_2,t_3}(0)  \|_\op & \le \| \dlyap(A+B\wtK(0),I) \|_\op \Big ( 2 \| B \|_\op \| A + B \wtK(0) \|_\op ( \| Q_{t_1,t_2}(0) \|_\op + \| Q_{t_1,t_3}(0) \|_\op  \\
& \qquad \qquad + \| Q_{t_2,t_3}(0) \|_\op ) + 2 \| B \|_\op^2 ( \| Q_{t_1}(0) \|_\op +  \| Q_{t_2} (0)\|_\op +  \| Q_{t_3}(0) \|_\op) \Big ) 
\end{align*}
Note that the norm bounds proved on $\| Q'(0) \|_\op$ given in Lemma \ref{lem:lqr_lr1_bound} still applies in this setting due to our restriction that $\| K_0 - \Kst \|_\op \le 1/(30\Bnorm \Pnorm^{3/2})$, and similarly our bound proved above on $\| Q_{t_1,t_2}(0) \|_\op$ can be used to bound each of the second derivatives. Combining these results, and using that $\Pnorm \ge 1$, gives
$$ \| Q_{t_1,t_2,t_3}(0)  \|_\op \le c \Bnorm^3 \Pnorm^{7/2} + c \Runorm \Bnorm ( 1 + \Bnorm) \Pnorm^{5/2} $$
The second bound then follows directly.
\end{proof}

\begin{lem}\label{lem:lqr_smooth_hessian}
Assume that $\thetast$ is stabilizable and $\Rx, \Ru \succeq I$. For $\theta_0,\theta_1$ satisfying $\| \theta_0 - \thetast \|_\op, \| \theta_1 - \thetast \|_\op \le  \min \{ 1/(150 \Pnorm^5), \Bnorm / 2 \}$ and $K_0$ satisfying $\| K_0 - \Kst \|_\op \le 1/(30 \Bnorm \Pnorm^{3/2})$, we have
\begin{align*}
& \| \nabla_K^2 \calR(K;\theta)|_{K = K_0} - \nabla_K^2 \calR(K;\theta')|_{K=K_0} \|_\op \le  \bigg ( \dimx \poly( \Pnorm, \Runorm, \Bnorm, 1/\Bnorm)  \\
& \qquad \qquad + \dimx^2 \Big ( \Bnorm^2 \Pnorm^4 + (1+\Bnorm) \Runorm \Pnorm^3 + \frac{\Runorm \Pnorm}{\Bnorm} \Big ) \bigg ) \cdot \| \theta_0 - \theta_1\|_\op.
\end{align*}
\end{lem}
\begin{proof}
Note that, since the Hessian is symmetric,
\begin{align*}
& \| \nabla_K^2 \calR(K;\theta_0)|_{K = K_0} - \nabla_K^2 \calR(K;\theta_1)|_{K=K_0} \|_\op \\
&  \qquad = \max_{\DelK : \| \DelK \|_\op  = 1} | \nabla_K^2 \calR(K;\theta_0)|_{K = K_0}[\DelK,\DelK] - \nabla_K^2 \calR(K;\theta_1)|_{K=K_0}[\DelK,\DelK] | 
\end{align*}
so it suffices to bound $| \nabla_K^2 \calR(K;\theta_0)|_{K = K_0}[\DelK,\DelK] - \nabla_K^2 \calR(K;\theta_1)|_{K=K_0}[\DelK,\DelK] | $ over all unit norm $\DelK$. Let $\Deltheta = (\DelA,\DelB)$ satisfy $\| \Deltheta \|_\op = 1$ and denote $A(s) = \Ast + s \DelA$, $B(s) = \Bst + s \DelB$, and $\thetatil(s) = (A(s),B(s))$. Assume that $\theta_1 = \theta_0 + s_1 \Deltheta$ for some $s_1$. Fix $\DelK$ with $\| \DelK \|_\op  = 1$. By the chain rule,
$$ \frac{d}{ds} \nabla_K^2 \calR(K;\thetatil(s))|_{K = K_0}[\DelK,\DelK] = \nabla_\theta (\nabla_K^2 \calR(K;\theta)|_{K = K_0}[\DelK,\DelK])|_{\theta=\thetatil(s)} [\Deltheta] $$
So by Taylor's Theorem,
$$ \nabla_K^2 \calR(K;\theta_0)|_{K = K_0}[\DelK,\DelK] =  \nabla_K^2 \calR(K;\theta_1)|_{K = K_0}[\DelK,\DelK] - \nabla_\theta (\nabla_K^2 \calR(K;\thetatil(s))|_{K = K_0}[\DelK,\DelK])|_{s=s_2} [\Deltheta] s_1 $$
for some $s_2 \in [0,s_1]$. Thus, since $\| \theta_0 - \theta_1 \|_\op = \| s_1 \Deltheta \|_\op = s_1$, denoting $\theta_2 := \thetatil(s_2)$,
\begin{align*}
 |  \nabla_K^2 \calR(K;\theta_0)|_{K = K_0}[\DelK,\DelK] & -  \nabla_K^2 \calR(K;\theta_1)|_{K = K_0}[\DelK,\DelK] | \\
& \le |\nabla_\theta (\nabla_K^2 \calR(K;\theta)|_{K = K_0}[\DelK,\DelK])|_{\theta = \theta_2} [\Deltheta]| \cdot \| \theta_0 - \theta_1\|_\op 
\end{align*}
So we can simply bound $|\nabla_\theta (\nabla_K^2 \calR(K;\theta)|_{K = K_0}[\DelK,\DelK])|_{\theta = \theta_2} [\Deltheta]|$ over all unit norm $\Deltheta$, and all $\theta_2 = \theta + s_2 \Deltheta$, $s_2 \in [0,s_1]$. Note that $\Deltheta = \frac{1}{s_1} (\theta_1 - \theta_0)$, so
\begin{align*}
\| \theta_2 - \thetast \|_\op &= \| \theta_0 + s_2 \Deltheta - \thetast \|_\op = \| (1-s_2/s_1) \theta_0 + (s_2/s_1) \theta_1 - \thetast \|_\op \\
&\le (1-s_2/s_1) \| \theta_0 - \thetast \|_\op + s_2/s_1 \| \theta_1 - \thetast \|_\op \le  \min \{ 1/(150 \Pnorm^5), \Bnorm / 2 \} 
\end{align*}
so we are in the domain where the bounds given in Lemmas \ref{lem:lqr_lr1_bound}, \ref{lem:lqr_lr23_bound}, and \ref{lem:lqr_random_norm_bounds} hold. By Lemma \ref{lem:lqr_lr23_bound}, we know that (where we drop the dependence on $s$ for brevity) 
\begin{align*}
 \nabla_K^2 \calR(K;\thetatil(s))|_{K = K_0}[\DelK,\DelK] & = \tr \Big ( \dlyap \Big ( A + BK_0,  2 (B \DelK)^\top (Q_{1}) (A+BK_0)  + 2 (A + BK_0)^\top (Q_{1})(B \DelK)  \\
& \qquad \qquad  + 2 (B\DelK)^\top Q (B \DelK)  +  2 \DelK^\top (\Ru + B^\top \Pthetatil B) \DelK \Big ) \\
& =: \tr(M(s))
\end{align*}
where $Q$ and $Q_1$ satisfy
$$ Q= (A + BK_0)^\top Q(A + BK_0) + (K_0 - \Kthetatil)^\top (\Ru + B^\top \Pthetatil B)(K_0 - \Kthetatil) $$
\begin{align*}
Q_{1} & = (A + BK_0)^\top (Q_{1}) (A+BK_0) + (B\DelK)^\top Q (A+BK_0) + (A+BK_0)^\top Q (B\DelK) \\
& \qquad \qquad + \DelK^\top (\Ru + B^\top \Pthetatil B)(K_0 - \Kthetatil) + (K_0 - \Kthetatil)^\top (\Ru + B^\top \Pthetatil B) \DelK
\end{align*}
and $\Pthetatil = P_\infty(\thetatil)$. It follows, by the definition of $\dlyap$, that $M$ satisfies,
\begin{align*}
M & = (A+ BK_0)^\top M (A + B K_0) + 2 (B \DelK)^\top (Q_{1}) (A+BK_0)  + 2 (A + BK_0)^\top (Q_{1})(B \DelK)  \\
& \qquad \qquad  + 2 (B\DelK)^\top Q (B \DelK)  +  2 \DelK^\top (\Ru + B^\top \Pthetatil B) \DelK
\end{align*}
and that
$$ \frac{d}{ds} \nabla_K^2 \calR(K;\thetatil)|_{K = K_0}[\DelK,\DelK] = \tr \Big (\frac{d}{ds} M \Big ) $$
Differentiating this expression for $M$ with respect to $s$ gives (where here we let $(.)'$ denote the derivative with respect to $s$)
\begin{align*}
 \tr(M') & = \tr \Big ( ( A + BK_0)^\top M' ( A + BK_0) + ( \DelA + \DelB K_0)^\top M ( A + BK_0) + ( A + BK_0)^\top M ( \DelA + \DelB K_0) \\
 & \qquad + 2 (\DelB \DelK)^\top (Q_{1}) (A+BK_0) + 2 (B \DelK)^\top (Q_{1}') (A+BK_0) + 2 (B \DelK)^\top (Q_{1}) (\DelA+\DelB K_0) \\
 & \qquad + 2 (\DelA + \DelB K_0)^\top (Q_{1})(B \DelK) + 2 (A + BK_0)^\top (Q_{1}')(B \DelK) + 2 (A + BK_0)^\top (Q_{1})(\DelB \DelK) \\
& \qquad + 2 (\DelB\DelK)^\top Q (B \DelK)  + 2 (B\DelK)^\top Q' (B \DelK)  + 2 (B\DelK)^\top Q (\DelB \DelK) \\
& \qquad +  2 \DelK^\top ( \DelB^\top \Pthetatil B + B^\top \Pthetatil' B + B^\top \Pthetatil \DelB) \DelK \Big )  \\
& = \tr \Big ( \dlyap \Big ( A + BK_0, ( \DelA + \DelB K_0)^\top M ( A + BK_0) + ( A + BK_0)^\top M ( \DelA + \DelB K_0) \\
 & \qquad + 2 (\DelB \DelK)^\top (Q_{1}) (A+BK_0) + 2 (B \DelK)^\top (Q_{1}') (A+BK_0) + 2 (B \DelK)^\top (Q_{1}) (\DelA+\DelB K_0) \\
 & \qquad + 2 (\DelA + \DelB K_0)^\top (Q_{1})(B \DelK) + 2 (A + B K_0)^\top (Q_{1}')(B \DelK) + 2 (A + B K_0)^\top (Q_{1})(\DelB \DelK) \\
& \qquad + 2 (\DelB\DelK)^\top Q (B \DelK)  + 2 (B\DelK)^\top Q' (B \DelK)  + 2 (B\DelK)^\top Q (\DelB \DelK) \\
& \qquad +  2 \DelK^\top (\DelB^\top \Pthetatil B + B^\top \Pthetatil' B + B^\top \Pthetatil \DelB) \DelK \Big )  \Big ) \\
& \le \dimx \| \dlyap(A+B K_0,I) \|_\op \Big (2 (1 + \|  K_0 \|_\op) \| A + B  K_0 \|_\op \| M \|_\op + 4 \| A + B  K_0 \|_\op \| Q_1 \|_\op \\
& \qquad + 4 \| B \|_\op (1 + \|  K_0 \|_\op) \| Q_1 \|_\op + 4 \| A + B  K_0 \|_\op \| B \|_\op \| Q_1' \|_\op + 4 \| B \|_\op \| Q \|_\op + 2 \| B \|_\op^2 \| Q' \|_\op \\
& \qquad  + 4 \| B \|_\op \| \Pthetatil \|_\op + 2 \| B \|_\op^2 \| \Pthetatil' \|_\op \Big ) 
\end{align*}
Then, using bounds proved in Lemmas \ref{lem:lqr_lr1_bound} and \ref{lem:lqr_lr23_bound} to upper bound this at $s = s_2$,
\begin{align*}
 \tr(M'(s_2))& \le c \dimx \Pnorm \Big ( (1 + \|  K_0 \|_\op) \sqrt{\Pnorm} \| M(s_2) \|_\op + \sqrt{\Pnorm} \| Q_1(s_2) \|_\op  \\
& \qquad + \Bnorm ( 1 + \|  K_0 \|_\op) \| Q_1(s_2) \|_\op + \Bnorm \sqrt{\Pnorm} \| Q_1'(s_2) \|_\op + \Bnorm \| Q(s_2) \|_\op \\
& \qquad + \Bnorm^2 \| Q'(s_2) \|_\op + \Bnorm \Pnorm +  \Bnorm^2 \| P_{\thetatil(s_2)}' \|_\op \Big ) 
\end{align*}
By Lemma \ref{lem:lqr_lr1_bound} we can bound
$$ \| Q(s_2) \|_\op \le \frac{c}{\Pnorm} + \frac{c \Runorm}{\Bnorm^2 \Pnorm^2}, \quad \| Q_1(s_2) \|_\op \le c \Bnorm \sqrt{\Pnorm} + \frac{c \Runorm ( 1 + 1/\Bnorm)}{\sqrt{\Pnorm}} $$
and by Lemma \ref{lem:lqr_lr23_bound} we can bound
$$ \| M(s_2) \|_\op \le \| \nabla_K^2 \calR(K;\theta_2)|_{K =  K_0} \|_\op \le c \dimx \Big ( \Bnorm^2 \Pnorm^2 + \Runorm (1 + \Bnorm) \Pnorm \Big ) $$
Furthermore, by our assumption on $ K_0$ and Lemma B.8 of \cite{simchowitz2020naive},
$$ \|  K_0 \|_\op \le \| \Kst \|_\op + 1/(30 \Bnorm \Pnorm^{3/2}) \le \sqrt{\Pnorm} + 1/(30 \Bnorm \Pnorm^{3/2}) $$
and by Lemma 3.2 of \cite{simchowitz2020naive} and Lemma \ref{lem:lqr_random_norm_bounds}
$$ \| P_{\thetatil(s_2)}' \|_\op \le 4 \| P_{\thetatil(s_2)} \|_\op^3 \le c \Pnorm^3 $$
It remains to bound $\| Q'(s_2) \|_\op$ and $\| Q_1'(s_2) \|_\op$.  Given the expression for $Q$, we can differentiate it to get
\begin{align*}
Q' & = \dlyap \Big ( A + B K_0,  (\DelA + \DelB K_0)^\top Q(A + B K_0) + (A + B K_0)^\top Q(\DelA + \DelB  K_0) \\
& \qquad -(\Kthetatil')^\top (\Ru + B^\top \Pthetatil B)( K_0 - \Kthetatil) - ( K_0 - \Kthetatil)^\top (\Ru + B^\top \Pthetatil B)(\Kthetatil) \\
& \qquad + ( K_0 - \Kthetatil)^\top ( \DelB^\top \Pthetatil B + B^\top \Pthetatil' B + B^\top \Pthetatil \DelB)( K_0 - \Kthetatil) \Big )
\end{align*}
so, 
\begin{align*}
\| Q' \|_\op & \le \| \dlyap(A + B  K_0,I) \|_\op \Big ( 2 (1 + \|  K_0 \|_\op) \| A + B  K_0 \|_\op \| Q \|_\op + 2 \| (\Ru + B^\top \Pthetatil B)( K_0 - \Kthetatil) \|_\op \| \Kthetatil' \|_\op \\
& \qquad + 2 \|  K_0 - \Kthetatil \|_\op  \| B ( K_0 - \Kthetatil) \|_\op \| \Pthetatil \|_\op + \| B ( K_0 - \Kthetatil) \|_\op^2 \| \Pthetatil' \|_\op  \Big ) 
\end{align*}
Lemma 3.2 of \cite{simchowitz2020naive} gives $\| K_{\thetatil(s_2)}' \|_\op \le 7 \| P_{\thetatil(s_2)} \|_\op^{7/2} \le c \Pnorm^{7/2}$. Then by the Mean Value Theorem,
\begin{align*}
\|  K_0 - K_{\thetatil(s_2)} \|_\op & \le \|  K_0 - \Kst \|_\op + \| \Kst - K_{\thetatil(s_2)} \|_\op \\
& \le 1/(30\Bnorm \Pnorm^{3/2}) + \max_{s : \| \thetatil(s) - \thetast \|_\op \le  \min \{ 1/(150 \Pnorm^5), \Bnorm / 2 \}} \| K_{\thetatil(s_2)}' \|_\op \| \theta - \thetast \|_\op \\
& \le  c(1 + 1/\Bnorm)/\Pnorm^{3/2}
\end{align*}
Using this, Lemma \ref{lem:lqr_random_norm_bounds}, and what we have shown above, we can then bound
\begin{align*}
\| Q'(s_2) \|_\op & \le c \Big ( \Bnorm \Pnorm^4 + (\Runorm^{1/2} + \Runorm/\Bnorm) \Pnorm^3 + \Pnorm + \Runorm / \Bnorm^2 \\
& \qquad + 1/(\Bnorm \Pnorm) + \Runorm/ (\Bnorm^3 \Pnorm^2) \Big )
\end{align*}
We now bound $\| Q_1'(s_2) \|_\op$. Differentiating the expression for $Q_1$ given above yields:
\begin{align*}
Q_{1}' & = \dlyap \Big ( A + B K_0, (\DelA + \DelB K_0)^\top Q_{1} (A+B K_0) + (A + B K_0)^\top Q_{1} (\DelA+\DelB K_0) \\
&\qquad  + (\DelB\DelK)^\top Q (A+B K_0) + (B\DelK)^\top Q' (A+B K_0) + (B\DelK)^\top Q (\DelA+\DelB K_0) \\
& \qquad + (\DelA+\DelB K_0)^\top Q (B\DelK) + (A+B K_0)^\top Q' (B\DelK) + (A+B K_0)^\top Q (\DelB\DelK) \\
& \qquad + \DelK^\top (\DelB^\top \Pthetatil B + B^\top \Pthetatil' B + B^\top \Pthetatil \DelB)( K_0 - \Kthetatil) - \DelK^\top (\Ru + B^\top \Pthetatil B)( \Kthetatil') \\
& \qquad - (\Kthetatil')^\top (\Ru + B^\top \Pthetatil B) \DelK + ( K_0 - \Kthetatil)^\top (\DelB^\top \Pthetatil B + B^\top \Pthetatil' B + B^\top \Pthetatil \DelB) \DelK \Big )
\end{align*}
Since we have already shown that the operator norms of all terms in this expression are polynomial in problem parameters, we can bound
$$ \| Q_1'(s_2) \|_\op \le \poly( \Pnorm, \Runorm, \Bnorm, 1/\Bnorm) $$
Plugging these quantities into our bound on $\tr(M'(s_2))$, it follows that
\begin{align*}
\tr(M'(s_2)) & \le \dimx^2 \left ( \Bnorm^2 \Pnorm^4 + (1+\Bnorm) \Runorm \Pnorm^3 + \frac{\Runorm \Pnorm}{\Bnorm} \right ) \\
& \qquad \qquad + \dimx \poly( \Pnorm, \Runorm, \Bnorm, 1/\Bnorm)
\end{align*}
As this holds regardless of $\DelK,\Deltheta$, and for all $\theta_2$ with $\| \theta_2 - \thetast \|_\op \le \min \{ 1/(150 \Pnorm^5), \Bnorm / 2 \}$, we have shown that
\begin{align*}
 |  \nabla_K^2 \calR(K;\theta_0)|_{K =  K_0}& [\DelK,\DelK]  -  \nabla_K^2 \calR(K;\theta_1)|_{K =  K_0}[\DelK,\DelK] | \\
& \le |\nabla_\theta (\nabla_K^2 \calR(K;\theta)|_{K =  K_0}[\DelK,\DelK])|_{\theta = \theta_2} [\Deltheta]| \cdot \| \theta_0 - \theta_1\|_\op \\
& \le \dimx^2 \left ( \Bnorm^2 \Pnorm^4 + (1+\Bnorm) \Runorm \Pnorm^3 + \frac{\Runorm \Pnorm}{\Bnorm} \right ) \cdot \| \theta_0 - \theta_1\|_\op \\
& \qquad + \dimx \poly( \Pnorm, \Runorm, \Bnorm, 1/\Bnorm) \cdot \| \theta_0 - \theta_1\|_\op
\end{align*}
from which the desired result follows.
\end{proof}

\subsection{Norm Bounds on Control Theoretic Quantities}

\begin{lem}[Lemmas 3.2, B.3, B.8, and C.5 of \cite{simchowitz2020naive}]\label{lem:lqr_param_bounds}
Let $\theta(t) = (\Ast + t \DelA, \Bst + t \DelB)$ and $P(t) := P_\infty(\theta(t)), K(t) := \Kopt(\theta(t))$. If $\max \{ \| \DelA \|_\op, \| \DelB \|_\op \} \leq \epsilon$ and $\Rx, \Ru \succeq I$, then, for $t$ where $(A(t),B(t))$ is stabilizable, 
\begin{enumerate}
\item  $\| P'(t) \|_\op \leq 4 \| P(t) \|_\op^3 \epsilon$.
\item $\| P''(t) \|_\op \leq \poly(\| P(t) \|_\op) \epsilon^2$.
\item $\| K(t) \|_\op \leq \sqrt{\| P(t) \|_\op}$.
\item $\| K'(t) \|_\op \leq 7 \| P(t) \|_\op^{7/2} \epsilon$.
\item $\| K''(t) \|_\op \leq \poly(\| P(t) \|_\op) \epsilon^2$. 
\end{enumerate}
\end{lem}

\begin{lem}\label{lem:lqr_random_norm_bounds}
Assume that $\thetast$ is stabilizable and $\Rx, \Ru \succeq I$. Consider some alternate $\theta = (A,B)$ with $\| \theta - \thetast \|_\op \le \min \{ 1/(150 \Pnorm^5), \Bnorm/2 \}$ and a controller $K_0$ with $\| K_0 - \Kopt(\thetast) \|_\op \le 1/(30 \Bnorm \Pnorm^{3/2})$. Denote $\Ptheta := P_{\infty}(\theta), \Ktheta := \Kopt(\theta)$. Then the following are true.
\begin{enumerate}
\item $\Pnorm, \| \Ptheta \|_\op \ge 1$.
\item $\theta$ is stabilizable.
\item $ \Pnorm \cong \| \Ptheta \|_\op$ and $\| \Ptheta \|_\op \le 5 \sqrt{\frac{3}{71}}\Pnorm$.
\item $\| B \|_\op \le \frac{3}{2} \Bnorm$.
\item $\| (\Ru + B^\top \Ptheta B) (\Kst - \Ktheta)\|_\op \le c_1 ( \Runorm^{1/2} + \Pnorm^{1/2} \Bnorm)/(\Pnorm^{3/2})$.
\item $\| (\Ru + B^\top \Ptheta B) (K_0 - \Kst)\|_\op \le c_2(\Runorm + \Bnorm^2 \Pnorm)/(\Bnorm \Pnorm^{3/2})$.
\item $\| \dlyap(A+B\Ktheta,I) \|_\op, \| \dlyap(A+BK_0,I) \|_\op \le c_3 \Pnorm$.
\item $\| A + B K_0 \|_\op  \le c_4 \sqrt{\Pnorm}$.
\end{enumerate}
for absolute constants $c_1,c_2,c_3,c_4$ and where $\cong$ denotes equality up to absolute constants.
\end{lem}
\begin{proof}
First, note that $\Rx \succeq I$ implies $\Pnorm, \| \Ptheta \|_\op \ge 1$ by Lemma 4.2 of \cite{simchowitz2020naive}, and that $\| \Ast - A \|_\op, \| \Bst - B \|_\op \le \| \theta - \thetast \|_\op$. Given our assumption on $\| \theta - \thetast \|_\op$, it follows that $8 \Pnorm^2 \| \theta - \thetast \|_\op \le 8/(150\Pnorm^3) \le \frac{4}{75} < 1$. Proposition 6 of \cite{simchowitz2020naive} then implies that $\theta$ is stabilizable and that 
$$ \| \Ptheta \|_\op \le (1-4/75)^{-1/2} \Pnorm \le 5 \sqrt{\frac{3}{71}} \Pnorm $$
Applying the same bound in the opposite direction, we have that $8 \| \Ptheta \|_\op^2 \| \theta - \thetast \|_\op \le 8 \frac{75}{71} \Pnorm^2 \| \theta - \thetast \|_\op \le 8 \frac{75}{71} \frac{1}{150} = \frac{4}{71} < 1$, so
$$ \Pnorm \le (1-4/71)^{-1/2} \| \Ptheta \|_\op \le \sqrt{\frac{71}{67}} \| \Ptheta \|_\op $$
From which 3. follows.

If $\| B( K_0 - \Ktheta) \|_\op \le 1/(5 \| \Ptheta \|_\op)^{3/2}$, $A + B  K_0$ is stable, by Proposition 7 of \cite{simchowitz2020naive}. By Proposition 6 of \cite{simchowitz2020naive}, $\| B ( K_0 - \Ktheta) \|_\op \le \| B ( K_0 - \Kst) \|_\op + \| B (\Kst - \Ktheta) \|_\op \le  \| B \|_\op \|  K_0 - \Kst \|_\op + c \| \Ptheta \|_\op^{7/2} \| \theta - \thetast \|_\op \le 1/(5\| \Ptheta \|_\op^{3/2})$, where the last inequality holds by our bounds on $\|  K_0 - \Kst \|_\op$ and $\| \theta - \thetast \|_\op$, and plugging in the appropriate constants. Thus, $A + BK_0$ is stable. Then the following hold.
\begin{itemize}
\item $\| B \|_\op \le \Bnorm + \| \theta - \thetast \|_\op \le \frac{3}{2} \Bnorm$, where the last inequality holds since $\| \theta - \thetast \|_\op \le \Bnorm/2$.
\item Let $\zeta = 8 \| \Ptheta \|_\op^2 \| \theta - \thetast \|_\op$. Since $\zeta < 1/2$ and $\| \theta - \thetast \|_\op \le 1/(32 \| \Ptheta \|_\op^3)$, then 
	\begin{align*}
	\| (\Kst - \Ktheta)^\top (\Ru + B^\top \Ptheta B) \|_\op & \le \| \Ru^{1/2} \|_\op \| \Ru^{1/2} (\Kst - \Ktheta) \|_\op + \| \Ptheta^{1/2} B \|_\op \| \Ptheta^{1/2} B(\Kst - \Ktheta) \|_\op \\
	& \overset{(a)}{\le} 4 (7\| \Ru^{1/2} \|_\op + 9 \| \Ptheta^{1/2} B \|_\op) \| \Ptheta \|_\op^{7/2}  \| \theta - \thetast \|_\op  \\
	& \le \frac{c ( \Runorm^{1/2} + \| \Ptheta \|_\op^{1/2} \| B \|_\op)}{\| \Ptheta \|_\op^{3/2}} \\
	& \le \frac{c ( \Runorm^{1/2} + \Pnorm^{1/2} \Bnorm)}{\Pnorm^{3/2}}
	\end{align*}
	$(a)$ holds by Proposition 6 of \cite{simchowitz2020naive}.
\item Since $\|  K_0 - \Kst \|_\op \le 1/(c \| B \|_\op \| \Ptheta \|_\op^{3/2})$
\begin{align*}
	\| ( K_0 - \Kst)^\top (\Ru + B^\top \Ptheta B) \|_\op & = \|  K_0 - \Kst \|_\op \| \Ru + B^\top \Ptheta B \|_\op \\
	&  \le \frac{ \| \Ru + B^\top \Ptheta B \|_\op}{c \| B \|_\op \| \Ptheta \|_\op^{3/2}} \le \frac{c(\Runorm + \Bnorm^2 \Pnorm)}{\Bnorm \Pnorm^{3/2}}
	\end{align*}
\item By Lemma B.5, $\| \dlyap(A+B\Ktheta,I) \|_\op \le \| \Ptheta \|_\op \le c \Pnorm$, so long as $\Rx \succeq I$.
\item By Lemma B.12 of \cite{simchowitz2020naive}, if $\| B( K_0 - \Ktheta) \|_\op \le 1/(5\| \dlyap(A+B\Ktheta,I) \|_\op^{3/2})$, then $\| \dlyap(A+B K_0,I) \|_\op \le 2 \| \dlyap(A+B\Ktheta,I) \|_\op$. Note that $\| B( K_0 - \Ktheta) \|_\op \le 1/(5\| \dlyap(A+B\Ktheta,I) \|_\op^{3/2})$ holds since $\| B ( K_0 - \Ktheta) \|_\op \le 1/(5 \| \Ptheta \|_\op^{3/2}) \le 1/(5 \| \dlyap(A+B\Ktheta,I)\|_\op^{3/2})$.
\item Note that $\| A + B \Ktheta \|_\op^2 \le \| \dlyap(A+B \Ktheta,I) \|_\op$, so, as long as $\zeta \le 1/2$,
	\begin{align*}
	\| A + B  K_0 \|_\op & \le \| A + B \Ktheta \|_\op + \| B( K_0 - \Ktheta) \|_\op \le \sqrt{\| \Ptheta \|_\op} + t \| B \|_\op + 32 \| \Ptheta \|_\op^{7/2} \| \theta - \thetast \|_\op \\
	& \le c  \sqrt{\Pnorm} + \frac{c}{\Pnorm^{3/2}}  \le c \sqrt{\Pnorm}.
	\end{align*}
\end{itemize}
\end{proof}

\begin{lem}\label{lem:lqr_p3d_bound}
Let $\theta(t) = (\Ast + t \DelA, \Bst + t \DelB)$ and $P(t) := P_\infty(\theta(t)), K(t) := \Kopt(\theta(t)),  \Acl(t) := A(t) + B(t) K(t)$. If $\max \{ \| \Ast - \Ahat \|_\op, \| \Bst - \Bhat \|_\op \} \leq \epsilon$ and $\Rx,\Ru \succeq I$, then for $t$ such that $(A(t),B(t))$ is stabilizable:
$$ \| P'''(t) \|_\op \leq (1 + \| \Acl(t) \|_\op)(1 + \| B(t) \|_\op)  \poly( \| P(t) \|_\op) \epsilon^3. $$
\end{lem}
\begin{proof}
For simplicity, we drop the $t$ throughout the remainder of the proof. By Lemma C.2 of \cite{simchowitz2020naive}, $P'' = \dlyap(\Acl, Q_2)$ where $Q_2 = \Acl'^\top P' \Acl + \Acl^\top P' \Acl' + Q_1'$, $Q_1' = \Acl'^\top P \DelAcl + \Acl^\top P' \DelAcl + \Acl^\top P B' K' + (B'K')^\top P \Acl + \DelAcl^\top P' \Acl + \DelAcl^\top P \Acl'$, and $\DelAcl(t) = A'(t) + B'(t) \Kopt(A(t),B(t)) = \DelA + \DelB \Kopt(A(t),B(t))$. By the definition of $\dlyap$:
$$ P'' = \Acl^\top P'' \Acl + Q_2 $$
so:
$$ P''' = {\Acl'}^\top P'' \Acl + \Acl^\top P'' \Acl' + \Acl^\top P''' \Acl + Q_2' = \dlyap(\Acl, {\Acl'}^\top P'' \Acl + \Acl^\top P'' \Acl'  + Q_2')$$
By Lemma B.5 of \cite{simchowitz2020naive}, $\| P''' \|_\op \leq \| P \|_\op \| {\Acl'}^\top P'' \Acl + \Acl^\top P'' \Acl'  + Q_2' \|_\op \leq \| P \|_\op ( 2\| {\Acl'}^\top P'' \Acl \|_\op + \| Q_2' \|_\op)$. $\Acl' = \DelA + \DelB K + B K'$ so $\| \Acl' \|_\op \le (\sqrt{\| P \|_\op} + 7 \| B \|_\op \| P \|_\op^{7/2}) \epsilon$. By Lemma \ref{lem:lqr_param_bounds}, $\| {\Acl'}^\top P'' \Acl \|_\op \le \| \Acl \|_\op (1 + \| B \|_\op) \poly(\| P \|_\op) \epsilon^3$.

It remains to upper bound the operator norm of $Q_2'$. By definition of $Q_2$, we see, for small absolute constant $c$:
\begin{align*}
\| Q_2' \|_\op & \le c \Big ( \| \Acl^\top P' \Acl'' \|_\op + \| {\Acl'}^\top P' \Acl' \|_\op + \| \Acl P'' \Acl' \|_\op + \| {\Acl''}^\top P \DelAcl \|_\op + \| {\Acl'}^\top P' \DelAcl \|_\op \\
& \qquad \qquad + \| {\Acl'}^\top P \DelAcl' \|_\op + \| {\Acl'}^\top P'' \DelAcl \|_\op + \| {\Acl'}^\top P' \DelAcl' \|_\op + \| {\Acl'}^\top P B' K' \|_\op \\
& \qquad \qquad + \| \Acl^\top P' B' K' \|_\op + \| \Acl^\top P B' K'' \|_\op \Big ) 
\end{align*}
By Lemma \ref{lem:lqr_param_bounds}, $\| \Acl'' \|_\op = \| \DelB K' + B K'' \|_\op \le 7 \| P \|_\op^{7/2} \epsilon^2 + \| B \|_\op \poly(\|P\|_\op) \epsilon^2$. By definition, $\| \DelAcl \|_\op \le \epsilon, \| B' \|_\op \le \epsilon$. Finally, $\DelAcl' = \DelB K'$ so $\| \DelAcl' \|_\op \le 7 \| P \|_\op^{7/2} \epsilon^2$. We see then that, by Lemma \ref{lem:lqr_param_bounds}, every term in the above sum is order $\epsilon^3$ so, combining everything, we have:
$$ \| P''' \|_\op \le (1 + \| \Acl \|_\op)(1 + \| B \|_\op) \poly(\| P \|_\op) \epsilon^3 $$
\end{proof}

\begin{lem}\label{lem:lqr_k3d_bound}
Let $\theta(t) = (\Ast + t \DelA, \Bst + t \DelB)$ and $K(t) := \Kopt(\theta(t)), \Acl(t) := A(t) + B(t) K(t)$. If $\max \{ \| \Ast - \Ahat \|_\op, \| \Bst - \Bhat \|_\op \} \leq \epsilon$, $\Rx, \Ru \succeq I$, then for $t$ such that $(A(t),B(t))$ is stabilizable:
$$ \| K'''(t) \|_\op \le \poly(\| P(t) \|_\op, \| B(t) \|_\op, \| \Acl(t) \|_\op) \epsilon^3. $$
\end{lem}
\begin{proof}
As before, we drop the $t$ throughout the remainder of the proof. By Lemma B.3 of \cite{simchowitz2020naive}:
$$K'' = R_0^{-1} Q_3'(t) + R_0^{-1} (\Ru + B^\top P B)' K' $$
where $Q_3 = \Delta_B^\top P \Acl + B^\top P \Delta_{\Acl} + B^\top P' \Acl $ and $R_0 = \Ru + B^\top P B$. 
So, using the identity $(X^{-1})' = -X^{-1} X' X^{-1}$:
\begin{align*}
K''' & = R_0^{-1} Q_3'' + R_0^{-1} (\Ru + B^\top P B)' R_0^{-1} Q_3' + R_0^{-1} (\Ru + B^\top P B)' R_0^{-1} (\Ru + B^\top P B)' K' \\
& \quad \quad + R_0^{-1} (\Ru + B^\top P B)'' K' + R_0^{-1} (\Ru + B^\top P B)' K'' 
\end{align*}
By Lemma C.3 of \cite{simchowitz2020naive}, $\| R_0^{-1} \|_\op \le 1$. $(\Ru + B^\top P B)' = \DelB^\top P B + B^\top P' B + B^\top P' \DelB$ so by Lemma \ref{lem:lqr_param_bounds}, $ \| (\Ru + B^\top P B)' \|_\op \le \poly(\| B \|_\op, \| P \|_\op) \epsilon$. Similarly, we see that $\| (\Ru + B^\top P B)'' \|_\op \le \poly(\| B \|_\op, \| P \|_\op) \epsilon^2$. Using Lemma \ref{lem:lqr_param_bounds} to bound $\|K'\|_\op$ and $\| K'' \|_\op$, we have:
$$ \| K''' \|_\op \le \| Q_3'' \|_\op + \| Q_3' \|_\op  \poly(\| B \|_\op, \| P \|_\op) \epsilon + \poly(\| B \|_\op, \| P \|_\op) \epsilon^3$$
It remains to bound $\| Q_3'' \|_\op$ and $\| Q_3' \|_\op$. By definition:
$$ Q_3' = 2 \DelB^\top P' \Acl + \DelB^\top P \Acl' + \DelB^\top P \DelAcl + B^\top P' \DelAcl + B^\top P \DelAcl'  + B^\top P'' \Acl + B^\top P' \Acl' $$
Using previously computed norm bounds, we have $\| Q_3' \|_\op \le \poly(\| B \|_\op, \| P \|_\op) \epsilon^2$. To bound $Q_3''$, we can differentiate the above, obtaining that, for a small absolute constant $c$:
\begin{align*}
\| Q_3'' \|_\op & \le c \Big ( \| \DelB^\top P'' \Acl \|_\op + \| \DelB^\top P' \Acl' \|_\op + \| \DelB^\top P \Acl'' \|_\op +  \| \DelB^\top P' \DelAcl \|_\op + \| \DelB^\top P \DelAcl' \|_\op \\
& \qquad \qquad + \| B^\top P'' \DelAcl \|_\op + \| B^\top P' \DelAcl' \|_\op + \| B^\top P \DelAcl'' \|_\op + \| B^\top P''' \Acl \|_\op \\
& \qquad \qquad + \| B^\top P'' \Acl' \|_\op + \| B^\top P' \Acl'' \|_\op \Big )
\end{align*}
$\| \DelAcl'' \|_\op \le \| \DelB K'' \|_\op \le \poly(\| P \|_\op)\epsilon^3$. Then, using Lemma \ref{lem:lqr_p3d_bound} and previously computed norm bounds, we have: 
$$ \| Q_3'' \|_\op \le \poly(\| P \|_\op, \| B \|_\op, \| \Acl \|_\op) \epsilon^3$$
Combining everything gives the stated result. 
\end{proof}

%% file: body/applications.tex
%!TEX root = ../main.tex

\newcommand{\DelAa}{\Delta_{A}^1}
\newcommand{\DelAb}{\Delta_{A}^2}
\newcommand{\DelBa}{\Delta_{B}^1}
\newcommand{\DelBb}{\Delta_{B}^2}
\newcommand{\Delthetaa}{\Delta_{\theta}^1}
\newcommand{\Delthetab}{\Delta_{\theta}^2}

\newcommand{\Thetarho}{\Theta_\rho}
\newcommand{\Thetarhod}{\Theta_{\rho,\dimx}}
\newcommand{\approxrho}{\approx_{\rho}}

\section{Provable Gains for Task-Optimal Design}\label{sec:examples}

\subsection{Preliminaries for Comparison of Designs}

\paragraph{Asymptotic Notation.} We assume that $\rho$ is close to 1, and are primarily concerned with the scaling in $\dimx$ and $\frac{1}{1-\rho}$. As such, we let $\Thetarho(\cdot)$ hide numerical constants and terms lower order in $\frac{1}{1-\rho}$. That is, we write $x = \Thetarho(\frac{1}{(1-\rho)^n})$ if $x = \frac{c_n}{(1-\rho)^n} + \sum_{j=1}^{n-1} \frac{c_j}{(1-\rho)^j}$, for numerical constants $c_1,\ldots,c_n$. Similarly, we write $x = \Thetarho(\frac{-1}{(1-\rho)^n})$ if $x = \frac{-c_n}{(1-\rho)^n} + \sum_{j=1}^{n-1} \frac{c_j}{(1-\rho)^j}$ for $c_n > 0$. In addition, we let $\approxrho$ denote that two quantities have the same scaling in $1-\rho$, up to absolute constants. 

$\Thetarho(\cdot )$ does not suppress dimension dependence, but in some cases it will be convenient to hide dimension dependence that is lower order in $\frac{1}{1-\rho}$. In such cases, to make clear that we are doing this, we will use $\Thetarhod( \cdot )$. Formally,  $x = \Thetarhod( \frac{1}{(1-\rho)^n} + \frac{\dimx^m}{(1-\rho)^p} )$ if  $x = \frac{c_n}{(1-\rho)^n} + \sum_{j=1}^{n-1} \frac{c_j}{(1-\rho)^j} + \frac{\dimx^m}{(1-\rho)^p} + \sum_{j=1}^{p-1} \frac{b_j \dimx^{m_j}}{(1-\rho)^j}$. In general we will only apply this notation to the final sample complexities when we are concerned with identifying the leading terms. We will also use $\calO(\cdot)$ in its standard form, suppressing lower order dependence on $\dimx$ and $\frac{1}{1-\rho}$.

\paragraph{Convex Representation of Inputs.} Recall that \Cref{prop:ce_optimal} show the task complexity achieves by any policy scales as $\tr(\taskhes(\thetast) \matGam_T(\pi; \thetast)^{-1})$ (where, throughout this section, we let $\Gamma_T(\pi; \thetast)$ denote the expected covariates under policy $\pi$). In particular, if we are playing periodic policies, as \Cref{lem:lds_lb_cov_ss_exp} shows we can approximate $\Gamma_T(\pi; \thetast)$ with $\Gamss_T(\pi; \thetast)$ so the complexity scales instead as $\tr(\taskhes(\thetast) \matGamss_T(\pi; \thetast)^{-1})$. Similarly, if we are playing only noise, the complexity scales as $\tr(\taskhes(\thetast) \matGamnoise_T(\pi; \thetast)^{-1})/T$.

Throughout, we will assume that $\gamma^2 \gg \sigma^2$ so we ignore the contribution of process noise to $\Gamss_T(\pi; \thetast)$. If our policy plays a periodic input $\bmU = (U_\ell)_{\ell=1}^k \in \calU_{\gamma^2,k}$, recall that  
\begin{align*}
\Gamss_T(\pi; \thetast) = \sum_{\ell=1}^k (e^{\imag \frac{2\pi \ell}{k}} I - \Atilst)^{-1} \Btilst U_\ell \Btilst^\herm (e^{\imag \frac{2\pi \ell}{k}} I - \Atilst)^{-\herm}
\end{align*}
Some algebra shows that
\begin{align*}
(e^{\imag \frac{2\pi \ell}{k}} I - \Atilst)^{-1} \Btilst = \begin{bmatrix} (e^{\imag \frac{2\pi \ell}{k}} I - \Ast)^{-1} \Bst \\ I \end{bmatrix}
\end{align*}
As we have already shown, a general matrix signal $\bmU$ can be realized in the time domain via a certain decomposition (see \Cref{sec:construct_time_input}) so in the following we will consider optimizing over $U_\ell$ so that $\Gamss_T(\pi; \thetast)$ satisfies our objective. 

We will consider the performance of \algname, optimal operator-norm identification, optimal Frobenius norm identification, and optimal noise excitation. By construction, \algname plays periodic inputs. Similarly, note that the optimal Frobenius norm identification algorithm is itself an instance of \algname---with $\taskhes_{\mathrm{fro}}(\thetast) = I$---so we can assume that the optimal Frobenius norm identification algorithm also plays periodic inputs. As \cite{wagenmaker2020active} show, the optimal operator-norm identification also plays periodic inputs. In all cases, then, we will consider the complexity $\tr(\taskhes(\thetast) \matGamss_T(\pi; \thetast)^{-1})$. For optimal noise excitation, we consider instead the complexity $\tr(\taskhes(\thetast) \matGamnoise_T(\pi; \thetast)^{-1})/T$.

\paragraph{Simplifications.} 
Note that due to the structure of $\matGam_T$, we have
\begin{align*}
\tr(\taskhes(\thetast) \matGam_T(\pi;\thetast)^{-1}) = \sum_{i=1}^d \tr(\taskhes_i \Gamma_T(\pi;\thetast)^{-1}) 
\end{align*}
where we let $\taskhes_i := [\taskhes(\thetast)]_{(i-1)(d+p)+1:i(d+p),(i-1)(d+p)+1:i(d+p)}$, the $i$th $(d+p)\times (d+p)$ block diagonal element of $\taskhes(\thetast)$. We will denote elements of $\taskhes(\thetast)$ with $[\taskhes(\thetast)]_{A_{ij},A_{nm}}$, where 
\begin{align*}
[\taskhes(\thetast)]_{A_{ij},A_{nm}} = \vectorize( \theta_{ij})^\top \taskhes(\thetast) \vectorize(\theta_{nm})
\end{align*}
and $\theta_{ij} = (e_i e_j^\top, 0)$ (and similarly for $B$). See \Cref{sec:lds_vec} for a more in-depth discussion of the vectorization of linear dynamical systems. By our construction of $\taskhes(\thetast)$, the elements $[\taskhes(\thetast)]_{A_{nm},A_{nm}}$ and $[\taskhes(\thetast)]_{B_{nm},B_{nm}}$ will lie on the diagonal of $\taskhes(\thetast)$, while other elements will not.  Furthermore, elements of the form $[\taskhes(\thetast)]_{A_{nm},A_{mn}}$ (and replacing $A$ with $B$) with $m \neq n$ will not be contained in any $\taskhes_i$, as these will lie off the block-diagonal. It follows that the expressions given above include all entries of $\taskhes(\thetast)$ that will appear in our calculations.

\paragraph{Computation of Inputs.}
For each exploration strategy, our goal will be to compute the inputs $\pi$ optimal for a given exploration criteria, and then  compute the value of $\Phi_T(\pi;\thetast) = \tr \left ( \taskhes(\thetast) \matGam_T(\pi;\thetast)^{-1} \right )$ for this input. As \Cref{prop:ce_optimal} shows, the task complexity of any given exploration strategy $\pi$ scales as $\Phi_T(\pi;\thetast) = \tr \left ( \taskhes(\thetast) \matGam_T(\pi;\thetast)^{-1} \right )$. 

Note that the optimal Frobenius norm identification algorithm is simply \algname, but with $\taskhes(\thetast) = I$. Thus, for the optimal task-specific strategy and the Frobenius norm identification strategy, we can compute the optimal inputs by choosing those inputs which minimize $\Phi_T(\pi;\thetast) = \tr \left ( \taskhes(\thetast) \matGam_T(\pi;\thetast)^{-1} \right )$, for each $\taskhes(\thetast)$. By \Cref{prop:phiss_phiopt} we have that $ \Phiopt(\gamma^2; \thetast)$ and $\Phiss(\gamma^2;\thetast)$ are equivalent up to constants, where 
\begin{align*}
 \Phiss(\gamma^2;\thetast) := \liminf_{T \rightarrow \infty} \min_{u \in \calU_{\gamma^2,T}} \tr \Big ( \taskhes(\thetast) \matGamss_{T,T}(\thetastbar,u,0)^{-1} \Big ) 
\end{align*}
Note that this corresponds to the covariates obtained when playing an input that is only sinusoidal and has no noise component. Furthermore, as \algname is optimal and itself plays periodic inputs, it suffices to consider only periodic inputs. Putting this together, for large enough $T$, for the task-optimal and Frobenius norm strategies, we simply analyze $\tr \left ( \taskhes(\thetast) \matGamss_T(\pi;\thetast)^{-1} \right )$ and only consider periodic, non-noise inputs.

The case of operator norm identification is similar. As is shown in \cite{wagenmaker2020active}, the optimal inputs here are also periodic, so it follows that $ \tr \left ( \taskhes(\thetast) \matGam_T(\pi;\thetast)^{-1} \right ) \approxrho \tr \left ( \taskhes(\thetast) \matGamss_T(\pi;\thetast)^{-1} \right )$, for $\pi$ the optimal operator norm inputs. Furthermore, an argument similar to that used in the proof of \Cref{thm:lds_lb} can be used to show that the optimal inputs are sinusoidal with no noise component. It follows that we can, in this case, also simply analyze the steady-state covariates with no noise component.

Note that \algname itself does mix the sinusoidal input with a noise component to ensure sufficient excitation. While the above argument shows that this does not improve the complexity of Frobenius or operator norm identification evaluated in the Frobenius or operator norm, one might hope that the inclusion of noise would help these exploration strategies more easily transfer to the actual task of interest. We make two remarks on this. First, as we are concerned with the inputs optimal on Frobenius and operator norm identification, and as these inputs do not require this noise component, the result we obtain would still hold even if this was the case. Second, our analysis shows that if we play the noise that optimally excites the system for completing the task of interest, the complexity obtained is still suboptimal. Thus, even if we were to mix the operator or Frobenius norm identification inputs with the optimal noise, the resulting strategy would still be suboptimal, so our conclusion holds regardless of whether noise is played or not. 

Finally, to simplify the analysis further we ignore the contribution from the excitation due to the process noise when computing the covariates. This is reasonable for small values of $\sigma_w^2$, which will make this contribution lower order.

\subsection{Computation of Task Hessian}
\begin{lem}[Computation of Task Hessian]\label{lem:ex_task_hes} Consider the following instance of the \lqrx problem:
$$ \Ast = \rho_1 e_1 e_1^\top + \rho_2 (I - e_1 e_1^\top), \quad \Bst = b I, \quad \Rx = \kappa_1  e_1 e_1^\top + \kappa_2 (I - e_1 e_1^\top), \quad \Ru = \mu I $$
Then, if $n \neq m$,
\begin{align*}
 [ \taskhes(\thetast)]_{A_{nm},A_{nm}}&  = \frac{1}{(\mu+b^2 p_n)(1-a_m^2)} (\frac{b p_n}{1 - a_n a_m} )^2 + \frac{1}{(\mu+b^2 p_m)(1-a_n^2)} (\frac{b p_n a_n^2}{1 - a_n a_m} )^2
\end{align*}
and if $n = m$,
\begin{align*}
 [ \taskhes(\thetast)]_{A_{nn},A_{nn}}&  =  \frac{1}{(\mu+b^2 p_n)(1-a_n^2)} \Big (\frac{b p_n (1 + a_n^2) }{1 - a_n^2} \Big )^2  
\end{align*}
If $n \neq m$,
\begin{align*}
[ \taskhes(\thetast)]_{B_{nm},B_{nm}} & =  \frac{1}{(\mu+b^2 p_n)(1-a_m^2)} (\frac{b p_n k_m}{1 - a_n a_m} )^2 + \frac{1}{(\mu+b^2 p_m)(1-a_n^2)} \Big ( p_n a_n - \frac{b p_n a_n^2 k_m}{1 - a_n a_m} \Big )^2  
\end{align*}
and if $n = m$,
\begin{align*}
[ \taskhes(\thetast)]_{B_{nn},B_{nn}} & =  \frac{1}{(\mu+b^2 p_n)(1-a_n^2)} \Big (p_n a_n - \frac{b p_n  k_n (1 + a_n^2)}{1 - a_n^2} \Big )^2    
\end{align*}
Finally, if $n \neq m$,
 \begin{align*}
 [\taskhes(\thetast)]_{A_{nm}, B_{nm}} & = \frac{-k_m}{(\mu+b^2 p_n)(1-a_m^2)} (\frac{b p_n}{1 - a_n a_m} )^2 + \frac{1}{(\mu+b^2 p_m)(1-a_n^2)} \Big ( p_n a_n - \frac{b p_n a_n^2 k_m}{1 - a_n a_m} \Big ) \frac{b p_n a_n^2}{1 - a_n a_m}    
 \end{align*}
 and if $n = m$,
  \begin{align*}
 [\taskhes(\thetast)]_{A_{nn}, B_{nn}} & =   \frac{1}{(\mu+b^2 p_n)(1-a_n^2)} \Big (\frac{b p_n (1 + a_n^2) }{1 - a_n^2} \Big ) \Big (p_n a_n - \frac{b p_n  k_n (1 + a_n^2)}{1 - a_n^2} \Big )   .
 \end{align*}
Furthermore, all other elements of $\taskhes(\thetast)$ which appear in $\sum_{i=1}^d \tr(\taskhes_i \Gamma_T(\pi;\thetast)^{-1})$ are 0.
\end{lem}
\begin{proof}
Note that here $\Pst$ is diagonal with diagonal elements $p_i = [\Pst]_{ii}$ satisfying
$$ p_i = \rho_i^2 p_i - \frac{\rho_i^2 b^2 p_i^2}{\mu + b^2 p_i} + \kappa_i \Longrightarrow p_i = \frac{1}{2b^2} \left (  b^2 \kappa_i - \mu + \mu \rho_i^2 + \sqrt{4 b^2 \mu \kappa_i + (\mu - b^2 \kappa_i - \mu \rho_i^2)^2} \right )$$
where we set $\rho_i = \rho_2, \kappa_i = \kappa_2, i \ge 2$. It follows that $\Kst$ is diagonal and that $k_i := [\Kst]_{ii} = \frac{\rho_i b p_i}{\mu + b^2 p_i}$, so $\Aclst$ is also diagonal. Let $a_i := [\Aclst]_{ii}$. 

By definition, $\taskhes(\thetast) =  (\nabla_\theta \Kopt(\theta)|_{\theta = \thetast})^\top (\nabla_K^2 \calR(K;\thetast)|_{K = \Kopt(\thetast)})(\nabla_\theta \Kopt(\theta)|_{\theta = \thetast})$. Our goal is to calculate how $\taskhes(\thetast)$ scales with the problem parameters, and from this determine the error rate of different exploration approaches. To this end, let $A(t_1,t_2) = \Ast + t_1 \DelAa + t_2 \DelAb$, $B(t_1,t_2) = \Bst + t_1 \DelBa + t_2 \DelBb$, $\Delthetaa = (\DelAa,\DelBb), \Delthetab = (\DelAb,\DelBb)$, and $K(t_1,t_2) = \Kopt(A(t_1,t_2),B(t_1,t_2))$. Then by the chain rule, 
\begin{align*}
\frac{d}{dt_2} \frac{d}{dt_1} \calR(K(t_1,t_2);\thetast) |_{t_1 = t_2 = 0} & = (\nabla_{\theta} \Kopt(\theta)|_{\theta=\thetast}[\Delthetaa])^\top (\nabla_K^2 \calR(K;\thetast)|_{K = \Kst}) (\nabla_{\theta} \Kopt(\theta)|_{\theta=\thetast}[\Delthetab]) \\
& = \vectorize(\Delthetaa)^\top \taskhes(\thetast) \vectorize(\Delthetab)
\end{align*}
Thus, to determine the value of $\taskhes(\thetast)$, we can simply evaluate $\frac{d}{dt_2} \frac{d}{dt_1} \calR(K(t_1,t_2);\thetast) |_{t_1 = t_2 = 0}$ for different values of $\Delthetaa,\Delthetab$. Now recall that,
\begin{align*}
\calR(K;\thetast) & = \tr \left ( \dlyap(\Ast + \Bst K, (K - \Kst)^\top (\Ru + \Bst^\top \Pst \Bst) (K - \Kst) ) \right ) \\
& =  \tr \left ( \sum_{s=0}^\infty (\Ast + \Bst K)^s (K - \Kst)^\top (\Ru + \Bst^\top \Pst \Bst) (K - \Kst) ((\Ast + \Bst K)^\top)^s \right ) 
\end{align*}
Setting $K = K(t_1,t_2)$ and differentiating this with respect to $t_1,t_2$, we find that
\begin{align*}
\vectorize(\Deltheta^1)^\top \taskhes(\thetast) \vectorize(\Deltheta^2) & = \frac{d}{dt_2} \frac{d}{dt_1}  \calR(K(t_1,t_2);\thetast) |_{t_1 = t_2 = 0} \\
& = \tr \left ( \tsum_{s=0}^\infty \Aclst^s (K^{t_1}(0,0))^\top (\mu I + b^2 \Pst) (K^{t_2}(0,0)) \Aclst^s \right ) \\
& = \tsum_{s=0}^\infty \tr \left (  \Aclst^{2s} (K^{t_1}(0,0))^\top (\mu I + b^2 \Pst) (K^{t_2}(0,0)) \right ) \\
& =  \tsum_{s=0}^\infty \tsum_{i=1}^{\dimx} a_i^{2s} [K^{t_1}(0,0)^\top (\mu I + b^2 \Pst) K^{t_2}(0,0)]_{ii} \\
& = \sum_{i=1}^{\dimx} \frac{ [K^{t_1}(0,0)^\top (\mu I + b^2 \Pst) K^{t_2}(0,0)]_{ii}}{1 - a_i^2}
\end{align*}
Recall that
$$ K^{t_i}(0,0) = -(\mu I + b^2 \Pst)^{-1} \left ( (\DelB^i)^\top \Pst \Aclst + b \Pst \DelAcl^i + b P^{t_i}(0,0) \Aclst \right ) $$
where $\DelAcl^i = \DelA^i - \DelB^i \Kst$, and
$$ P^{t_i}(0,0) = \dlyap(\Aclst,Q_i) = \sum_{s=0}^\infty \Aclst^s Q_i \Aclst^s, \quad Q_i = \Aclst^\top \Pst \DelAcl^i + (\DelAcl^i)^\top \Pst \Aclst $$
Thus, 
\begin{align*}
[K^{t_1}(0,0)^\top (\mu I + b^2 \Pst) K^{t_2}(0,0)]_{ii} & = \Big [ \left ( (\DelB^1)^\top \Pst \Aclst + b \Pst \DelAcl^1 + b P^{t_1}(0,0) \Aclst \right )^\top \cdot (\mu I+b^2 \Pst)^{-1} \\
& \qquad \qquad \cdot \left ( (\DelB^2)^\top \Pst \Aclst + b \Pst \DelAcl^2 + b P^{t_2}(0,0) \Aclst \right ) \Big ]_{ii} \\
& =  [ (\DelB^1)^\top \Pst \Aclst + b \Pst \DelAcl^1 + b P^{t_1}(0,0) \Aclst]_{:,i}^\top \cdot (\mu I+b^2 \Pst)^{-1} \\
& \qquad \qquad \cdot [(\DelB^2)^\top \Pst \Aclst + b \Pst \DelAcl^2 + b P^{t_2}(0,0) \Aclst]_{:,i} \\
& = \left ( p_i a_i [\DelB^1]_{i,:} + b \Pst [\DelAcl^1]_{:,i} + b a_i [P^{t_1}(0,0)]_{:,i} \right )^\top \cdot (\mu I+ b^2 \Pst)^{-1} \\
& \qquad \qquad \cdot \left ( p_i a_i [\DelB^2]_{i,:} + b \Pst [\DelAcl^2]_{:,i} + b a_i [P^{t_2}(0,0)]_{:,i} \right )
\end{align*}
and
\begin{align*}
[P^{t_1}(0,0)]_{:,i} & = \sum_{s=0}^\infty [\Aclst^s Q_1 \Aclst^s]_{:,i}  = \sum_{s=0}^\infty a_i^s \Aclst^s [Q_1]_{:,i} = (I - a_i \Aclst)^{-1} [Q_1]_{:,i} \\
& =  (I - a_i \Aclst)^{-1} (\Aclst \Pst [\DelAcl^1]_{:,i} + p_i a_i [\DelAcl^1]_{i,:})
\end{align*}
Putting this together, we have that
\begin{align}
& \vectorize(\Deltheta^1)^\top \taskhes(\thetast) \vectorize(\Deltheta^2) \nonumber \\
& \qquad  = \sum_{i=1}^{\dimx} \frac{1}{(1-a_i^2)} \Big  ( p_i a_i [\DelB^1]_{i,:} + b (I + a_i (I - a_i \Aclst)^{-1} \Aclst ) \Pst [\DelAcl^1]_{:,i} + b p_i a_i^2 (I - a_i \Aclst)^{-1}[\DelAcl^1]_{i,:} \Big )^\top \nonumber \\
& \qquad \qquad \cdot (\mu I+b^2 \Pst)^{-1}  \Big  ( p_i a_i [\DelB^2]_{i,:} + b (I + a_i (I - a_i \Aclst)^{-1} \Aclst ) \Pst [\DelAcl^2]_{:,i} + b p_i a_i^2 (I - a_i \Aclst)^{-1}[\DelAcl^2]_{i,:} \Big ) \nonumber \\
& \qquad = \sum_{i=1}^{\dimx} \sum_{j=1}^{\dimx} \frac{1}{(\mu+b^2 p_j)(1-a_i^2)} \Big ( p_i a_i [\DelB^1]_{i,j} + \frac{b p_j}{1 - a_i a_j} ([\DelA^1]_{j,i} - k_i [\DelB^1]_{j,i}) + \frac{b p_i a_i^2}{1 - a_i a_j} ([\DelA^1]_{i,j} - k_j [\DelB^1]_{i,j}) \Big ) \nonumber \\
& \qquad \qquad \cdot  \Big ( p_i a_i [\DelB^2]_{i,j} + \frac{b p_j}{1 - a_i a_j} ([\DelA^2]_{j,i} - k_i [\DelB^2]_{j,i}) + \frac{b p_i a_i^2}{1 - a_i a_j} ([\DelA^2]_{i,j} - k_j [\DelB^2]_{i,j}) \Big ) \label{eq:taskhes_entry}
\end{align}
We now evaluate the above when $\DelB^1 = \DelB^2 = 0$ and $\DelA^1 = e_\ell e_o^\top, \DelA^2 = e_n e_m^\top$. For this to be non-zero, we must have that either $\ell = n, o = m$ or $\ell = m, o = n$ and, as noted previously, we can ignore the case when $\ell = m, o = n$. Therefore, if $n \neq m$, 
\begin{align*}
 [ \taskhes(\thetast)]_{A_{nm},A_{nm}}&  = \frac{1}{(\mu+b^2 p_n)(1-a_m^2)} (\frac{b p_n}{1 - a_n a_m} )^2 + \frac{1}{(\mu+b^2 p_m)(1-a_n^2)} (\frac{b p_n a_n^2}{1 - a_n a_m} )^2
\end{align*}
and if $n = m$,
\begin{align*}
 [ \taskhes(\thetast)]_{A_{nn},A_{nn}}&  =  \frac{1}{(\mu+b^2 p_n)(1-a_n^2)} \Big (\frac{b p_n (1 + a_n^2) }{1 - a_n^2} \Big )^2  
\end{align*}

Now consider $\DelA^1 = \DelA^2 = 0$ and $\DelB^1 = e_\ell e_o^\top, \DelB^2 = e_n e_m^\top$. As before, for $[ \taskhes(\thetast)]_{B_{\ell o},B_{n m}} $ to be non-zero, we need either $\ell = n, o = m$ or $\ell = m, o = n$. Therefore, if $n \neq m$,
\begin{align*}
[ \taskhes(\thetast)]_{B_{nm},B_{nm}} & =  \frac{1}{(\mu+b^2 p_n)(1-a_m^2)} (\frac{b p_n k_m}{1 - a_n a_m} )^2 + \frac{1}{(\mu+b^2 p_m)(1-a_n^2)} \Big ( p_n a_n - \frac{b p_n a_n^2 k_m}{1 - a_n a_m} \Big )^2  
\end{align*}
and if $n = m$,
\begin{align*}
[ \taskhes(\thetast)]_{B_{nn},B_{nn}} & =  \frac{1}{(\mu+b^2 p_n)(1-a_n^2)} \Big (p_n a_n - \frac{b p_n  k_n (1 + a_n^2)}{1 - a_n^2} \Big )^2    
\end{align*}

 Finally, we consider the case where $\DelA^1 = e_\ell e_o^\top, \DelA^2 = 0$ and $\DelB^1 = 0, \DelB^2 = e_n e_m^\top$. Again, we must have that either $\ell = n, o = m$ or $\ell = m, o = n$ for $[\taskhes(\thetast)]_{A_{\ell o}, B_{nm}}$ to be non-zero. Therefore, if $n \neq m$,
 \begin{align*}
 [\taskhes(\thetast)]_{A_{nm}, B_{nm}} & = \frac{-k_m}{(\mu+b^2 p_n)(1-a_m^2)} (\frac{b p_n}{1 - a_n a_m} )^2 + \frac{1}{(\mu+b^2 p_m)(1-a_n^2)} \Big ( p_n a_n - \frac{b p_n a_n^2 k_m}{1 - a_n a_m} \Big ) \frac{b p_n a_n^2}{1 - a_n a_m}    
 \end{align*}
 and if $n = m$,
  \begin{align*}
 [\taskhes(\thetast)]_{A_{nn}, B_{nn}} & =   \frac{1}{(\mu+b^2 p_n)(1-a_n^2)} \Big (\frac{b p_n (1 + a_n^2) }{1 - a_n^2} \Big ) \Big (p_n a_n - \frac{b p_n  k_n (1 + a_n^2)}{1 - a_n^2} \Big )   .
 \end{align*}
\end{proof}

\subsection{Proof of \Cref{lem:lqrex1}}
Here we choose $\rho_1 = \rho, \rho_2 = 0, b = \sqrt{1-\rho}, \kappa_1 = \kappa_2 = \frac{1}{\sqrt{1-\rho}},$ and $\mu = \frac{1}{(1-\rho)^2}$. With these constants, some algebra shows that
\begin{align*}
& a_1 = \calO(\rho), \quad \tfrac{1}{1-a_1} = \Thetarho(\tfrac{1}{1-\rho}), \quad a_i = 0, i \ge 2 \\
& k_1 = \calO(1-\rho), \quad k_i = 0, i \ge 2 \\
& p_1 = \Thetarho(\tfrac{1}{(1-\rho)^{3/2}}), \quad p_i = \Thetarho(\tfrac{1}{\sqrt{1-\rho}}), i \ge 2
\end{align*}
Plugging these values into the expression given for $\taskhes(\thetast)$ in \Cref{lem:ex_task_hes}, we have
\begin{align*}
& [\taskhes(\thetast)]_{A_{11},A_{11}} = \Thetarho \left ( \frac{1}{(1-\rho)^3} \right ), \quad [\taskhes(\thetast)]_{A_{1m},A_{1m}} = \Thetarho \left ( \frac{1}{1-\rho} \right ), m \ge 2 \\
& [\taskhes(\thetast)]_{B_{11},B_{11}} = \Thetarho \left ( \frac{1}{(1-\rho)^2} \right ), \quad [\taskhes(\thetast)]_{B_{1m},B_{1m}} = \Thetarho \left ( \frac{1}{(1-\rho)^2} \right ), m \ge 2 \\
& [\taskhes(\thetast)]_{A_{11},B_{11}} = \Thetarho \left ( \frac{1}{(1-\rho)^{5/2}} \right ), \quad [\taskhes(\thetast)]_{A_{1m},B_{1m}} = \Thetarho \left ( \frac{1}{(1-\rho)^{3/2}} \right ), m \ge 2
\end{align*}
All other terms are 0 or do not scale with $\frac{1}{1-\rho}$ and can therefore be ignored. It follows that the sample complexity will scale as
\begin{align*}
\tr(\taskhes(\thetast) \matGam_T(\pi;\thetast)^{-1}) \approxrho \tr(\taskhes_1 \Gamma_T(\pi;\thetast)^{-1})
\end{align*}
where here
\begin{align*}
\taskhes_1 \approxrho \frac{1}{(1-\rho)^3} e_1 e_1^\top + & \frac{1}{1-\rho} \sum_{j=2}^{\dimx} e_j e_j^\top + \frac{1}{(1-\rho)^2} \sum_{j=1}^{\dimx} e_{\dimx +j} e_{\dimx +j}^\top + \frac{1}{(1-\rho)^{5/2}} (e_1 e_{\dimx+1}^\top + e_{\dimx+1} e_1^\top) \\
& + \frac{1}{(1-\rho)^{3/2}} \sum_{j=2}^{\dimx} (e_j e_{\dimx + j}^\top + e_{\dimx +j} e_j^\top)
\end{align*}

\paragraph{Sample Complexity of \algname.} Our results show that the sample complexity of \algname scale with steady state covariates, and we can therefore analyze $\tr(\taskhes_1 \Gamss_T(\pi;\thetast)^{-1})$. As \algname plays the optimal inputs and we are concerned with obtaining an upper bound on its performance, we will simply construct a feasible input, which will then upper bound the actual performance.

 In particular, we will set $U_\ell = 0$ for all but two $\ell$ (and their conjugate partners), and for those $\ell$ will set $U_\ell = U' := \begin{bmatrix} u_1 & 0 \\ 0 & u_2 I_{\dimx-1} \end{bmatrix}$. In that case, we will have 
\begin{align*}
\Gamss_T &(\pi;\thetast)  \propto 2 \text{re} \bigg ( \begin{bmatrix} (e^{\imag \omega_1} I - \Ast)^{-1} \Bst \\ I \end{bmatrix} U' \begin{bmatrix} (e^{\imag \omega_1} I - \Ast)^{-1} \Bst \\ I \end{bmatrix}^\herm + \begin{bmatrix} (e^{\imag \omega_2} I - \Ast)^{-1} \Bst \\ I \end{bmatrix} U' \begin{bmatrix} (e^{\imag \omega_2} I - \Ast)^{-1} \Bst \\ I \end{bmatrix}^\herm \bigg )  \\
& = 2  \begin{bmatrix} \tfrac{(1-\rho) u_1}{|e^{\imag \omega_1} - \rho|^{2}} & 0  & \text{re} ( \tfrac{1}{e^{\imag \omega_1} - \rho}) \sqrt{1-\rho} u_1 & 0 \\ 
0 & 2 (1-\rho) u_2 I & 0 &  \text{re} (\tfrac{1}{e^{\imag \omega_1}}) \sqrt{1-\rho} u_2 I \\
\text{re} (\tfrac{1}{e^{\imag \omega_1} - \rho}) \sqrt{1-\rho} u_1 & 0 & 2 u_1 &0 \\
0 & \text{re} (\tfrac{1}{e^{\imag \omega_1}} ) \sqrt{1-\rho} u_2 I & 0 & 2 u_2 I
\end{bmatrix} \\
& + 2  \begin{bmatrix} \tfrac{(1-\rho) u_1}{|e^{\imag \omega_2} - \rho|^{2} } & 0  & \text{re} ( \tfrac{1}{e^{\imag \omega_2} - \rho}) \sqrt{1-\rho} u_1 & 0 \\ 
0 & 0 & 0 &  \text{re} (\tfrac{1}{e^{\imag \omega_2} }) \sqrt{1-\rho} u_2 I \\
\text{re} (\tfrac{1}{e^{\imag \omega_2} - \rho}) \sqrt{1-\rho} u_1 & 0 & 0& 0\\
0 & \text{re} (\tfrac{1}{e^{\imag \omega_2} }) \sqrt{1-\rho} u_2 I & 0 & 0
\end{bmatrix} 
\end{align*}
where we simplify using the values of $\Ast,\Bst$, $\omega_1$ and $\omega_2$ are the input frequencies we choose, and we use $\mathrm{re}(x)$ to denote the real part of $x$, which comes from the conjugate symmetry. We write ``$\propto$'' instead of ``$=$'' as additional normalization by factors of $T$ and $k$ are necessary to yield equality, and for simplicity we currently ignore. We will handle these factors later. For large enough $T$, $\omega_1$ and $\omega_2$ can be chosen essentially as desired, so we set $\omega_1 = 1 - \rho$ and $\omega_2 = \pi + 1 - \rho$. As we take $\rho$ close to 1, we have
\begin{align*}
e^{\imag \omega_1}  = \cos(\omega_1)  + \imag \sin(\omega_1) = 1  + \imag (1-\rho) + o(1-\rho), \quad e^{\imag \omega_2} = -1 - \imag (1-\rho) + o(1-\rho)
\end{align*} 
It follows that
\begin{align*}
& \text{re}((e^{\imag \omega_1})^{-1}) = - \text{re}((e^{\imag \omega_2})^{-1})  \\
& \text{re} ((e^{\imag \omega_1} - \rho)^{-1}) =\frac{1-\rho}{2 ( 1 - \rho)^2} + o((1-\rho)^2), \quad |e^{\imag \omega_1} - \rho|^{-2} = \frac{1}{2(1-\rho)^2} + o((1-\rho)^2) \\
& \text{re} ((e^{\imag \omega_2} - \rho)^{-1}) = \frac{-1-\rho}{( 1 + \rho)^2} + o(1-\rho), \quad |e^{\imag \omega_2} - \rho|^{-2} =  \calO(1)
\end{align*}
Plugging these in, we get that the above is equal to:
\begin{align*}
\begin{bmatrix}\tfrac{u_1}{1-\rho}  +  \calO((1-\rho)u_1)& 0  & \tfrac{u_1}{\sqrt{1-\rho}} - \calO((1-\rho)^{3/2} u_1)& 0 \\ 
0 &  4(1-\rho) u_2 I & 0 &  0 \\
\tfrac{u_1}{\sqrt{1-\rho}}  - \calO((1-\rho)^{3/2} u_1)& 0 &  4u_1 &0 \\
0 &0 & 0 &  4u_2 I
\end{bmatrix}
\end{align*}
This has the form given in \Cref{lem:inverse_diagish} so, applying this result and approximating $\tfrac{u_1}{1-\rho}  +  \calO((1-\rho)u_1)$ as $\tfrac{u_1}{1-\rho}$ and $\tfrac{u_1}{\sqrt{1-\rho}} - \calO((1-\rho)^{3/2} u_1)$ as $\tfrac{u_1}{\sqrt{1-\rho}}$, (we note that this approximation will not affect the leading terms in the inverse due to the form of the inverse given in \Cref{lem:inverse_diagish}) we have that the inverse of this matrix will be
\begin{align}\label{eq:lqrex1_tople_gaminv}
\begin{bmatrix} \tfrac{4(1-\rho)}{3u_1} & 0  & \tfrac{-\sqrt{1-\rho}}{3u_1} & 0 \\ 
0 &  \tfrac{1}{4(1-\rho) u_2} I & 0 &  0 \\
\tfrac{-\sqrt{1-\rho}}{3u_1}  & 0 &  \frac{1}{3 u_1} &0 \\
0 &0 & 0 &  \tfrac{1}{4 u_2} I
\end{bmatrix}
\end{align}
Plugging this into our expression for $\tr(\taskhes(\thetast) \matGam_T(\pi;\thetast)^{-1})$ gives
\begin{align*}
T \tr(\taskhes(\thetast) \matGam_T(\pi;\thetast)^{-1}) \approxrho \frac{1}{(1-\rho)^2 u_1'} + \frac{\dimx}{(1-\rho)^2 u_2'} + \frac{1}{(1-\rho)^2 u_1'} + \frac{\dimx}{(1-\rho)^2 u_2'} - \frac{1}{(1-\rho)^2 u_1'}
\end{align*}
where $u_i'$ denotes $u_i/k^2$. As we are interested in obtaining an upper bound on the sample complexity of \algname, we upper bound this by
\begin{align*}
\calO \left ( \frac{1}{(1-\rho)^2 u_1'} + \frac{\dimx}{(1-\rho)^2 u_2'} \right )
\end{align*}
It remains to choose $u_1',u_2'$ that satisfy $u_1' + (\dimx-1) u_2' \le \gamma^2/2$. Choosing the values that minimize the above gives the final complexity:
\begin{align*}
\calO \left ( \frac{\dimx^2}{(1-\rho)^2 \gamma^2} \right ).
\end{align*}

\paragraph{Sample Complexity of Operator Norm Identification.}
As was shown in \cite{wagenmaker2020active}, the optimal operator norm identification algorithm will choose the inputs that maximize $\lammin(\Gamss_T(\pi;\thetast))$. We will first construct an input with diagonal $\bmU_\ell$ and, as in the previous section, with $u_2 = \ldots = u_{\dimx}$, and will then show that this input is in fact optimal. 

Intuitively, the optimal operator norm identification algorithm seeks to input energy at frequencies which best excite the system (maximize $\lammin(\Gamss_T(\pi;\thetast))$), and that balance the gain in each direction. Note that, regardless of the input frequencies, the computation in the previous section shows that, other than the first element, the diagonal components of $\Gamss_T(\pi;\thetast)$ will scale as $(1-\rho)u_2,u_1,$ and $u_2$, respectively. Ignoring for a minute the contribution of the off-diagonal terms, the value of the first coordinate will be maximized if energy is input at the frequency $\omega_1$ which maximizes $| e^{\imag \omega_1} I - \rho|^{-2} (1-\rho) u_1$. It is easy to see that this is maximized at $\omega_1 = 0$, which gives the value to the first coordinate of $\frac{u_1}{1-\rho}$, and, ignoring off diagonal entries, a minimum eigenvalue of
\begin{align*}
\lammin(\Gamss_T(\pi;\thetast)) = 2 T \min \{ \tfrac{u_1}{1-\rho}, (1-\rho) u_2, u_1, u_2 \} = 2 T \min \{ (1-\rho) u_2, u_1 \}
\end{align*}
To maximize this while respecting the power constraint, operator norm-identification will choose $u_2 \approxrho \gamma^2/\dimx$ and $u_1 \approxrho (1-\rho) \gamma^2/\dimx$, which will yield the minimum eigenvalue of
\begin{align*}
\lammin(\Gamss_T(\pi;\thetast)) \approxrho  T (1-\rho) \gamma^2/\dimx
\end{align*}
Now taking into account the off-diagonal terms and applying Lemma \ref{lem:diagish_mineig}, using the expression for $\Gamss_T(\pi;\thetast)$ given in the analysis of \algname, we see that the contribution of the off-diagonal entries causes the minimum eigenvalue to instead be 0. However, if we instead set the inputs to those chosen in the previous section, as we saw there we have
\begin{align}\label{eq:lqrex1_op_cov}
\Gamss_T(\pi;\thetast) = 2 T \begin{bmatrix}\tfrac{u_1}{1-\rho}  +  \calO((1-\rho)u_1)& 0  & \tfrac{u_1}{\sqrt{1-\rho}} - \calO((1-\rho)^{3/2} u_1)& 0 \\ 
0 &  4(1-\rho) u_2 I & 0 &  0 \\
\tfrac{u_1}{\sqrt{1-\rho}}  - \calO((1-\rho)^{3/2} u_1)& 0 &  4u_1 &0 \\
0 &0 & 0 &  4u_2 I
\end{bmatrix}
\end{align}
Applying \Cref{lem:diagish_mineig} to this, after some algebra we see that
\begin{align*}
\lammin(\Gamss_T(\pi;\thetast)) \approxrho T \min \{ \tfrac{u_1}{1-\rho}, u_1, u_2, (1-\rho) u_2 \} = T \min \{ u_1, (1-\rho) u_2 \}
\end{align*}
Choosing $u_2 \approxrho \gamma^2/\dimx$ and $u_1 \approxrho (1-\rho) \gamma^2/\dimx$ to balance this as before, we have that this input yields
\begin{align*}
\lammin(\Gamss_T(\pi;\thetast)) \approxrho T (1-\rho) \gamma^2/\dimx
\end{align*}
Observe that this achieves the same minimum eigenvalue as that achieved ignoring off-diagonal terms (up to constants) and that, furthermore, the form of the minimum eigenvalue given in \Cref{lem:diagish_mineig} implies that the off-diagonal terms will only decrease the minimum eigenvalue. It follows that $ T (1-\rho) \gamma^2/\dimx$ is an upper bound on the minimum achievable eigenvalue when the inputs are diagonal so, since this input achieves this value, this is the near-optimal diagonal input for operator norm identification. It follows that the optimal covariance with diagonal inputs will take the form given in \eqref{eq:lqrex1_op_cov}. 

We now show that the globally optimal inputs are diagonal. We have just shown that the optimal covariance, when playing a diagonal input, will take the form
\begin{align}\label{eq:ex1op_cov_diag_opt}
\Gamss_T(\pi^\star;\thetast) \approxrho 2 T \begin{bmatrix}\tfrac{\gamma^2}{\dimx}  & 0  & \tfrac{\sqrt{1-\rho} \gamma^2}{\dimx} & 0 \\ 
0 &  \tfrac{4(1-\rho) \gamma^2}{\dimx} I & 0 &  0 \\
\tfrac{\sqrt{1-\rho} \gamma^2}{\dimx} & 0 &  \tfrac{4 (1-\rho) \gamma^2}{\dimx} &0 \\
0 &0 & 0 &  \tfrac{4 \gamma^2}{\dimx} I
\end{bmatrix}
\end{align}
Now consider some $\Delta \in \C^{\dimu \times \dimu}$, and consider perturbing our optimal diagonal input at some frequency $\omega$ by $\Delta$ to form the new input $\bmU_\omega + \Delta$. For our new input to be in our feasible set, we must have that $\tr(\Delta) \le 0$ and that $\Delta$ is symmetric. We want to show that, for every such perturbation, $\lammin(\Gamss_T(\pi';\thetast) ) \le \lammin(\Gamss_T(\pi^\star;\thetast) )$ (where $\pi'$ denotes the perturbed input). By first-order optimality conditions, this will imply that the diagonal input is optimal. 

The resulting perturbation to the input will perturb $\Gamss_T(\pi^\star;\thetast)$ as
\begin{align*}
\Gamss_T(\pi';\thetast) = \Gamss_T(\pi^\star;\thetast) + 2T \mathrm{real} \left ( \begin{bmatrix} G \Delta G^\herm & G \Delta \\ \Delta G^\herm & \Delta \end{bmatrix} \right ), \quad G := \begin{bmatrix} \sqrt{1-\rho} (e^{j\omega}  - \rho)^{-1} & 0 \\ 0 & \sqrt{1-\rho} e^{-j\omega} I \end{bmatrix}
\end{align*}
where here $G$ is equal to $(e^{j\omega} I - \Ast)^{-1} \Bst$. Note that the eigenvectors corresponding to the minimum eigenvalues of $\Gamss_T(\pi^\star;\thetast)$ are $v_2,\ldots, v_{\dimx} = e_2,\ldots,e_{\dimx}$ and, some algebra shows,
\begin{align*}
v_1 & = \left [ \frac{1 - 4(1-\rho) - \sqrt{1-4(1-\rho)+16(1-\rho)^2}}{2\sqrt{1-\rho}}, 0, \ldots, 0, 1, 0, \ldots, 0 \right ] \\
&=  \left [ -\sqrt{1-\rho} + o(1-\rho), 0, \ldots, 0, 1, 0, \ldots, 0 \right ]
\end{align*}
where the $1$ as at the index $\dimx + 1$. Some algebra shows that
\begin{align*}
v_i^\top \mathrm{real} \left ( \begin{bmatrix} G \Delta G^\herm & G \Delta \\ \Delta G^\herm & \Delta \end{bmatrix} \right ) v_i = (1-\rho) \Delta_{ii}, \quad i \ge 2
\end{align*}
where, since $v_i = e_i$ the inner products select only the diagonal elements, and
\begin{align*}
v_1^\top \mathrm{real} \left ( \begin{bmatrix} G \Delta G^\herm & G \Delta \\ \Delta G^\herm & \Delta \end{bmatrix} \right ) v_1 & \approxrho \Big ( (1-\rho)^2 | e^{j\omega} - \rho|^{-2} -2(1-\rho) \mathrm{real}((e^{j\omega} - \rho)^{-1}) + 1 \Big ) \Delta_{11} \\
& = \left ( \frac{(1-\rho)^2}{(\cos \omega - \rho)^2 + \sin^2 \omega} + 1- \frac{2(1-\rho)(\cos \omega - \rho)}{(\cos \omega - \rho)^2 + \sin^2 \omega}  \right ) \Delta_{11} \\
& \ge \left ( 1- \frac{(1-\rho)(\cos \omega - \rho)}{(\cos \omega - \rho)^2 + \sin^2 \omega}  \right ) \Delta_{11} \\
& = \left ( 1- \frac{(1-\rho)(\cos \omega - \rho)}{1 + \rho^2 - 2 \rho \cos \omega}  \right ) \Delta_{11} 
\end{align*}
and note that $\left ( 1- \frac{(1-\rho)(\cos \omega - \rho)}{1 + \rho^2 - 2 \rho \cos \omega}  \right ) \ge 0$ for all $\omega$. We must have that $\sum_{i=1}^{\dimx} \Delta_{ii} \le 0$ to meet our constraint. This implies that either all $\Delta_{ii} = 0$, or there exists $i'$ such that $\Delta_{i'i'} < 0$. By the above expressions, it follows that, if $i' \ge 2$, the latter case will cause the minimum eigenvalue to decrease, and if $i' = 1$, the minimum eigenvalue cannot increase. It follows that our perturbation $\Delta$ cannot increase $\lammin(\Gamss_T(\pi^\star;\thetast) )$, which implies that the optimal input is in fact diagonal.

Returning to the optimal covariates obtained with diagonal inputs, using the inverse expression for $\Gamss_T(\pi^\star;\thetast)$ given in the analysis of \algname \eqref{eq:lqrex1_tople_gaminv}, we then have that
\begin{align*}
T \tr(\taskhes(\thetast) \matGam_T(\pi^\star;\thetast)^{-1}) = [\taskhes(\thetast)]_{A_{11},A_{11}} \frac{4(1-\rho)}{3 u_1} & + 2 [\taskhes(\thetast)]_{A_{11},B_{11}} \frac{-\sqrt{1-\rho}}{3 u_1} + [\taskhes(\thetast)]_{B_{11},B_{11}} \frac{1}{3 u_1} \\
& + \Thetarhod \left ( \frac{\dimx}{(1-\rho)^2 u_2} \right )
\end{align*}
From the expression for $\taskhes(\thetast)$ given in \eqref{eq:taskhes_entry}, we see that we can express
\begin{align*}
[\taskhes(\thetast)]_{A_{11},A_{11}} = c_1^2, \quad  [\taskhes(\thetast)]_{B_{11},B_{11}} = c_2^2, \quad  [\taskhes(\thetast)]_{A_{11},B_{11}} =  c_1 c_2
\end{align*}
for some values $c_1,c_2$. It follows that,
\begin{align*}
 & [\taskhes(\thetast)]_{A_{11},A_{11}} \frac{4(1-\rho)}{3u_1} + 2 [\taskhes(\thetast)]_{A_{11},B_{11}} \frac{-\sqrt{1-\rho}}{3u_1} + [\taskhes(\thetast)]_{B_{11},B_{11}} \frac{3}{u_1} \\
 & \qquad = \frac{1}{3u_1} \left ( \sqrt{1-\rho} c_1 - c_2 \right )^2 + \frac{1}{u_1} (1-\rho) c_1^2 
\end{align*}
Now plugging in values of $c_1,c_2$, we have
\begin{align*}
\frac{1}{3u_1} \left ( \sqrt{1-\rho} c_1 - c_2 \right )^2 + \frac{1}{u_1} (1-\rho) c_1^2 = \Thetarho \left ( \frac{1}{(1-\rho)^2 u_1} \right )
\end{align*}
Which gives
\begin{align}\label{eq:ex1_op_lb}
T \tr(\taskhes(\thetast) \matGam_T(\pi^\star;\thetast)^{-1}) = \Thetarhod \left ( \frac{1}{(1-\rho)^2 u_1} +  \frac{\dimx}{(1-\rho)^2 u_2} \right )
\end{align}
Plugging in our values for $u_1,u_2$ gives the complexity:
\begin{align*}
\Thetarhod \left ( \frac{\dimx}{(1-\rho)^3 \gamma^2} + \frac{\dimx^2}{(1-\rho)^2 \gamma^2}  \right ).
\end{align*}
Note that our analysis is somewhat sensitive to the constants present in the entries of $\tr(\taskhes(\thetast) \matGam_T(\pi;\thetast)^{-1})$ that correspond to $u_1$. It is difficult to determine the precise constants that will appear in actual operator norm identification allocation. However, we note that any increase to the value of the constant in the off-diagonal term, $\frac{u_1}{\sqrt{1-\rho}}$, will cause the minimum eigenvalue to decrease, by \Cref{lem:diagish_mineig}, and we can therefore expect the constants to be no larger than their stated values. If we use constants smaller than what is stated here, this will only cause the magnitude of the off-diagonal terms in the inverse, $\frac{-\sqrt{1-\rho}}{3u_1}$, to decrease, which will further reduce the contribution of the term $2 [\taskhes(\thetast)]_{A_{11},B_{11}} \frac{-\sqrt{1-\rho}}{3u_1}$, causing the final complexity to be larger. However, this will not change the fact that 
\begin{align*}
 [\taskhes(\thetast)]_{A_{11},A_{11}} \frac{4(1-\rho)}{3u_1} + 2 [\taskhes(\thetast)]_{A_{11},B_{11}} \frac{-\sqrt{1-\rho}}{3u_1} + [\taskhes(\thetast)]_{B_{11},B_{11}} \frac{3}{u_1} = \Thetarho \left ( \frac{1}{(1-\rho)^2 u_1} \right )
\end{align*}
So it follows the true complexity is as stated.

\paragraph{Sample Complexity of Frobenius Norm Identification.}
Note that \algname is the optimal Frobenius norm identification algorithm. In this case, $\taskhes_{\mathrm{fro}}(\thetast) = I$, so the optimal Frobenius norm identification algorithm minimizes $\tr(\Gamss_T(\pi;\thetast)^{-1})$. A similar argument to that used in determining the optimal operator norm identification inputs can be used to show that here the optimal covariance is again of the form \eqref{eq:lqrex1_op_cov}. Then using our inverse expression from the previous sections \eqref{eq:lqrex1_tople_gaminv}, we have that 
\begin{align*}
T \tr(\Gamss_T(\pi;\thetast)^{-1}) \approxrho \frac{1-\rho}{u_1} + \frac{\dimx}{(1-\rho) u_2} + \frac{1}{u_1} + \frac{\dimx}{u_2} \approxrho   \frac{\dimx}{(1-\rho) u_2} + \frac{1}{u_1}
\end{align*}
$u_1,u_2$ will be chosen to minimize this while respecting the constraint $u_1 + (\dimx-1) u_2 \le \gamma^2$. Some algebra shows that this is minimized for values
\begin{align*}
u_2 \approxrho \frac{\dimx \gamma^2 - \gamma^2 \sqrt{1-\rho}}{\dimx^2}, \quad u_1 \approxrho \frac{\gamma^2 \sqrt{1-\rho}}{\dimx}
\end{align*}
Plugging these into the complexity expression for operator norm identification, \eqref{eq:ex1_op_lb}, gives the complexity
\begin{align*}
\Thetarhod \left ( \frac{\dimx}{(1-\rho)^{5/2} \gamma^2} + \frac{\dimx^2}{(1-\rho)^2 \gamma^2}  \right ).
\end{align*}

\newcommand{\Lamstar}{\Lambda_\star}
\paragraph{Sample Complexity of Optimal Noise Identification.} Finally, we turn to the policy which plays the inputs $u_t \sim \calN(0,\Lambda_\star)$ for optimal $\Lambda_\star$ satisfying $\tr(\Lambda_\star) \le \gamma^2$. In this case, our results show that the sample complexity will scale as $\tr(\taskhes(\thetast) \matGamnoise_T(\thetast,\Lambda_\star)^{-1})/T$. Our goal is then to determine the optimal choice of $\Lambda_\star$. A simple application of the KKT conditions shows that the optimal $\Lamstar$ is diagonal (see the proof of \Cref{prop:lqrex2} for this stated explicitly in a similar setting), and will take the form $\Lamstar = \diag([u_1,u_2,\ldots,u_2])$. In this case, some algebra shows that
\begin{align*}
\Gamnoise_T(\thetast,\Lambda_\star) \approxrho \begin{bmatrix} u_1 & 0 & 0 & 0 \\
0 & (1-\rho) u_2 I & 0 & 0 \\
0 & 0 & u_1 & 0 \\
0 & 0 & 0 & u_2 I \end{bmatrix}
\end{align*}
so
\begin{align*}
\tr(\taskhes(\thetast) \matGamnoise_T(\thetast,\Lambda_\star)^{-1}) = \Thetarho \left ( \frac{1}{(1-\rho)^3 u_1} + \frac{\dimx}{(1-\rho)^2 u_2} \right ) 
\end{align*}
Choosing $u_1,u_2$ that minimizes this gives the complexity
\begin{align*}
\Thetarhod \left (  \frac{1}{(1-\rho)^3 \gamma^2} + \frac{\dimx^4}{(1-\rho) \gamma^2} \right ) .
\end{align*}
\qed

\begin{lem}\label{lem:inverse_diagish}
Consider diagonal matrices $D_1, D_2, D_3 \in \R^{d \times d}$. Then,
\begin{align*}
D = \begin{bmatrix}
D_1 & D_2 \\ D_2 & D_3
\end{bmatrix}^{-1} = \begin{bmatrix} C_1 & C_2 \\ C_2 & C_3 \end{bmatrix}
\end{align*}
where $C_1, C_2$, and $C_3$ are diagonal and 
\begin{align*}
[C_1]_{ii} = \frac{[D_3]_{ii}}{[D_1]_{ii} [D_3]_{ii} - [D_2]_{ii}^2}, \quad [C_3]_{ii} = \frac{[D_1]_{ii}}{[D_1]_{ii} [D_3]_{ii} - [D_2]_{ii}^2}, \quad [C_2]_{ii} = \frac{-[D_2]_{ii}}{[D_1]_{ii} [D_3]_{ii} - [D_2]_{ii}^2}
\end{align*}
provided these quantities are well-defined.
\end{lem}
\begin{proof}
Note that we can permute the columns and rows of $D$ with some permutation $P$ such to create a block diagonal matrix $D'$:
\begin{align*}
PDP^\top = D' =  \begin{bmatrix} D_1' & \ldots & 0 \\
\vdots & \ddots & \vdots \\
0 & \ldots & D_d' \end{bmatrix}, \quad D'_i = \begin{bmatrix} [D_1]_{ii} & [D_2]_{ii} \\ [D_2]_{ii} & [D_3]_{ii} \end{bmatrix}, i = 1,\ldots,d
\end{align*}
From the block diagonal structure and the inverse of $2\times 2$ matrices, we have
\begin{align*}
(D')^{-1} =  \begin{bmatrix} (D_1')^{-1} & \ldots & 0 \\
\vdots & \ddots & \vdots \\
0 & \ldots & (D_d')^{-1} \end{bmatrix}, \quad (D_i')^{-1} = \frac{1}{ [D_1]_{ii}[D_3]_{ii} - [D_2]_{ii}^2} \begin{bmatrix} [D_3]_{ii} & -[D_2]_{ii} \\ -[D_2]_{ii} & [D_1]_{ii} \end{bmatrix}
\end{align*}
The orthogonality of a permutation gives that $D = P^\top D' P$, so $D^{-1} = P^\top (D')^{-1} P$. Permuting the form of $(D')^{-1} $ gives the expression in the statement of the result.
\end{proof}

\begin{lem}\label{lem:diagish_mineig}
Consider diagonal matrices $D_1, D_2, D_3 \in \R^{d \times d}$. Then,
\begin{align*}
\lammin \left ( \begin{bmatrix}
D_1 & D_2 \\ D_2 & D_3
\end{bmatrix} \right ) = \min_{i \in \{1,\ldots, d\} } \frac{1}{2} \left ( [D_1]_{ii} + [D_3]_{ii} - \sqrt{([D_1]_{ii} + [D_3]_{ii})^2 - 4 ([D_1]_{ii}  [D_3]_{ii} - [D_2]_{ii}^2) } \right ).
\end{align*}
\end{lem}
\begin{proof}
Let $P$ be the permutation described in the proof of \Cref{lem:inverse_diagish}. Note that the eigenvalues of $PDP^\top$ are the same as those of $D$ since, if we write the eigendecomposition of $D$ as $V \Lambda V^\top$, we see that $PDP^\top = (PV) \Lambda (PV)^\top$, and that $PV$ is orthogonal, so this is the eigendecomposition of $PDP^\top$. Given this, we have that $\lammin(D) = \min_{i \in \{ 1,\ldots,d \}} \lammin(D_i')$. The eigenvalues, $\lambda$, of $D_i'$ satisfy
\begin{align*}
([D_1]_{ii} - \lambda)([D_3]_{ii} - \lambda) - [D_2]_{ii}^2 = 0
\end{align*}
Solving this for $\lambda$ and taking the minimum solution gives the result.
\end{proof}

\subsection{Proof of Proposition \ref{prop:lqrex2}}
We now choose $\rho_1 = \rho_2 = 1$, $b = 1$, $\kappa_1 = \frac{1}{(1-\rho)^4}$, $\kappa_2 = 1$, $\mu = \frac{1}{(1-\rho)^2}$. With this choice, some algebra shows that
\begin{align}\label{eq:lqr_ex2_simp_quant}
\begin{split}
& a_1 = \calO((1-\rho)^2), \quad a_i = \calO(1/(2-\rho)), i \ge 2 \\
& \tfrac{1}{1-a_1} = \calO(1), \quad \tfrac{1}{1-a_i} = \Thetarho( \tfrac{1}{1-\rho} ), i \ge 2\\
&  k_1 = \calO(1 ), \quad k_i = \calO( 1-\rho), i \ge 2 \\
& p_1 = \Thetarho(\tfrac{1}{(1-\rho)^4}), \quad p_i = \Thetarho(\tfrac{1}{1-\rho}), i \ge 2
\end{split}
\end{align}
Plugging these values into the expression given for $\taskhes(\thetast)$ in \Cref{lem:ex_task_hes} gives: 
 \begin{align*}
 &  [ \taskhes(\thetast)]_{A_{11},A_{11}} = \Thetarho \left (  \frac{1}{(1-\rho)^4} \right ), \quad [ \taskhes(\thetast)]_{A_{1m},A_{1m}} = \Thetarho \left (\frac{1}{(1-\rho)^5} \right ), m > 1 \\
 &  [ \taskhes(\thetast)]_{A_{n1},A_{n1}} = \Thetarho \left ( 1 \right ), n > 1, \quad [ \taskhes(\thetast)]_{A_{nm},A_{nm}} = \Thetarho \left ( \frac{1}{(1-\rho)^{3}} \right ), n > 1,m > 1 \\
 & [ \taskhes(\thetast)]_{B_{11},B_{11}} = \Thetarho \left ( \frac{1}{(1-\rho)^4} \right ), \quad [ \taskhes(\thetast)]_{B_{n1},B_{n1}} = \Thetarho(1), n > 1  \\
 & [ \taskhes(\thetast)]_{B_{1m},B_{1m}} = \Thetarho \left ( \frac{1}{(1-\rho)^{3}} \right ), m > 1, \quad   [ \taskhes(\thetast)]_{B_{nm},B_{nm}} = \Thetarho \left ( \frac{1}{1-\rho} \right ), n > 1, m > 1 \\
&   [\taskhes(\thetast)]_{A_{11}, B_{11}} = \Thetarho \left ( \frac{-1}{(1-\rho)^4} \right ), \quad   [\taskhes(\thetast)]_{A_{1m}, B_{1m}} = \Thetarho \left ( \frac{-1}{(1-\rho)^4} \right ), m > 1\\
&  [\taskhes(\thetast)]_{A_{n1}, B_{n1}} = \Thetarho \left ( 1 \right ), n > 1, \quad [\taskhes(\thetast)]_{A_{nm}, B_{nm}} = \Thetarho \left ( \frac{1}{(1-\rho)^2} \right ), n > 1, m > 1
 \end{align*}

\paragraph{Sample Complexity of \algname.} As we will consider $\rho$ close to $1$, elements scaling as $\Thetarho(\frac{1}{(1-\rho)^n})$ for $n \le 2$ will be dominated by elements scaling as $\Thetarho(\frac{1}{(1-\rho)^n})$ for $n > 2$. For simplicity, we henceforth ignore these elements. Given these approximations, we see that $\taskhes_i$ for $i > 1$ is approximately diagonal and therefore:
\begin{align*}
\tr(\taskhes_i \Gamma_T(\pi;\thetast)^{-1}) \approxrho \sum_{j=1}^{\dimx + \dimu} [\taskhes_i]_{jj} [\Gamss_T(\pi;\thetast)^{-1}]_{jj}
\end{align*}
 $\taskhes_i$, however, contains non-negligible off-diagonal elements, $[\taskhes(\thetast)]_{A_{1m}, B_{1m}}$, and will take the form:
\begin{align*}
\taskhes_1 & \approxrho \frac{1}{(1-\rho)^4} e_1 e_1^\top + \frac{1}{(1-\rho)^5} \sum_{j=2}^{\dimx} e_j e_j^\top + \frac{1}{(1-\rho)^4} e_{\dimx + 1} e_{\dimx +1}^\top +  \frac{1}{(1-\rho)^3} \sum_{j=2}^{\dimx} e_{\dimx + j} e_{\dimx + j}^\top \\
& \qquad \qquad - \frac{1}{(1-\rho)^4} \sum_{j=1}^{\dimx} (e_j e_{\dimx+j}^\top + e_{\dimx+j} e_j^\top) \\
& =  \frac{1}{(1-\rho)^4} (e_1 - e_{\dimx +1}) (e_1 - e_{\dimx +1})^\top + \sum_{j=2}^{\dimx} \frac{1}{(1-\rho)^5} ( e_j - (1-\rho) e_{\dimx + j}) ( e_j - (1-\rho) e_{\dimx + j})^\top 
\end{align*}
As we are concerned with showing an upper bound on the performance of \algname, we can simply choose a feasible set of inputs and compute the sample complexity obtained by them. Since \algname obtains the optimal sample complexity, it follows that this will be a valid upper bound on performance. Given this, let $\bmU_\ell = 0$ for all but a single $\ell$ to be chosen (and it's conjugate partner), and assume that the nonzero $\bmU_\ell = \diag([u_1,\ldots,u_{\dimx}])$ is real and diagonal. We will choose $u_2 = \ldots = u_{\dimx}$. In that case, we will have 
\begin{align*}
\Gamss_T(\pi;\thetast) & \propto 2 \text{re} \bigg ( \begin{bmatrix} (e^{\imag \omega_\ell} I - \Ast)^{-1} \Bst \\ I \end{bmatrix} \bmU_\ell \begin{bmatrix} (e^{\imag \omega_\ell} I - \Ast)^{-1} \Bst \\ I \end{bmatrix}^\herm \bigg ) \\
& = 2  \begin{bmatrix} |e^{\imag \omega_\ell} - \rho|^{-2} \bmU_\ell  & \text{re} ((e^{\imag \omega_\ell} - \rho)^{-1}) \bmU_\ell \\ \text{re} ((e^{\imag \omega_\ell} - \rho)^{-1}) \bmU_\ell & \bmU_\ell \end{bmatrix} 
\end{align*}
where we simplify using the values of $\Ast,\Bst$, $\omega_\ell = \imag 2 \pi \ell / k$, and the real comes from the conjugate symmetry. Since $\bmU_\ell$ is diagonal, we can apply \Cref{lem:inverse_diagish} to invert this:
\begin{align*}
 \begin{bmatrix} |e^{\imag \omega_\ell} - \rho|^{-2} \bmU_\ell  & \text{re} ((e^{\imag \omega_\ell} - \rho)^{-1}) \bmU_\ell \\ \text{re} ((e^{\imag \omega_\ell} - \rho)^{-1}) \bmU_\ell & \bmU_\ell \end{bmatrix}^{-1} = \begin{bmatrix} C_1 & C_2 \\ C_2 & C_3 \end{bmatrix}
\end{align*}
where
\begin{align*}
& [C_1]_{ii} = \frac{u_i^{-1}}{|e^{\imag \omega_\ell} - \rho|^{-2}  - \text{re} ((e^{\imag \omega_\ell} - \rho)^{-1})^2  }, \quad [C_2]_{ii} = \frac{-\text{re} ((e^{\imag \omega_\ell} - \rho)^{-1}) u_i^{-1}}{|e^{\imag \omega_\ell} - \rho|^{-2}  - \text{re} ((e^{\imag \omega_\ell} - \rho)^{-1})^2 } \\
& [C_3]_{ii} = \frac{|e^{\imag \omega_\ell} - \rho|^{-2} u_i^{-1}}{|e^{\imag \omega_\ell} - \rho|^{-2}  - \text{re} ((e^{\imag \omega_\ell} - \rho)^{-1})^2  }
\end{align*}
We choose $\omega_\ell = 1 -\rho$. Then,
\begin{align*}
e^{\imag \omega_\ell} - \rho = \cos(\omega_\ell) - \rho + \imag \sin(\omega_\ell) =  1 - \rho + \imag (1-\rho) + o(1-\rho)
\end{align*} 
and
\begin{align*}
\text{re} ((e^{\imag \omega_\ell} - \rho)^{-1}) = \frac{1}{2(1-\rho)} + o((1-\rho)^2), \quad |e^{\imag \omega_\ell} - \rho|^{-2} =  \frac{1}{2(1-\rho)^2} + o((1-\rho)^2)
\end{align*}
So it follows that
\begin{align*}
[C_1]_{ii} \approxrho \frac{4(1-\rho)^2}{u_i}, \quad [C_2]_{ii} \approxrho \frac{- 2(1-\rho)}{u_i}, \quad [C_3]_{ii} \approxrho \frac{2}{u_i} 
\end{align*}
Using our expressions for $\taskhes_i$ given above and this expression, we have that
\begin{align*}
T \tr(\taskhes_1\Gamss_T(\pi;\thetast)^{-1}) & = \Thetarhod \left (  \frac{1}{(1-\rho)^2 u_1'} + \frac{\dimx}{(1-\rho)^3 u_2'} + \frac{1}{(1-\rho)^4 u_1'} + \frac{\dimx}{(1-\rho)^3 u_2'} + \frac{1}{(1-\rho)^3 u_1'} + \frac{\dimx}{(1-\rho)^3 u_2'} \right ) 
\end{align*}
\begin{align*}
T \tr(\taskhes_i \Gamss_T(\pi;\thetast)^{-1}) = \Thetarhod \left ( \frac{\dimx}{(1-\rho) u_2'} + \frac{\dimx}{(1-\rho) u_2'} \right ), \quad i > 1
\end{align*}
where $u_i'$ denotes $u_i/k^2$. Thus,
\begin{align*}
T \tr(\taskhes(\thetast) \matGamss_T(\pi;\thetast)^{-1}) = \Thetarhod \left ( \frac{1}{(1-\rho)^4 u_1'} + \frac{\dimx}{(1-\rho)^3 u_2'} \right ) 
\end{align*}
We can choose $u_1$ and $u_2$ as we wish as long as they meet the power constraint $u_1' + (\dimx-1) u_2' \le \gamma^2$.  Choosing the values that minimize the complexity yields:
\begin{align*}
T \tr(\taskhes(\thetast) \matGamss_T(\pi;\thetast)^{-1}) = \Thetarhod \left ( \frac{1}{(1-\rho)^4 \gamma^2} + \frac{\dimx^2}{(1-\rho)^3 \gamma^2} \right ) 
\end{align*}
By our construction, this is an upper bound on the performance of \algname.

\paragraph{Sample Complexity of Operator Norm and Frobenius Norm Identification.} We can follow a similar argument as that used in the proof of \Cref{lem:lqrex1} to show that the optimal inputs will be diagonal for both operator norm and Frobenius norm identification. Furthermore, as both $\Ast$ and $\Bst$ are scalings of the identity, the optimal operator norm and Frobenius norm identification algorithms will allocate the same energy to each coordinate. Note that the input constructed in the previous section will yield the maximum gain, so it follows that both operator and Frobenius norm identification will play inputs at similar frequencies, and the analysis in the preceding section can be applied here. However, given that the inputs are isotropic, we will have $u_1 = u_2 = \gamma^2/\dimx$, which will yield a sample complexity of
\begin{align*}
T \tr(\taskhes(\thetast) \matGamss_T(\pi;\thetast)^{-1}) = \Thetarhod \left ( \frac{\dimx}{(1-\rho)^4 \gamma^2} + \frac{\dimx^2}{(1-\rho)^3 \gamma^2} \right ). 
\end{align*}

\paragraph{Sample Complexity of Optimal Noise Identification.} When playing noise, the complexity will scale as $\tr(\taskhes(\thetast) \matGamnoise_T(\thetast,\Lambda_u)^{-1})/T$. To analyze the sample complexity of this approach, we must first determine the $\Lambda_u \succeq 0$ that minimizes this and satisfies $\tr(\Lambda_u) \le \gamma^2$. In our setting, we will have
\begin{align*}
\Gamnoise_T(\thetast,\Lambda_u) = \begin{bmatrix} \sum_{s=0}^T \rho^{2s}  \Lambda_u & 0 \\ 0 & \Lambda_u \end{bmatrix}\approxrho \begin{bmatrix} \frac{1}{1 - \rho} \Lambda_u & 0 \\ 0 & \Lambda_u \end{bmatrix}
\end{align*}
So it follows that
\begin{align*}
\tr(\taskhes(\thetast) \matGamnoise_T(\thetast,\Lambda_u)^{-1}) \approxrho \sum_{i=1}^{\dimx} [ (1-\rho) \tr(\taskhes_{i,1} \Lambda_u^{-1} ) + \tr(\taskhes_{i,2} \Lambda_u^{-1} )]
\end{align*}
where $\taskhes_{i,1},\taskhes_{i,2}$ denote the first and second $\dimx \times \dimx$ block diagonals of $\taskhes_i$, respectively. Computing the gradient of this expression and the constraint $\tr(\Lambda_u) \le \gamma^2$ and applying the KKT conditions gives that the optimal $\Lamstar$ must satisfy:
\begin{align*}
-\sum_{i=1}^{\dimx} [ (1-\rho) (\Lamstar)^{-1} \taskhes_{i,1} (\Lamstar)^{-1}  +  (\Lamstar)^{-1} \taskhes_{i,2} (\Lamstar)^{-1} ] + \mu I = 0
\end{align*}
As $\taskhes_{i,1},\taskhes_{i,2}$ are diagonal, up to lower order terms, it follows that a diagonal $\Lamstar$ will satisfy this expression. Furthermore, given the symmetry of $\Ast,\Bst$, it is clear that $\Lamstar$ will then take the form $\diag([u_1,u_2,\ldots,u_2])$ for $u_1 + (\dimx-1) u_2 \le \gamma^2$. Plugging this into the expressions given above, we have that the complexity scales as
\begin{align*}
\tr(\taskhes(\thetast) \matGamnoise_T(\thetast,\Lamstar)^{-1}) & = \Thetarhod \left ( \frac{\dimx}{(1-\rho)^4 u_2} + \frac{1}{(1-\rho)^4 u_1} \right ) 
\end{align*}
Choosing $u_1,u_2$ that minimize this yields the complexity
\begin{align*}
\tr(\taskhes(\thetast) \matGamnoise_T(\thetast,\Lamstar)^{-1}) & = \Thetarhod \left ( \frac{\dimx^2}{(1-\rho)^4 \gamma^2}  \right ) .
\end{align*}
\qed

\section{Details on Numerical Results}\label{sec:exp_details}

\begin{figure*}[h]
     \centering
     \hfill
     \begin{minipage}[b]{0.335\textwidth}
         \centering
          \includegraphics[width=\linewidth]{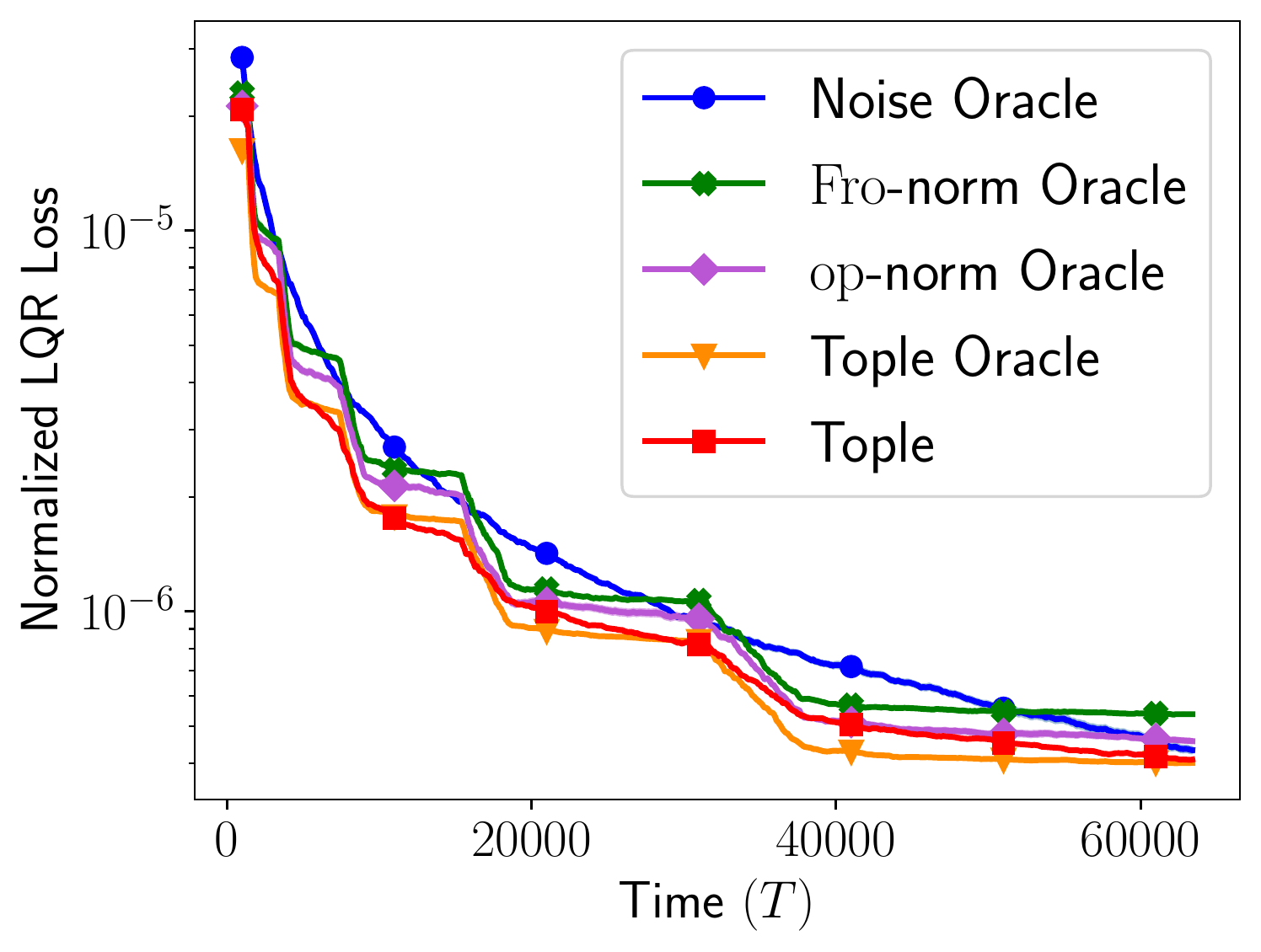}
  \caption{\lqrx loss versus time on $\Ast$ a Jordan block and $\Bst,\Rx,\Ru$ randomly generated.}
       \label{fig:jordan_error}
     \end{minipage}
     \hfill
     \begin{minipage}[b]{0.315\textwidth}
         \centering
          \includegraphics[width=\linewidth]{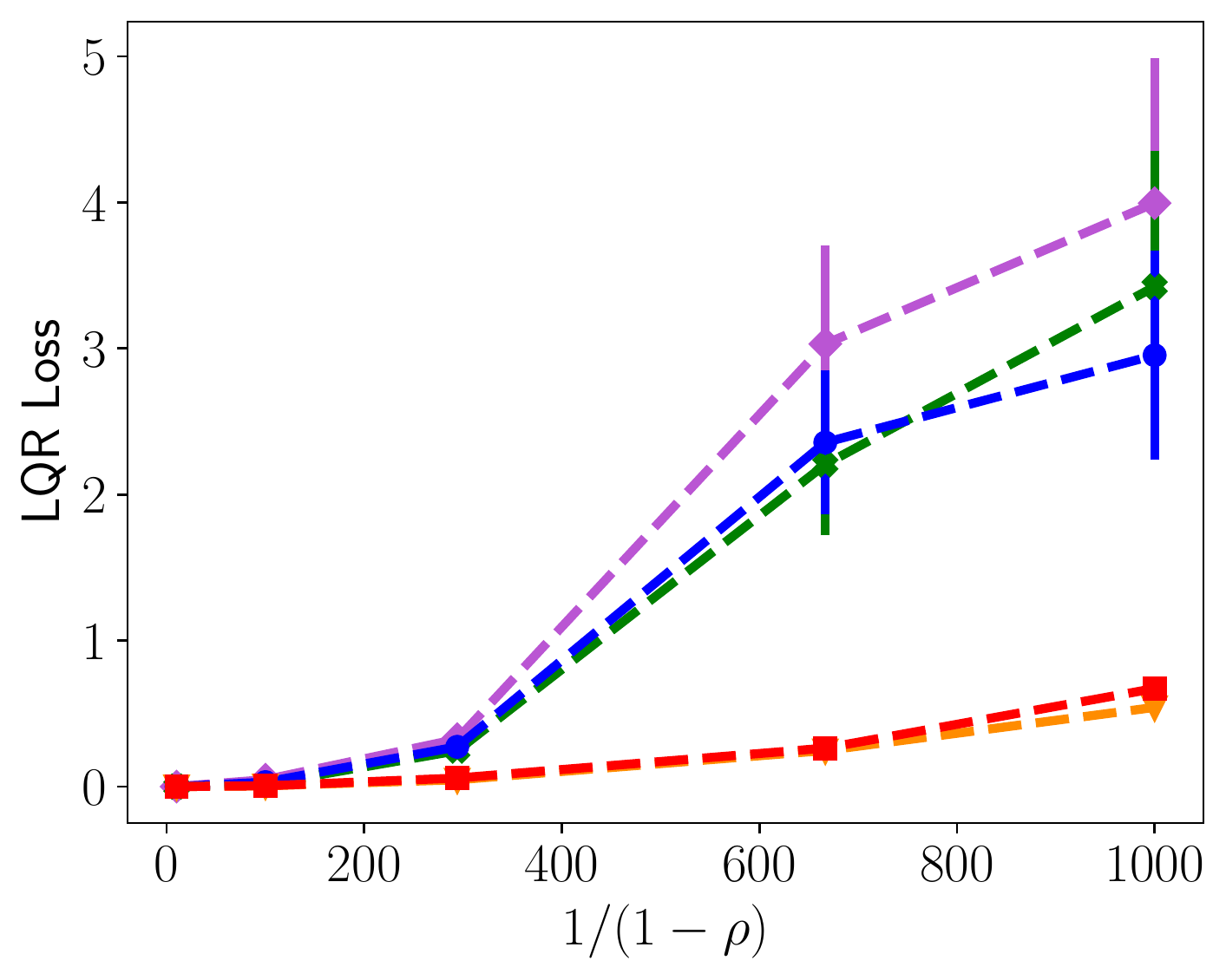}
  \caption{\lqrx loss when varying $\rho$ on example stated in \Cref{lem:lqrex1}.}
  \label{fig:ex1_error}
     \end{minipage}
     \hfill
     \begin{minipage}[b]{0.33\textwidth}
         \centering
          \includegraphics[width=\linewidth]{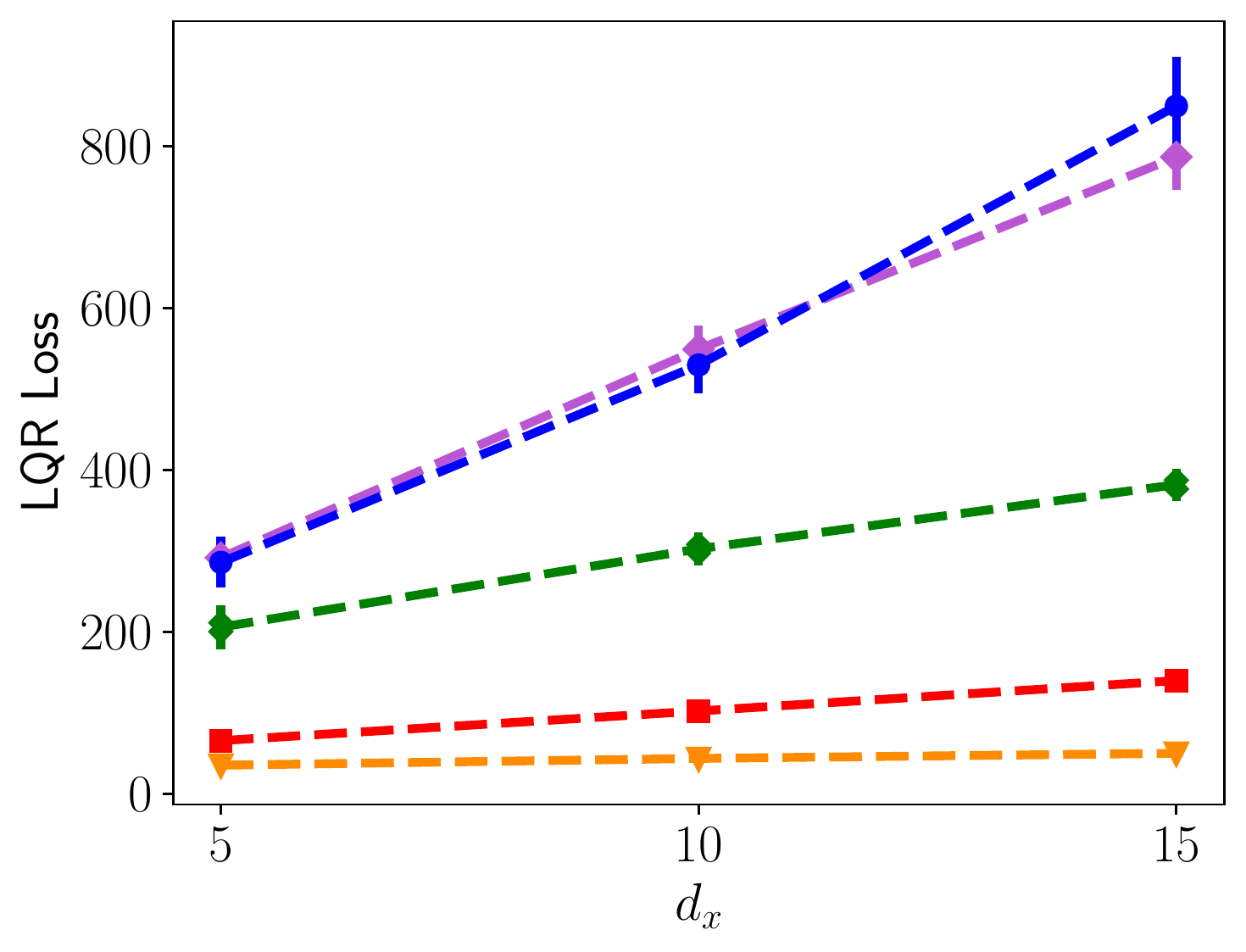}
  	\caption{\lqrx loss when varying $\dimx$ on example stated in \Cref{prop:lqrex2}. }
	\label{fig:ex2_error}
     \end{minipage}
\end{figure*}

In Figures \ref{fig:jordan_error}, \ref{fig:ex1_error}, and \ref{fig:ex2_error}, we plot Figures \ref{fig:jordan}, \ref{fig:ex1}, and \ref{fig:ex2} with error bars. In all cases the error bars indicate a standard error. We make several additional remarks on the experiments. For \Cref{fig:jordan_error}, we chose $\dimx = \dimu = 5$ and chose $\rho = 0.8$, which gave us
\begin{align*}
\Ast = \begin{bmatrix} 0.8 & 1 & 0 & 0 & 0 \\
0 & 0.8 & 1 & 0 & 0 \\
0 & 0 & 0.8 & 1 & 0 \\
0 & 0 & 0 & 0.8 & 1 \\
0 & 0 & 0 & 0 & 0.8 
\end{bmatrix}
\end{align*}
As was stated in the main text, we generated $\Bst,\Rx$, and $\Ru$ randomly. For each realization, we ran 15 trials, so \Cref{fig:jordan_error} is, in total, the average over 225 trials. As different $\Rx$ and $\Ru$ would cause $\| \taskhes(\thetast) \|_\op$ to vary widely, we divided the loss of each realization by  $\| \taskhes(\thetast) \|_\op$ to ensure they were on the same scale. The reader may wonder why the error decays in a stepwise fashion. This is due to the convex relaxation of the inputs. In this example, the majority of the energy is concentrated in the first eigenvalue of the input, and thus, when the matrix input is decomposed, the majority of the energy is played in only a fraction of $1/\dimu$ of the time. We therefore see a much steeper decrease in this time. As we show, however, our convex relaxation is tight and nothing is lost by playing inputs in this way.

For \Cref{fig:ex1_error}, we chose the values of $\Rx$ and $\Ru$ as given in the proof of \Cref{lem:lqrex1} and set $\dimx = \dimu = 5$. For \Cref{fig:ex2_error}, we chose $\Rx$ and $\Ru$ as given in the proof of \Cref{prop:lqrex2} and set $\rho = 0.99$.

Our implementation of \algname uses the convex relaxation and projected gradient descent solution given in \Cref{sec:freq_domain_summary}. While \cite{wagenmaker2020active} does not provide a computationally efficient solution to their proposed operator norm identification algorithm, we note that the convex relaxation given in \Cref{sec:freq_domain_summary} can be applied to this problem as well, yielding a computationally efficient version of the algorithm given in \cite{wagenmaker2020active}. We rely on this computationally efficient relaxation for our implementation of the operator norm identification algorithm.